\documentclass{dlp}
\begin{document}
  \frontmatter

\title{From Dependence to Causation}
\author{David Lopez-Paz}
\date{June 20, 2016}

\maketitle

\begin{center}
\vspace*{1cm}
\begin{huge}
\noindent\textbf{From Dependence to Causation}
\end{huge}
\vskip 0.75 cm
\begin{Large}
\noindent David Lopez-Paz
\end{Large}
\vskip 0.75 cm
\begin{Large}
  \textbf{Abstract}
\end{Large}
\end{center}

Machine learning is the science of discovering statistical dependencies in
data, and the use of those dependencies to perform predictions.  During the
last decade, machine learning has made spectacular progress, surpassing human
performance in complex tasks such as object recognition, car driving, and
computer gaming.  However, the central role of prediction in machine learning
avoids progress towards general-purpose artificial intelligence. As one way
forward, we argue that \emph{causal inference} is a fundamental component of
human intelligence, yet ignored by learning algorithms.

Causal inference is the problem of uncovering the cause-effect relationships
between the variables of a data generating system.  Causal structures provide
understanding about how these systems behave under changing, unseen
environments. In turn, knowledge about these causal dynamics allows to answer
``what if'' questions, describing the potential responses of the system under
hypothetical manipulations and interventions. Thus, understanding cause and
effect is one step from machine learning towards machine reasoning and
machine intelligence. But, currently available causal inference algorithms
operate in specific regimes, and rely on assumptions that are difficult to
verify in practice.

This thesis advances the art of causal inference in three different ways.
First, we develop a framework for the study of statistical dependence based on
copulas (models NPRV and GPRV) and random features (models RCA and RDC).
Second, we build on this framework to interpret the problem of causal inference
as the task of distribution classification.  This new interpretation conceives
a family of new causal inference algorithms (models RCC and NCC), which are widely
applicable under mild learning theoretical assumptions.  Third, we showcase NCC
to discover causal structures in convolutional neural network 
features. All of the algorithms presented in this thesis are applicable to big
data, exhibit strong theoretical guarantees, and achieve state-of-the-art
performance in a variety of real-world benchmarks.

This thesis closes with a discussion about the state-of-affairs in machine
learning research, and a review about the current progress on novel ideas such
as machines-teaching-machines paradigms, theory of nonconvex optimization, and the
supervision continuum. We have tried to provide our exposition with a
philosophical flavour, as well as to make it a self-contained book.

\tableofcontents
\listoffigures
\listoftables

\makeatletter
\def\ll@theorem{%
  \protect\numberline{\csname the\thmt@envname\endcsname}%
  \ifx\@empty\thmt@shortoptarg
    \thmt@thmname
  \else
    \thmt@shortoptarg
  \fi}
\def\ll@definition{%
  \protect\numberline{\csname the\thmt@envname\endcsname}%
  \ifx\@empty\thmt@shortoptarg
    \thmt@thmname
  \else
    \thmt@shortoptarg
  \fi}
\def\ll@remark{%
  \protect\numberline{\csname the\thmt@envname\endcsname}%
  \ifx\@empty\thmt@shortoptarg
    \thmt@thmname
  \else
    \thmt@shortoptarg
  \fi}
\def\ll@example{%
  \protect\numberline{\csname the\thmt@envname\endcsname}%
  \ifx\@empty\thmt@shortoptarg
    \thmt@thmname
  \else
    \thmt@shortoptarg
  \fi}
\def\ll@lemma{%
  \protect\numberline{\csname the\thmt@envname\endcsname}%
  \ifx\@empty\thmt@shortoptarg
    \thmt@thmname
  \else
    \thmt@shortoptarg
  \fi}
\def\ll@corollary{%
  \protect\numberline{\csname the\thmt@envname\endcsname}%
  \ifx\@empty\thmt@shortoptarg
    \thmt@thmname
  \else
    \thmt@shortoptarg
  \fi}
\makeatother

\renewcommand{\listtheoremname}{List of Definitions}
\listoftheorems[ignoreall,show={definition}]
\renewcommand{\listtheoremname}{List of Theorems}
\listoftheorems[ignoreall,show={theorem,lemma,corollary}]
\renewcommand{\listtheoremname}{List of Remarks}
\listoftheorems[ignoreall,show={remark}]
\renewcommand{\listtheoremname}{List of Examples}
\listoftheorems[ignoreall,show={example}]

\chapter*{Notation}
\begin{table}[h!]
\begin{tabular}{l|p{10cm}}
  symbol & meaning \\\hline
  $a$ & scalar or vector\\
  $a_i$ & entry at $i$th position of vector $a$\\
  $a_{\mathcal{I}}$ & vector $(a_i)_{i\in\mathcal{I}}$, for set of indices $\mathcal{I}$\\
  $A$ & matrix or tensor\\
  $A_{i,j}$ & entry at $i$th row and $j$th column of the matrix $A$\\
  $A_{i,:}$ & row vector from the $i$th row of the matrix $A$\\
  $A_{:,j}$ & column vector from the $j$th column of the matrix $A$\\
  $A_{i,j,k}$ & similar notations apply to higher-order tensors\\\hline
  $\mathcal{A}$ & set\\
  $\{x_i\}_{i=1}^n$ & set $\{x_1, \ldots, x_n\}$\\
  $\R$ &   the set of real numbers\\
  $\Rn$ &  the set of vectors of size $n$ with real entries\\
  $\Rnm$ & the set of matrices of size $n \times m$ with real entries\\
  $\R^{n \times m \times d}$ & similar notations apply to tensors\\\hline
  $\bm a$ & scalar-valued or vector-valued random variable\\
  $\bm A$ & matrix-valued or tensor-valued random variable\\
  $\bm a \equiv P$ & $\bm a$ follows the distribution $P$\\
  $a \sim P$ & $a$ is sampled from $P$\\
  $P^n$ & $n$-dimensional product distribution built from $P$\\
  $\Pr_{\bm x}(e)$ & probability of event $e$\\
  $\E{\bm x}{f(\bm x)}$ & expectation of $f(\bm x)$ over the distribution $P$.\\
  $\V{\bm x}{f(\bm x)}$ & variance of $f(\bm x)$.\\\hline
  $\bm x \indep \bm y$ & $\bm x$ is independent from $\bm y$\\
  $\bm x \indep \bm y \given \bm z$ & $\bm x$ is conditionally independent from $\bm y$ given $\bm z$\\
  $\bm x \to \bm y$ & $\bm x$ causes $\bm y$\\
\end{tabular}

\vskip 0.5cm

The elements $p(\bm x) = p$ are the probability density \emph{function} of 
$\bm x$. On the other hand, the elements $p(\bm x = x) = p(x)$ are the
\emph{value} of the probability density function at $x$. The same notations
apply to cumulative distribution functions, denoted with an upper case $P$.
\caption{Notations.}
\end{table}

\mainmatter

  \chapter{Introduction}
As put forward by David Hume over three centuries ago, our experience is shaped
by the observation of constant conjunction of events. Rain follows drops in
atmospheric pressure, sunlight energizes our mornings with warmth, the orbit of
the Moon dances with the tides of the sea, mirrors shatter into pieces when we
throw stones at them, heavy smokers suffer from cancer, our salary relates to
the car we drive, and bad political decisions collapse stock markets.
Such systematic variations suggest that these pairs of variables rely on each
other to instantiate their values.  These variables, we say, \emph{depend} on
each other.

Dependence is the necessary substance for statistics and machine learning.  It
relates the variables populating our world to each other, and enables the
prediction of values for some variables given values taken by others.  Let me
exemplify. There exists a strong linear dependence between the chocolate
consumption and the amount of Nobel laureates per country \citep{Messerli}.
Therefore, we could use the data about these two variables from a small amount
of countries to construct a linear function from chocolate consumption to
number of Nobel laureates.  Using this linear function we could, given the
chocolate consumption in a new country, predict their national Nobel
prize sprout.  Dependencies like these leave patterns in the
joint probability distribution of the variables under study. The goal of
machine learning and statistics is then, as summarized by Vladimir
\citet{Vapnik82}, the inference of such patterns from empirical data, and their
use to predict new aspects about such joint probability distribution.

But, how does dependence arise? The answer hides in the most fundamental of the
connections between two entities: causation.  According to the \emph{principle of
common cause} pioneered by Hans \citet{reichenbach56}, every dependence
between two variables $\bm x$ and $\bm y$ is the observable footprint of one
out of three possible causal structures: either $\bm x$ causes $\bm y$, or $\bm
y$ causes $\bm x$, or there exists a third variable $\bm z$, called
\emph{confounder}, which causes both $\bm x$ and $\bm y$.  The third structure
reveals a major consequence: \emph{dependence does not imply causation}. Or, when
the dependence between two variables $\bm x$ and $\bm y$ is due to a confounder
$\bm z$, this dependence does not imply the existence of a causal
relationship between $\bm x$ and $\bm y$.  Now, this explains the bizarre
connection between chocolate eating and Nobel prize winning from the previous
paragraph!  It may be that this dependence arises due to the existence of an
unobserved confounder: for example, the strength of the economy of the country.

The study of causation is not exclusive to philosophy, as it enjoys
far-reaching consequences in statistics. While dependence is the tool to
describe patterns about the distribution generating our data, causation is the
tool to describe the reactions of these patterns when intervening on the distribution. 
In plain words, the difference between dependence and causation
is the difference between \emph{seeing} and \emph{doing}. In terms of our running
example: if we were a politician interested in increasing the number of Nobel
prizes awarded to scientists from our country, the causal structure of the
problem indicates that we should boost the national economy (the alleged common cause), instead of
force-feeding chocolate to our fellow citizens. Thus, causation does not only
describe which variables depend on which, but also how to manipulate
them in order to achieve a desired effect.

More abstractly, causation bridges the distribution that generates the observed
data to some different but related distribution, which is more relevant to
answer the questions at hand \citep{jonasscript}. For instance, the question
``Does chocolate consumption \emph{cause} an increase in national Nobel
laureates?'' is not a question about the distribution generating the observed
data. Instead, it is a question about a different distribution that we could
obtain, for instance, by randomizing the chocolate consumption across
countries.  To answer the question we should check, after some
decades of randomization, if the dependence between chocolates and Nobels remains
in this new induced distribution. Although randomized experiments are
considered the golden standard for causal inference, these are often
unpractical, unethical, or impossible to realize. In these
situations we face the need for \emph{observational causal inference}: the
skill to infer the causal structure of a data generating process without
intervening on it.

As humans, we successfully leverage observational causal inference to reason
about our changing world, and about the outcome of the interventions
that we perform on it (\emph{Will she reply if I text her?}). Observational
causal inference is key to reasoning and intelligence
\citep{bottou2014machine}. Changing environments are a nuisance not only known to
humans: machines face the same issues when dealing with changing distributions
between training and testing times; multitask, domain adaptation, and transfer
learning problems; and dynamic environments such as online learning
and reinforcement learning.  Causal inference is a promising
tool to address these questions, yet ignored in most machine learning
algorithms. Here we take a stance about the central importance of causal
inference for artificial intelligence, and contribute to the \emph{cause} by
developing novel theory and algorithms.

Let us begin this journey; one exploration into the fascinating concepts of
statistical dependence and causation.  We will equip ourselves with the
necessary mathematical background in Part I. To understand causation one must
first master dependence, so we will undertake this endeavour in \mbox{Part II}.
Finally, Part III crosses the bridge from dependence to causation, and argues 
about the central role of the latter in machine learning, machine
reasoning, and artificial intelligence. This thesis has a philosophical taste
rare to our field of research; we hope that this is for the
enjoyment of the reader.

\begin{remark}[The origin of dependence]
The word \emph{dependence} originates from the Old French vocable
\emph{dependre}, which was first used around the 15$^\text{th}$ century. The
concept of statistical dependence was explicitly introduced in Abraham de
Moivre's \emph{The Doctrine of Chances} (1718), where he defines two events to
be independent ``when they have no connection one with the other, and that the
happening of one neither forwards nor obstructs the happening of the other''.
On the other hand, he describes two events to be dependent ``when they are so
connected together as that the probability of either happening alters the
happening of the other''.  In the same work, de Moivre's correctly calculates
the joint probability of two independent events as the product of their
marginal probabilities. Gerolamo Cardano (1501-1576) hinted the multiplication
rule before, but not explicitly. The first precise mathematical
characterization of statistical dependence is Pierre-Simon Laplace's
\emph{Th\'eorie analytique des probabilit\'es}, in 1812.
\end{remark}

\section{Outline}
The rest of this thesis is organized in seven chapters.
  \begin{enumerate}
  \item Chapter~\ref{chapter:mathematics} introduces the necessary mathematics
  to understand this thesis. We will review well known but important results
  about linear algebra, probability theory, machine learning, and numerical
  optimization. 

  \item Chapter~\ref{chapter:representing-data} reviews four techniques to
  represent data for its analysis: kernel methods, random features, neural
  networks, and ensembles.  Data representations will be a basic building block
  to study statistical dependence and causation throughout this thesis.
  Chapters \ref{chapter:mathematics} and \ref{chapter:representing-data} are a
  personal effort to make this thesis a self-contained book.

  \item Chapter~\ref{chapter:generative-dependence} starts the study of
  statistical dependence by means of \emph{generative models of dependence},
  which estimate the full dependence structure of a
  multidimensional probability distribution. This chapter contains novel
  material from \citep{dlp-ssl,dlp-gp}, where cited.
  
  \item Chapter~\ref{chapter:discriminative-dependence} concerns
  \emph{discriminative models of dependence}, which, in contrast to generative
  models, summarize the dependence structure of a multidimensional probability
  distribution into a low-dimensional statistic. This chapter contains novel
  material from \citep{dlp-rdc,dlp-rca}, where cited.

  \item Chapter~\ref{chapter:language-causality} crosses the bridge from
  dependence to causation, introducing the language of causal modeling, and
  reviewing the state-of-the-art on algorithms for observational causal
  inference. This chapter contains novel material from \citep{dlp-gauss}, where
  cited.

  \item Chapter~\ref{chapter:learning-causal-relations} phrases observational
  causal inference as probability distribution classification.  Under this
  interpretation, we describe a new framework of observational causal inference
  algorithms, which exhibit provable guarantees and state-of-the-art
  performance. Furthermore, we apply our algorithms to infer the existence of
  causal signals in convolutional neural network features.  This chapter
  contains novel material from \citep{dlp-clt,dlp-jmlr,dlp-img}, where cited.

  \item Chapter~\ref{chapter:conclusion} closes the exposition with some
  reflections on the state-of-affairs in machine learning research, as well as
  some preliminary progress on three research questions:
  machines-teaching-machines paradigms, theory of nonconvex optimization, and the
  supervision continuum.  This chapter contains novel material from
  \citep{dlp-distillation}, where cited.
\end{enumerate}

The code implementing all the algorithms and experiments presented in this
thesis is available at \url{https://github.com/lopezpaz}.

\section{Contributions}
We summarize the contributions contained in this thesis, as well as their
location in the text, in both Table~\ref{table:contributions} and the
corresponding back-references from the Bibliography.  Most of these are works
in collaboration with extraordinary scientists, including my wonderful advisors
Bernhard Sch\"olkopf and Zoubin Ghahramani. Our contributions are:
\begin{enumerate}
  \item We introduce nonparametric vine copulas (\underline{NPRV}), and their
  use to address semisupervised domain adaptation problems \citep{dlp-ssl}. Vine copulas
  factorize multivariate densities into a product of marginal distributions and
  bivariate copula functions. Therefore, each of these factors can be adapted
  independently to learn from different domains. Experimental results on
  regression problems with real-world data illustrate the efficacy of the
  proposed approach when compared to the state-of-the-art.
  \item We relax the ``vine simplifying assumption'' by modeling the latent
  functions that specify the shape of a conditional copula given its
  conditioning variables \citep{dlp-gp}. We learn these functions by bringing sparse
  Gaussian processes and expectation propagation into the world of vines. We term
  our method \underline{GPRV}.  Our experiments show that
  modeling these previously ignored conditional dependencies leads
  to better estimates of the copula of the data.
  \item We propose the Randomized Component Analysis (\underline{RCA}) framework \citep{dlp-rca}. RCA
  extends linear component analysis algorithms, such as principal component
  analysis and canonical correlation analysis, to model nonlinear dependencies.
  We stablish theoretical guarantees for RCA using recent 
  concentration inequalities for matrix-valued random variables, and provide
  numerical simulations that show the state-of-the-art performance of the
  proposed algorithms.
  \item We extend the RCA framework into the Randomized Dependence Coefficient
  (\underline{RDC}), a measure of dependence between multivariate random
  variables \citep{dlp-rdc}. RDC is invariant with respect to monotone transformations in
  marginal distributions, runs in log-linear time, has provable theoretical
  guarantees, and is easy to implement. RDC has a competitive performance
  when compared to the state-of-the-art measures of dependence.
  \item We extend RCA to pose causal inference as the problem of learning to
  classify probability distributions \citep{dlp-clt,dlp-jmlr}.  In particular,
  we will featurize samples from probability distributions using the kernel
  mean embedding associated with some characteristic kernel. Using these
  embeddings, we train a binary classifier (the Randomized Causation
  Coefficient or \underline{RCC}) to distinguish between causal structures. We
  present generalization bounds showing the statistical consistency and
  learning rates of the proposed approach, and provide a simple implementation
  that achieves state-of-the-art cause-effect inference. Furthermore, we extend
  RCC to multivariate causal inference.
  \item We propose a variant of RCC based on neural networks, termed
  \underline{NCC}. We use NCC to reveal the existence of observable causal
  signals in computer vision features. In particular,  NCC effectively
  separates contextual features from object features in collections of static
  images \citep{dlp-img}.  This separation proves the existence of a relation between the
  direction of causation and the difference between objects and their context,
  as well as the existence of observable causal signals in collections of
  static images.
  \item We introduce {\underline{generalized distillation}} \citep{dlp-distillation},
  a framework to learn from multiple data modalities and machines
  semisupervisedly. Compression \citep{Bucilua06}, distillation
  \citep{Hinton15} and privileged information \citep{Vapnik09}
  are shown particular instances of generalized distillation.
  \item In our conclusion chapter, we provide research discussions about the
  concepts of \underline{supervision continuum} and the \underline{theory of
  nonconvex optimization}. 
  \item We provide a self-contained exposition, which provides all the
  necessary mathematical background, and allows to read this thesis as a book.
\end{enumerate}

\begin{landscape}
\begin{table}
  \begin{center}
  \resizebox{\textwidth}{!}{
  \begin{tabular}{p{16cm}l}
  publication & cited in\\\hline\hline
  \emph{Semi-Supervised Domain Adaptation with Non-Parametric Copulas} & \\
  David Lopez-Paz, Jos\'e Miguel Hern\'andez-Lobato and Bernhard Sch\"olkopf& Sections~\ref{sec:nonparametric}, \ref{sec:adaptation}\\
  NIPS, 2012 \citep{dlp-ssl}\\\hline
  \emph{Gaussian Process Vine Copulas for Multivariate Dependence} & \\
  David Lopez-Paz, Jos\'e Miguel Hern\'andez-Lobato and Zoubin Ghahramani & Sections~\ref{sec:conditional}, \ref{sec:conditional_experiments}\\
  ICML, 2013 \citep{dlp-gp} & \\\hline
  \emph{The Randomized Dependence Coefficient} \\
  David Lopez-Paz, Philipp Hennig and Bernhard Sch\"olkopf & Section~\ref{sec:dependence-measures}, \ref{sec:rca-experiments}\\
  NIPS, 2013 \citep{dlp-rdc} & \\\hline
  \emph{Two Numerical Models of Saturn Rings Temperature as Measured by Cassini}&\\
  Nicolas Altobelli, David Lopez-Paz et al. & ---\\
  Icarus, 2014 \citep{dlp-icarus} & \\\hline
  \emph{Randomized Nonlinear Component Analysis} & \\
  David Lopez-Paz, Suvrit Sra, Alex Smola, Zoubin Ghahramani and Bernhard Sch\"olkopf & Section~\ref{sec:component-analysis}, \ref{sec:rca-experiments}\\
  ICML, 2014 \citep{dlp-rca}  & \\\hline
  \emph{The Randomized Causation Coefficient} & \\
  David Lopez-Paz, Krikamol Muandet and Benjamin Recht & Chapter~\ref{chapter:learning-causal-relations}\\
  JMLR, 2015 \citep{dlp-jmlr} & \\\hline
  \emph{Towards A Learning Theory of Cause-Effect Inference} &\\
  David Lopez-Paz, Krikamol Muandet, Bernhard Sch\"olkopf and Iliya Tolstikhin& Chapter~\ref{chapter:learning-causal-relations}\\
  ICML, 2015 \citep{dlp-clt} & \\\hline
  \emph{No Regret Bound for Extreme Bandits} & \\
  Robert Nishihara, David Lopez-Paz and L\'eon Bottou & Section~\ref{sec:learning-nns}\\
  AISTATS, 2016 \citep{dlp-bandits} & \\\hline
  \emph{Non-linear Causal Inference using Gaussianity Measures} & \\
  Daniel Hernandez-Lobato, Pablo Morales Mombiela, David Lopez-Paz and Alberto Suarez& Section~\ref{sec:causal-inference-algorithms} \\
  JMLR, 2016 \citep{dlp-gauss} & \\\hline
  \emph{Unifying distillation and privileged information} & \\
  David Lopez-Paz, L\'eon Bottou, Bernhard Sch\"olkopf, Vladimir Vapnik & Section~\ref{sec:iclr}\\
  ICLR, 2016 \citep{dlp-distillation} & \\\hline\hline
  \emph{Discovering causal signals in images} & \\
  David Lopez-Paz, Robert Nishihara, Soumith Chintala, Bernhard Sch\"olkopf, L\'eon Bottou & ---\\
  Under review, 2016 \citep{dlp-img} & \\\hline\hline
  \emph{Lower bounds for realizable transductive learning} & \\
  Ilya Tolstikhin, David Lopez-Paz & ---\\
  Under review, 2016 \citep{dlp-colt} & \\\hline\hline
  \end{tabular}
  }
  \end{center}
  \caption{Main publications of the author.}
  \label{table:contributions}
\end{table}
\end{landscape}

  \part{Background}
  \chapter{Mathematical preliminaries}\label{chapter:mathematics}
\vspace{-1.25cm}
  \emph{This chapter is a review of well-known results.}
\vspace{1.25cm}

\noindent This chapter introduces the necessary mathematics to understand this
thesis. We will review well known but important results about linear algebra,
probability theory, machine learning, and numerical optimization.

\section{Linear algebra}
This section studies vector spaces over the field of the real numbers. It
reviews basic concepts about vectors and matrices, as well as their
respective infinite-dimensional generalizations as functions and operators.

\subsection{Vectors}\label{sec:vectors}
Vectors $u \in \Rd$ are one-dimensional arrays of $d$ numbers
\begin{equation*}
  u = (u_1, \ldots, u_d)^\top.
\end{equation*}
The \emph{inner product} between two vectors $u, v \in \R^d$ is 
\begin{equation*}
  \dot{u}{v}_{\Rd} = \sum_{i=1}^n u_i v_i,
\end{equation*}
where we omit the subscript $\Rd$ whenever this causes no confusion.  Using the
inner product, we measure the ``size'' of a vector $u \in \Rd$ using its
\emph{norm}
\begin{equation*}
  \|u\| = \sqrt{\dot{u}{u}}.
\end{equation*}
Using the norm, we define the \emph{distance} between two vectors as
\begin{equation*}
  \|u-v\| = \sqrt{\dot{u-v}{u-v}}.
\end{equation*}
Two vectors $u, v$ are \emph{orthogonal} if $\dot{u}{v} = 0$. A vector $u$ is an \emph{unit
vector} if $\|u\|=1$. If two vectors are unit and orthogonal, they are
\emph{orthonormal}.  For any two vectors $u, v \in \Rd$, the
\emph{Cauchy-Schwartz inequality} states that
\begin{equation*}
  |\dot{u}{v}| \leq \|u\| \, \|v\|.
\end{equation*}
One consequence of the
Cauchy-Schwartz inequality is the \emph{triangle inequality}
\begin{equation*}
  \|u+v\| \leq \|u\| + \|v\|,
\end{equation*}
where the two previous inequalities are valid for all $u,v\in\Rd$.

The previous results hold for any norm, although we will focus in the Euclidean
norm $\|\cdot\|_2 = \|\cdot\|$, one special case of the $p$-norm 
\begin{equation*}
  \|x\|_p = \pa{\sum_{i=1}^d \left|x_i\right|^p}^{1/p},
\end{equation*}
when $p=2$. 

\subsection{Matrices}\label{sec:matrices}
Real matrices $X\in \R^{n\times d}$ are two-dimensional arrangements of real numbers
\begin{equation*}
X = \left( \begin{array}{ccc}
X_{1,1} & \cdots & X_{1,d} \\
\vdots & \ddots & \vdots \\
X_{n,1} & \cdots & X_{n,d} \end{array} \right).
\end{equation*}
We call the vector $X_{i,:} \in\Rd$ the $i$-th row of $X$, the vector $X_{:,j}
\in \Rn$ the $j$-th column of $X$, and the number $X_{i,j} \in \R$ the
$(i,j)$-entry of $X$. Unless stated otherwise, vectors $u \in
\R^{d}$ are \emph{column matrices} $u \in \R^{d\times 1}$. We adopt the usual
associative, distributive, but not commutative matrix multiplication. Such 
product of two matrices $A \in \R^{n \times d}$ and $B \in \R^{d\times m}$ has
entries
\begin{equation*}
  (AB)_{i,j} = \sum_{k=1}^d A_{i,k}B_{k,j},
\end{equation*}
for all $1 \leq i \leq n$ and $1 \leq j \leq m$.
We call the matrix $X^\top$ the \emph{transpose} of $X$, and it satisfies
$X^\top_{j,i} = X_{i,j}$ for all $1 \leq i \leq n$ and $1 \leq j \leq d$.  Real
matrices $X \in \R^{n \times d}$ are \emph{square} if $n = d$, and
\emph{symmetric} if $X=X^\top$.  \emph{Orthogonal} matrices $X \in
\R^{n\times n}$ have orthonormal vectors for rows and columns. \emph{Unitary matrices} $U$ satisfy $U^\top U = I$.
The vector $\text{diag}(X) = (X_{1,1}, \ldots, X_{\min(n,d),\min(n,d)})$ is the
\emph{diagonal} of the matrix $X \in \R^{n\times d}$.  \emph{Diagonal} matrices
have nonzero elements only on their diagonal. The \emph{identity} matrix $I_n
\in \R^{n \times n}$ is the diagonal matrix with
$\text{diag}(I_n)=(1,\ldots,1)$.

Real symmetric matrices $X \in \R^{n \times n}$ are \emph{positive-definite} 
if $z^\top X z > 0$ for all nonzero $z$, or if all its
eigenvalues are positive. Similarly, a real symmetric matrix is \emph{positive-semidefinite}
if $z^\top X z \geq 0$ for all nonzero $z$. For positive-definite
matrices we write $X \succ 0$, and for positive-semidefinite matrices we write
$X \succeq 0$. If $X$, $Y$ and $X-Y$ are three positive-definite matrices, we may
establish the \emph{L\"owner order} between $X$ and $Y$ and say $X \succ Y$. Positive
semidefinite matrices $X$ satisfy $X = SS =: S^2$ for an unique $S$, called the
square root of $X$. Finally, if $X$ is positive-definite, then $Q^\top X Q$ is
positive-definite for all $Q$. 

The \emph{rank} of a matrix is the number of linearly independent columns.
These are the columns of a matrix that we can not express as a linear
combination of the other columns in that same matrix. Alternatively, the rank of a matrix is
the dimensionality of the vector space spanned by its columns. The rank of a
matrix is equal to the rank of its transpose. A matrix $X \in \Rnd$ is
\emph{full rank} if $\text{rank}(X) = \min(n,d)$.

Given a full rank matrix $X \in \R^{n \times n}$, we call the unique matrix $B$
satisfying $AB=BA=I_n$ the \emph{inverse matrix} of $A$. Orthogonal matrices
satisfy $X^\top = X^{-1}$.  Square diagonal matrices $D$ have diagonal inverses
$D^{-1}$ with $\text{diag}(D^{-1}) = (D^{-1}_{1,1}, \ldots, D^{-1}_{n,n})$.
Positive-definite matrices have positive-definite inverses. One useful matrix
identity involving inverses is the \emph{Sherman-Morrison-Woodbury formula}:
\begin{equation}\label{eq:smwf}
  \left(A+UCV \right)^{-1} = A^{-1} - A^{-1}U \left(C^{-1}+VA^{-1}U
  \right)^{-1} VA^{-1}.
\end{equation}

As we did with vectors, we can calculate the ``size'' of a matrix using its norm.
There are a variety of matrix norms that we can use. \emph{Subordinate norms}
have form
\begin{equation*}
  \|A\|_{\alpha, \beta} = \sup\left\lbrace \frac{\|Ax\|_\alpha}{\|x\|_\beta} :
  x \in \Rn, x \neq 0 \right\rbrace,
\end{equation*}
where the norms $\|\cdot\|_\alpha$ and $\|\cdot\|_\beta$ are vector norms
satisfying
\begin{equation*}
  \|Ax\|_\beta \leq \|A\|_{\alpha, \beta} \|x\|_\alpha.
\end{equation*}
We adopt the short hand notation $\|A\|_{\alpha, \alpha} = \|A\|_\alpha$. The
particular case $\|A\|_2$ is the \emph{operator norm}. Another example
of matrix norms are \emph{entrywise norms}:
\begin{equation*}
  \|A\|_{p,q} = \left(\sum_{i=1}^n \left(\sum_{j=1}^d |A_{i,j}|^p\right)^{q/p}
  \right)^{1/q}.
\end{equation*}
When $p=q=2$, we call this norm the \emph{Frobenius norm}. All matrix norms are
equivalent, in the sense that, for two different matrix norms $\|\cdot\|_a$ and
$\|\cdot\|_b$, there exists two finite constants $c,C$ such that $c\|A\|_a \leq
\|A\|_b \leq C\|A\|_a$, for all matrices $A$. The operator and Frobenius norms
are two examples of \emph{unitary invariant norms}: $\|A\| = \|UA\|$ for any matrix $A$ and unitary
matrix $U$. Unless specified otherwise, $\|A\|$ will denote the operator norm
of $A$.

Given a symmetric matrix $A \in \R^{n\times n}$ and a real number $\lambda$,
the nonzero vectors $v$ that satisfy
\begin{equation*}
  Av = \lambda v
\end{equation*}
are the \emph{eigenvectors of $A$}. Each \emph{eigenvector} $v$ is orthogonal to the others,
and has an \emph{eigenvalue} $\lambda$ associated with it. Geometrically,
eigenvectors are those vectors that, when we applied to the linear
transformation given by some matrix $A$, change their magnitude by $\lambda$
but remain constant in direction. For symmetric matrices, eigenvectors and eigenvalues
provide with the \emph{eigendecomposition}
\begin{equation*}
  A = Q \Lambda Q^\top,
\end{equation*}
where the columns of $Q \in \R^{n\times n}$ are the eigenvectors of $A$, and
the real entries of the diagonal matrix $\Lambda \in \R^{n\times n}$ contain the
associated eigenvalues. By convention, we arrange the decomposition such that
$\Lambda_{1,1} \geq \Lambda_{2,2} \geq \cdots \geq \Lambda_{n,n}$. If the
matrix $A$ is asymmetric, different formulas apply \citep{horn2012matrix}.
In short, the eigendecomposition of a matrix informs about the 
directions and magnitudes that the linear operation $Ax$ shrinks or
expands vectors $x$.

\emph{Singular values} generalize the concept of eigenvectors, eigenvalues, and
eigendecompositions to rectangular matrices. The singular value decomposition of
matrix $A \in \R^{n \times d}$ is
\begin{equation*}
  A = U \Sigma V^\top,
\end{equation*}
where $U \in \R^{n\times n}$ is an orthogonal matrix whose columns we call the
\emph{left singular vectors} of $A$, $\Sigma \in \R^{n\times d}$ is a diagonal
matrix whose positive entries we call the \emph{singular values} of $A$, and $V \in
\R^{d\times d}$ is an orthogonal matrix whose columns we call the \emph{right
singular vectors} of $A$.  The eigenvalue and singular value decompositions
relate to each other. The left singular vectors of $A$ are the eigenvectors of
$AA^\top$. The right singular vectors of $A$ are the eigenvectors of $A^\top
A$. The nonzero singular values of $A$ are the square root of the nonzero
eigenvalues of both $AA^\top$ and $A^\top A$.  The operator norm relates to the
largest eigenvalue $\Lambda_{1,1}$ and the largest singular value
$\Sigma_{1,1}$ of a square matrix $A$ as
\begin{equation*}
  \|A\|_2 = \sqrt{\Lambda_{1,1}} = \Sigma_{1,1}.
\end{equation*}
Thus, the operator norm upper bounds how much does the matrix $A$ 
modify the norm of a vector.

The product of all the eigenvalues of a matrix $A$ is the \emph{determinant}
$|A|$. The sum of all the eigenvalues of a matrix $A$ is the
\emph{trace} $\text{tr}(A)$. The trace is also equal to the sum of the elements
in the diagonal of the matrix.  The trace and the determinant are
similarity-invariant: the trace and the determinant of two matrices $A$ and
$B^{-1} A B$ are the same, for all $B$.

For a more extensive exposition on matrix algebra, 
consult \citep{golub2012matrix,horn2012matrix,petersen2008matrix}.

\subsection{Functions and operators}
Section~\ref{sec:vectors} studied $d$-dimensional vectors, which live in the
$d$-dimensional Euclidean space $\Rd$. These are vectors $u$ with $d$
components $u_1, \ldots, u_d$, indexed by the integers $\{1, \ldots, d\}$.  In
contrast, it is possible to define \emph{infinite-dimensional vectors} or
\emph{functions}, which live in a infinite-dimensional \emph{Hilbert Space}. A Hilbert space $\H$ is
a vector space, equipped with an inner product $\dot{f}{g}_\H$, such that the
norm $\|f\| = \sqrt{\dot{f}{f}_\H}$ turns $\H$ into a complete metric space.

The key intuition here is the analogy between infinite-dimensional vectors and
functions. Let us consider the Hilbert space $\H$ of functions $f: \R \to \R$.
Then, the infinite-dimensional vector or function $f \in \H$ has shape
$f=(f(x))_{x\in\R}$, where the indices are now real numbers $x$, arguments to
the function $f$.

Similarly, \emph{linear operators} are the infinite-dimensional
extension of matrices. While matrices $A \in \R^{n\times d}$ are linear
transformations of vectors $u \in \Rd$ into vectors $Au \in \Rn$, linear
operators $L : \F \to \H$ are linear transformations of  functions $f \in \H$
into functions $g \in \H$. In the following, assume that $\F$ and $\H$ contain
functions from $\X$ to $\R$.  We say that the linear operator $L : \F \to \H$
is bounded if there exists a $c
> 0$ such that
\begin{equation*}
  \|Lf\|_\H \leq c\|f\|_\F,
\end{equation*}
for all nonzero $f\in\F$. A linear operator is bounded if and only if it is
continuous. If a bounded operator $L$ has finite \emph{Hilbert-Schmidt} norm
\begin{equation*}
  \|L\|^2_{\text{HS}} = \sum_{i\in \mathcal{I}}  \|L f_i\|^2,
\end{equation*}
we say the operator is a \emph{Hilbert-Schmidt operator}. In the previous, the
set $\{f_i : i \in \mathcal{I}\}$ is an orthonormal basis on $\F$.  Finally, we
say that an operator $T$ is an integral transform if it admits the expression
\begin{equation*}
  (Tf)(u) = \int K(t,u) f(t) \d t,
\end{equation*}
for some kernel function $K : \X \times \X$. For example, by choosing the
  kernel $$K(t,u) = \frac{e^{-\text{\i} u t}}{\sqrt{2 \pi}},$$ there
  $\text{\i}$ denotes the imaginary unit, we obtain the Fourier transform.

Most of the material presented for vectors and matrices extends to functions
and operators: the Cauchy-Schwartz inequality, the triangle inequality, eigen
and singular value decompositions, and so on. We recommend the monograph of
\citet{reed1972functional} to learn more about functional analysis.

\section{Probability theory}
Probability theory studies \emph{probability
spaces}.  A probability space is a triplet $(\Omega, \B(\Omega),
\Pr)$. Here, the \emph{sample space} $\Omega$ is the collection of outcomes of a 
random experiment. For example, the sample space of a ``coin flip'' is the set $\Omega= \{\text{heads},
\text{tails}\}$.  The
\emph{$\sigma$-algebra} $\B(\Omega)$ is a nonempty collection of subsets of
$\Omega$ such that i) $\Omega$ is in $\B(\Omega)$, ii) if $A \in \B(\Omega)$,
so is the complement of $A$, and iii) if $A_n$ is a sequence of elements of
$\B(\Omega)$, then the union of $A_n$ is in $\B(\Omega)$.  Using De Morgan's
law, one can also see that if $A_n$ is a sequence of elements of $\B(\Omega)$,
then the intersection of $A_n$ is in $\B(\Omega)$. The power set of $\Omega$ is
the largest $\sigma$-algebra of $\Omega$. In plain words, the $\sigma$-algebra
$\B(\Omega)$ is the collection of all the events (subsets of the sample space)
that we would like to consider.  Throughout this thesis, $\B(\Omega)$ will be
the \emph{Borel} $\sigma$-algebra of $\Omega$. The \emph{probability measure} $\Pr$ is a
function $\B(\Omega) \to [0,1]$, such that $\Pr(\emptyset) = 0$, $\Pr(\Omega) =
1$, and $\Pr(\cup_i A_i) = \sum_i \Pr(A_i)$ for all countable collections
$\{A_i\}$ of pairwise disjoint sets $A_i$. For a fair coin, we could have
$\Pr(\{\text{heads}\}) = \Pr(\{\text{tails}\}) = \frac{1}{2}$, $\Pr(\emptyset)
= 0$ and $\Pr(\{\text{heads}, \text{tails}\}) = 1$.

\begin{remark}[Interpretations of probability]
  There are two main interpretations of the concept of \emph{probability}.
  \emph{Frequentist probability} is the limit of the relative frequency of an
  event. For instance, if we get heads $h(n)$ times in $n$ tosses of the same
  coin, the frequentist probability of the event ``heads'' is
  $\lim_{n\to \infty} h(n)/n$. On the other hand, \emph{Bayesian probability}
  measures the degree of belief or plausibility of a given event.  One way to
  understand the difference between the two is that frequentism considers data
  a random quantity used to infer a fixed parameter. Conversely, Bayesianism
  considers data a fixed quantity used to infer the distribution of a random
  parameter.

  We say that Frequentist interpretations of probability are \emph{objective},
  since they rely purely on the observation of repetition of events.
  Conversely, Bayesian interpretations of probability are \emph{subjective},
  since they combine observations of events with prior beliefs not contained in
  the data nor the statistical model. It is beneficial to see both approaches
  as complementary: frequentist methods offer a formalism to study repeatable
  phenomena, and Bayesian methods offer a formalism to replace repeatability
  with uncertainty modeled as subjective probabilities. 
\end{remark}

Probability spaces $(\Omega, \B(\Omega), \Pr)$ are the basic building blocks to
define \emph{random variables}.  Random variables take different
values at random, each of them with probability given by the
probability measure $\Pr$.  More specifically, let $(\X, \B(\X))$ be some
measurable space.  Then, a random variable taking values in $(\X, \B(\X))$ is a
$(\B(\Omega), \B(\X))$-measurable function $\bm x \colon \Omega \to \X$.  For
example, consider real-valued random variables, that is $\X=\R$. Then, the
answer to the question ``What is the probability of the random variable $\bm x$
taking the value $3.5 \in \R$?'' is
\begin{equation*}
  \Pr(\{ \omega \in\Omega : \bm x(\omega) = 3.5 \}) =: \Pr(\bm x = 3.5).
\end{equation*}
Intuitively, random variables measure some property of an stochastic system.
Then, the probability of the stochastic system $\bm x$ taking a particular value $x \in
\X$ is the probability of the set of possible outcomes $\omega \in
\Omega$ satisfying $\bm x(\omega) = x$.
This thesis studies dependence and causation by characterizing sets of
\emph{random variables} and their relationships. For a cheat sheet on
statistics, see \citep{statistics-zone}.

\subsection{Single random variables}
We are often interested in the probability of a random variable taking values
over a certain range. \emph{Cumulative distribution functions} use probability
measures to compute such probabilities.
\begin{definition}[Cumulative distribution function]
  The cumulative distribution function (cdf) or distribution of a
  real random variable $\bm x$ is
  \begin{equation*}
    P(\bm x = x) = P(x) = \Pr(\bm x \leq x).
  \end{equation*}
\end{definition}
  
Distribution functions are nondecreasing and right-continuous. If the cdf $P$
is strictly increasing and continuous, the inverse cdf $P^{-1}$ is the
\emph{quantile} function. One simple way to estimate distributions 
from data is to use the empirical measure. 

\begin{definition}[Empirical measure]\label{def:empirical_measure}
  Consider the sample $x_1, \ldots, x_n \sim P(\bm x)$. Then, the empirical
  probability measure of this sample is
  \begin{equation*}
    \Pr_n(\omega) = \frac{1}{n} \sum_{i=1}^n \I(x_i \in \omega)
  \end{equation*}
  for all events $\omega \subseteq \Omega$.
\end{definition}

One central use of the empirical measure is to define the empirical
distribution function:

\begin{definition}[Empirical cumulative distribution function]\label{def:ecdf}
  The empirical cumulative distribution function (ecdf) of $x_1, \ldots, x_n
  \sim P(\bm x)$ is
  \begin{equation*}
    P_n(\bm x = x) = \Pr_n((-\infty, x]) = \frac{1}{n} \sum_{i=1}^n \I(x_{i} \leq x).
  \end{equation*}
\end{definition}

The ecdf converges uniformly to the true cdf, as the sample size $n$
grows to infinity. This uniform convergence is exponential, as stated in the
next fundamental result.

\begin{theorem}[Dvoretzky-Kiefer-Wolfowitz-Massart inequality]\label{thm:massart}
Let $x_1, \ldots, x_n \sim P(\bm x)$ be a real-valued sample. Then, for all $t > 0$,
\begin{equation*}
  \Pr\left(\sup_{x \in \R} \left| P_n(x) - P(x)\right| > t\right) \leq 2e^{-2nt^2}.
\end{equation*}
\end{theorem}
\begin{proof}
  See \citep{massart1990tight}.
\end{proof}

Sometimes we want to determine how likely it is that a random variable takes a
certain value.  For continuous random variables with differentiable cdfs, 
\emph{the probability density function} provides us with these likelihoods.

\begin{definition}[Probability density function]
  The probability density function (pdf) of a real random variable $\bm
  x$ is
  \begin{equation*}
    p(\bm x = x) = p(x) = \frac{\d}{\d x} P(x) 
  \end{equation*}
  Pdfs satisfy $p(x) \geq 0$ for all $x \in \X$, and
  $\int_\X p(x) \d x = 1$.
\end{definition}

For discrete random variables, these likelihoods are given by
the probability mass function.
\begin{definition}[Probability mass function]
  The probability mass function (pmf) of a random variable $\bm x$ over a discrete space $\X$ is 
  \begin{equation*}
    p(x) = \Pr(\bm x = x).
  \end{equation*}
  Pmfs are nonnegative for all $x \in \X$, and zero for all $x \notin \X$. Pmfs
  satisfy $\sum_{x \in \X} p(x) = 1$.
\end{definition}

In many cases we are interested in summaries of random variables.
One common way to summarize a random variable into $k$ numbers is to use
its first $k$ moments. The $n$-th \emph{moment} of a random variable $\bm x$ is
\begin{equation*}
  \E{\bm x}{\bm x^n} = \E{}{\bm x^n} = \int_{-\infty}^\infty x^n \mathrm{d} P(\bm x = x),
\end{equation*}
and the $n$th \emph{central moment} of a random variable $\bm x$ is 
\begin{equation*}
  \E{}{(\bm x-\E{}{\bm x})^n}.
\end{equation*}
The first moment of a random variable $\E{}{\bm x}$ is the mean or expected
value of $\bm x$, and characterizes how does the ``average'' sample from $P(\bm
x)$ looks like.  The second central moment of a random variable $\V{}{\bm x} =
\E{}{(\bm x - \E{}{\bm x })^2}$ is the variance of $\bm x$, and
measures the spread of samples drawn from $P(\bm x)$ around its mean $\E{}{\bm x}$.

\subsection{Multiple random variables}

Now we turn to the joint study of collections of random variables. We can study
a collection of real-valued random variables $\bm x_1, \ldots, \bm x_d$ as the
vector-valued random variable $\bm x = (\bm x_1, \ldots, \bm x_d)$. Thus, $\bm
x$ is a random variable taking values in $\Rd$.  The cdf of $\bm x$ is
\begin{equation*}
  P(t) = \Pr(\bm x \leq t) = \Pr(\bm x_1 \leq t_1, \ldots, \bm x_d \leq t_d),
\end{equation*}
In the multivariate case, the empirical measure from
Definition~\ref{def:empirical_measure} takes the same form, and the ecdf is 
\begin{equation*}
  P_n(t) = \frac{1}{n} \sum_{i=1}^n \I(x_{i,1} \leq t_1, \ldots, x_{i,d} \leq t_d),
\end{equation*}
where $x_1, \ldots, x_n \sim P(\bm x)$.

The generalization of Theorem~\ref{thm:massart} to multivariate random variables is
a groundbreaking result by Vapnik and Chervonenkis. 

\begin{theorem}[Vapnik-Chervonenkis] Let $\X$ be a collection of measurable
sets in $\Rd$. Then, for all $n t^2 \geq 1$,
\begin{equation*}
  \Pr\left(\sup_{x\in \X} \left| \Pr(x) -\Pr_n(x) \right| >
  t\right) \leq 4 s(\X, 2n)e^{-nt^2/8},
\end{equation*}
where 
\begin{equation*}
  s(\X,n) = \max_{x_1, \ldots, x_n \in \Rd} \left| \{\{x_1, \ldots, x_n\} \cap
  X : X \in \X\}\right| \leq 2^n
\end{equation*}
is known as the \emph{$n$-th shatter coefficient} of $\X$.
\end{theorem}
\begin{proof}
  See \citet{vapnik1971uniform}.
\end{proof}

The density function of a vector-valued random variable $\bm x$ is
\begin{equation*}
  p(x) = \left.\frac{\partial^d P(x)}{\partial x_1 \cdots \partial x_d} \right|_x.
\end{equation*}

Now we review two fundamental properties relating the density or mass functions
of two random variables. First, the \emph{total probability rule}
\begin{align*}
  p(\bm x) &= \sum_{y \in \Y} p(\bm x, \bm y=y),\\
  p(\bm x) &= \int p(\bm x, \bm y = y) \d y,
\end{align*}
for discrete and continuous variables, respectively, is useful to compute the \emph{marginal distribution} $p(\bm x)$ of a single
random variable $\bm x$ given the \emph{joint distribution} $p(\bm x, \bm y)$
of two random variables $\bm x$ and $\bm y$. Second, \emph{conditional
probability rule}
\begin{align*}
  p(\bm x\given \bm y = y) = p(\bm x\given y) = \frac{p(\bm x,y)}{p(y)},
\end{align*}
is useful to compute the \emph{conditional distribution} $p(\bm x \given y)$ of
the random variable $\bm x$ when another random variable $\bm y$ takes the
value $y$, whenever $p(y) \neq 0$. Applying the conditional probability rule
in both directions yields \emph{Bayes' rule}
\begin{align*}
  p(\bm x = x \given \bm y = y) &= \frac{p(\bm y = y \given \bm x = x)p(\bm x = x)}{p(\bm y=y)}.
\end{align*}

\begin{remark}[One slight abuse of notation]
Throughout this thesis, the cumulative distribution \emph{function} $P(\bm x)$
takes \emph{values}
\begin{equation*}
  P(\bm x = x) = P(x),
\end{equation*}
and the probability density \emph{function} $p(\bm x)$ takes \emph{values} 
\begin{equation*}
  p(\bm x = x) = p(x).
\end{equation*}
All these notations will be used interchangeably whenever this causes no
confusion.
The notation of conditional distributions is more subtle. While the element
$p(\bm x \given y)$ is a function, the element $p(\bm x = x \given \bm y = y) =
p(x \given y)$ is a number. 
\end{remark}

Let $\bm x$ and $\bm y$ be two continuous random variables taking values in
$\X$ and $\Y$ respectively, and let $g : \X \to \Y$ be a bijection. Then:
\begin{equation}\label{eq:transformation-densities}
  p(\bm y = y) = \left| \frac{\d}{\d y} g^{-1}(y) \right| p(\bm x = g^{-1}(y)).
\end{equation}

For two random variables taking values $x \in \X$, the distance between their
respective density functions $p$ and $q$ is often measured using the
Kullback-Liebler (KL) divergence
\begin{equation}
  \label{eq:kl}
  \mathrm{KL}(p \,\|\, q) = \int_\X p(x) \log \frac{p(x)}{q(x)} \mathrm{d}x.
\end{equation}

As it happened with single random variables, we can create summaries of
multiple random variables and their relationships. One of these 
summaries are the $(r_x,r_y)$-mixed central moments: 
\begin{equation*}
  \E{}{(\bm x-\E{}{\bm x})^{r_x}(\bm y -\E{}{\bm y})^{r_y}},
\end{equation*}
For example, the $(2,2)$-mixed central moment of two random variables $\bm x$ and $\bm
y$ is their covariance 
\begin{equation*}
  \text{cov}(\bm x, \bm y) = \E{}{(\bm x-\E{}{\bm x})^2(\bm y -\E{}{\bm y})^2}.
\end{equation*}
When normalized, the covariance statistic becomes the correlation statistic 
\begin{equation*}
 \rho(\bm x, \bm y) = \frac{\text{cov}(\bm x, \bm y)}{\sqrt{\V{}{\bm x}\V{}{\bm y}}} \in [-1,1].
\end{equation*}
The correlation statistic describes to what extent the joint distribution of
$\bm x$ and $\bm y$ can be described with a straight line. In other words,
correlation measures to what extent two random variables are \emph{linearly}
dependent. Similar equations follow to derive the mixed central moments of a
collection of more than two random variables.

\subsection{Statistical estimation}\label{sec:estimation}

The crux of statistics is to identify interesting aspects $\theta$ about
random variables $\bm x$, and to approximate them as estimates ${\theta}_n(x) = \theta_n$ using $n$
samples $x = x_1, \ldots, x_n \sim P(\bm x)^n$. One can design multiple estimates
${\theta}_n$ for the same quantity $\theta$; therefore, it is interesting to
quantify and compare the quality of different estimates, in order to favour one of
them for a particular application. Two of the most important quantities about
statistical estimators are their \emph{bias} and \emph{variance}.

On the one hand, the \emph{bias} measures the deviation between the quantity of
interest $\theta$ and the expected value of our estimator $\theta_n$:
\begin{equation*}
  \textrm{Bias}({\theta}_n,\theta) = \E{x \sim P^n}{{\theta}_n} - \theta.
\end{equation*}
Estimators with zero bias are \emph{unbiased estimators}. Unbiasedness is
unrelated to consistency, where consistency means that the estimator converges
in probability to the true value being estimated, as the sample size increases
to infinity. Thus, unbiased estimators can be inconsistent, and consistent
estimators can be biased.

On the other hand, the \emph{variance} of an estimator
\begin{equation*}
  \textrm{Variance}({\theta}_n) = \E{}{({\theta}_n-\E{x\sim P^n}{{\theta}_n})^2}
\end{equation*}
measures its dispersion around the mean. The sum of the variance
and the square of the bias is equal to the mean squared error of the estimator
\begin{equation*}
  \textrm{MSE}({\theta}_n) = \textrm{Variance}({\theta}_n) + \textrm{Bias}^2({\theta}_n,\theta).
\end{equation*}
This reveals a key trade-off: different estimators achieving the same mean
square error can have a different bias-variance decompositions. As a matter of
fact, bias and variance are in many cases competing quantities.  We discuss 
this fundamental issue in Section~\ref{sec:bias}.

\subsection{Concentration inequalities}

We now review useful results concerning the concentration of averages of
independent random variables. For a more extensive treatment, consult
\citep{Boucheron13}. 

\begin{theorem}[Jensen's inequality] Let $\bm x$ be a random variable taking
values in $\X$, and let $f : \X \to \R$ be a convex function. Then, for all
$\bm x$ and $f$,
\begin{equation*}
  f(\E{}{\bm x}) \leq \E{}{f(\bm x)}.
\end{equation*}
\end{theorem}

\begin{theorem}[Markov's inequality] Let $\bm x$ be a random variable taking
values in the nonnegative reals. Then,
\begin{equation*}
  \Pr(\bm x \geq a) \leq \frac{\E{}{\bm x}}{a}.
\end{equation*}
\end{theorem}

Markov's inequality is tightly related to Chebyshev's inequality.

\begin{theorem}[Chebyshev's inequality] Let $\bm x$ be a random variable with
finite expected value $\mu$ and finite variance $\sigma^2 \neq 0$. Then, for
all $k > 0$,
\begin{equation*}
  \Pr(|\bm x - \mu| \geq k\sigma) \leq \frac{1}{k^2}.
\end{equation*}
\end{theorem}

As opposed to the polynomial concentration of Markov's and Chebyshev's
inequalities, the more sophisticated \emph{Chernoff bounds} offer exponential
concentration. The simplest Chernoff bound is Hoeffding's inequality.

\begin{theorem}[Hoeffding's inequality] Let $\bm x_1, \ldots, \bm x_n$ be a
collection $n$ independent random variables, where $\bm x_i$ takes values in
$[a_i,b_i]$, for all $1 \leq i \leq n$. Let $\bar{\bm x} = \sum_{i=1}^n
\bm x_i$. Then, for all $t > 0$, 
\begin{equation*}
  \Pr(\bar{\bm x} - \E{}{\bar{\bm x}} \geq t) \leq
  \exp\left(-\frac{2nt^2}{\sum_{i=1}^n (b_i-a_i)^2}\right).
\end{equation*}
\end{theorem}

One can sharpen Hoeffding's inequality by taking into account the variance of
$\bm x_1, \ldots, \bm x_n$, giving rise to Bernstein's inequality.

\begin{theorem}[Bernstein's inequality]\label{thm:bernstein} Let $\bm x_1, \ldots, \bm x_n$ be a
collection of $n$ independent random variables with zero-mean, where $|\bm x_i|
\leq M$ for all $1 \leq i \leq n$ almost surely. Let $\bar{\bm x} =
\sum_{i=1}^n \bm x_i$. Then, for all $t > 0$
\begin{equation*}
  \Pr(\bar{\bm x} \geq t) \leq
  \exp\left(-\frac{1}{2}\frac{t^2}{\sum_{i=1}^n \E{}{\bm x_i^2} + \frac{1}{3}Mt}\right).
\end{equation*}
Furthermore, random variables $\bm z$ with Bernstein bounds of the form
\begin{equation*}
  \Pr(\bm z \geq t) \leq
  C\exp\left(-\frac{1}{2}\frac{t^2}{A+Bt}\right)
\end{equation*}
admit the upper bound 
\begin{equation*}
  \E{}{\bm z} \leq 2\sqrt{A}(\sqrt{\pi}+\sqrt{\log C}) + 4B(1+\log C).
\end{equation*}
\end{theorem}

We can also achieve concentration not only over random averages, but over more
general \emph{functions} $f$ of random variables, assuming that the function $f$ is
well behaved. One example of such results is McDiarmid's inequality.

\begin{theorem}[McDiarmid's inequality]\label{thm:mcdiarmids}
Let $\bm x = (\bm x_1, \ldots, \bm x_n)$ be a collection of $n$ independent random
variables taking real values, and let $f : \Rn \to \R$ be a function satisfying
\begin{equation*}
  \sup_{x_1, \ldots, x_n, x'_i} \left|f(x_1, \ldots, x_n) - f(x_1, \ldots,
  x'_i, \ldots, x_n)\right| \leq c_i
\end{equation*}
for all $1 \leq i \leq n$. Then, for all $t > 0$,
\begin{equation*}
  \Pr_{}(f(\bm x) - \E{}{f(\bm x)} \geq t)
  \leq \exp\left(-\frac{2t^2}{\sum_{i=1}^n c_i^2}\right).
\end{equation*}
\end{theorem}

A key tool in the analysis of the presented algorithms in this thesis is the
Matrix-Bernstein inequality \citep{tropp}, which mirrors
Theorem~\ref{thm:bernstein} for matrix-valued random variables.

\begin{theorem}[Matrix Bernstein's inequality]\label{thm:matrix-bernstein}
Let $\bm X_1, \ldots, \bm X_n$
be a collection of $n$ independent random variables taking values in $\R^{d_1
\times d_2}$, where $\E{}{\bm X_i} = 0$ and $\|\bm X_i\|_2 \leq M$. Let
$\bar{\bm X} = \sum_{i=1}^n \bm X_i$, and define
\begin{equation*}
  \sigma^2 = \max\left(\pn{\E{}{\bar{\bm X}^\top\bar{\bm X}}},
  \pn{\E{}{\bar{\bm X}\bar{\bm X}^\top}}\right).
\end{equation*}
Then, for all $t > 0$,
\begin{equation*}
  \Pr\left(\|\bar{\bm X}\|_2 \geq t\right) \leq (d_1 + d_2) \exp\left(-
  \frac{\frac{1}{2}t^2}{\sigma^2 + \frac{1}{3}Mt}\right).
\end{equation*}
Furthermore,
\begin{equation*}
  \E{}{\|\bar{\bm X}\|_2} \leq \sqrt{2 \sigma^2 \log(d_1+d_2)} + \frac{1}{3} M \log (d_1+d_2).
\end{equation*}
\end{theorem}

One last fundamental result that we would like to mention is the Union bound.

\begin{theorem}[Union bound] Let $E_1, \ldots, E_n$ be a collection of events. Then,
\begin{equation*}
   \Pr(\cup_{i=1}^n E_i) \leq \sum_{i=1}^n \Pr(E_i).
\end{equation*}
\end{theorem}

\section{Machine learning}

Imagine that I give you the sequence
\begin{equation*}
1,2,3,\ldots
\end{equation*}
and I ask: \emph{What number comes next?}

Perhaps the more natural answer is \emph{four}, assuming that the given sequence is
the one of the positive integers. A more imaginative answer could be
\emph{two}, since that agrees with the sequence of the greatest primes dividing
$n$. Or maybe \emph{five}, which agrees with the sequence of numbers not
divisible by a square greater than one. A more twisted mind would prefer the
answer \emph{two hundred and eleven}, since that is the next ``home'' prime.
In any case, the more digits that we observe from the sequence and the less
paranoid we are, the larger the amount of hypothesis we will be able to reject
and the closer we will get to inferring the correct sequence.
Machine learning uses the tools of probability theory and statistics to
formalize \emph{inference} problems like these.

This section reviews the fundamentals of learning theory, regression,
classification, and model selection. For a more extensive treatment on machine
learning topics, we refer the reader to the monographs
\citep{mohri2012foundations,murphy2012machine,shalev2014understanding}.

\subsection{Learning theory}\label{sec:learning-theory}
Consider two random variables: one \emph{input} random variable $\bm x$ taking
values in $\X$, and one \emph{output} random variable $\bm y$ taking values in
$\Y$.  The usual problem in learning theory is to find a \emph{function},
\emph{dependence}, or \emph{pattern} that ``best'' predicts values for the
output variable given the values taken by the input variable. We have three
resources to our disposal to solve this problem.  First, a \emph{sample} or \emph{data}
\begin{align}\label{eq:the-data}
 \D &= \{ (x_1,y_1) \ldots, (x_n,y_n) \} \sim P^n(\bm x, \bm y),
 \,\, x_i \in \X, \, y_i \in \Y.
\end{align}
Second, a \emph{function class} $\F$, which is a set containing
functions $f : \X \to \Y$. And third, a \emph{loss function} $\ell : \Y \to \Y$,
which penalizes departures between predictions $f(x)$ and true output values
$y$.  Using these three ingredients, one way to solve the learning problem is
to find the function $f
\in \F$ minimizing the \emph{expected risk}
\begin{equation}
  \label{eq:poprisk}
  R(f) = \int_{\X \times \Y} \ell(f(x),y)\d P(\bm x, \bm y).
\end{equation}

Unfortunately, we can not compute the expected risk \eqref{eq:poprisk}, since
we do not have access to the \emph{data generating distribution} $P$. Instead,
we are given a finite sample $\D$ drawn from $P^n$.  Therefore, we may use instead 
the available data to minimize the \emph{empirical risk}
\begin{equation*}
  R_n(f) = \frac{1}{n} \sum_{i=1}^n \ell(f(x_i),y_i),
\end{equation*}
which converges to the expected risk as the sample size grows, due to the law
of large numbers.

\begin{remark}[Some notations in learning]
  We call the set \eqref{eq:the-data} \emph{data} or \emph{sample}, where each
  contained $(x_i,y_i)$ is one \emph{example}. Examples contain \emph{inputs}
  $x_i$ and \emph{outputs or targets} $y_i$. When the targets are categorical, we
  will call them \emph{labels}. When the inputs are vectors in $\Rd$, then
  $x_{i,j} \in \R$ is the $j$th feature of the $i$th example.  Sometimes we
  will arrange the data \eqref{eq:the-data} in two matrices: the \emph{feature
  matrix} $X \in \Rnd$, where $X_{i,:} = x_i$, and the \emph{target matrix} $Y
  \in \R^{n\times q}$, where $Y_{i,:} = y_i$. Finally, we sometimes refer to
  the data \eqref{eq:the-data} as the \emph{raw representation} or the
  \emph{original representation}.
\end{remark}

Using the definitions of expected and empirical risk,
construct the two functions 
\begin{align*}
  f^\star   &= \argmin_{f \in \F} R(f),\\
  f_n &= \argmin_{f \in \F} R_n(f),
\end{align*}
called the \emph{expected risk minimizer}, and the \emph{empirical risk minimizer}.
We say that a learning algorithm is \emph{consistent} if, as the amount of
available data grows ($n \to \infty$), the output of the algorithm converges to the expected
risk minimizer. The speed at which this convergence happens with respect to $n$ is the
\emph{learning rate}.
Also, consider the function 
\begin{align*}
  g^\star &= \argmin_{g : \X \to \Y} R(g).
\end{align*}
The function $g^\star$ is the function from \emph{the set of all measurable
functions} attaining minimal expected risk in our learning problem.  We call
$g^\star$ the \emph{Bayes predictor}, and $R(g^\star)$ the \emph{Bayes error}.
Note that perhaps $g^\star \notin \F$!
In this setup, the goal of learning theory is
\begin{center}
``How well does $f_n$ describe the dependence between $\bm x$ and $\bm y$,
\\when compared to $g^\star$?''
\end{center}
Mathematically, the answer to this question splits in two parts:
\begin{align*}
  R(f_n) - R(g) = \underbrace{R(f_n) - R(f^\star)}_{\text{estimation error}} +
  \underbrace{R(f^\star) - R(g^\star)}_{\text{approximation error}}.
\end{align*}
The \emph{estimation error} arises because we approximate the expected risk
minimizer with the empirical risk minimizer. The \emph{approximation error}
arises because we approximate the best possible function $g^\star$ with the
best function from our function class $\F$.  Let's take a look at the analysis
of the estimation error. 
\begin{align}
R(f_n) &= R(f_n) - R(f^\star) + R(f^\star)\nonumber\\
       &\leq R(f_n) - R(f^\star) + R(f^\star) + R_n(f^\star) - R_n(f_n)\label{eq:emprisk1}\\
       &\leq 2 \sup_{f\in\F} | R(f) - R_n(f) | + R(f^\star).\label{eq:emprisk2}
\end{align}
The inequality \eqref{eq:emprisk1} follows because we know that the empirical
risk minimizer $f_n$ satisfies
\begin{equation}\label{eq:empcondition}
R_n(f^\star) - R_n(f_n) \geq 0.
\end{equation}
Importantly, $f_n$ is the only function for which we can assure
\eqref{eq:empcondition}. Thus, the guarantees of empirical risk minimization
only hold for function classes allowing the efficient computation of their
empirical risk minimizers. In practical terms, this often means that finding
$f_n$ is a convex optimization problem.  The inequality \eqref{eq:emprisk2}
follows by assuming twice the worst difference between the empirical and
expected risk of one function. Summarizing, the estimation error allows the
upper bound
\begin{equation}
  R(f_n) - R(f^\star) \leq 2 \sup_{f\in\F} |R(f)-R_n(f)|.\label{eq:suprema}
\end{equation}
The right-hand side of this inequality is the suprema of the empirical process
$\{|R(f)-R_n(f)|\}_{f\in\F}$. To upper bound this suprema in a meaningful way, we
first {measure} the complexity of the function class $\F$. Defined
next, Rademacher complexities are one choice to do this \citep{koltchinskii2001}.
\begin{definition}[Rademacher complexity]\label{def:rademacher}
Let $\F$ be a class of functions from $\X$ to $[a,b]$, $\D = (x_1,
\ldots, x_n)$ a vector in $\X^n$, and $\bm \sigma = (\bm \sigma_1, \ldots, \bm
\sigma_n)$ be a vector of uniform random variables taking values in
$\{-1,+1\}$.  Then, the \emph{empirical Rademacher complexity} of $\F$ is
  \begin{equation*}
    {\Rad}_\D(\F) = \E{}{\sup_{f\in\F} \frac{1}{n}
    \sum_{i=1}^n \bm \sigma_i f(x_i)}
  \end{equation*}
Given $\mathcal{D} \sim P^n$, the \emph{Rademacher} complexity of $\F$ is
  \begin{equation*}
    \Rad_n(\F) = \E{\mathcal{D} \sim P^n}{{\Rad}_\D(\F)}.
  \end{equation*}
\end{definition}

The Rademacher complexity of a function class $\F$ measures the ability of
functions $f \in \F$ to hallucinate patterns from random noise. Like the
flexible mind of children imagining dragons in clouds, only flexible
functions are able to imagine regularities in randomness.  Thus,
Rademacher complexities measure how flexible or rich the functions $f \in \F$
are.  Rademacher complexities have a typical order of $O(n^{-1/2})$
\citep{K11}.  Although in this thesis we use Rademacher complexities,
there exist other measures of capacity, such as the VC-Dimension, VC-Entropy,
fat-shattering dimension, and covering numbers. The
relationships between these are explicit, due to results by Hussler and Dudley
\citep{BBL05}. To link Rademacher
complexities to the suprema~\eqref{eq:suprema}, we need one last technical
ingredient: the symmetrization inequality.

\begin{theorem}[Symmetrization inequality]\label{thm:symm}
Let $\ell : \Y \times \Y \to \R$. Then, for any function class $\F : \X \to \Y$,
\begin{equation*}
\E{\D\sim P^n}{\sup_{f\in\F} |R(f)-R_n(f)|} \leq 2 \Rad_n(\ell \circ \F).
\end{equation*}
\end{theorem}
\begin{proof}
  See \citep[page 5]{BBL05}
\end{proof}

Using Theorems~\ref{thm:mcdiarmids} and \ref{thm:symm}, we can upper bound the
suprema \eqref{eq:suprema}, which in turn upper bounds the excess risk between
the empirical and expected risk minimizers in $\F$.  The resulting upper bound 
depends on the error attained by the empirical risk minimizer, the
Rademacher complexity of $\F$, and the size of training
data. 
\begin{theorem}[Excess risk of empirical risk minimization]\label{thm:fundamental}
  Let $\F$ be a set of functions $f : \X \to [0,1]$. Then, for all 
  $\delta > 0$ and $f \in \F$,
  \begin{align*}
    R(f) &\leq R_n(f) + 2 \Rad_n(\F) + \sqrt{\frac{\log
    \frac{1}{\delta}}{2n}},\\
    R(f) &\leq R_n(f) + 2 \Rad_\D(\F) + \sqrt{\frac{\log
    \frac{2}{\delta}}{2n}},
  \end{align*}
  with probability at least $1 - \delta$.
\end{theorem}
\begin{proof}
  See, for example, \citep[Theorem 3.2]{BBL05}.
\end{proof}
Theorem~\ref{thm:fundamental} 
unveils two important facts. First, one sufficient condition
for the consistency of empirical risk minimization is that the Rademacher
complexity of $\F$ tends to zero as the amount of training data $n$
tends to infinity.  Second, the $O(n^{-1/2})$ speed of convergence, at least
without further assumptions, is optimal \citep[Theorem
6.8]{shalev2014understanding}. 

\begin{remark}[Learning faster]
  In some situations, it is possible to obtain a faster learning rate than the
  $O(n^{-1/2})$ rate from Theorem~\ref{thm:fundamental}. 

  In binary classification, we can obtain a $O(n^{-1})$ learning rate for
  empirical risk minimization if i) our function class $\F$ has finite
  VC-Dimension, ii) the Bayes predictor $g^\star$ is in $\F$, and iii) the
  problem is not ``too noisy''. \citet{massart2000some} formalizes the third
  condition as
  \begin{equation*}
    \inf_{x \in \X} \left| 2\Pr(\bm y = 1 \given \bm x = x) - 1 \right| > 0.
  \end{equation*}
  The fast rate \citep[Corollary
  5.3]{BBM05} stems from Talagrand's inequality, which refines the
  result from McDiarmid's inequality by taking into account second order statistics.

  In regression, we can obtain a $O(n^{-1})$ learning rate if i) the loss is
  Lipschitz-continuous and bounded, ii) our function class $\F$ is a convex set
  containing uniformly bounded functions, and iii) the \emph{local} Rademacher
  complexity of $\F$ is $o(n^{-1/2})$ \citep[Corollary 5.3]{BBM05}.
\end{remark}

Throughout the rest of this thesis, we will consider function classes of the form
\begin{equation*}
  \F_\phi = \left\lbrace f : f(x) = \langle A, \phi(x) \rangle_{\H_\phi} \right\rbrace,
\end{equation*}
where $\phi : \X \to \H_\phi$ is a feature map transforming the raw data $x$
into the representation $\phi(x)$, and $A : \X
\to \Y$ is a linear operator summarizing the representation $\phi(x)$ into the
target function or pattern. The next chapter studies different techniques to
construct the feature map $\phi$, responsible for computing the data
representations $\phi(x)$. 

\begin{remark}[Subtleties of empirical risk minimization]
  Throughout this thesis we will consider \emph{identically and independently
  distributed} (iid) data. Mathematically, we write this as $x_1, \ldots, x_n
  \sim P^n$, where $P^n$ is the $n$-product measure built from $P$. This will
  be the main assumption permitting learning: the relationship between examples
  in the past (the training data) and examples in the future (the test data) is
  that all of them are described by the same distribution $P$.
  This is the ``machine learning way'' to resolve Hume's \emph{the problem of
  induction}: without assumptions, learning and generalization are impossible. 

  For the empirical risk minimization learning theory to work, one must choose
  the triplet formed by the data, the function class, and the loss function
  \emph{independently}. This means that theory only holds when we train
  \emph{once}, and we do not adapt our algorithms and parameters to the
  training outcome.
\end{remark}

\begin{remark}[Universal consistency]
  \emph{Universally consistent} learning algorithms provide with a sequence of
  predictors that converge to the Bayes predictor as the training data grows to
  infinity, for all data generating distributions.  But, when considering all
  data generating distributions, universally consistent algorithms do not
  guarantee any learning rate. Formally, for any learning algorithm and
  $\varepsilon > 0$, there exists a distribution $P$ such that
  \begin{equation*}
    \Pr\left(R(f_n) \geq \frac{1}{2} - \varepsilon\right) = 1,
  \end{equation*}
  where $f_n$ is the output of the learning algorithm when given data $D \sim
  P^n$ \citep[Theorem 9]{bousquet2004}. Simply put, we can always construct
  data generating distributions under which a given algorithm will require an
  exponential amount of data, or said differently, will learn exponentially
  slow. Since these distributions exist for all learning algorithms, we can
  conclude that \emph{there is no free lunch} \citep{wolpert1997no}, and that
  all learning algorithms are equally ``bad''. But there is hope for good
  learning algorithms, since natural data is not arbitrary, but has rich
  structure. 
\end{remark}

\subsection{Model selection}\label{sec:model-selection}

Given different models ---for instance, different function classes--- to solve one
learning task, which one should we prefer?  This is the question of 
\emph{model selection}.

Model selection is problematic when learning from finite noisy data. In such
situations, the complexity of our learning algorithm will determine how well we
can tell apart patterns from noise. If using a too flexible learning algorithm,
we may hallucinate patterns in the random noise polluting our data.  Such
hallucinations will not be present in the test data, so our model will
generalize poorly, and have high expected risk. We call this situation
\emph{overfitting}. On the other hand, if using a too simple learning
algorithm, we will fail to capture all of the pattern of interest, both at
training and test data, having high empirical and expected risk. We call this
situation \emph{underfitting}.

\begin{figure}
  \begin{center}
  \includegraphics[width=\textwidth]{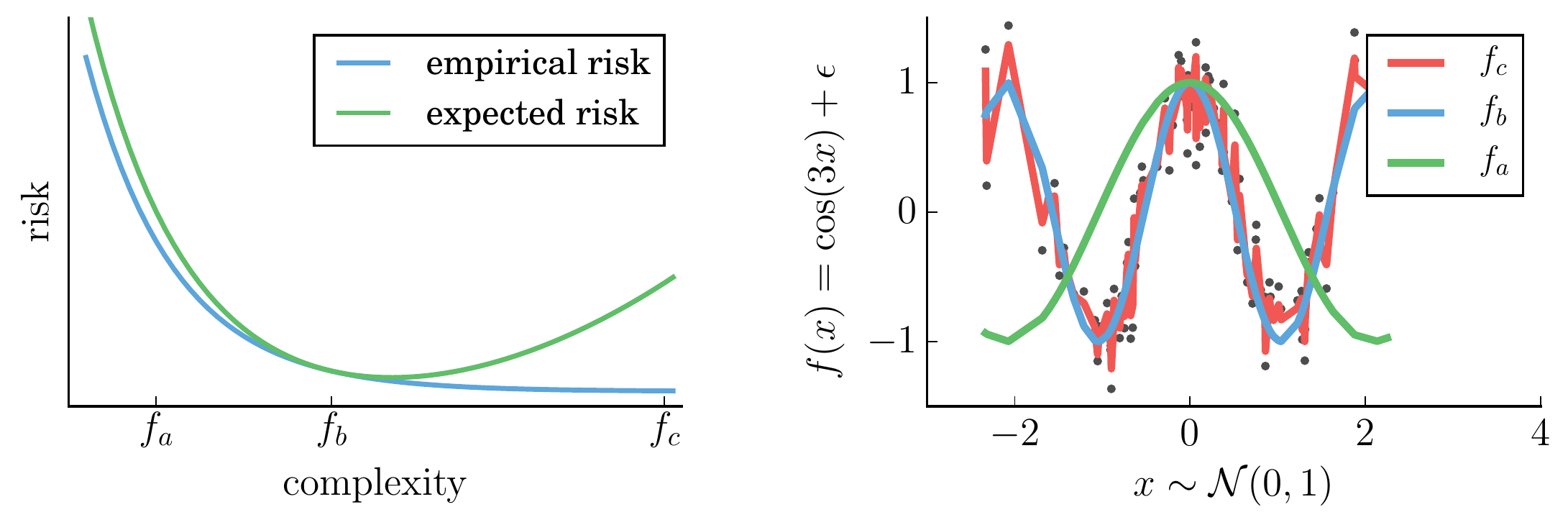}
  \end{center}
  \caption[Model selection]{Model selection.}
  \label{fig:model_selection}
\end{figure}

Figure~\ref{fig:model_selection} illustrates model selection.
Here, we want to learn the pattern 
\begin{equation*}
  f(x) = \cos(3x),
\end{equation*}
hinted by the noisy data depicted as gray dots. We offer three different
solutions to the problem: $f_a$, $f_b$, and $f_c$. First, see the ``complex'' model
$f_c$, depicted in red in the right-hand side of
Figure~\ref{fig:model_selection}. We say that $f_c$ \emph{overfits} the data,
because it incorporates the random noise polluting the data into the learned
pattern.  Since future test data will have different random noise, $f_c$ will
wiggle at random and generalize poorly. This is seen in the left-hand side of
Figure~\ref{fig:model_selection}, where the expected risk of $f_c$ is higher
than its empirical risk. Second, the ``simplistic'' model $f_b$. We say that
$f_b$ \emph{underfits} the data, because it is not flexible enough to
describe the high frequency of the sinusoidal pattern of interest. In the left-hand side of
Figure~\ref{fig:model_selection}, this translates in both the empirical and
expected risks of $f_a$ being high.
However, the model $f_b$ achieves a good balance between complexity and
simplicity, as it accommodates the pattern but ignores the noise in the data.
This balance translates into minimal expected risk, as illustrated in the
left-hand side of the figure.  The model $f_b$ allows an increase in empirical
risk to ignore the noise, lower its expected risk, and improve generalization. The techniques
sacrificing empirical risk in exchange to improved expected risk are known as
\emph{regularization}. As we will see in the next chapter, the differences
between these three predictors relate to the bias-variance trade off, which
will be discussed in Section~\ref{sec:bias}.

The question of model selection often follows \emph{Occam's razor}: prefer the
``simplest'' model (in terms of complexity) that explains the data ``well'' (in
terms of empirical risk). Different model selection strategies give 
different meanings to the phrases ``being simple'' and ``explaining the data
well''. Next, we review three of the most important model selection techniques.

\subsubsection{Structural Risk Minimization}

One alternative to model selection is the use of the theoretical results
reviewed in this section. Observe that Theorem~\ref{thm:fundamental} 
upper bounds the expected risk of a predictor $f$ as the sum of three
terms: the training error $R_n(f)$ of the model, the complexity $\Rad_n(\F)$ of
the model class $\F$, and the amount of available training data $n$. For a fixed
amount of training data $n$, we can perform model selection by considering
increasingly complex models, and selecting the one minimizing the sum $R_n(f) +
2\Rad_n(\F)$ from Theorem~\ref{thm:fundamental}. This is \emph{Structural Risk
Minimization} \citep{Vapnik98}. Unfortunately, the upper bounds provided by results
such as Theorem~\ref{thm:fundamental} are often too loose to use in practice,
and function class complexity measures are too difficult or impossible to compute. 

\subsubsection{Bayesian model selection}

One central quantity in Bayesian statistics is the \emph{evidence} or
\emph{marginal likelihood}:
\begin{equation}
  p(\bm x = x \given \bm m = m) = \int p(\bm x = x \given \bm \theta = \theta, \bm
  m = m) p(\bm \theta = \theta \given \bm m = m) \d \theta.\label{eq:marginal_likelihood}
\end{equation}
In words, this integral expresses the probability of the data $x$ coming from
the model $m$ as the integral over the \emph{likelihood} of all possible model
parameters $\theta$, weighted by their \emph{prior}. When deciding between two models
$m_1$ and $m_2$, a Bayesian statistician will use the marginal likelihood to
construct the ratio of posteriors or \emph{Bayes factor}
\begin{equation*}
  \frac{p(\bm m = m_1 \given \bm x = x)}{p(\bm m = m_2 \given \bm x = x)} =
  \frac{p(\bm m = m_1) \cdot p(\bm x = x \given \bm m = m_1)}{p(\bm m = m_2)
  \cdot p(\bm x = x \given \bm m = m_2)},
\end{equation*}
where $p(\bm m = m_1)$ is his prior belief about the correct model being $m_1$,
and similarly for $m_2$. If the Bayes factor is greater than $1$, the Bayesian
statistician will prefer the model $m_1$; otherwise, she will prefer the model $m_2$.

What is special about this procedure?  Assume for simplicity that $p(\bm m) =
\frac{1}{2}$ for both models $m_1$ and $m_2$. Since the marginal likelihood $p(\bm x \given
\bm m)$ is a probability distribution, it has to normalize to one when
integrated over all possible datasets $x \sim P^n$.  Thus, flexible models
need to assign small likelihoods to the large amount of datasets that they
can describe, but simple models can assign large likelihoods to the small
amount of datasets that they can describe. This trade-off serves as a model
selection criteria: simpler models able to explain the data well give higher
marginal likelihood.

In some situations, we need to select a model from an infinite amount of
candidates, all of them parametrized as a continuous random variable $\bm
m$. In these situations, Bayesian model selection is solving the optimization
problem
\begin{equation}\label{eq:model_selection_optimization}
  m^\star = \argmax_m p(\bm m = m) \cdot p(\bm x = x \given \bm m = m).
\end{equation}

Bayesian model selection faces some difficulties. First, the computation of the
marginal likelihood, which is solving the integral~\eqref{eq:marginal_likelihood}, is
often intractable. Second, even if the computation of the marginal likelihood
is feasible, the Bayesian model selection optimization
problem~\eqref{eq:model_selection_optimization} is often nonconvex;
thus, we are not protected from selecting an arbitrarily suboptimal model.
Third, Bayesian methods are inherently subjective. In the context of model
selection, this means that different prior beliefs about models and their
parameters can lead to two different Bayesian statisticians choosing different
models, even if the data at hand is the same.  Optimizing the marginal
likelihood is yet another optimization problem, and there is no free lunch
about it: if our models are over-parametrized we still
risk overfitting. However, nonparametric Bayesian models often have a small amount of  
parameters, making Bayesian model selection a
very attractive solution.

\subsubsection{Cross-validation}\label{sec:cross-validation}
  \index{cross-validation} In order to select the best model from a set of
  candidates, cross-validation splits the available data $\D$ in two random
  disjoint subsets: the training set $\D_\text{tr}$ and the validation set
  $\D_\text{va}$.  Then, cross-validation trains each of the candidate models
  using the training data $\D_\text{tr}$,  and chooses the model with the
  smallest risk on the unseen validation set $\D_\text{va}$. When the space of
  models is parametrized as a continuous random variable, cross-validation
  monitors the model error in the validation set, and stops the
  optimization when such error starts increasing. This is known as \emph{early
  stopping}.

  There are extensions of cross-validation which aim to provide a more robust
  estimate of the quality of each model in the candidate set. On the one hand,
  \emph{leave-$p$-out cross-validation} uses $p$ samples from the data as the
  validation set, and the remaining samples as the training set.  Leave-$p$-out
  cross-validation selects the model with the smallest average error over all
  such splits.  On the other hand, \emph{$k$-fold cross-validation}
  divides the data into $k$ disjoint subsets of equal size and performs
  cross-validation $k$ times, each of them using as validation set one of the
  $k$ subsets, and as training set the remaining $k-1$ subsets. Again, $k$-fold
  cross-validation selects the model achieving the smallest average error 
  over all such splits.  But beware! No theoretical guarantees are known for the
  correctness of the leave-$p$-out ($p > 1$) and $k$-fold cross-validation schemes,
  since they involve the repeated use of the same data.

  On the positive side, cross-validation is easy to apply and only requires iid
  data. On the negative side, applying cross-validation involves intensive
  computation and throwing away training data, to be used as a validation set.

\subsection{Regression as least squares}\label{sec:least-squares}
Assume data $\D =\{(x_i, y_i)\}^n_{i=1}$ coming from the model 
\begin{align*}
  x_i &\sim P(\bm x),\\
  \varepsilon_i &\sim \N(\bm \varepsilon; 0,\lambda^2),\\
  y_i &\leftarrow f(x) + \varepsilon_i,
\end{align*}
where $x_i \in \Rd$ for all $1 \leq i \leq n$, and 
\begin{align*}
  \alpha & \sim \N(0, \Sigma_\alpha),\\
  f(x) &\leftarrow \dot{\alpha}{x}.
\end{align*}
Therefore, we here assume a Gaussian prior over the parameter vector $\alpha$,
and additive Gaussian noise over the measurements $y_i$. To simplify notation,
we do not include a bias term in $f$, but assume that $x_{i,d} = 1$ for all $1
\leq i \leq n$. Using Bayes' rule and averaging over all possible linear
models, the distribution over the function value $f = f(x)$ is
\begin{align}
  p(\bm f \given x, X, y) &= \int p(f \given x, w) p(w \given X, y) \d w\nonumber\\
                         &= \N\left(\bm f; \lambda^{-2} x A^{-1} X^\top y, x A^{-1}
                         x^\top\right)\label{eq:bayes-ls}
\end{align}
where $X = (x_1, \ldots, x_n)^\top \in \Rnd$, $y = (y_1, \ldots, y_n)^\top \in
\R^{n\times 1}$, and $A = \lambda^{-2} X^\top X + \Sigma_\alpha^{-1}$ \citep{Rasmussen06}.
The mean of~\eqref{eq:bayes-ls} is 
\begin{equation*}
  \hat{\alpha} = (X^\top X + \lambda^2 \Sigma_\alpha)^{-1} X^\top y
\end{equation*}
and equals the \emph{maximum a posteriori} solution of the Bayesian least
squares problem. When $\Sigma_\alpha = I_d$, it coincides with the global
minima of the least-squares empirical risk 
\begin{equation}\label{eq:least-squares}
  R(\alpha, \lambda, \D) = \frac{1}{n} \sum_{i=1}^n
  \pa{\dot{\alpha}{x_i} - y_i}^2 + \frac{\lambda^2}{2} \pn{\alpha}_2^2.
\end{equation}

The term $\lambda^2$ in \eqref{eq:least-squares} is a \emph{regularizer}:
larger values of $\lambda$ will favour simpler solutions, which prevent
absorbing the noise $\varepsilon_i$ into the inferred pattern $\hat{\alpha}$.
In least-squares, we can search for the best regularization value at
essentially no additional computation \citep{rifkin2007notes}.

\subsection{Classification as logistic regression}\label{sec:logreg}
Logistic regressors $f : \R^d \to \R^q$ have form
\begin{equation*}
  f(x; W)_k = s(\dot{W}{x})_k,
\end{equation*}
for all $1 \leq k \leq q$,
where $x \in \Rd$ and $W \in \R^{d \times q}$, and the softmax operation
\begin{equation*}
  s(z)_k = \frac{\exp(z_k)}{\sum_{j=1}^q \exp(z_j)}
\end{equation*}
outputs probability vectors
$s(z)$, meaning that $s(z)_k \geq 0$ for all $1\leq k \leq q$, and
$\sum_{k=1}^q s(z)_k = 1$.
Using a dataset $\{(x_i, y_i)\}_{i=1}^n$, we can learn a logistic regressor by maximizing the Multinomial likelihood
\begin{equation*}
  L(X,Y,W) = \prod_{i=1}^n \prod_{k=1}^q f(X_{i,:}; W)_k^{Y_{i,k}},
\end{equation*}
or equivalently, the log-likelihood
\begin{equation*}
  \log L(X,Y,W) = \sum_{i=1}^n \sum_{k=1}^q Y_{i,k} \log f(X_{i,:}; W)_k,
\end{equation*}
where $X = (x_1, \ldots, x_n)^\top \in \Rnd$ and $Y = (y_1, \ldots, y_n)^\top
\in \R^{n\times q}$. Here, the target vectors $y_i$ follow a
\emph{one-hot-encoding}: if the $i$th example belongs to the $k$th
class, $y_{i,k} =1$ and $y_{i,k'} = 0$ for all $k' \neq k$.
Maximizing the log-likelihood is minimizing the risk 
\begin{equation}\label{eq:logistic-opt}
  E(W) = \frac{1}{n} \sum_{i=1}^n \ell(f(X_{i,:}; W), Y_{i,:}) 
\end{equation}
where $\ell$ is the \emph{cross-entropy loss}
\begin{equation*}
  \ell(\hat{y},y) = - \sum_{k=1}^q y_k \log \hat{y}_k.
\end{equation*}
Therefore, classification as logistic regression \emph{is} a multivariate (or
\emph{multitask}) linear regression $\dot{W}{x}$ under a different loss
function: the composition of the softmax and the cross-entropy operations. 

Minimizing~\eqref{eq:logistic-opt} with respect to the parameters $W$ is a
convex optimization problem. The next section reviews how to solve these and
other optimization problems, ubiquitous in this thesis.

\section{Numerical optimization}\label{sec:numerical-optimization}
\begin{figure}
  \includegraphics[width=\textwidth]{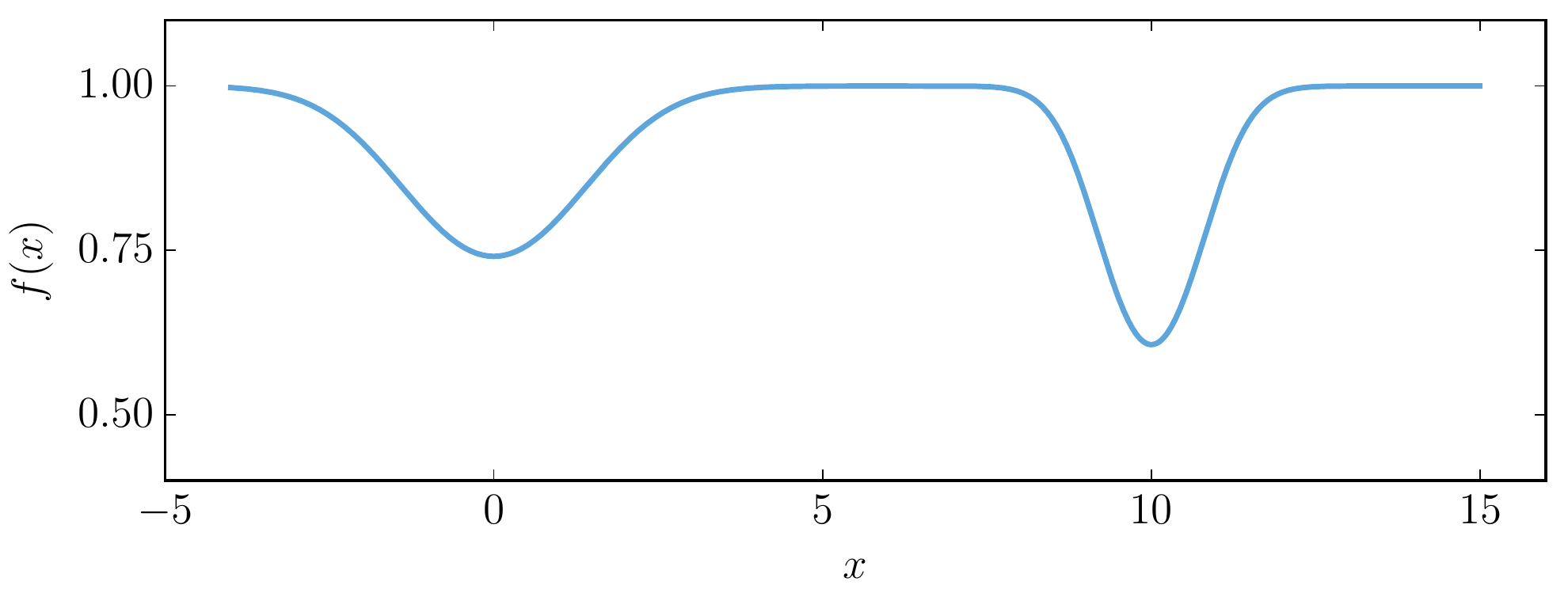}
  \caption[A one-dimensional nonconvex function]{A one-dimensional nonconvex
  function, with a local minima at $x=0$, a global minima at $x=10$, and a
  saddle point at $x=6$.}
  \label{fig:numerical_optimization}
\end{figure}

Numerical optimization algorithms deal with the problem of computing the
minimum value of functions, and where such minimum is. When we do not impose
any assumptions over the functions that we minimize, optimization is an NP-hard
problem. In particular, numerical optimization is challenging because general
functions have \emph{local minima} and saddle points, that can be far away from
their global minima.  Figure~\ref{fig:numerical_optimization} illustrates these
challenges for a one-dimensional function. Think of rolling a marble down the
graph of the function, starting at a random location, with the goal of landing
the marble at the global minima $x=10$. Then, we risk at getting the marble
stuck at the local minima $x=0$, or at the saddle point or \emph{plateau}
around $x=6$. In higher dimensions, problems do only get worse.

The rest of this section reviews basic concepts about numerical optimization,
such as function derivatives and gradients, convex functions, and gradient
based methods for numerical optimization.  Numerical optimization underlies
much of this thesis and the whole field of machine learning. For a extensive
treatise on numerical optimization, we recommend the monographs
\citep{boyd2004convex,nesterov,bubeck2014theory}.

\subsection{Derivatives and gradients}
First, we recall some basic definitions about multidimensional functions and
their derivatives. Let $f : \Rd \to \R$ be a differentiable function, with
partial derivatives 
\begin{equation*}
  \frac{\partial f}{\partial x_i},
\end{equation*}
for all $1 \leq i \leq d$.
Then, the gradient of $f$ is the vector of all $d$ partial derivatives
\begin{equation*}
  \nabla f(x) = \left(\frac{\partial f}{\partial x_1}, \ldots, \frac{\partial f}{\partial x_d} \right).
\end{equation*}
If we take all second derivatives and arrange them in an $d \times d$ matrix,
we get the Hessian of $f$, with entries
\begin{equation*}
  H(f(x))_{i,j} = \frac{\partial^2 f}{\partial x_i \partial x_j}.
\end{equation*}
If the function $f$ maps $\Rd$ into $\Rq$, then we can arrange all first
derivatives into a matrix called the Jacobian of $f$, with entries
\begin{equation*}
  J(f(x))_{i,j} = \frac{\partial f_i}{\partial x_j}.
\end{equation*}
It is easy to verify that, if $f : \Rd \to \R$, then $J(\nabla f(x)) =
H(f(x))$. 

Now, two definitions to characterize the good behaviour of a function. First, we
say that the function $f : \Rd \to \Rq$ is $L$-Lipschitz if it satisfies
\begin{equation*}
  \|f(x_1) - f(x_2)\| \leq L \|x_1-x_2\|,
\end{equation*}
for all $x_1, x_2 \in \Rd$. $L$-Lipschitz functions have bounded gradients,
$\|\nabla f(x)\| \leq L$. Second, we say that a function is $\beta$-smooth if
its gradients are $\beta$-Lipschitz:
\begin{equation*}
  \|\nabla f(x) - \nabla f(x')\| \leq \beta \|x-x'\|.
\end{equation*}

\subsection{Convex sets and functions}
A set $\X$ is convex if for all $x_1,x_2\in\X$ and $t \in [0,1]$,
$tx_1+(1-t)x_2 \in \X$. A function $f : \X \to \R$ is convex if, for all $x_1,
x_2 \in \X$ and $t \in [0,1]$
\begin{equation*}
  tf(x_1)+(1-t)f(x_2) \geq f(tx_1 + (1-t)x_2).
\end{equation*}
A geometrical interpretation of the previous inequality is that if we draw a
convex function in a paper, the straight line joining any two points in the
graph of the function will lay above the graph of the function. 

Convex sets $\X$ together with convex functions $f$ define convex optimization problems:
\begin{equation*}
  \min_x f(x) \text{ such that } x \in \X.
\end{equation*}
Convex optimization problems are important because their local minima are
global minima. This is in contrast to the nonconvex function depicted in
Figure~\ref{fig:numerical_optimization}. 

Let $f : \Rd \to \R$ be differentiable and convex. Then,
\begin{equation*}
  f(x_2) \geq f(x_1) + \dot{\nabla f(x_1)}{x_2-x_1}.
\end{equation*}
A geometrical interpretation of the previous inequality is that the tangent
line of a convex function at any point $x_1$ underestimates the function at all
locations $x_2 \in \X$.
Now let $f : \Rd \to \R$ be twice differentiable and convex. Then,
\begin{equation*}
  \nabla^2 f(x) \succeq 0.
\end{equation*}

A convex function is \emph{strictly convex} if the previous three inequalities
hold when replacing the ``$\geq$'' and ``$\succeq$'' symbols with the ``$>$''
and ``$\succ$'' symbols.

\subsection{Gradient based methods}

Gradient based methods start at a random location in the domain of the function
of interest, and perform minimization by taking small steps along the direction
of most negative gradient. Gradient based methods are therefore vulnerable to
get stuck in local minima, that is, places where the gradient is very small but
the function value is suboptimal. To understand this, see the example in
Figure~\ref{fig:numerical_optimization}. If we start our gradient based
optimization method at $x=-5$ and take sufficiently small steps, we will
converge at the sub-optimal local minima $x=0$. Another danger in this same
example would be to get stuck in the \emph{saddle point} around
$x=6$.

\subsubsection{First order methods}
The \emph{first order} Taylor approximation of $f : \Rd \to \R$ at $x_2$ is
\begin{equation*}
  f(x_2) \approx f(x_1) + (x_2-x_1)^\top\nabla f(x_1).
\end{equation*}
Imagine that we are optimizing $f$, and that we are currently positioned at
$x$. Using the previous equation, and moving in the direction given by the unit
vector $u$, we obtain
\begin{equation*}
  f(x+u)- f(x) \approx u^\top \nabla f(x).
\end{equation*}
Since we want to minimize $f(x+u)-f(x)$ using only function evaluations and
function derivative evaluations, we should minimize $u^\top
\nabla f(x)$ with respect to the direction unit vector $u$. This happens for $u
= -\nabla f(x)/\|\nabla f(x)\|$. Thus, we can update our position \emph{following the gradient descent}
\begin{equation*}
  x_{t+1} = x_t - \gamma \nabla f(x_t),
\end{equation*}
where $\gamma \in (0,1)$ is the \emph{step size}, chosen smaller than the
inverse of the Lipschitz constant of the function $f$.

\begin{remark}[Choosing the step size]
There exists a range of algorithms that provide a recipe to dynamically adjust
the step size over the course of optimization. Most of these algorithms
maintain a running average of the gradient, and adjust an individual step size
per optimized variable, as a function of how much individual partial
derivatives change over time. Some examples are the Adagrad algorithm
\citep{duchi2011adaptive}, and the RMSProp algorithm
\citep{tieleman2012lecture}.  Another solution is to run a small amount of
iterations of stochastic gradient descent with different step sizes, and select
the step size giving best results for the rest of the optimization. 
\end{remark}

Under some additional assumptions over the optimized function, it is possible
to accelerate gradient descent methods using \emph{Nesterov's accelerated
gradient descent} \citep{nesterov}.

\begin{example}[Stochastic gradient descent in learning]\label{remark:sgd}
  In machine learning, we often optimize functions of the form 
  \begin{equation*}
    f(x) = \frac{1}{n} \sum_{i=1}^n f_i(x),
  \end{equation*}
  where $n$ can be in the millions. Therefore, evaluating the gradients $\nabla
  f(x)$  is computationally prohibitive.
  \emph{Stochastic gradient descent} \citep{bottou2010large} is a modification
  of the gradient descent method, where the exact function gradients $\nabla
  f(x)$ are replaced with approximate gradients $\nabla
  f_i(x)$. Therefore, the update rules in stochastic gradient descent are
  \begin{equation*}
    x_{t+1} = x_t - \gamma \nabla f_{i}(x).
  \end{equation*}
  The approximate gradients are also stochastic, because $i \sim \U[1,n]$ at each step.
  Stochastic gradients $\nabla f_i(x)$ are estimators of the gradients $\nabla
  f(x)$, and therefore exhibit variance. One compromise between the
  computational properties of stochastic gradient descent and the low
  variance of exact gradient descent is to consider minibatches. In \emph{minibatch gradient descent},
  the update rules are
  \begin{equation*}
    x_{t+1} = x_t - \frac{\gamma}{|\mathcal{B}|} \sum_{i \in \mathcal{B}} \nabla f_i(x),
  \end{equation*}
  where $\mathcal{B}$ is a random subset of $m \ll n$ elements drawn from $\{1, \ldots, m\}$.
\end{example}

\subsubsection{Second order methods}
The \emph{second order} Taylor approximation of $f : \Rd \to \R$ at $x_2$ is
\begin{equation*}
  f(x_2) \approx f(x_1) + (x_2-x_1)^\top\nabla f(x) + \frac{1}{2}(x_2-x_1)^\top \nabla^2 f(x)(x_2-x_1).
\end{equation*}
Therefore, moving from $x$ in the direction given by $u$, we obtain
\begin{equation*}
  f(x + u) - f(x) \approx u^\top\nabla f(x) + \frac{1}{2}u^\top \nabla^2 u.
\end{equation*}
Therefore, based on function evaluations, first derivative evaluations, and
second derivative evaluations, if we aim at minimizing $f(x+u)-f(x)$ we should
minimize $u^\top\nabla f(x) + \frac{1}{2}u^\top \nabla^2 u$. This happens when
\begin{equation*}
  u = - (\nabla^2f(x))^{-1} \nabla f(x).
\end{equation*}
Therefore, second order gradient descent methods implement the update rule 
\begin{equation*}
  x_{t+1} = x_t - \gamma (\nabla^2f(x_t))^{-1} \nabla f(x_t),
\end{equation*}
for some small step size $\gamma \in (0,1)$. This is often called the Newton's
update.

\begin{remark}[First order versus second order methods]\label{remark:opt-rates}
  First order and second order gradient descent methods perform a local
  approximation of the optimized function at each point of evaluation. While
  first order methods perform a linear approximation, second order methods
  perform a quadratic approximation to learn something about the curvature of
  the function. Thus, second order
  methods use more information about $f$ per iteration, and this translates in
  a fewer number of necessary iterations for convergence. After $t$ iterations,
  the convex optimization error of first order methods is $O(L/\sqrt{t})$ for
  $L$-Lipschitz functions, $O(\beta/t)$ for $\beta$-smooth functions, and
  $O(\beta/t^2)$ for $\beta$-smooth functions
  when using Nesterov's accelerated gradient descent method.
  On the other hand, Newton's method generally
  achieves an optimization error of $O(1/t^2)$ \citep{bubeck2014theory}.

  Although second order methods need fewer iterations, each of their iterations
  is slower due to the inversion of the Hessian matrix. This operation requires
  $O(d^3)$ computations when optimizing a $d$-dimensional function. To
  alleviate these issues, \emph{quasi-Newton} methods replace the Hessian
  matrix with a low-rank approximation which allows for faster inversion.  One
  example is the Broyden-Fletcher-Goldfarb-Shanno (BFGS) algorithm (see
  \citet{nesterov}).
\end{remark}

  \chapter{Representing data}\label{chapter:representing-data}
\vspace{-1.25cm}
  \emph{This chapter is a review of well-known results.}
\vspace{1.25cm}

\noindent Pattern recognition is conceived in two steps. First, finding a \emph{feature
map}\index{feature map}
\begin{equation*}
  \phi : \X \to \H
\end{equation*}
that transforms the \emph{data} $\{x_i\}_{i=1}^n$, $x_i \in \X$, into the
\emph{representation} or \emph{features}\index{feature}
$\{\phi(x_i)\}_{i=1}^n$, $\phi(x_i) \in \H$.  Second, revealing the pattern of
interest in data as a linear transformation
\begin{equation}
  \label{eq:pattern}
  \dot{A}{\phi(x)}_\H
\end{equation}
of the representation.

Because of the simplicity of \eqref{eq:pattern}, most of the responsibility in
learning from data falls in the feature map $\phi$. Therefore, finding good
feature maps is key to pattern recognition \citep{dlbook}.
Good feature maps translate nonlinear statistics of data into linear statistics
of their representation: they turn dependencies into correlations, stretch
nonlinear relationships into linear regressions, disentangle the explanatory
factors of data into independent components, and arrange different classes of
examples into linearly separable groups. The following example illustrates the
key role of representations in learning. 

\begin{example}[Rings data]\label{example:two-rings}
  \begin{figure}
    \begin{subfigure}{0.45\textwidth}
    \includegraphics[width=\textwidth]{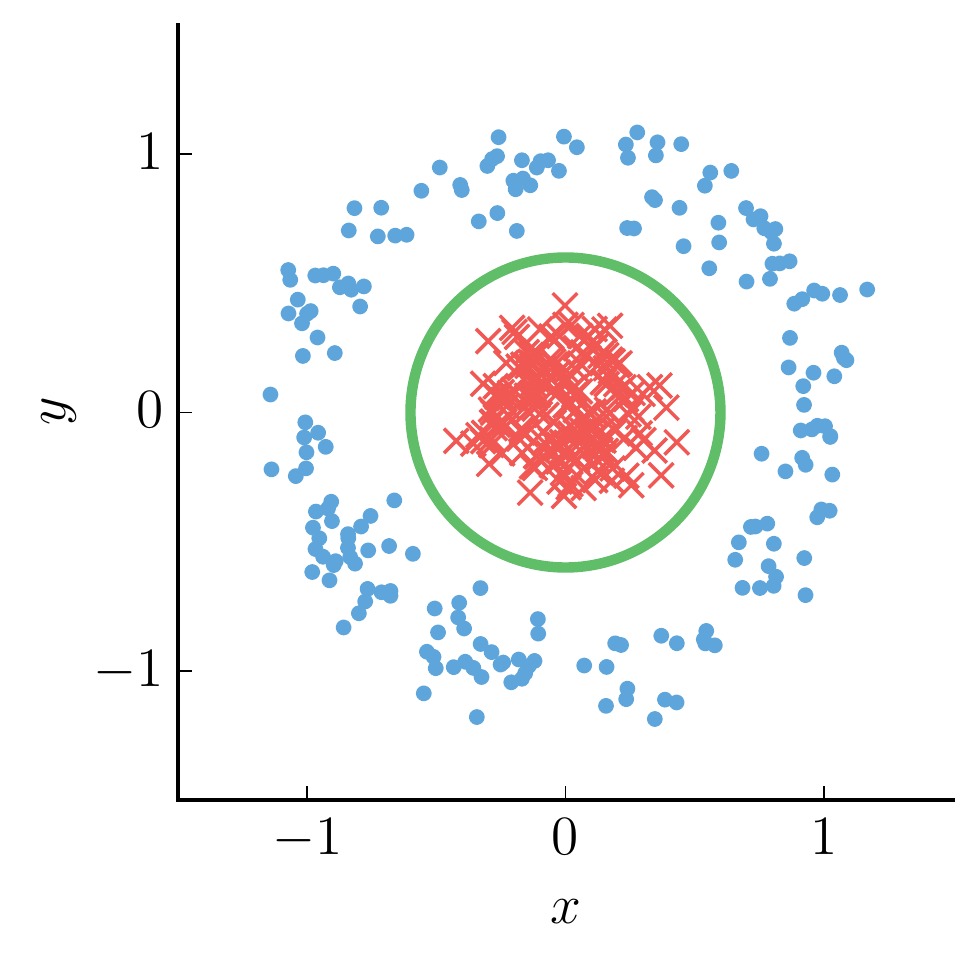}
    \caption{Raw data}
    \label{fig:two-rings-2d}
    \end{subfigure}
    \hspace{.5cm}
    \begin{subfigure}{0.45\textwidth}
    \includegraphics[width=\textwidth]{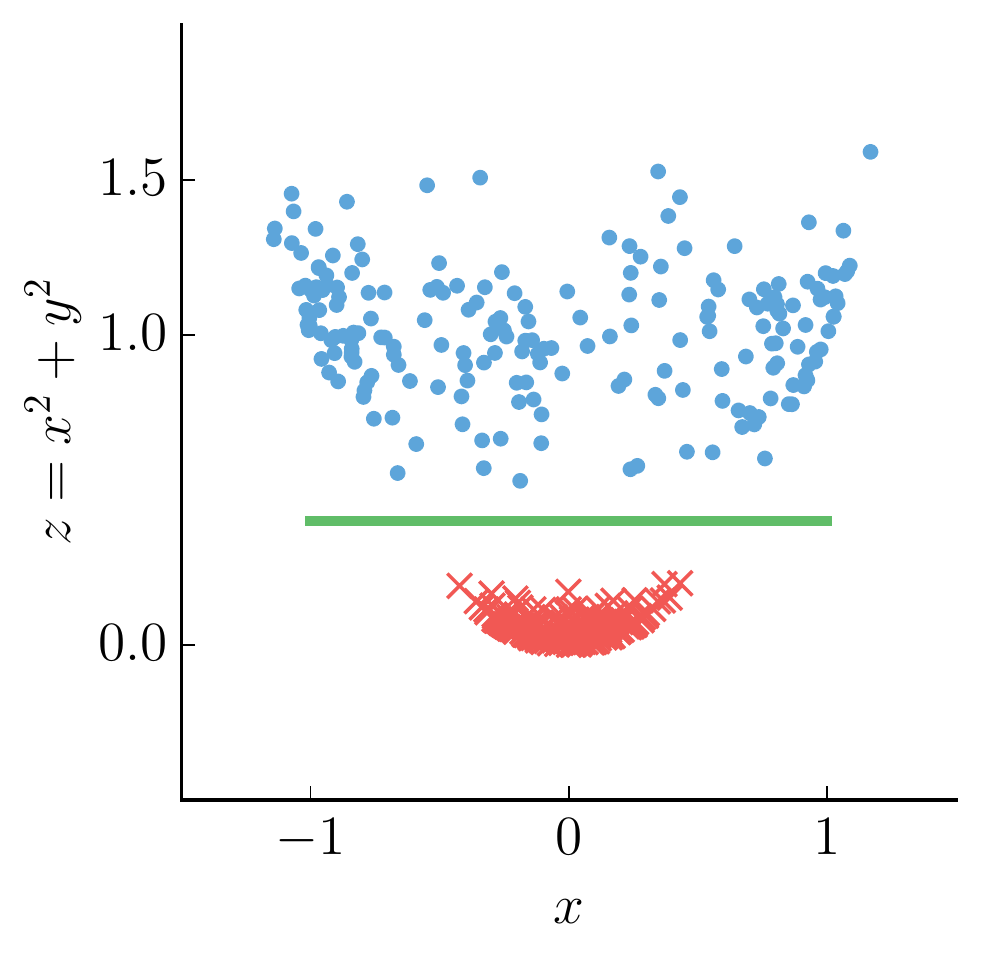}
    \caption{Data representation}
    \label{fig:two-rings-3d}
    \end{subfigure}
    \caption{The rings data}
  \end{figure}
  Consider the problem of finding a linear function that separates the two
  classes of examples from
  Figure~\ref{fig:two-rings-2d}.  After some struggle, we conclude that under
  the raw representation $(x,y) \in \X$, no such function exists.  To solve
  this issue, we engineer a third feature $z = x^2 + y^2$. This new feature
  elevates each example to an altitude proportional to its distance to the
  origin.  Under the representation $(x,z) \in \H$, depicted in
  Figure~\ref{fig:two-rings-3d}, there exists a linear function that separates
  the two classes of examples, solving the problem at hand.
\end{example}

There exists a variety of methods to construct representations of data. In
this chapter, we review four of them: kernel methods, random features, neural
networks, and ensembles.

\section{Kernel methods}\index{kernel!methods}\label{sec:kernels}

The central object of study in kernel methods \citep{Scholkopf01} is the kernel
function. Throughout this section, we assume that $\mathcal{X}$ is a
compact metric space.

\begin{definition}[Kernel function]\index{kernel!function}
  A symmetric function $k : \X \times \X \to \R$ is a
  positive definite kernel function, or kernel, if for all $n \geq 1$, $x_1,
  \ldots, x_n \in \R$, and $c_1, \ldots, c_n \in \R$ 
  \begin{equation*}
    \sum_{i=1}^n c_i c_j k(x_i, x_j) \geq 0.
  \end{equation*}
\end{definition}

Each kernel $k$ provides with a fixed feature map $\phi_k$.

\begin{definition}[Kernel representation]\index{kernel!representation}
  A function $k : \X \times \X \to \R$ is a kernel if and only if there exists
  a Hilbert space $\H$ and a \emph{feature map} $\phi_k : \X \to \H$ such that
  for all $x, x' \in \X$
  \begin{equation*}
    k(x,x') = \dot{\phi_k(x)}{\phi_k(x')}_\H,
  \end{equation*}
  We refer to $\phi_k(x) \in \H$ as a kernel representation of $x \in \X$.
\end{definition}

Kernel representations often lack explicit closed forms, but we can access them
implicitly using the inner products $\dot{\phi_k(x)}{\phi_k(x')}$ computed as
$k(x,x')$.  In general, there exists more than one feature map $\phi_k$ and
Hilbert space $\H$ satisfying $k(x,x') = \dot{\phi_k(x)}{\phi_k(x')}_\H$, for a fixed
given $k$.  But, every kernel $k$ is associated to an unique
\emph{Reproducing Kernel Hilbert Space} (RKHS) $\Hk$\index{RKHS}, with
corresponding unique \emph{canonical feature map} $k(x,\cdot) \in \Hk$, such
that
\begin{equation*}
  k(x,x') = \dot{k(x,\cdot)}{k(x',\cdot)}_\Hk.
\end{equation*}

The following important
result highlights a key property of reproducing kernel Hilbert spaces.

\begin{theorem}[Moore-Aronszajn] Let $\Hk$ be a
  Hilbert space of functions from $\X$ to $\R$. Then, $\Hk$ is a RKHS if and only
  if there exists a kernel $k : \X \times \X \to \R$ such that
    \begin{align*}
      \forall x \in \X,  \quad & k(x, \cdot) \in \Hk,\\
      \forall f \in \Hk, \quad & \dot{f(\cdot)}{k(x, \cdot)} \text{ (reproducing property) }
    \end{align*}
  If such $k$ exists, it is unique, and $k$ is the \emph{reproducing kernel of $\Hk$}.
  Every kernel $k$ reproduces a unique RKHS $\Hk$.
  \begin{proof}
    See Theorem 3 in \citep{Berlinet11}.
  \end{proof}
\end{theorem}

The reproducing property is
attractive from a computational perspective, because it allows to express any
function $f \in \Hk$ as the linear combination of evaluations of the 
reproducing kernel $k$. Presented next, the representer theorem leverages the reproducing
property to learn patterns from data using kernels.

\begin{theorem}[Representer]\index{kernel!representer theorem}
  Let $k : \X \times \X \to \R$ be a kernel with
  corresponding RKHS $\Hk$. Assume data $\{(x_1, y_1), \dots, (x_n, y_n)\} \subseteq
  \X \times \R$, a strictly monotonically increasing function $g
  \colon [0, \infty) \to \R$, and an arbitrary risk function $R \colon
  (\X \times \R^2)^n \to \R \cup \lbrace \infty \rbrace$; then, any
  $f^\star \in H_k$ satisfying
  \begin{equation*}
   f^\star = \argmin_{f \in H_k} R\left( (x_1, y_1,
   f(x_1)), ..., (x_n, y_n, f(x_n)) \right) + g\left( \lVert f \rVert \right)
  \end{equation*}
  admits the representation
  \begin{equation*}
    f(\cdot) = \sum_{i=1}^n \alpha_i k(x_i,\cdot),
  \end{equation*}
  where $\alpha_i \in \R$ for all $1 \le i \le n$.
  \begin{proof}
    See Section 4.2. of \citep{Scholkopf01}.
  \end{proof}
\end{theorem}

Simply put, the representer theorem states that if we use a kernel function
associated with a rich RKHS, we will be able to use it to learn rich 
patterns from data.

\subsection{Learning with kernels}\index{kernel!ridge regression}
Learning with kernels involves three steps. First, stating the learning problem
of interest in terms of the Gram matrix $G \in \R^{n \times n}$, with entries
$G_{ij} = \dot{x_i}{x_j}$, for all pairs of inputs $(x_i, x_j)$. Second,
replacing the Gram matrix $G$ by the kernel matrix $K \in \R^{n \times n}$,
with entries $K_{ij} = k(x_i,x_j)$. Third, solving the learning problem by
computing linear statistics of the kernel matrix $K$.  This manipulation is
known as the \emph{kernel trick}.

The following example illustrates the use of the kernel trick to extend the
capabilities of least-squares regression to model nonlinear relations between
random variables. 

\begin{example}[Kernel least-squares regression]
  \label{ex:kernel-least-squares}
  \begin{figure}
    \begin{subfigure}{0.32\textwidth}
    \includegraphics[width=\textwidth]{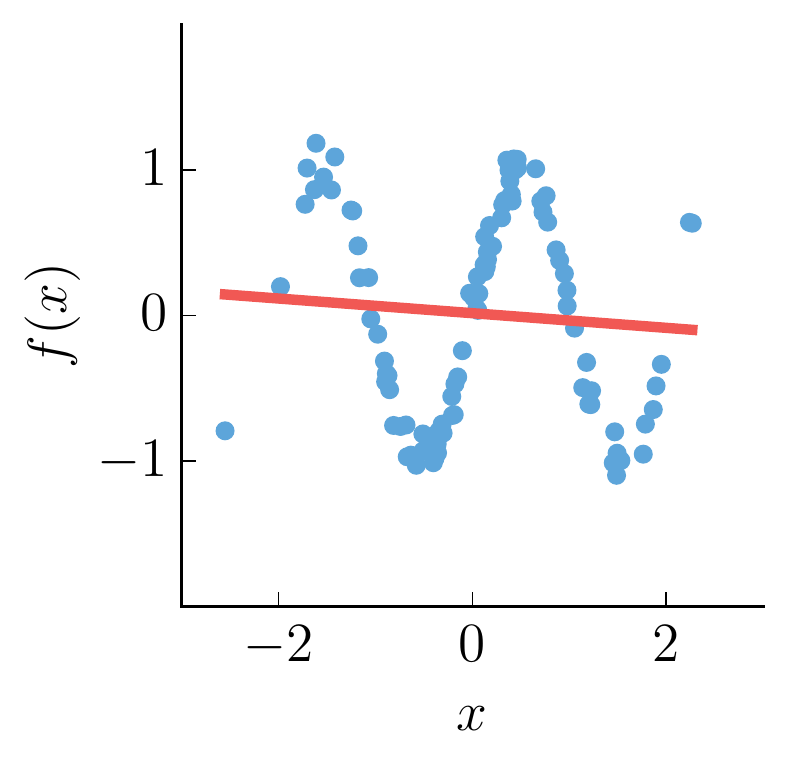}
    \caption{Linear LS}
    \label{fig:linear-linear}
    \end{subfigure}
    \begin{subfigure}{0.32\textwidth}
    \includegraphics[width=\textwidth]{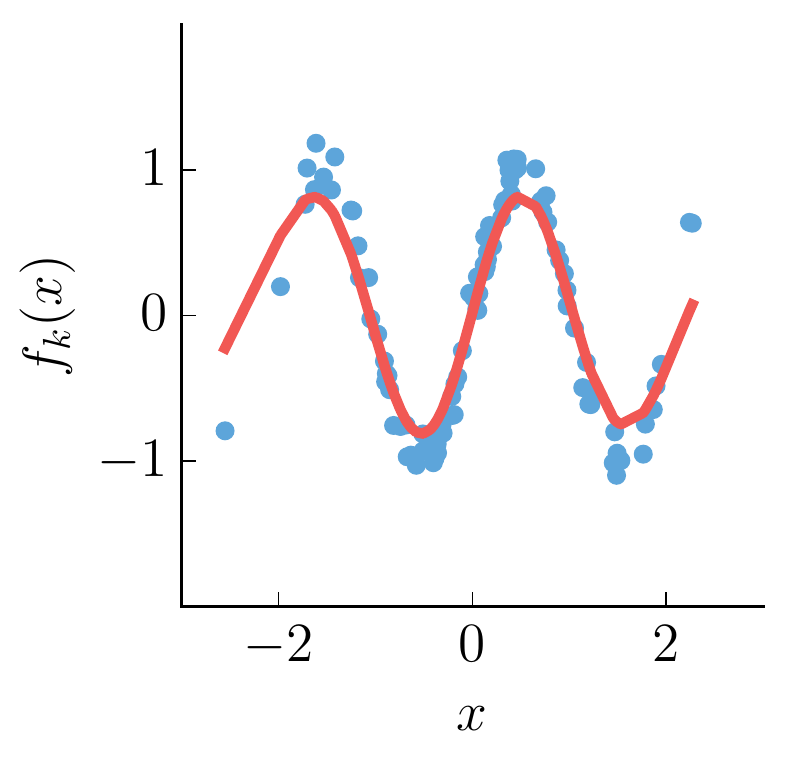}
    \caption{Kernel LS}
    \label{fig:linear-kernel}
    \end{subfigure}
    \begin{subfigure}{0.32\textwidth}
    \includegraphics[width=\textwidth]{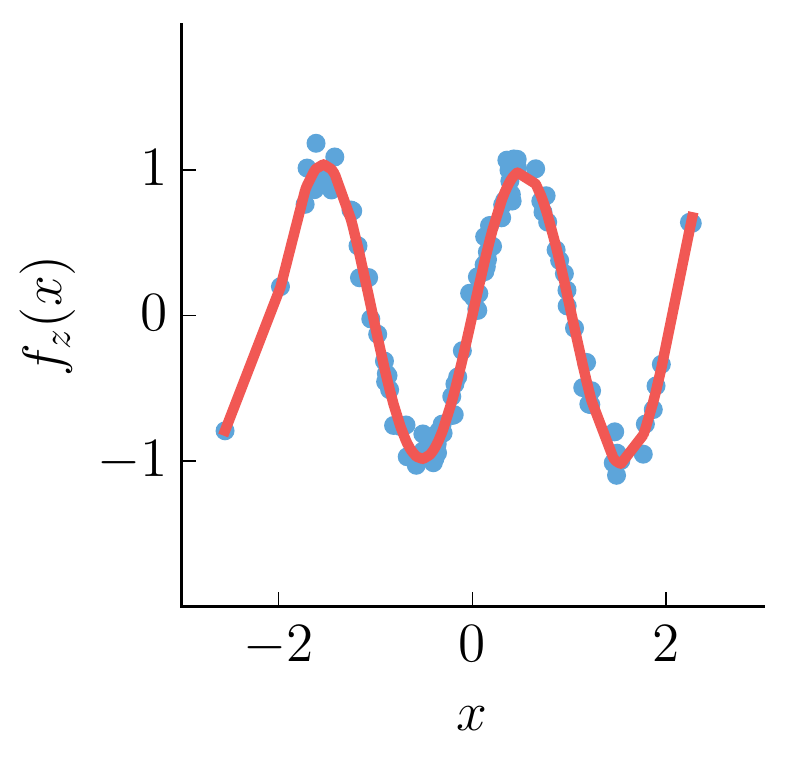}
    \caption{Randomized LS}
    \label{fig:linear-random}
    \end{subfigure}
    \caption{Different types of least-squares regression.}
    \label{fig:linear-regression}
  \end{figure}
  Recall the problem of least squares, described in Section~\ref{sec:least-squares}. 
  Figure~\ref{fig:linear-regression} illustrates a one-dimensional dataset
  where $\bm x \equiv \N(0,1)$ and $\bm y = \sin(3\bm x) +
  \N(0,\lambda^2)$.  As shown in Figure~\ref{fig:linear-linear}, \emph{linear}
  least-squares regression fails to recover the true nonlinear relationship
  $f(x) = \E{}{\bm y| x} = \sin(3x)$.  We solve this issue by performing
  least-squares regression on some kernel representation $\phi_k$.  To apply
  the kernel trick, we must first state \eqref{eq:least-squares} in
  terms of the $n \times n$ Gram matrix $G
  = XX^\top$.
  For this, we use the Sherman-Morrison-Woodbury formula \eqref{eq:smwf} to
  rewrite \eqref{eq:least-squares} as
  \begin{equation*}
    \hat{\alpha} = ( X^\top I_n  X + \lambda I_d)^{-1}X^\top I_n  y = 
     X^\top( XX^\top + \lambda I_n)^{-1} y.
  \end{equation*}
  Then, our regression function is
  \begin{equation*}
    f(x) = \dot{\hat{\alpha}}{x} = \pb{( XX^\top + \lambda I_n)^{-1} y}^\top Xx^\top.
  \end{equation*}
  Next, replace the Gram matrix $XX^\top \in
  \R^{n \times n}$ by the kernel matrix $K \in \R^{n\times n}$ with
  entries $K_{ij} = k(x_i,x_j)$, and the vector $Xx^\top \in\Rn$ by the vector
  $k_x \in \Rn$ with entries $k_{x,i} = k(x,x_i)$:
  \begin{equation}\label{eq:kernel-machine}
    f_k(x) = \dot{\hat{\alpha}}{x} = \pa{( K + \lambda I_n)^{-1} y}^\top k_x = \sum_{i=1}^n \beta_i k(x_i, x),
  \end{equation}
  where $\beta_i = \pa{( K + \lambda I_n)^{-1} y}_i$, for all $1 \leq i \leq
  n$.  Figure~\ref{fig:linear-kernel} illustrates the least-squares regression
  obtained using the kernel representation, which successfully captures the nonlinear
  pattern describing the data. 
\end{example}

Example~\ref{ex:kernel-least-squares} reveals a key property of kernel
representations.  As shown in \eqref{eq:kernel-machine}, the nonlinear
regression function $f_k$ is a linear transformation of the $n$-dimensional
representation 
\begin{equation*}
  (k(x_1, x), \ldots, k(x_n, x)),
\end{equation*}
also called the \emph{empirical kernel map}.  Such 
kernel representations are \emph{nonparametric}: given $n$ data, kernels
representations are effectively $n$-dimensional.  Nonparametric representations are a
double-edged sword. On the positive side, nonparametric representations allow
each point $x_i$ in the data to speak by itself, as one dedicated dimension of
the representation. This makes intuitive sense, because when having more data,
we should be able to afford a more sophisticated representation. On the
negative side, learning using $n$-dimensional representations requires 
computations prohibitive for large $n$. In the previous example, the
computational burden is $O(n^3)$ due to the construction and inversion of the
$n\times n$ kernel matrix $K$.  Furthermore, kernel machines
\eqref{eq:kernel-machine} need access to all the data $\{x_i\}_{i=1}^n$ for
their evaluation, thus requiring $O(nd)$ permanent storage.

\index{representation!parametric}
\index{representation!nonparametric}
\begin{remark}[Nonparametric versus parametric representations]
  \label{remark:nonparametric}
  The dimensionality of nonparametric representations grows linearly with the
  amount of data $n$, but not with respect to the complexity of the pattern
  of interest. Even when recovering a simple pattern from
  $n=10^6$ samples of $d=10^3$ dimensions, an orthodox use of kernels will
  require $O(10^{18})$ computations and $O(10^9)$ memory storage. As we will
  see later in this chapter, \emph{parametric} representations are attractive
  alternatives to deal with big data, since we can tune their size 
  according to the difficulty of the learning problem at hand.
  In any case, nonparametric representations are essentially parameter-free,
  since they use the given training data as parameters. This translates in
  learning algorithms with a small amount tunable parameters, which is 
  a desirable property.

  Moreover, nonparametric representations are useful when the
  dimensionality of our data $d$ is greater than the sample size $n$. In
  this case, computing the (dual) $n\times n$ kernel matrix is cheaper than
  computing the (primal) $d\times d$ covariance matrix of the data.
\end{remark}

\subsection{Examples of kernel functions}\label{sec:kernelex}
There exists a wide catalog of kernel functions \citep{Souza10}. Favouring the
choice of one kernel over another is a problem specific issue, and will 
depend on the available prior knowledge about the data under study. For
example, the Gaussian kernel is an effective choice to discover smooth
patterns.  Alternatively, the arc-cosine kernel is a better choice to model
patterns with abrupt changes. Or, if data contains patterns 
that repeat themselves, periodical kernels induce more suitable data
representations. 

Kernel functions have closed-form expressions and a small number of tunable parameters.
The simplest kernel function is the \emph{polynomial kernel}
\begin{equation*}
  k(x,x') = (\dot{x}{x'} + c)^d,
\end{equation*}
with offset parameter $c \geq 0$ and degree parameter $d \in \mathbb{N}$. 
For $c = 0$ and $d = 1$, the representation $\phi_k(x)$ induced by
the polynomial kernel matches the original data $x$. As $d$
grows, the polynomial kernel representation captures increasingly complex
patterns, described by polynomials of degree $d$. The offset parameter $c$
trades-off the influence between the higher-order and the lower-order terms in
the polynomial representation.

The most widely-used kernel function is the \emph{Gaussian kernel}
\begin{equation}\label{eq:gaussian-kernel}
  k(x,x') = \exp\pa{-\gamma \|x-x'\|_2^2}.
\end{equation}
The \emph{bandwidth} parameter $\gamma \geq 0$ controls the complexity of the
representation. Large values of $\gamma$ induce more complex representations,
and small values of $\gamma$ induce representations closer to the original raw
data.  In practice, one sets $\gamma$ to roughly match the scale of the data.
One alternative to do this is the \emph{median heuristic}, which selects $\gamma$ to be the inverse of
the empirical median of the pairwise distances $\|x_i - x_j\|_2^2$, with $(x_i,
x_j)$ subsampled from the data.  The Gaussian kernel is differentiable an
infinite amount of times, making it appropriate to model smooth patterns.

Another important kernel function is the \emph{arc-cosine kernel}
\begin{align}\label{eq:kernel-arc-cosine}
  k(x,x') &= 2 \int \frac{\exp(-\frac{1}{2}\pn{w}^2)}{(2\pi)^{d/2}}
  \Theta(\dot{w}{x})\Theta(\dot{w}{x'}) \dot{w}{x}^q \dot{w}{x'}^q \d 
  w,\nonumber\\
  &= \frac{1}{\pi} \pn{x}^q \pn{x'}^q (-1)^q(\sin \theta)^{2q
  +1}\pa{\frac{1}{\sin \theta} \frac{\partial}{\partial
  \theta}}^q\pa{\frac{\pi-\theta}{\sin\theta}},
\end{align}
where $\theta := \cos^{-1}\pa{\dot{x}{x'}/\pa{\pn{x}\pn{x'}}}$ and $q \in
\mathbb{N}$ \citep{Cho11}. For $q = 1$, the arc-cosine kernel data
representation is piece-wise linear, and properly describes patterns 
exhibiting abrupt changes.

Finally, we can construct new kernels as the combination of other kernels.  For instance, if
$k_1(x,x')$ and $k_2(x,x')$ are two kernels, then $k(x,x') = k_1(x,x') +
k_2(x,x')$ and $k(x,x') = k_1(x,x')k_2(x,x')$ are also kernels. Or, for any
kernel $k$ and function $f : \Rd \to \R$, $k(x,x') = f(x)k(x,x')f(x')$ and
$k(x,x') = k(f(x),f(x'))$ are also kernels. \citet[page 296]{Bishop06} offers a
detailed table of rules to build kernels out of other kernels.
\citet{Duvenaud13} proposes a genetic algorithm that explores these kind of
compositions to evolve complex kernels with interpretable meanings.

\section{Random features}\label{sec:random-features}
  \index{random features}
  We have seen that kernel methods induce nonparametric representations, that
  is, representations that have $n$ effective dimensions when learning from $n$ data.
  One drawback of nonparametric representations is their associated
  computational requirements. For instance, kernel least-squares requires $O(n^3)$ computations and $O(nd)$ memory storage. As $n$
  grows to the tens of thousands, these computational and memory requirements
  become prohibitive.

  This section proposes two alternatives to approximate $n$-dimensional
  nonparametric kernel representations as $m$-dimensional parametric
  representations, where $m$ can depend on the complexity of the learning
  problem at hand. In some cases, $m$ will be much smaller than $n$. 

  \subsection{The Nystr\"om method}\label{sec:nystrom}
  \index{Nystr\"om method}
  The Nystr\"om method \citep{Williams01}
  approximates the representation induced by a kernel $k$ as
  \begin{equation}
  \label{eq:nystrom}
  \phi(x) = M^{-1/2} \pa{k(w_1, x), \ldots, k(w_m, x)}^\top
  \in \R^m,
  \end{equation}
  where the matrix $M \in \R^{m \times m}$ has entries $M_{ij} = k(w_i,
  w_j)$, with $w_i \in \Rd$ for all $1 \leq j \leq m$. In practice, the set
  $\{w_1, \ldots, w_m\}$ is a subset of the data $\{x_i\}_{i=1}^n$ sampled at random,
  or $m$ representative data prototypes computed using a clustering algorithm
  \citep{Kumar12}.

  The analysis of the Nystr\"om approximation considers the rank-$m$
  approximate kernel matrix
  \begin{equation*}
    {K}_m = \Phi\Phi^\top \in \R^{n\times n},
  \end{equation*}
  where 
  \begin{equation*}
    \Phi := (\phi(x_i), \ldots,
    \phi(x_i))^\top \in \R^{n\times m},
  \end{equation*}
  and $\phi$ follows \eqref{eq:nystrom}. If the set $\{w_1, \ldots, w_m\}$ is a
  subset of the data sampled using a carefully chosen probability distribution
  \citep{Drineas05}, then
  \begin{equation*}
    \|K - {K}_m\|_2 \leq \|K- K_m^\star \|_2 + O\pa{\frac{n}{\sqrt{m}}},
  \end{equation*}
  where $K^\star_m$ is the best rank-$m$ approximation to $K$.

  The Nystr\"om method has two main advantages. First, it allows the
  approximation of arbitrary kernel representations.  Second, the set $\{w_1,
  \ldots, w_m\}$ is an opportunity to adapt the representation to the geometry
  of the data at hand.  This adaptation results in a reduction of the Nystr\"om
  approximation error from $O(\frac{n}{\sqrt{m}})$ to $O(\frac{n}{m})$
  when the gap between the two largest eigenvalues of the true kernel matrix
  $K$ is large \citep{Yang12}. On the negative side, Nystr\"om approximations
  face the same problems than regular kernel representations: it is necessary
  to construct and invert the $m \times m$ matrix $M$, multiply against it to
  construct the approximate kernel representation, and store the set $\{w_1,
  \ldots, w_m\}$ in $O(md)$ memory at all times. Like in exact kernel methods,
  this requires a prohibitive amount of computation when a large amount $m$ of
  representation features is necessary.

  In the following, we review \emph{random Mercer features}, an alternative
  approximation to a specific class of kernel representations which overcomes
  the two short-comings of the Nystr\"om method.

  \subsection{Random Mercer features}\label{sec:random-mercer-features}
  \index{random Mercer features}
  \index{random kitchen sinks}
  Random Mercer features approximate kernel functions which satisfy
  \emph{Mercer's condition}, by exploiting their expansion as a sum.
  \begin{theorem}[Mercer's condition]
    Let $\mathcal{X}$ be a compact metric space, and let $k : \X \times \X \to
    \R$ be a continuous kernel which is square-integrable on
    $\mathcal{X}\times\mathcal{X}$ and satisfies 
    \begin{equation*}
      \int_\X \int_\X k(x,x') f(x) f(x') \d x \d x' \geq 0
    \end{equation*}
    for all $f \in L^2(\X)$. Then, $k$ admits a representation 
    \begin{equation}\label{eq:kernel-expansion}
      k(x,x') = \sum_{j=1}^\infty \lambda_j \phi_{\lambda_j}(x) \phi_{\lambda_j}(x'),
    \end{equation}
    where $\lambda_j \geq 0$, $\dot{\phi_i}{\phi_j} = \delta_{ij}$, and the
    convergence is absolute and uniform.
    \begin{proof}
      See \citep{Mercer09}.
    \end{proof}
  \end{theorem}
  As pioneered by \citet{Rahimi07,Rahimi08,fastfood}, one can approximate
  the expansion \eqref{eq:kernel-expansion} by random sampling.  More
  specifically, for trace-class kernels, those with finite $\|\lambda\|_1 :=
  \sum_{i=1}^\infty \lambda_j$, we can normalize the kernel expansion
  \eqref{eq:kernel-expansion} to mimic an expectation
  \begin{equation}\label{eq:kernel-expectation}
    k(x,x') = \|\lambda\|_1 \E{\lambda \sim p(\lambda)}{\phi_\lambda(x)\phi_\lambda(x')},
  \end{equation}
  where
  \begin{equation*}
    p(\lambda) = 
    \begin{cases} 
      \|\lambda\|_1^{-1} \lambda &\mbox{if } \lambda \in \{\lambda_1, \ldots \}, \\ 
      0 & \mbox{else.}
    \end{cases}
  \end{equation*}
  Now, by sampling $\lambda_j \sim p(\lambda)$, for $1 \leq j \leq m$, we
  can approximate the expectation \eqref{eq:kernel-expectation} with the
  Monte-Carlo sum
  \begin{equation*}
    k(x,x') \approx \frac{\|\lambda\|_1}{m} \sum_{j=1}^m
    \phi_{\lambda_j}(x)\phi_{\lambda_j}(x'),
  \end{equation*}
  from which we can recover the $m$-dimensional, parametric representation
  \begin{equation}
    \label{eq:mercer-map}
    \phi(x) = \sqrt{\frac{\|\lambda\|_1}{m}} \pa{\phi_{\lambda_1}(x), \ldots, \phi_{\lambda_m}(x)}^\top.
  \end{equation}

  The functions $\{\phi_{\lambda_j}\}_{j=1}^\infty$ are often
  unknown or expensive to compute. Fortunately, there are some exceptions. For
  example, the arc-cosine kernel \eqref{eq:kernel-arc-cosine}
  follows the exact form of an expectation under the $d$-dimensional Gaussian
  distribution. Therefore, by sampling $w_1, \ldots, w_m \sim
  \N(0,1)$, we can approximate \eqref{eq:kernel-arc-cosine} by
  \begin{align}
    k(x,x') &= 2 \int \frac{\exp(-\frac{1}{2}\pn{w}^2)}{(2\pi)^{d/2}}
    \Theta(\dot{w}{x})\Theta(\dot{w}{x'}) \dot{w}{x}^q \dot{w}{x'}^q \d 
    w\nonumber\\
    &\approx \frac{2}{m} \sum_{j=1}^m 
    \Theta(\dot{w}{x})\Theta(\dot{w}{x'}) \dot{w}{x}^q \dot{w}{x'}^q.\label{eq:arc-cosine-approx}
  \end{align}
  For instance, consider $q=1$. Then, we can combine \eqref{eq:mercer-map} and
  \eqref{eq:arc-cosine-approx} to construct the $m$-dimensional representation
  \begin{equation*}
    \phi(x) = \sqrt{\frac{2}{m}} \pa{\max(\dot{w_1}{x}), \ldots,
    \max(\dot{w_m}{x'})}^\top \in \Rm,
  \end{equation*}
  formed by rectifier linear units, which approximates the arc-cosine kernel,
  in the sense that $\dot{\phi(x)}{\phi(x')}$ converges to
  \eqref{eq:kernel-arc-cosine} pointwise as $m \to \infty$.
  
  Another class of kernels with easily computable basis
  $\{\phi_j\}_{j=1}^\infty$ is the class of continuous shift-invariant kernels,
  those satisfying $k(x,x') = k(x-x',0)$ for all $x,x'\in\X$. In this case,
  $\{\phi_j\}_{j=1}^\infty$ is the Fourier basis, as hinted by the following
  result due to Salomon Bochner.

  \begin{theorem}[Bochner]\label{thm:bochner}
    \index{theorem|Bochner's@\textit{Bochner's}}
    A function $k$ defined on a locally compact Abelian group $G$ with dual
    group $G'$ is the Fourier transform of a positive measure $p$ on $G'$ if
    and only if it is continuous and positive definite.
    \begin{proof}
      See Section 1.4.3 from \citep{Rudin62}.
    \end{proof}
  \end{theorem}

  One consequence of
  Bochner's theorem is that continuous shift-invariant kernels $k : \R^d \times
  \R^d \to \R$ are the Fourier transform of a positive measure defined on
  $\R^d$ \citep{Rahimi07, Rahimi08}. Then, 
  \begin{align*}
    k(x-x',0) &=  c_k\int_\Omega p_k(w) \exp\pa{\text{\i} \dot{w}{x-x'}} \d w\\
            &=  c_k\int_\Omega p_k(w) \cos(\dot{w}{x-x'}) + \text{\i}
            \sin(\dot{w}{x-x'}) \d w
  \end{align*}
  where $p_k(w)$ is a positive measure and $c_k$ is a normalization constant,
  both depending on $k$.
  If 
  the kernel function $k$ and the probability measure $p$ are real, 
  \begin{align*}
    k(x-x',0) &=  c_k\int_\Omega p_k(w) \cos(\dot{w}{x-x'}) + \text{\i}
    \sin(\dot{w}{x-x'}) \d w\\ &=  c_k\int_\Omega p_k(w) \cos(\dot{w}{x-x'}) \d
    w.
  \end{align*}
  Next, using the trigonometric identity
  \begin{equation*}
    \cos(a-b) = \frac{1}{\pi} \int_0^{2\pi} \cos(a+x) \cos(b+x) \d x,
  \end{equation*}
  it follows that
  \begin{align}
    k(x-x',0) &=  c_k\int_\Omega p_k(w) \cos(\dot{w}{x-x'}) \d w\nonumber\\
            &=  \frac{c_k}{\pi}\int_\Omega \int_0^{2\pi} p_k(w) \cos(\dot{w}{x} + b)
            \cos(\dot{w}{x'} + b) \d w \d b\nonumber\\ &=  2c_k\int_\Omega
            \int_0^{2\pi} p_k(w) u(b) \cos(\dot{w}{x} + b) \cos(\dot{w}{x'} + b)
            \d w \d b,\label{eq:bochner-2}
  \end{align}
  where $u(b) = (2\pi)^{-1}$ is the uniform distribution on the closed interval
  $[0,2\pi]$. We can approximate this expression  by drawing $m$ samples $w_1,
  \ldots, w_m \sim p$, $m$ samples $b_1, \ldots, b_m \sim u$, and replacing the
  integral \eqref{eq:bochner-2} with the sum
  \begin{align*}
    k(x,x') &=  2c_k\int_\Omega \int_0^{2\pi} p_k(w) u(b) \cos(\dot{w}{x} + b)
            \cos(\dot{w}{x'} + b) \d w \d b\\
            &\approx \frac{2c_k}{m} \sum_{j=1}^m \cos(\dot{w_j}{x}+b_j)
            \cos(\dot{w_j}{x'}+b_j).
  \end{align*}
  From this, we can recover the $m$-dimensional, explicit
  representation 
  \begin{equation*}
    \phi(x) = \sqrt{\frac{2c_k}{m}} \pa{\cos(\dot{w_1}{x} + b_1),
    \ldots, \cos(\dot{w_m}{x} + b_m)}^\top \in \R^m,
  \end{equation*}
  which approximates the associated shift-invariant kernel $k$ in the pointwise convergence 
  \begin{equation*}
    \dot{\phi(x)}{\phi(x')}_{\Rm} \to k(x-x',0)
  \end{equation*}
  as $m \to \infty$.

  \begin{example}[Gaussian kernel]
    The Gaussian kernel \eqref{eq:gaussian-kernel} is shift-invariant, and its
    Fourier transform is the Gaussian distribution $\N(0, 2\gamma I_d)$.
    Therefore, the map 
    \begin{equation}
      \label{eq:gauss-map}
      \phi(x) = \sqrt{\frac{2}{m}} \pa{\cos(\dot{w_1}{x} + b_1),
      \ldots, \cos(\dot{w_m}{x}+b_m)}^\top \in \Rm
    \end{equation}
    with $w_j \sim \N(0, 2\gamma I_d)$ and $b_j \sim
    \mathcal{U}[0,2\pi]$ for all $1 \leq j \leq m$ approximates the Gaussian
    kernel in the sense of the pointwise convergence 
    \begin{equation*}
      \dot{\phi(x)}{\phi(x')}_{\Rm} \to
      \exp\pa{-\gamma\pn{x-x'}}
    \end{equation*}
    as $m \to \infty$.
  \end{example}

  \begin{remark}[Computing Gaussian random features faster]\label{remark:fastfood}
    Constructing the representation \eqref{eq:gauss-map} involves computing the dot
    product $\dot{W}{x}$, where $W \in \R^{d\times m}$ is a matrix of Gaussian
    random numbers. Na\"ively, this is a $O(md)$ computation.
    \citet{fastfood} introduce \emph{Fastfood}, a technique to approximate
    dot products involving Gaussian matrices $W$, accelerating their
    computation from $O(md)$ to $O(m \log d)$ operations. Fastfood replaces the
    Gaussian matrix $W$ with a concatenation of $d\times d$ blocks with structure
    \begin{equation*} V := \frac{1}{\sigma \sqrt{d}} SHG\Pi HB, \end{equation*}
    where $\Pi \in \pb{0,1}^{d\times d}$ is a permutation matrix, and $H$ is the
    Walsh-Hadamard matrix. $S$, $G$ and $B$ are \emph{diagonal} matrices
    containing, in order, kernel function dependent scaling coefficients,
    Gaussian random numbers, and random $\{-1,+1\}$ signs. All matrices allow
    sub-quadratic computation, and the only storage requirements are the $m
    \times m$ diagonal matrices $S$, $G$, $B$.  \citet{fastfood} provide with an
    analysis of the quality of the Fastfood approximation. 
    
    The benefits of Fastfood are most noticeable when representing
    high-dimensional data.  For instance, when working with color images of $32
    \times 32$ pixels, Fastfood allows to compute the representations
    \eqref{eq:gauss-map} up to $265$ times faster.
  \end{remark}
  
  For other examples of shift-invariant kernel approximations using Bochner's
  theorem, see Table 1 of \citep{Yang14}. \citet{Sriperumbudur15} characterize
  the approximation error of $d$-dimensional shift-invariant kernels on
  $\mathcal{S} \subset \mathbb{R}^d$ using Bochner's theorem
  \begin{equation}
    \label{eq:szabo1}
    \Pr\pa{\sup_{x,x'\in \mathcal{S}} \left|\hat{k}(x,x')-k(x,x')\right| \geq
    \frac{h(d,|\mathcal{S}|,c_k)+\sqrt{2t}}{\sqrt{m}}} \leq
    \exp\pa{{-t}},
  \end{equation}
  where $\mathcal{S} \subset \R^d$ is a compact set of diameter
  $|\mathcal{S}|$, and
  \begin{equation}
    \label{eq:szabo2}
    h(d,|\mathcal{S}|,c_k) = 
    32 \sqrt{2d \log(|\mathcal{S}| + 1)} +
    32 \sqrt{2d \log(c_k + 1)} +
    16 \sqrt{2d (\log(|\mathcal{S}| + 1))^{-1}}.
  \end{equation}
  
  \begin{remark}[Multiple kernel learning]\label{remark:mkl}
    Random feature maps allow the use of different representations
    simultaneously. For instance, we could sample 
    \begin{align*}
      w_1, \ldots, w_m        &\sim \N(0, 2\cdot 0.1 \cdot I_d),\\
      w_{m+1}, \ldots, w_{2m} &\sim \N(0, 2\cdot 1 \cdot I_d),\\
      w_{2m+1}, \ldots, w_{3m} &\sim \N(0, 2\cdot 10 \cdot I_d),
    \end{align*}
    to construct a $3m$-dimensional representation approximating the sum of
    three Gaussian kernels with bandwidths $\gamma$ of $0.1$, $1$, and $10$.
    Or, for example, we could construct a $2m$-dimensional representation where
    the first half $m$ random features approximate a Gaussian kernel, and the
    second half of $m$ random features approximate an arc-cosine kernel.  This
    strategy is closely related to multiple kernel learning \citep{Gonen11}.
    The concatenation of random feature maps approximate the sum of their
    associated kernels. The outer-product of random feature maps approximates
    the product of their associated kernels.
    Finally, it is possible to learn the distribution $p_k(w)$ from which
    we sample the random features 
    \citep{buazuavan2012fourier,wilson2014thesis}.
  \end{remark}

  As opposed to the Nystr\"om method, random features do not require the
  multiplication of any $m \times m$ matrix for their construction, a costly
  operation for large $m$. Furthermore, random features do not require storage,
  since they can be efficiently resampled at test time. On the negative side,
  and as opposed to the Nystr\"om method, random features are independent from the data under study.  Therefore,
  complex learning problems require the use of large amounts of random features.
  For example, \citet{Huang14} used $400.000$ random features to build a
  state-of-the-art speech recognition system. To sum up, the intuition behind
  random features is that each random feature provides with a random summary or
  view of the data. Thus, when using large amounts of random features, chances
  are that linear combinations of these random views can express any reasonable
  pattern of interest. 

  \begin{remark}[Boltzmann brains]\label{remark:boltzmann}
  An early consideration of structure arising from randomness is due to
  Ludwig Boltzmann (1844-1906).  Under the second law of thermodynamics, our
  universe evolves (modulo random fluctuations) from low to high entropy
  states, that is, from highly ordered states to highly unordered states.  Such
  direction of time, imposed by increasing entropy, strongly contradicts the
  existence and evolution of organized life forms. Therefore, Boltzmann argues
  that our existence is a random departure from a higher-entropy universe.
  Using this argument, Boltzmann concludes that it is much more
  likely for us to be self-aware entities floating in a near-equilibrium
  thermodynamic soup (and be called \emph{Boltzmann brains}) instead of
  highly-organized physical beings embedded in a highly-organized environment,
  like our perception suggests.  Consequently, our knowledge about the universe
  is highly biased: we observe this unlikely low-entropy universe because it
  is the only one capable of hosting life; this bias is the \emph{anthropic
  principle}.
\end{remark}

As with kernels, we exemplify the use of random features on a regression
problem.

\subsection{Learning with random features}

Learning with random features involves two steps. First, transforming the data
$\{x_i\}_{i=1}^n$ into the representation $\{\phi(x_i)\}_{i=1}^n$, where $\phi$
follows \eqref{eq:mercer-map} for some kernel $k$. Second, solving the learning
problem by performing linear statistics on the random representation.

\begin{example}[Randomized least-squares regression] 
  Recall Example~\ref{ex:kernel-least-squares}. The solution to the
  least-squares regression problem \eqref{eq:least-squares} is
  \begin{equation*}
    \hat{\alpha} = \argmin_{\alpha \in \Rd} R(\alpha, \lambda, \D) = (X^\top X +
    \lambda I_d)^{-1} X^\top y,
  \end{equation*}
  where $X = (x_1, \ldots, x_n)^\top \in \Rnd$  and $y = (y_1, \ldots,
  y_n)^\top \in \Rn$. 

  To model nonlinear relationships using random features, replace the data
  $\{x_i\}_{i=1}^n$ with the $m$-dimensional random representation
  $\{\phi(x_i)\}_{i=1}^n$ from \eqref{eq:mercer-map}.  Then solve again the
  least-squares problem, this time to obtain the $m$-dimensional vector of
  coefficients
  \begin{equation*}
    \hat{\beta} = \argmin_{\beta \in \Rm} R(\alpha, \lambda, \{(\phi(x_i),y_i)\}) = (\Phi^\top \Phi +
    \lambda I_d)^{-1} \Phi^\top y,
  \end{equation*}
  where $\Phi = (\phi(x_1), \ldots, \phi(x_n))^\top \in \Rnm$  and $y = (y_1,
  \ldots, y_n)^\top \in \Rn$. This produces the regression function
  \begin{equation*}
    f_z(x) = \dot{\beta}{\phi(x)},
  \end{equation*}
  which successfully captures the nonlinear pattern in data, as depicted in
  Figure~\ref{fig:linear-random}. Learning this function takes $O(m^2n)$ time,
  while the exact kernel solution from Example~\ref{ex:kernel-least-squares}
  took $O(n^3)$ time, a much longer computation for $n \gg m$. 
\end{example}

Kernel and random representations are independent from the data under study.
Instead of relying on fixed data representations, it should be possible to
\emph{learn} them from data. This is the philosophy implemented by neural
networks, reviewed next.

\section{Neural networks}\label{sec:neural-networks}
In the beginning of this chapter, we claimed that pattern recognition involves two steps.
First, transforming the data $x_i \in \X$ into a
suitable representation $\phi(x_i) \in \H$. Second, inferring the pattern
of interest in the data as a linear statistic $\dot{A}{\phi(x)}$ of
the representation.

Kernel and random representations, reviewed in the previous two sections,
approach pattern recognition in a rather simple way: they apply a \emph{fixed} feature
map $\phi$ to the data, and then perform linear operations in the associated
fixed representation. More specifically, kernel and random representations have
form 
\begin{equation*}
  \phi(x) = \pa{\sigma(x, w_1, b_1), \ldots, \sigma(x, w_m, b_m)}^\top \in \Rm
\end{equation*}
for some \emph{fixed} set of parameters $w_1, \ldots, w_m \in \Rd$, $b_1,
\ldots, b_m \in \R$, and \emph{nonlinearity} function $\sigma : \Rd \times
\Rd \to \R$. In kernels, $\sigma(x, w_j, b_j) = k(x,w_j)$,
with $w_j = x_j$ for all $1 \leq j \leq m=n$. In random features, 
we choose the representation dimensionality $m$ \emph{a priori}, by considering the
complexity of the learning problem at hand, the amount of available data, and
our computational budget. Then, each parameter $w_j \sim p_k(w)$, $b_j \sim
\U[0,2\pi]$, and $\sigma(x,w_j, b_j) = \sqrt{\frac{2}{m}}\cos(\dot{w_j}{x}+b_j)$ for all $1 \leq
j \leq m$. In both cases, the feature map $\phi$ is independent from the
data: only a small number of tunable parameters, such as the degree for the polynomial
kernel or the variance of the Gaussian random features, are adaptable to the
problem at hand using cross-validation (see Section~\ref{sec:model-selection}).

We can parallel the previous exposition to introduce neural networks,
and highlight an structural equivalence between them and kernel methods. In neural networks,
the nonlinearity function $\sigma(z)$ is fixed to the rectifier linear unit $\sigma(z) =
\max(z,0)$, the hyperbolic tangent $\sigma(z) = \tanh(z)$, or the sigmoid
$\sigma(z) = (1+\exp(-z))^{-1}$, to name a few. However, the representation parameters
$\{(w_j,b_j)\}_{j=1}^m$ are not fixed but learned from data.  This is a challenging task: neural
networks often have millions of parameters, so training them requires the
approximation of a high-dimensional, nonconvex optimization problem. This
is a challenging task both from a computational perspective (solving such high-dimensional
optimization problems), and an statistical perspective (properly tuning millions of parameters
calls for massive data). For a historical review on artificial
neural networks, we recommend the introduction of \citep{dlbook}.

  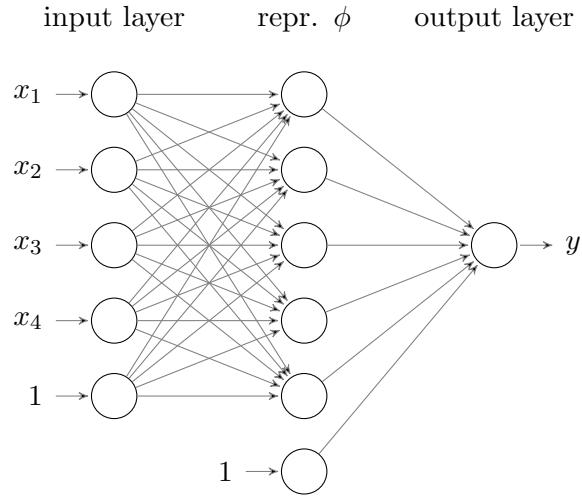
\begin{figure}
    \begin{center}
      \begin{tikzpicture}[shorten >=1pt,->,draw=black!50, node distance=2.5cm]
          \tikzstyle{every pin edge}=[<-,shorten <=1pt]
          \tikzstyle{neuron}=[circle,fill=black!25,minimum size=17pt,inner sep=0pt]
          \tikzstyle{input neuron}=[neuron, draw=black, fill=white];
          \tikzstyle{output neuron}=[neuron, draw=black, fill=white]; 
          \tikzstyle{hidden neuron}=[neuron, draw=black, fill=white];
          \tikzstyle{annot} = [text width=6em, text centered]
      
          \foreach \name / \y in {1,...,4}
            \path[yshift=0.5cm]
              node[input neuron, pin=left:$x_\y$] (I-\name) at (0,-\y) {};
      
          \foreach \name / \y in {1,...,5}
              \path[yshift=0.5cm]
                  node[hidden neuron] (H1-\name) at (2.5cm,-\y cm) {};
          
          \node[input neuron, pin=left:$1$] (I-bias) at (0,-4.5) {};
          \node[input neuron, pin=left:$1$] (H1-bias) at (2.5cm,-5.5) {};
          \node[output neuron,pin={[pin edge={->}]right:$y$}, right of=H1-3] (O) {};
      
          \foreach \source in {1,...,4}
              \foreach \dest in {1,...,5}
                  \path (I-\source) edge (H1-\dest);
  
          \foreach \source in {1,...,5}
              \path (H1-\source) edge (O);
          
          \foreach \dest in {1,...,5}
            \path (I-bias) edge (H1-\dest);
            
          \path (H1-bias) edge (O);
      
          \node[annot,above of=H1-1, node distance=1cm] (hl1) {repr. $\phi$};
          \node[annot,left of=hl1] {input layer};
          \node[annot,right of=hl1] {output layer};
      \end{tikzpicture}
    \end{center}
    \caption[A shallow neural network]{A shallow neural network.}
    \label{fig:shallow-net}
  \end{figure}

Neural networks are organized in a sequence of layers, where each layer contains 
a vector of neurons.  The neurons between two subsequent layers
are connected by a matrix of \emph{weights}.  The strength of the weights
connecting two neurons is one real number contained in the parameter set
$\{(w_j,b_j)\}_{j=1}^m$.  Neural networks contain three types of layers: input,
hidden, and output layers.  First, the input layer receives the data.  Second,
the data propagates forward from the input layer to the hidden layer, who is in
charge of computing the data representation.  In particular, the $j$-th neuron in the
hidden layer computes the $j$-th feature $\phi(x)_j = \sigma(\dot{w_j}{x}+b_j)$
of the representation, for all $1\leq j\leq m$.  Third, the representation
propagates forward from the hidden layer to the output layer.  Finally, the
output layer returns the pattern of interest, computed as the linear
transformation $\dot{A}{\phi(x)}$ of the hidden layer representation.
Figure~\ref{fig:shallow-net} illustrates a neural network, where each circle
depicts a neuron, and each arrow depicts a weight connecting two neurons from
subsequent layers together.  The depicted network accepts as input
four-dimensional data through its input layer, transforms it into a
five-dimensional representation on its hidden layer, and outputs the
one-dimensional pattern
\begin{equation*}
  y=f(x_1, x_2,x_3,x_4) = \dot{\alpha}{\sigma(\dot{(W,b)}{(x,1)})}+\beta
\end{equation*}
through its output layer. Neural networks like the one
depicted in Figure~\ref{fig:shallow-net} are \emph{fully connected}
neural networks, since all the neurons in a given layer connect to all the
neurons in the next layer.

\subsection{Deep neural networks}\label{sec:deep-nets}

  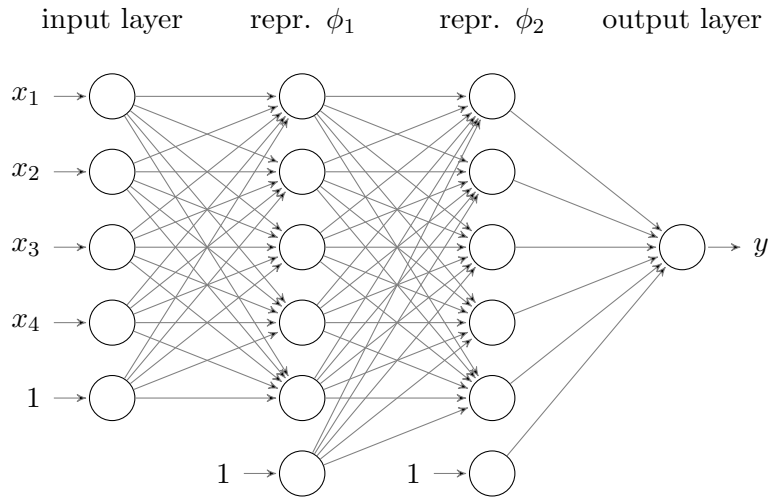
\begin{figure}
    \begin{center}
      \begin{tikzpicture}[shorten >=1pt,->,draw=black!50, node distance=2.5cm]
          \tikzstyle{every pin edge}=[<-,shorten <=1pt]
          \tikzstyle{neuron}=[circle,fill=black!25,minimum size=17pt,inner sep=0pt]
          \tikzstyle{input neuron}=[neuron, draw=black, fill=white];
          \tikzstyle{output neuron}=[neuron, draw=black, fill=white];
          \tikzstyle{hidden neuron}=[neuron, draw=black, fill=white];
          \tikzstyle{annot} = [text width=6em, text centered]
      
          \foreach \name / \y in {1,...,4}
            \path[yshift=0.5cm]
              node[input neuron, pin=left:$x_\y$] (I-\name) at (0,-\y) {};
      
          \foreach \name / \y in {1,...,5}
              \path[yshift=0.5cm]
                  node[hidden neuron] (H1-\name) at (2.5cm,-\y cm) {};
          
          \foreach \name / \y in {1,...,5}
              \path[yshift=0.5cm]
                  node[hidden neuron] (H2-\name) at (5cm,-\y cm) {};
          
          \node[input neuron, pin=left:$1$] (I-bias) at (0,-4.5) {};
          \node[input neuron, pin=left:$1$] (H1-bias) at (2.5cm,-5.5) {};
          \node[input neuron, pin=left:$1$] (H2-bias) at (5cm,-5.5) {};
      
          \node[output neuron,pin={[pin edge={->}]right:$y$}, right of=H2-3] (O) {};
      
          \foreach \source in {1,...,4}
              \foreach \dest in {1,...,5}
                  \path (I-\source) edge (H1-\dest);
  
          \foreach \source in {1,...,5}
              \foreach \dest in {1,...,5}
                  \path (H1-\source) edge (H2-\dest);
              
          \foreach \source in {1,...,5}
              \path (H2-\source) edge (O);
          
          \foreach \dest in {1,...,5}
            \path (I-bias) edge (H1-\dest);
            
          \foreach \dest in {1,...,5}
            \path (H1-bias) edge (H2-\dest);
            
          \path (H2-bias) edge (O);
      
          \node[annot,above of=H1-1, node distance=1cm] (hl1) {repr. $\phi_1$};
          \node[annot,above of=H2-1, node distance=1cm] (hl2) {repr. $\phi_2$};
          \node[annot,left of=hl1] {input layer};
          \node[annot,right of=hl2] {output layer};
      \end{tikzpicture}
    \end{center}
    \caption[A deep neural network]{A \emph{deep} neural network.}
    \label{fig:nnet}
  \end{figure}

  \emph{Deep} neural networks implement data representations computed
  as the composition of \emph{multiple} hidden layers. For
  instance, the neural network depicted in Figure~\ref{fig:nnet} has two
  layers, which implement the representation 
  \begin{align*}
    \phi(x) = \phi_2(\phi_1(x)) \in \R^5,
  \end{align*}
  where
  \begin{align*}
    \phi_2 : \R^{5} \to \R^{5}, \quad \phi_2(z)&= \pa{\sigma(\dot{z}{w_{2,1}}+b_{2,1}), \ldots, \sigma(\dot{z}{w_{2,5}}+b_{2,5})}^\top \in \R^{5},\\
    \phi_1 : \R^4 \to \R^{5}, \quad \phi_1(x) &= \pa{\sigma(\dot{x}{w_{1,1}}+b_{1,1}), \ldots, \sigma(\dot{x}{w_{1,5}}+b_{1,5})}^\top \in \R^{5},
  \end{align*}
  and $z=\phi_1(x)$, $w_{1,j} \in \R^4$, $w_{2,j} \in \R^5$, and $b_{1,j}, b_{2,j}
  \in \R$ for all $1 \leq j \leq 5$.

  More generally, the parameters of a deep neural network are  a collection of
  weight matrices $W_1, \ldots, W_L$, with $W_i \in \R^{m_{i-1} \times m_{i}}$,
  $m_0 = d$, $m_L = m$, biases $b_1, \leq, b_L \in \R$, and compute the
  representation
  \begin{align*}
    \phi(x) &= \phi_L,\\
    \phi_l  &= \sigma(\dot{W_l}{\phi_{l-1}}+b_l),\\
    \phi_0  &= x,
  \end{align*}
  where the nonlinearity $\sigma : \R \to \R$ operates entrywise. The number of
  free parameters in a deep representation is $O(m_0m_1 + m_1m2 + \ldots +
  m_{L-1}m_L)$. Note the contrast with the number of parameters of a
  nonparametric Gaussian kernel machine; most likely, two: the Gaussian kernel
  bandwidth, and the regression regularizer.

  Deep neural networks are hierarchical compositions of representations,
  each capturing increasingly complex patterns from data. Each hidden layer takes as input the
  output of the previous layer, and processes it to learn a slightly more
  abstract representation of the data. For example, when training deep neural networks
  to recognize patterns from images, the first representation $\phi_1$
  detects edges of different orientations from raw
  pixels in the image, and the subsequent representations $\phi_2,
  \ldots, \phi_L$ learn how to combine those edges into parts, those parts into
  objects, and so on. 
  
  Representing data with \emph{deep models}, also known as
  \emph{deep learning}, has been the most successful technique to learn
  intricate patterns from large data in recent years, defining the new
  state-of-the-art in complex tasks such as image or speech recognition
  \citep{dlbook,LeCun15}.
  
  One key property fueling the power of deep representations is that these are
  \emph{distributed} representations \citep{hinton1986distributed}. This
  concept is better understood using a simple example.  Consider the task of
  classifying images of cars.  If using a kernels, our $n$-dimensional
  representation $\phi(x)$ would contain one feature $\phi(x)_j = k(x,x_j)$ per
  car image $x_j$, for all $1 \leq j \leq n$. This means that our
  representation would contain one dedicated feature describing the image
  ``small yellow Ferrari'', and another dedicated feature describing the image
  ``big red Tesla''.  These representations are \emph{local} representations,
  and partition our data in an number of groups \emph{linear} in the sample
  size $n$.  On the other hand, compositional architectures such as deep neural
  networks could arrange their representation to depict the three binary
  features ``yellow or red'', ``small or big'', and ``Ferrari or Tesla''. Each
  of these three binary features is the computation implemented by a sequence
  of hidden layers, which process and recombine the data in multiple different
  ways.  The key point here is that these three binary features exhibit a
  many-to-many relationship with respect to the data: many samples in the data
  are partially described by the same feature, and many features describe data
  example.  Importantly, these \emph{distributed} representations, these
  ``attribute sharing'' structure of data, allow a separation in a number of
  groups \emph{exponential} in the dimensionality $d$: the three binary
  features in our example can describe an exponential amount of $2^3$ different
  images of cars.

  \begin{remark}[Is it necessary to be deep?]
    Universal kernels learn, up to an arbitrary precision, any continuous
    bounded pattern from data.  Therefore, why should we care about deep neural
    network representations, their millions of parameters, and their
    complicated numerical optimization?

    Because when learning some functions, restricting the representation to have one
    single layer results in requiring its dimensionality to be exponentially
    large.  For example, Gaussian kernel machines $f(x) = \sum_{i=1}^n \alpha_i
    k(x,x_i)$ need at least $n = O(2^d)$ terms to represent the parity function
    of a binary string $x$ of $d$ bits.  In the language of neural networks,
    learning some functions require an exponential amount of hidden neurons
    when the network has only one hidden layer. In contrast, the parity
    function is learnable using a $O(d)$ dimensional representation with 
    two layers \citep[Section 14.6]{dlbook}. In sum,
    deep representations incorporate the compositional
    structure of the world as their prior knowledge.  Such compositional structure, constructing
    features out of features, leads to exponential gains in representational
    power. Rephrasing the comparison in terms of sample complexity, functions with an
    exponential amount of different regions may require an exponential amount
    of data when learned using shallow representations, and a linear amount of
    data when learned using deep representations.
    Moreover, deep models are a generalization of shallow models, making them
    an object of both theoretical and practical interest.
  \end{remark}

  \subsection{Convolutional neural networks}\label{sec:cnns}
  Kernel methods, random features and neural networks are general-purpose tools
  to construct representations from data.  In particular, all of them are \emph{permutation
  invariant}: they learn the same representation from two
  different versions of the same data, if the only difference between the two 
  is the order of their variables.  But for some data, the order of variables
  is rich prior knowledge, exploitable to build better representations.

  For example, consider the design of a machine to classify the
  hand-written digits from Figure~\ref{fig:digits-original} into ``fives'' or
  ``eights''. As humans, solving this task is easy because of the way on which
  the pixels, edges, strokes, and parts of the digits are arranged on the
  two-dimensional surface of the paper. The task becomes much more difficult if
  we scramble the pixels of the digit images using a fixed random permutation, as
  illustrated in Figure~\ref{fig:digits-permuted}.  Although the transformation
  from Figure~\ref{fig:digits-original} to Figure~\ref{fig:digits-permuted}
  destroys the spatial dependencies between the variables under study,
  permutation invariant methods treat equivalently both versions of the data.
  Permutation invariant methods therefore would ignore the local spatial dependence
  structures between neighbouring pixels in natural images. To some extent,
  permutation invariant methods will search
  patterns over the space of all images, including images formed by random pixels,
  instead of focusing their efforts on the smaller set of images that feel
  natural to perception.  Therefore, discarding spatial dependencies is a waste
  of our resources!  How can we leverage these dependence structures, instead
  of ignoring them? 

  One way is to apply the same feature
  map along different local spatial groups of variables.  In the case of
  images, this means extracting the same representation from different small
  neighbourhoods of pixels in the image, and returning the concatenation of all
  of these local representations as the image representation. After all,
  to locate an object in an image, all we care about is \emph{what} features are
  present in the image, regardless of \emph{where}.  This is known as
  \emph{translational invariance}.
  
  \emph{Convolutional
  neural networks} implement this idea by extending the architecture of
  feedforward neural networks. Deep convolutional neural networks alternate
  three different types of layers: convolutional layers, nonlinearity layers,
  and pooling layers. 
  We now detail the inner workings of these three types of layers. For
  simplicity, assume that the data under study are color images. The mathematical
  representation of an image is the three-dimensional volume or \emph{tensor} $X \in
  \R^{w \times h \times d}$, where $w$ and $h$ are the width and the height of the
  image in pixels, and $d$ is the depth of the image in channels or features.
  
  First, convolution layers accept three inputs: the input image $X \in \R^{w \times h
  \times d}$, the filter bank $W \in \R^{s\times s \times d \times d'}$
  containing $d'$ filters of size $s \times s \times d$, and the bias vector $b \in
  \R^{d'}$.  Convolution layers return one output image $X' \in \R^{(w-s+1)
  \times (h-s+1) \times d'}$, with entries 
  \begin{align*}
    X' &= \text{conv}(X; W,b),\\
    X'_{i',j',k'} &= \sum_{i=1}^s \sum_{j=1}^s \sum_{k=1}^d X_{i'+i-1,j'+j-1,k} W_{i,j,k,k'} + b_{k'},
  \end{align*}
  for all $i' \in \{1, \ldots, w'\}$, $ j' \in \{1, \ldots, h'\}$, and $d' \in
    \{1, \ldots, d'\}$. In practice, the input images $X$ are \emph{padded}
    with zeros before each convolution, so that the input and output images
    have the same size. The intensity of the output pixel $X'_{i,j,k}$ relates to 
    the presence of the filter $W_{:,:,:,k}$ near the input pixel $X_{i,j,:}$.
    Figure~\ref{fig:convolution} exemplifies the convolution operation.
 
  \begin{figure}
    \begin{center}
    \includegraphics[width=\textwidth]{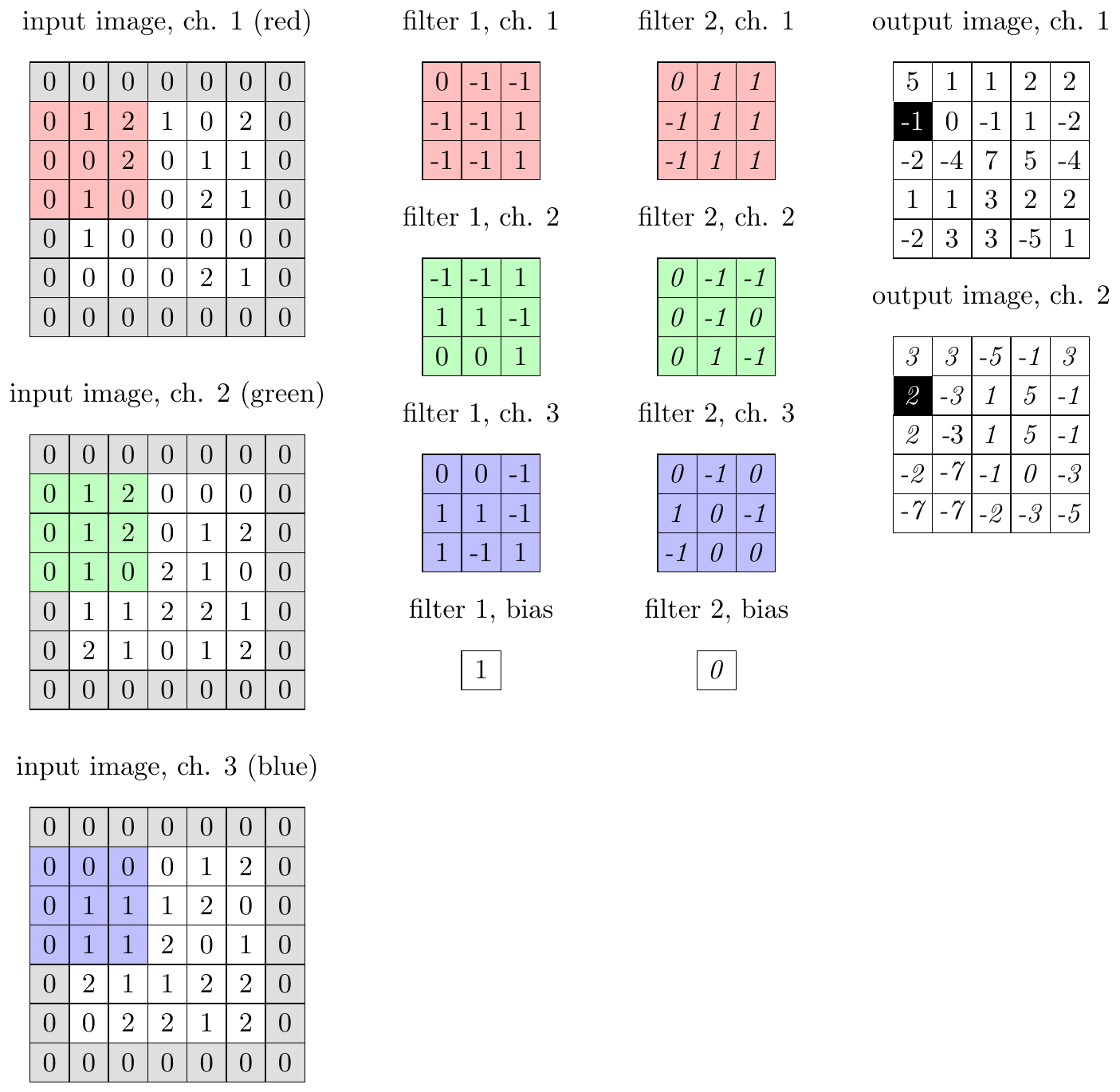}
    \end{center}
    \caption[Operation of a convolution layer]{A convolution layer transforming
    a zero-padded input color image $X \in \R^{5 \times 5 \times 3}$ into an
    output image $X' \in \R^{5\times 5\times 2}$, using a filter bank $W \in
    \R^{3 \times 3 \times 3 \times 2}$ and its corresponding two biases $b = (b_1,
    b_2)$. Highlighted in black, the output pixel computed by the
    color-highlighted patches of the input image and filter bank. Figure
    adapted from \citep{karpathycnn}.}
    \label{fig:convolution}
  \end{figure}
  
  Second, nonlinearity layers $\sigma(\cdot)$ apply a nonlinear function
  entrywise 
   \begin{align*}
    X' &= \sigma(X),\\
    X'_{i,j,k} &= \sigma(X_{i,j,k}),
  \end{align*}
  for all $i \in \{1, \ldots, w\}$, $j \in \{1, \ldots, h\}$, and $k \in \{1, \ldots, d\}$.

  Third, pooling layers summarize each neighbourhood of $\alpha \times \alpha$ pixels
  in a given input image $X \in \R^{w \times h \times d}$ into one pixel of the
  output image $X' \in \R^{(w/\alpha) \times (h/\alpha) \times d}$.  For
  instance, in \emph{max pooling} each of the pixel values of the output image is the
  maximum value of the pixel values within each $\alpha \times \alpha$
  neighbourhood in the output image.  In most applications, $\alpha = 2$; in
  this case, simply write $X' = \text{pool}(X)$. Pooling layers reduce the
  computational requirements of deep convolutional neural networks, since they
  reduce the size of the input image passed to the next convolution. Pooling layers operate
  independently per channel.  To remove the need of pooling layers, some
  authors suggest to implement convolution layers with \emph{large stride}. In these
  large stride convolutions, the filter slides multiple pixels at a time, effectively
  reducing the size of the output image \citep{springenberg2014striving}.

  In short, the representation implemented by a deep convolutional neural
  network has form 
  \begin{align*}
   \phi(X) &= \phi_L,\\
    \phi_l &= \text{pool}(\sigma(\text{conv}(\phi_{l-1}; W_l, b_l))),\\
    \phi_0 &= X,
  \end{align*}
  where $L$ can be in the dozens \citep{dlbook}. The feature map of a convolutional deep
  neural network is ``elastic'', in the sense that it accepts images of
  arbitrary size. The only difference is that the convolution operation will
  slide over a larger input image, thus producing a larger output image. If we
  require a final representation of a fixed dimensionality, we can use the last
  pooling layer to downscale the dimensionality of the final output image
  appropriately. 

  \begin{figure}
    \begin{subfigure}{\textwidth}
    \includegraphics[width=\textwidth]{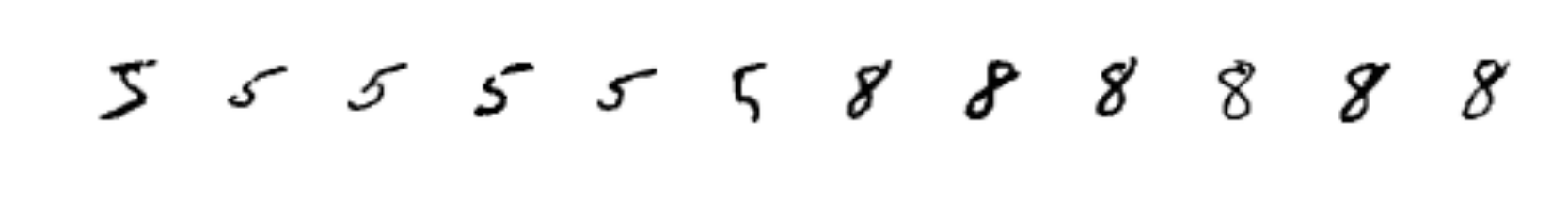}
    \vspace{-0.8cm}
    \caption{Original handwritten digit images.}
    \label{fig:digits-original}
    \end{subfigure}
    \begin{subfigure}{\textwidth}
    \includegraphics[width=\textwidth]{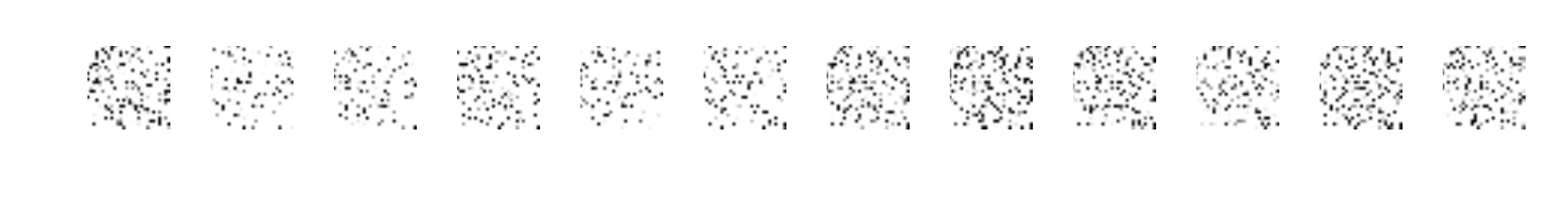}
    \vspace{-0.8cm}
    \caption{Handwritten digit images with randomly permuted pixels.}
    \label{fig:digits-permuted}
    \end{subfigure}
    \caption{The MNIST handwritten digits dataset.}
  \end{figure}

\begin{remark}[Recurrent neural networks]
  Some data, such as speech, video, and stock quotes, are naturally presented
  as a time series. The temporal dependence structure in these data is yet
  another instance of prior knowledge that can be conveniently exploited to
  build better representations. \emph{Recurrent} neural networks (see, for example, 
  \citep{sutskever2013training}) are neural networks adapted to learn
  from time series.
\end{remark}
  
\subsection{Learning with neural networks}\label{sec:learning-nns}
  Neural networks, fully connected or convolutional, shallow or deep, are
  trained using the \emph{backpropagation} algorithm \citep{rumelhart}.
  Usually, before employing backpropagation, we fill each weight matrix $W_l \in
  \R^{d_{l-1} \times d_L}$ in the neural network with random
  numbers sampled from
  \begin{equation*}
    \U\left[-\sqrt{\frac{6}{d_{l-1}+d_L}}, +\sqrt{\frac{6}{d_{l-1}+d_L}}\right],
  \end{equation*}
  where $\U$ denotes the uniform distribution \citep{glorot2010understanding}.
  
  Once the neural network has been randomly initialized, the backpropagation
  algorithm runs for a number of iterations.  Each backpropagation iteration
  implements two computations. First, the raw data makes a \emph{forward pass}
  through the network, from the input layer to the output layer, producing
  predictions.  Second, the prediction errors make a \emph{backwards pass}
  through the network, from the output layer to the input layer. In this
  backward pass, the backpropagation algorithm computes how should we modify
  each of the weights in the network to lower its average prediction error.
  Backpropagation proceeds recursively: the weight updates in one layer depend
  on the prediction errors made by the next layer. Thanks to the
  differentiation chain rule, backpropagation is effectively implemented as a
  gradient descent routine on neural networks with architectures described by
  directed acyclic graphs. The backpropagation algorithm updates the network
  for a number of iterations, until the average error over some held-out
    validation data stops decreasing or starts to increase. For a full
    description of the backpropagation algorithm and its history, refer to
    \citep[Section 6.4]{dlbook}.
  
  Bear in mind that, in opposition to kernels and random
  features, training neural networks requires approximating the solution to a
  high-dimensional nonconvex optimization problem. Nonconvex optimization problems have multiple
  local minima, so initializing the network to a different set of weights will
  result in backpropagation converging to a different solution, and this solution will have a different generalization error
  (Section~\ref{sec:numerical-optimization}).  To alleviate this issue,
  practitioners train multiple neural networks on the same data, starting from
  different random initializations, and then average their outputs for a final
  prediction. The nonconvexity of deep neural networks is a double edged sword:
  it allows the learning of highly complex patterns, but hinders the
  development of theoretical guarantees regarding their generalization
  performance.

  We now exemplify how to learn a single-hidden-layer neural network to perform
  nonlinear least-squares regression. 

\begin{example}[Neural least-squares]
  \label{ex:neural}
  As in Example~\ref{ex:kernel-least-squares} the goal here is to minimize the
  least-squares regression error 
  \begin{equation}
    \label{eq:neural-E}
    E = \frac{1}{n} \sum_{i=1}^n (f(x_i)-y_i)^2,
  \end{equation}
  with respect to the parameters $\{(\alpha_j,w_j,b_j)\}_{j=1}^m$, and $\beta$ of the neural network
  \begin{equation*}
    f(x) = \sum_{j=1}^m \alpha_j \sigma(\dot{w_j}{x}+b_j) + \beta
  \end{equation*}
  for some nonlinearity function $\sigma : \R \to \R$.
  We use the backpropagation algorithm. First, propagate all the training data
  through the network.  Then, compute the derivatives of the error function
  \eqref{eq:neural-E} with respect to each parameter of the network:
  \begin{align}
    \frac{\partial E}{\partial w_{j,k}} &= \frac{2}{n} \sum_{i=1}^n
    (f(x_i)-y_i) \cdot \alpha_j \cdot \sigma'(\dot{w_j}{x_i}+b_j) \cdot
    x_{i,k},\label{eq:neural-gradients}\\
    \frac{\partial E}{\partial b_{j}} &= \frac{2}{n} \sum_{i=1}^n
    (f(x_i)-y_i) \cdot \alpha_j \cdot \sigma'(\dot{w_j}{x_i}+b_j),\nonumber\\
    \frac{\partial E}{\partial \alpha_j} &= \frac{2}{n} \sum_{i=1}^n
    (f(x_i)-y_i) \cdot  \sigma(\dot{w_j}{x_i}+b_j),\nonumber\\
    \frac{\partial E}{\partial \beta} &= \frac{2}{n} \sum_{i=1}^n
    (f(x_i)-y_i).\nonumber
  \end{align}
  We can observe the recursive character in \eqref{eq:neural-gradients}: the
  updates of the weights in a given layer depend on the next layer. 
  Similar, slightly more complicated formulas follow for deep and
  convolutional neural networks. Using the gradients
  \eqref{eq:neural-gradients}, we update $T$ times each parameter in the
  network using the update rule 
  \begin{align*}
    w_{j,k} &= w_{j,k} - \gamma \frac{\partial E}{\partial w_{j,k}},
  \end{align*}
  where $\gamma \in (0,1)$ is a small \emph{step size}
  (Section~\ref{sec:numerical-optimization}).  Similar update rules follow for
  $\{\alpha_j, b_j\}_{j=1}^m$ and $b$. To decide the number of  gradient
  descent iterations $T$, we can monitor the performance of the neural network
  on some held-out validation set, and stop the optimization when the validation error error
  stops decreasing. 
  The computation of the gradients of \eqref{eq:neural-E} takes
  $O(n)$ time, a prohibitive requirement for large $n$ or large number of
  iterations $T$.  Because of this reason, neural networks are commonly trained
  using stochastic gradient descent (Remark~\ref{remark:sgd}).
\end{example}

The previous example illustrates how to tune the network parameters
$\{(\alpha_j, w_j, b_j)\}$ and $b$, but it does not comment on how to choose the
architectural aspects of the network, such as the nonlinearity function,
the step size in the gradient descent optimization, the
number of hidden layers, the number of neurons in each hidden layer,
and so on. These parameters are usually tuned using cross-validation, as
detailed in Section~\ref{sec:cross-validation}. The candidate set of neural
network architectures is often chosen at random from some reasonable distribution over
the architecture parameters \citep{bergstra2012random,dlp-bandits}. Then, the
final neural network is the best or the average of the top best performing on
the validation set.

Because of the great flexibility of deep neural network representations, it is
important to implement regularization schemes along with their optimization.
Three popular alternatives are dropout regularization
\citep{srivastava2014dropout} batch normalization \citep{ioffe2015batch}, and
early stopping.  Dropout regularization reduces the risk of overfitting by
deactivating a random subset of the neurons at each iteration of gradient
descent, so the network can not excessively rely on any single neuron.  Batch
normalization readjusts the parameters of the network periodically during
learning, so that the neuron pre-nonlinearity activations have zero mean and
unit variance.  Early stopping stops the training of the neural network as soon
as possible, since the generalization error of algorithms trained
with stochastic gradient descent increases with the number of iterations
\citep{hardtrecht}.

\section{Ensembles}\label{sec:ensembles}
 \emph{Ensembles} are combinations of different predictors, or \emph{weak
 learners}, to solve one single learning problem. Ensembling is a powerful
 technique: the winning entry of the \$1,000,000 \emph{Netflix
 Prize} was a combination of more than 100 different weak learners
 \citep{bell2007bellkor}.  There are two main ways of combining weak learners
 together: \emph{boosting} and \emph{stacking}.
 
 First, boosting ensembles learn a sequence of weak learners, where each weak learner
 corrects the mistakes made by previous ones. Given some data
 $\{(x_i,y_i)\}_{i=1}^n$, gradient boosting machines \citep{friedman2001greedy}
 perform regression as follows. First, compute the constant
 \begin{equation*}
   f_0(x) = \gamma_0 = \argmin_\gamma \sum_{i=1}^n \ell(\gamma, y_i),
 \end{equation*}
 where $\ell : \R \times \R \to \R$ is a differentiable loss
 function. 
 Second, for a number of boosting iterations $1 \leq t \leq T$, use the 
 \emph{pseudo-residual} data
 \begin{equation*}
   \left\lbrace \left( x_i, - \frac{\partial \ell(f_{t-1}(x_i), y_i)}{\partial f_{t-1}(x_i)}\right)\right\rbrace_{i=1}^n
 \end{equation*}
 to fit a weak learner $h_t$, and incorporate it into the ensemble as 
 \begin{equation*}
   f_t(x) = f_{t-1}(x) + \gamma h_t(x),
 \end{equation*}
 where 
 \begin{equation*}
   \gamma = \argmin_\gamma \sum_{i=1}^n \ell(f_{t-1}(x_i) + \gamma h_t(x_i), y_i).
 \end{equation*}
 The ensemble $f_T$ is the final predictor.
 
 Second, stacking ensembles construct $T$ weak learners independently and in parallel, and
 their predictions are the input to another machine, that learns
 how to combine them into the final prediction of the ensemble.  Bagging is one
 popular variation of stacking, where one trains each of the $T$ independent
 weak learners on a subset of the data sampled at random with replacement. The
 predictions of a bagging ensemble are simply the average of all the weak
 learners. Bagging reduces the error variance of individual predictions. To see
 this, write the error variance of the ensemble as
 \begin{align*}
   \E{}{\left(\frac{1}{T} \sum_{i=1}^n \bm \varepsilon_i\right)^2} &=
   \frac{1}{T^2} \E{}{\sum_{i} \left(\bm \varepsilon_i^2 + \sum_{j\neq i} \bm \varepsilon_i \bm\varepsilon_j\right)}\\
   &= \frac{1}{T} \E{}{\bm \varepsilon_i^2} + \frac{k-1}{k} \E{}{\bm \varepsilon_i \bm \varepsilon_j}.
 \end{align*}
 We see that if the weak learners are independent, the error covariances
 $\E{}{\bm \varepsilon_i \bm \varepsilon_j}$ tend to zero, so the ensemble will
 have an average error variance $T$ times smaller than the individual weak
 learner error variances \citep{dlbook}.
 
 Random forests are one popular example of bagging ensembles
 \citep{Breiman01}, considered one of the most successful learning
 algorithms \citep{fernandez2014we}. Random forests are bags of decision trees,
 each of them trained on a random subset of both the data examples and the data
 features.  Random
 forests induce a random representation, like the ones studied in
 Section~\ref{sec:random-features}.  A random forest with $m$ decision trees of
 $l$ leafs each implements a $lf$-dimensional random feature map $\phi$, with
 features
  \begin{align}
    \label{eq:random-forest-features}
    \phi(x)_j = \I\left(\text{leaf}\left(\left\lfloor
    \frac{j-1}{m}+1\right\rfloor, x\right) = (\text{mod}(j-1,m)+1)\right),
  \end{align}
  where $\text{leaf}(t,x)$ returns the leaf index from the $t$-th tree where
  the sample $x$ fell, for all $1 \leq j \leq lf$.

\section{Trade-offs in representing data}\label{sec:bias}

Finding good representations is both the most important and challenging
part of pattern recognition.
It is important, because they allow to extract nontrivial intelligence from data. 
And it is challenging, because it involves multiple 
intertwined trade-offs. The only way of favouring one representation over
another is the use of prior knowledge about the specific data under study.
Every representation learning algorithm excelling at one task will fail miserably when applied to 
others. As a matter of fact, when averaged over all possible pattern recognition tasks, no method is
better than other. In mathematical jargon, \emph{there is no free lunch}
\citep{wolpert1997no}.

The first major trade-off is the one between the flexibility of a representation and
its \emph{sample complexity}. Learning flexible patterns calls for
flexible feature maps, and flexible feature maps contain a large amount of
tunable parameters. In turn, a larger amount of data is necessary to
tune a larger amount of parameters.  For
instance, consider representing $d$ dimensional
data using a feature map with $O(md)$ free parameters. In the simplest case,
where each of the parameters is binary can only take two different values, we
face a search amongst $2^{md}$ possible representations. As a
modest example, if learning from data containing $d=10$ dimensions, there is
an exponential amount
\begin{equation*}
  2^{d \times m} = 2^{100} \approx 1.25 \times 10^{30}
\end{equation*}
of single-hidden-layer neural networks with $m = 10$ hidden neurons connected
by binary weights. \citet{bellman1956dynamic} termed this exponential rate of
growth in the size of optimization problems \emph{the curse of
dimensionality}.

Second, flexible feature maps call for nonconvex numerical optimization
problems, populated by local minima and saddle point solutions (recall
Figure~\ref{fig:numerical_optimization}).  But flexibility also contradicts
invariance. For example, if learning to classify handwritten digit images like
the ones depicted in Figure~\ref{fig:digits-original}, we may favour
representations that are invariant with respect to slight rotations of the
digits, given that the same digit can appear in the data at different angles,
when written by different people. However, representations taking this
invariance to an extreme would deem ``sixes'' indistinguishable from ``nines'',
and perform poorly.

\begin{figure}
  \begin{subfigure}{0.24\textwidth}
  \includegraphics[width=\textwidth]{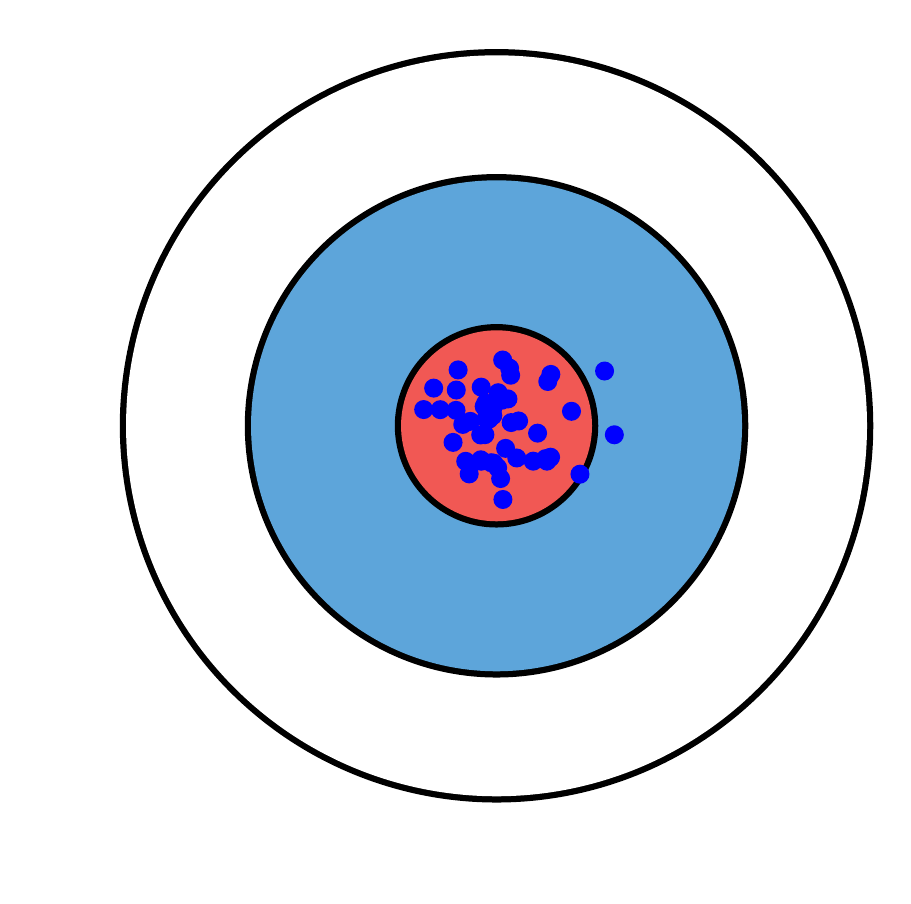}
  \end{subfigure}
  \begin{subfigure}{0.24\textwidth}
  \includegraphics[width=\textwidth]{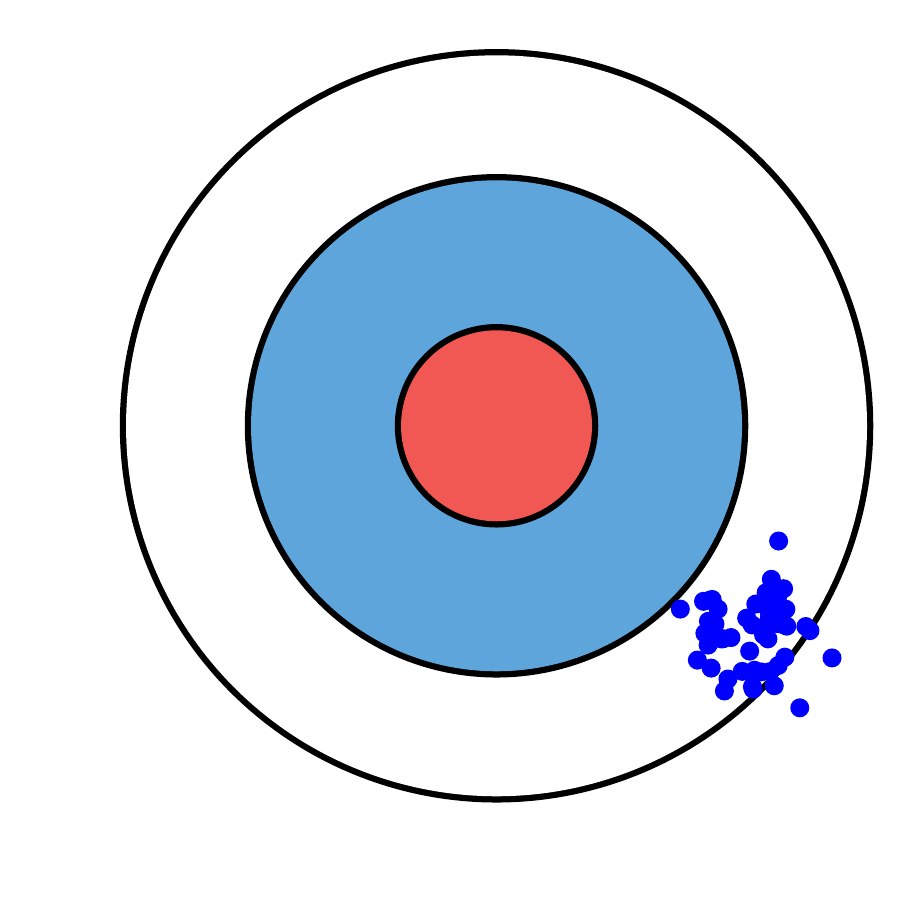}
  \end{subfigure}
  \begin{subfigure}{0.24\textwidth}
  \includegraphics[width=\textwidth]{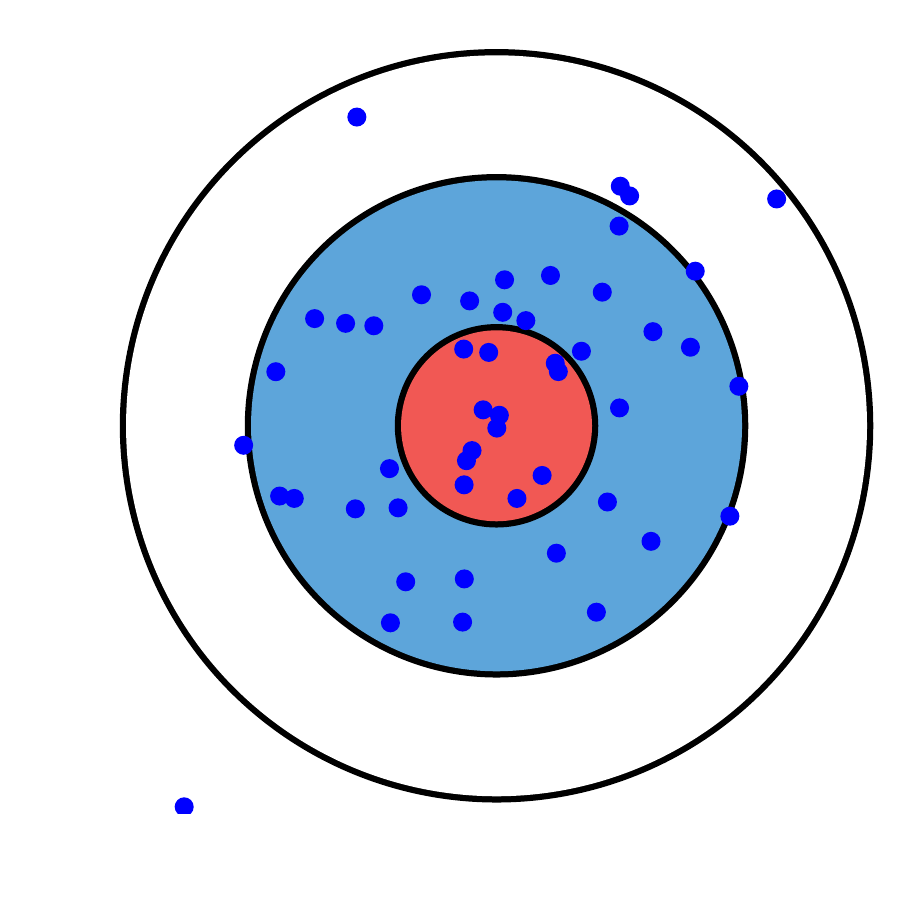}
  \end{subfigure}
  \begin{subfigure}{0.24\textwidth}
  \includegraphics[width=\textwidth]{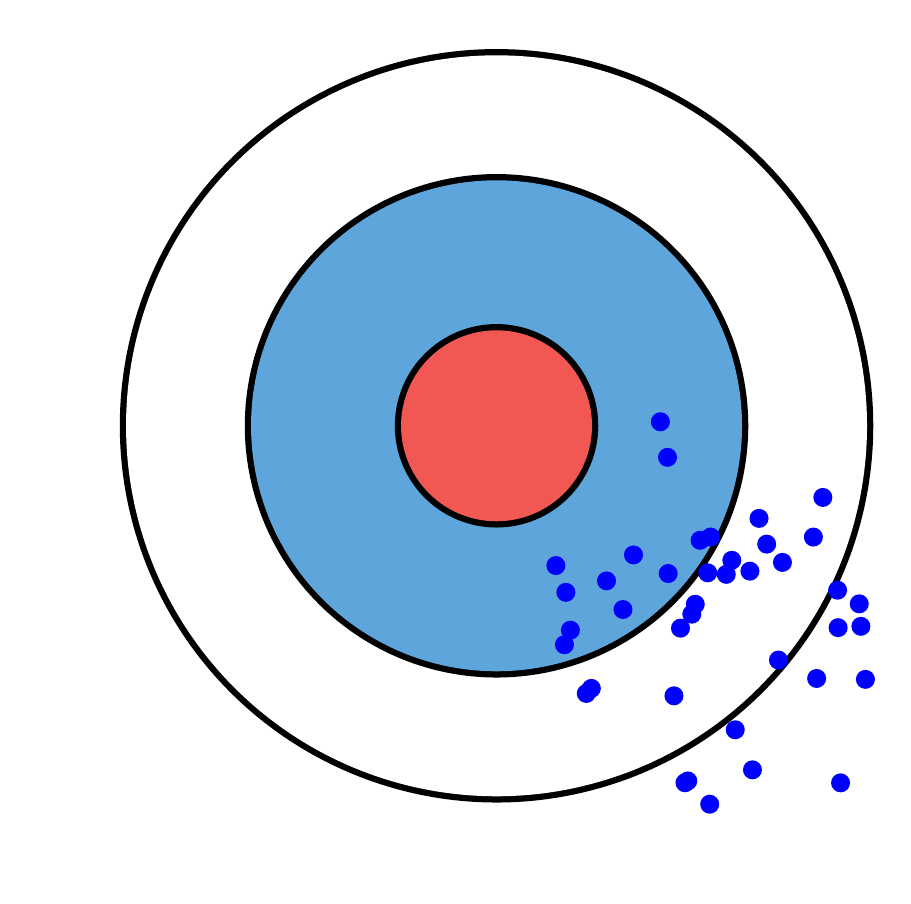}
  \end{subfigure}
  \caption[Bias versus variance]{Illustration of bias and variance, when playing
  darts. From left to right, low bias and low variance, high bias and low
  variance, low bias and high variance, high bias and high variance.}
  \label{fig:bias-variance}
\end{figure}

Third, from a statistical point of view, flexibility controls the
\emph{bias-variance trade-off} discussed in
Sections~\ref{sec:estimation}~and~\ref{sec:model-selection}. The trade-off
originates from the fact that learning drinks from two simultaneous, competing
sources of error. First, the bias, which is the error derived from erroneous
assumptions built in our representation. For example, linear feature maps
exhibit high bias when trying to unveil a complex nonlinear pattern.  High bias
results in over-simplifying the pattern of interest, that is, underfitting.
Second, the variance, which is the error derived from the sensitivity to noise
in the training set. A learning algorithm has large variance when small changes
in the training data produce large deviations on its predictions. High variance
causes overfitting, which is the undesirable effect of hallucinating patterns
from the noise polluting the data. In short, too-simple models have high bias and
low variance, while too-complex models have low bias and high variance.
Figure~\ref{fig:bias-variance} illustrates the bias-variance trade-off when
playing to hit the bullseye in the game of darts\footnote{Figure based on
\url{http://scott.fortmann-roe.com/docs/BiasVariance.html}}. Good
representations should aim at optimally balancing bias and variance to 
maximize performance at subsequent learning tasks.

\section{Representing uncertainty}
Uncertainty is ubiquitous in data. It arises due to human or mechanical errors
in data collection, incomplete models, or fluctuations of unmeasured or missing
variables.  Even if we have the most Laplacian deterministic view of the
universe, our limited knowledge and perception turns deterministic systems into
partially random. Furthermore, describing complex processes using a few uncertain
rules is simpler than describing them using a large amount of deterministic
rules.

We can accommodate uncertainty in learning by assuming that predictions
are not deterministic quantities $\hat{f}(x)$, but
\emph{predictive distributions} $\hat{P}(\bm y \given x)$.  For
instance, consider access to some data $\{(x_i, y_i)\}_{i=1}^n$, where $y_i =
f(x_i) + \epsilon_i$ for some function $f : \Rd \to \R$ that we wish to learn,
and some additive noise $\epsilon_i \sim \N(0,\lambda^2)$, for all $1 \leq i
\leq n$. 
Before seeing the measurements $y_i$, we can use our prior knowledge about the
data under study, and define a \emph{prior distribution} over the kind of
functions $f$ that we expect to see linking the random variables $\bm x$ and
$\bm y$. For instance, we may believe that the possible regression functions
$f$ follow a Gaussian process \citep{Rasmussen06} prior: 
\begin{equation*}
  f \sim \N(0, K),
\end{equation*}
where the $n\times n$ covariance matrix is the kernel matrix $K$, with entries
$K_{ij} = k(x_i,x_j)$. Here, the kernel function $k$ describes the shape of the
interactions between pairs of points $(x_i, x_j)$, for all $1 \leq i, j \leq
n$, and depends on prior knowledge, but not on the
data. Given a new observation $x$, the $n+1$ measurement locations $(x_1, \ldots, x_n,
x)$ are still jointly Gaussian: 
\begin{equation}\label{eq:gp-prior}
  \begin{pmatrix} y \\ f(x) \end{pmatrix}
  \sim \N \left(0, \begin{pmatrix} K + \lambda^2 I_n & k_x \\ k^\top_x & k(x,x)
  \end{pmatrix}\right),
\end{equation}
where the column vector $k_x \in \Rn$ has entries $k_{x,i} = k(x,x_i)$ for all
$1 \leq i \leq n$. 

Now, let us take into account the measurements $\{y_i\}_{i=1}^n$. By
applying the conditional distribution rule of multivariate
Gaussians~\eqref{eq:mvnat}, we can transform the Gaussian process prior
\eqref{eq:gp-prior} into the Gaussian process \emph{posterior} or predictive
distribution
\begin{align}
  f(x) &\sim \N(\mu(x), \sigma(x))\nonumber\\
  \mu(x) &= k^\top_x (K+\lambda^2 I_n)^{-1}y,\nonumber\\
  \sigma(x) &= k(x,x)-k^\top_x (K+\lambda^2 I_n)^{-1} k_x.\label{eq:gp-equations}
\end{align}

As seen in Equation~\eqref{eq:gp-equations}, the predictions from Gaussian
processes are Gaussian distributions. In some situations, however,
predictive distributions can be far from Gaussian: heavy-tailed, multimodal,
and so forth. One method to approximate arbitrary predictive
distributions is the \emph{bootstrap method} \citep{efron1979bootstrap}. 
The bootstrap method trains $K$ weak learners that solve the learning problem
at hand, each of them on a different \emph{bootstrap set}
$\{x_{k(i)}\}_{i=1}^m$, for all $1 \leq k \leq K$. Each bootstrap set is a
random subset of $m$ examples of the data sampled with replacement
\citep{kleiner2014scalable}.  At test time, the bootstrap method returns $K$
different answers, one per weak learner. The ensemble then summarizes the $K$
bootstrap answers into a predictive distribution.
Random forests (Section~\ref{sec:ensembles}) are one simple
form of bootstrapping. The predictions provided by random forests are a
collection of predictions made by the individual decision trees forming the
forest.  Thus, one can use these individual predictions to estimate a
predictive distribution.

\begin{figure}[t!]
  \begin{subfigure}{0.32\textwidth}
  \includegraphics[width=\textwidth]{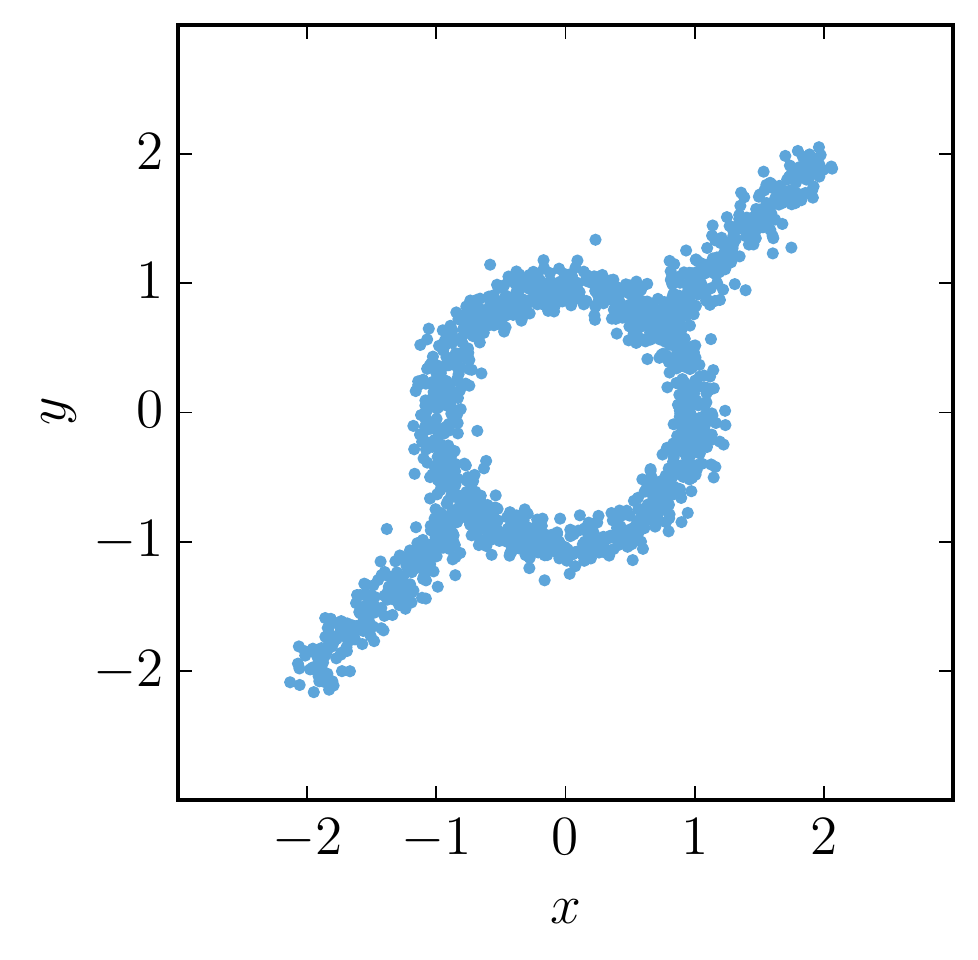}
  \caption{Data}
  \label{fig:variance-data}
  \end{subfigure}
  \begin{subfigure}{0.32\textwidth}
  \includegraphics[width=\textwidth]{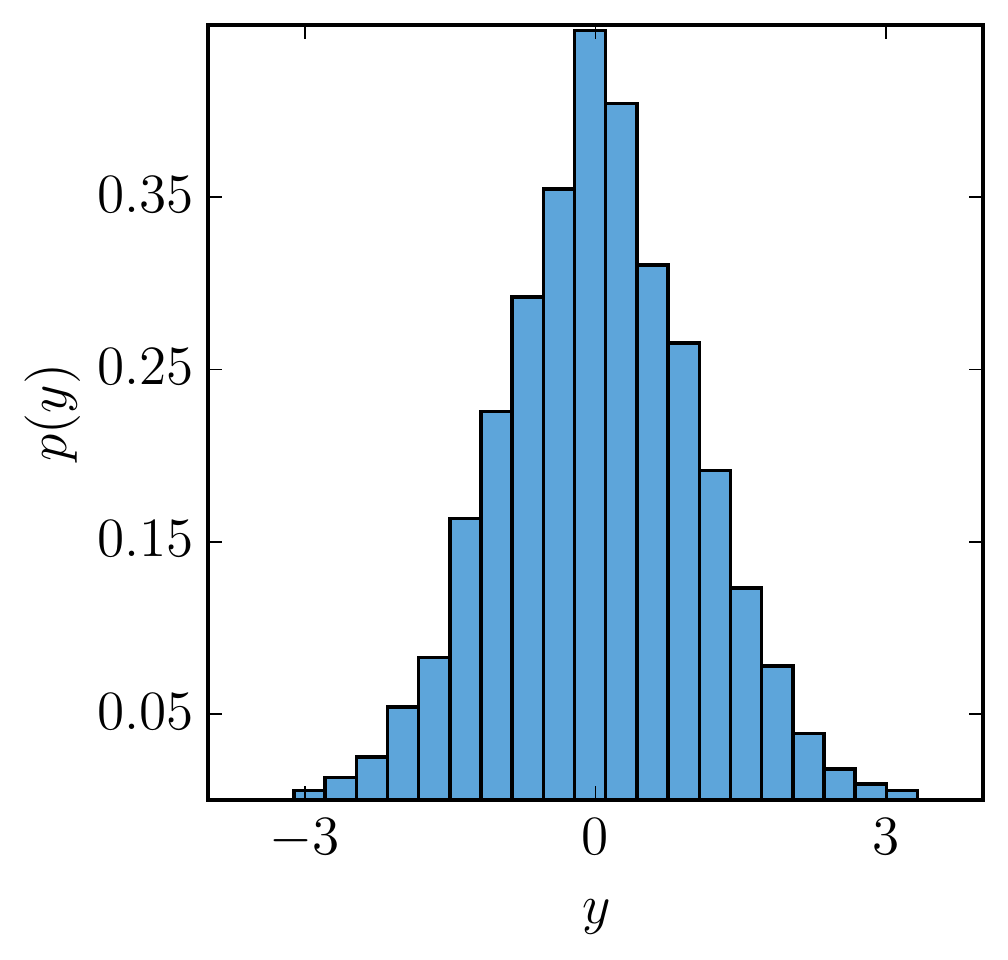}
  \caption{Gaussian process}
  \label{fig:variance-gp}
  \end{subfigure}
  \begin{subfigure}{0.32\textwidth}
  \includegraphics[width=\textwidth]{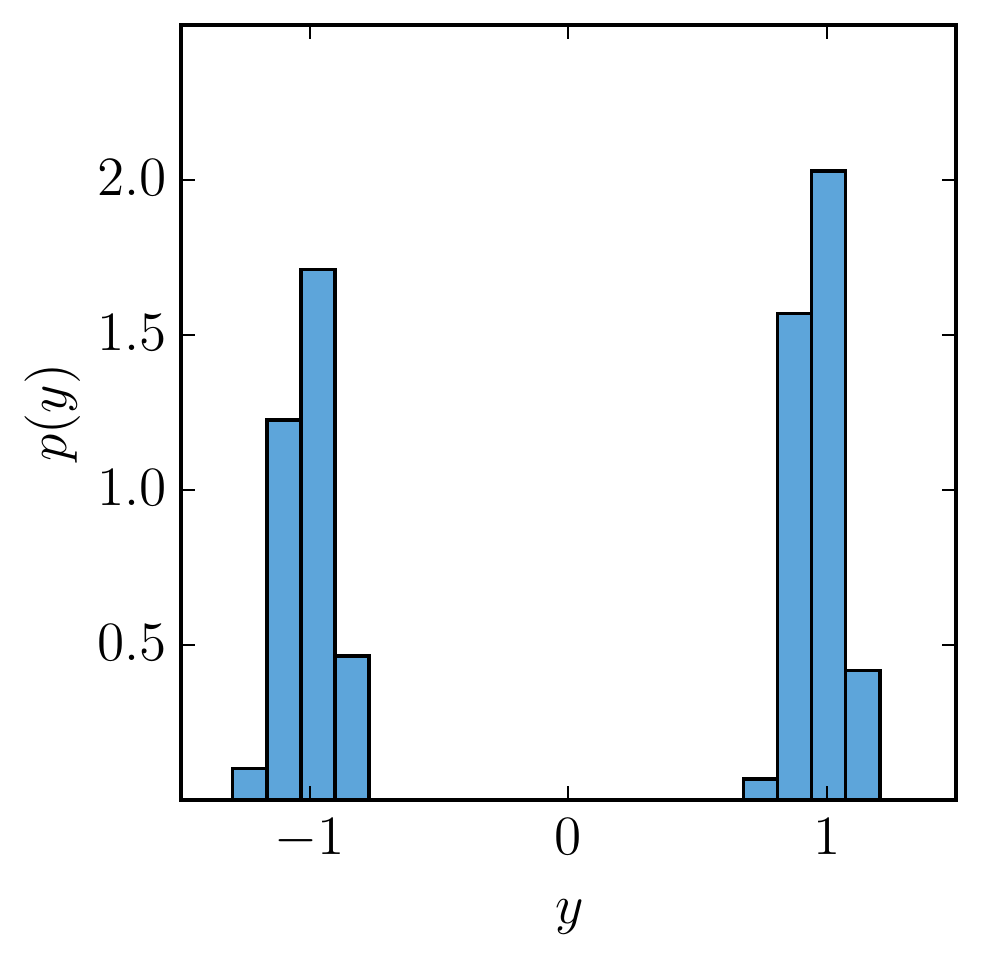}
  \caption{Random forest}
  \label{fig:variance-forest}
  \end{subfigure}
  \caption{Measuring uncertainty with predictive distributions.}
  \label{fig:variance}
\end{figure}

Figure~\ref{fig:variance} illustrates the predictive distributions ${P}(\bm y
\given \bm x = 0)$ estimated by a Gaussian process and a random forest, using
the data from Figure~\ref{fig:variance-data}.
On the one hand, the Gaussian process returns a Gaussian predictive
distribution, depicted in Figure~\ref{fig:variance-gp}, which erroneously
characterizes the true, bimodal predictive distribution at $x=0$.  On the other
hand, the random forest is able to determine, as seen in
Figure~\ref{fig:variance-forest}, that the true predictive distribution at $x=0$ has
two pronounced modes. In any case, the Gaussian process correctly captures the
variance (uncertainty) of the true predictive distribution. And
this is everything we could hope for, since Gaussian process predictive
distributions are Gaussian, and therefore unimodal.

  \part{Dependence}
  \chapter{Generative dependence}\label{chapter:generative-dependence}
\vspace{-1.25cm}
  \emph{This chapter contains novel material. First,
  Section~\ref{sec:conditional} introduces the use of expectation propagation
  and sparse Gaussian processes to model multivariate conditional dependence in
  copulas \citep{dlp-gp}. We illustrate the effectiveness of our approach in the
  task of modeling regular vines (Section~\ref{sec:conditional_experiments}).
  We call this model the Gaussian Process Regular Vine (GPRV).  Second,
  Section~\ref{sec:nonparametric} proposes a nonparametric copula model, along
  with its associated conditional distributions \citep{dlp-ssl}.  We exemplify
  the effectiveness of our approach in the task of semisupervised domain
  adaptation using regular vines (Section~\ref{sec:adaptation}). We call this 
  model the Non-Parametric Regular Vine (NPRV).}
\vspace{1.25cm}

\noindent \emph{Generative models} use samples
\begin{equation*}
  x = \{x_1, \ldots, x_n\} \sim P^n(\bm x), \, x_i \in \mathbb{R}^d
\end{equation*}
to \emph{estimate} the probability density function
\begin{equation*}
  p(\bm x) = \frac{\partial^d P(\bm x)}{\partial x_1 \cdots \partial x_d},
\end{equation*}
where $P(\bm x)$ is a continuously differentiable cdf.  So, generative models
aim at describing all the marginal distributions and dependence structures
governing the multivariate data $x$ by estimating its density function $p(\bm
x)$. This task of \emph{density estimation problem} is often posed as a
\emph{maximum likelihood estimation}\footnote{We call \emph{estimation} the process
of obtaining point-estimates of parameters from observations. We call 
\emph{inference} the process of deriving posterior distributions from previous
beliefs and observations.}, and solved in two steps. First, choose one
\emph{generative model}, that is, a collection of density functions $\P_\Theta
= \{p_\theta\}_{\theta \in \Theta}$ indexed by their parameter vector $\theta
\in \Theta$.  Second, choose the density $p_{\hat{\theta}} \in \P_\Theta$ that
best describes the samples $x$, by maximizing the log-likelihood objective 
\begin{equation*}
  L(\theta) = \sum_{i=1}^n \log p_{\theta}(x_i),
\end{equation*}
with respect to the distribution parameters $\theta$. Let $\hat{\theta}$ be the
parameter vector maximizing the previous objective on the data $x$. Then, the maximum
likelihood solution to the density estimation problem is the density 
$p_{\hat{\theta}}$.

Why is {generative modeling} of interest? A 
good estimate for the data generating density function $p(\bm x)$ allows all
sorts of complex manipulations, including:
\begin{enumerate}
  \item \emph{Evaluating the probability of data}. This allows to detect
  outliers, or to manipulate samples as to increase or decrease their likelihood
  with respect to the model. 
  \item \emph{Sampling new data}. Generating new samples is useful to
  synthesize artificial data, such as images
  and sounds. 
  \item \emph{Computing conditional distributions} of output variables $\bm
  x_{\mathcal{O}}$ given input variables $\bm x_{\mathcal{I}} = x$. The conditional
  distribution $p(\bm x_{\mathcal{O}} \given \bm x_{\mathcal{I}} = x_\mathcal{I})$
  could characterize, for instance, the distribution of missing variables: their expected
  value, variance (uncertainty), and so on. Conditional distributions also
  allow to use generative models for discriminative tasks, like regression
  and classification.
  \item \emph{Computing marginal distributions} of variables $\bm
  x_{\mathcal{M}}$, by integrating out (or marginalizing out) all the variables
  in $\bar{\mathcal{M}}$:
  \begin{equation*}
    p(\bm x_{\mathcal{M}}) = \int_{\bar{\mathcal{M}}} p(\bm x_\mathcal{M}, \bm
    x_{\bar{\mathcal{M}}} = x_{\bar{\mathcal{M}}}) \d x_{\bar{\mathcal{M}}}.
  \end{equation*}
\end{enumerate}

Under mild conditions, the probability density function $p(\bm x)$ contains all
the observable information about the data generating distribution $P(\bm x)$. Thus,
accurately estimating the density function of our data amounts to solving multiple
statistical learning problems at once, including regression, classification,
and so forth.  This erects density estimation as the silver bullet to all
statistical learning problems. But, with great powers comes great
responsibility: density estimation, the most general of statistical problems,
is also a most challenging task.  To better understand this, take a look at
Figure~\ref{fig:density-discriminative}. In both regression and classification tasks on the depicted density,
the statistic of interest is shown as a black line. Either the depicted
regressor or the depicted classifier is a much simpler object than the full
density of the data.  Thus, for problems such as regression or classification,
density estimation is often a
capricious intermediate step. In these situations, density estimation is a
living antagonist of Vapnik's principle:
\begin{center}
  \emph{When solving a problem of interest, do not solve a more general problem
  as an intermediate step. \citep{Vapnik98}}
\end{center}

\begin{figure}
  \begin{subfigure}{0.48\textwidth}
  \begin{center}
  \includegraphics[width=\textwidth]{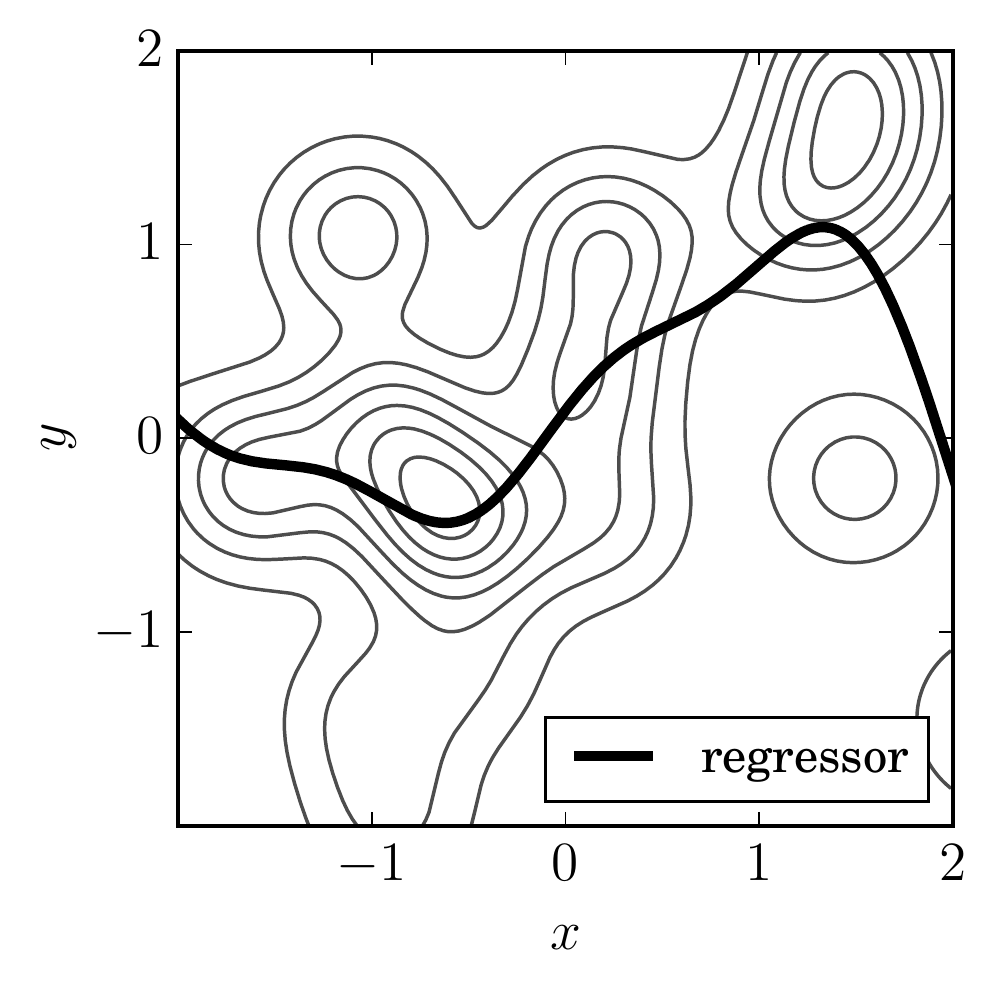}
  \end{center}
  \vspace{-0.5cm}
  \caption{regression}
  \label{fig:from-density-to-reg}
  \end{subfigure}
  \hspace{.5cm}
  \begin{subfigure}{0.48\textwidth}
  \begin{center}
  \includegraphics[width=\textwidth]{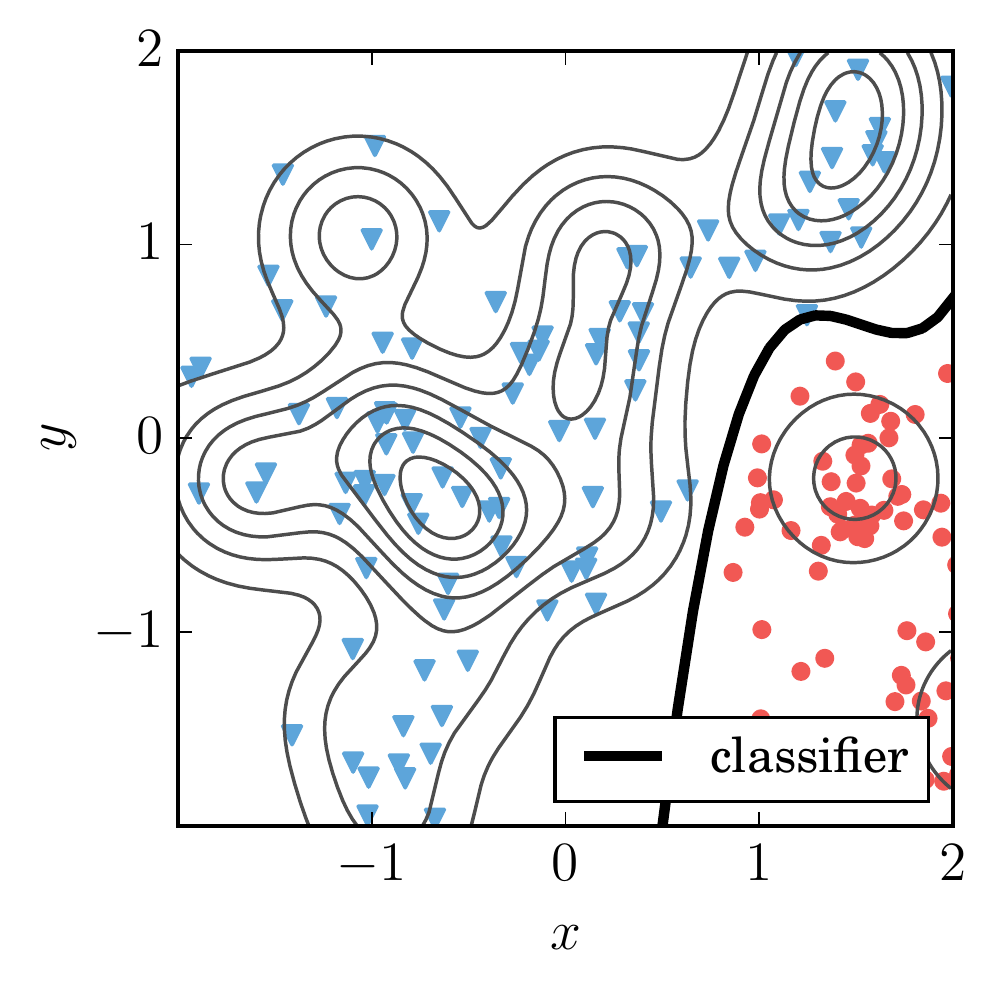}
  \end{center}
  \vspace{-0.5cm}
  \caption{classification}
  \label{fig:from-density-to-cla}
  \end{subfigure}
  \caption[Density estimation is difficult]{Density estimation is more
  difficult than (a) regression, and (b) classification.}
  \label{fig:density-discriminative}
\end{figure}

A second challenge of density estimation is its computational intractability. This
difficulty arises because probabilities require normalization, and such normalization
involves solving challenging and multidimensional integrals over the density
function.  In fact, different density computations pose different trade-offs; for instance, generative models
allowing for easy sampling may be difficult to condition and marginalize, or
vice versa \citep[Table 2]{goodfellow2014generative}.  Luckily, normalized
probabilities are necessary only when combining different generative models 
together: for example, when evaluating the likelihood of a sample with respect to two
different generative models.

A third challenge of generative modeling is their evaluation: 
generative models trained for different purposes should be evaluated differently.
\citet{theis2015note} illustrates this dilemma for generative models of natural
images.  Using a fixed dataset, the authors construct a generative model with
high log-likelihood but producing poor samples, and a generative model with low
log-likelihood but producing great samples. The latter model simply memorizes
the training data. This memorization allows to produce perfect samples (the
training data itself), but assigns almost zero log-likelihood to unseen test
data. More formally, when we approximating a density function
$p$ with a model $p_\theta$ using a metric $d$ over probability measures, there
are multiple ways to be $d(p,p_\theta)=\varepsilon > 0$ wrong.  Unsurprisingly,
some of these ways to be wrong are more appropriate to solve some problems (like
log-likelihood maximization for data compression), and less appropriate for
others (for instance, a higher degree of memorization leads to better sample
quality). This relates to the notion of loss functions in supervised learning,
since different losses aim at different goals.

This chapter explores five models for density estimation: Gaussian models,
transformation models, mixture models, copula models, and product models. Each
model has different advantages and disadvantages, and excels at modeling
different types of data.

\begin{remark}[Wonders and worries in maximum likelihood estimation]
  Maximum likelihood relies on two assumptions: the \emph{likelihood principle}
  and the \emph{law of likelihood}. The {likelihood principle} assumes that,
  given a generative model like $\P_\Theta$, the log-likelihood function
  $L(\theta)$ contains all the relevant information to estimate the parameter
  $\theta \in \Theta$. On the other hand, the law of likelihood states that the
  ratio $p_{\theta_1}(x)/p_{\theta_2}(x)$ equals the amount of evidence
  supporting the model $p_{\theta_1}$ in favour of the model $p_{\theta_2}$,
  given the data $x$.
  
  Maximum likelihood estimation is \emph{consistent}: the sequence of
  maximum likelihood estimates converges to the true value under estimation, as
  the sample size grows to infinity. Maximum likelihood is also
  \emph{efficient}: no other consistent estimator has lower
  asymptotic mean squared error. Technically, this is because maximum
  likelihood estimation achieves the absolute Cram\'er-Rao bound.

  When working with finite samples, there are alternative estimators that
  outperform maximum likelihood estimation in mean squared error. A notable
  example is the James-Stein estimator of the mean of a $d$-dimensional
  Gaussian, for $d \geq 3$. We exemplify it next. Consider observing one sample
  $y \sim \N(\mu,\sigma^2 I)$; then, the maximum likelihood estimation of the
  mean is $\hat{\mu}_{\textrm{MLE}} = y$, which is an estimation taking into
  account each of the $d$ coordinates separately. In contrast, the James-Stein
  estimator is $\hat{\mu}_{\textrm{JS}} = (1-(m-2)\sigma^2 \|y\|^{-2})y$, which
  is an estimation taking into account the norm of $y$ to estimate each of the
  $d$ coordinates jointly.
\end{remark}

\section{Gaussian models}

The Gaussian distribution is the most important of probability distributions
because of two reasons. First, due to the \emph{central limit theorem}, the sum
of $n$ independent random variables converges to an unnormalized Gaussian
distribution, as $n \to \infty$. Second, Gaussian distributions model linear
dependencies, and this enables a linear algebra
over Gaussian distributions convenient for computation.

The Gaussian distribution\index{Gaussian distribution}
is a distribution over the real line, with density function 
\begin{equation*}
  \N(\bm x; \mu, \sigma^2) = p(\bm x = x;\mu,\sigma^2) =
  \frac{1}{\sigma\sqrt{(2\pi)}}\exp\pa{-\frac{(x-\mu)^2}{2\sigma^2}}
\end{equation*}
fully parametrized by its first two moments: the mean $\mu$ and the variance $\sigma^2$. The special case
$\N(\bm x|0,1)$ is the \emph{Normal distribution}\index{normal
distribution}. The Gaussian cumulative distribution function does not have a
closed form, but is approximated numerically.
Figure~\ref{fig:gaussian_dfs} plots the probability density function,
cumulative distribution function, and empirical cumulative distribution
function (see Definition~\eqref{def:ecdf}) of a Normal distribution.

\begin{figure}
  \begin{center}
  \includegraphics[width=\textwidth]{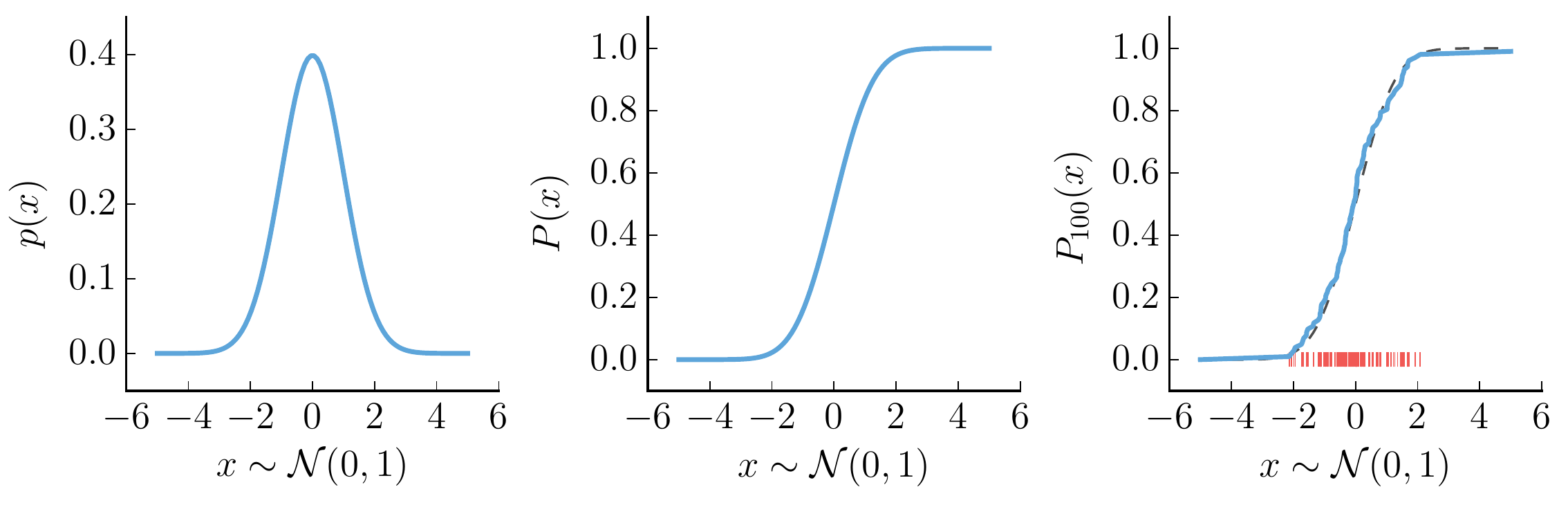}
  \end{center}
  \caption[Normal PDF, CDF, and ECDF]{Probability density function (pdf),
  probability cumulative function (cdf), and empirical cdf based on $100$
  samples (depicted in red) for a Normal distribution.}
  \label{fig:gaussian_dfs}
\end{figure}

Assume now $d$ Gaussian random variables $\bm x = (\bm x_1, \ldots, \bm x_d)$
with $\bm x_i \equiv \N(\mu_i,\Sigma_{ii})$. If the dependencies between the
components in $x$ are linear, the joint distribution of $\bm x$ is 
a \emph{multivariate Gaussian distribution}, with a density function 
\begin{equation*}
  \N(\bm x; \mu, \Sigma) = p(\bm x;\mu,\Sigma) =
  \frac{1}{\sqrt{(2\pi)^d|\Sigma|}}\exp\pa{-\frac{1}{2}(x-\mu)^\top
  \Sigma_{}^{-1} (x-\mu)}
\end{equation*}
fully characterized by its mean vector $\mu = (\mu_1, \ldots, \mu_d) \in \R^d$,
and the positive-definite covariance matrix $\Sigma \in
\R^{d\times d}$\index{covariance matrix}. The $d$ diagonal terms $\Sigma_{i,i}$ are the
$d$ variances of each of the Gaussian random variables $\bm x_i$
forming the random vector $\bm x$, and each off-diagonal term $\Sigma_{ij}$ ($i\neq
j$) is the covariance between $\bm x_i$ and $\bm x_j$. Therefore, uncorrelated Gaussian
random variables have diagonal covariance matrices.  In particular, for any
$\sigma^2 > 0$, we call the distribution $\N(\mu,\sigma^2 I_d)$ \emph{isotropic or
spheric}, see the left side of Figure~\ref{fig:two_gaussians}. Two
Gaussian random variables may not be jointly Gaussian; in this case, their
dependencies are nonlinear. One important consequence of this fact is that two
random variables can be simultaneously uncorrelated and dependent. So remember:
independent implies uncorrelated, but uncorrelated does not imply independent! 

\begin{figure}
  \begin{center}
  \includegraphics[width=0.8\textwidth]{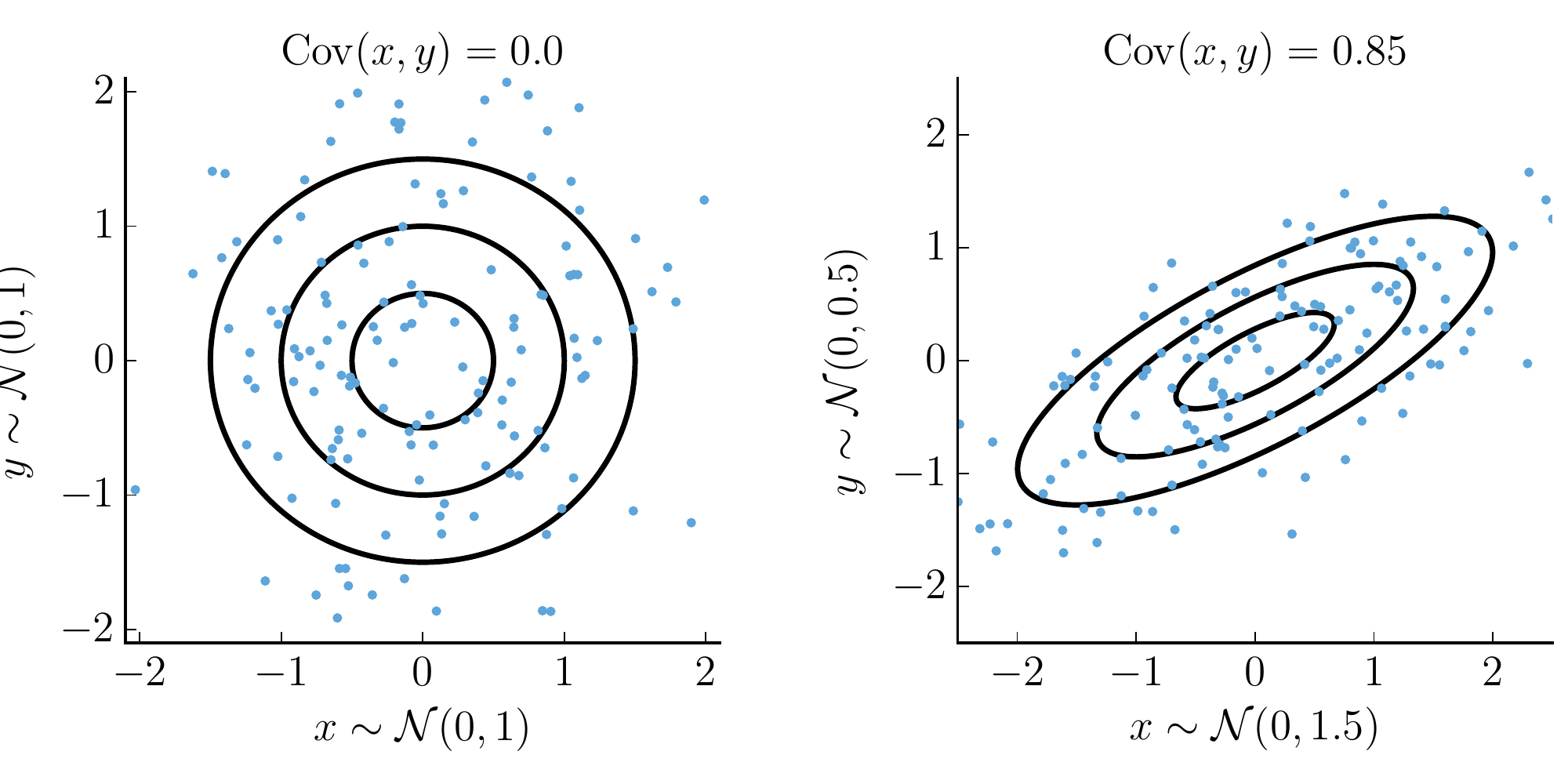}
  \end{center}
  \caption[Samples from two Gaussian distributions]{Samples drawn from two
  different Gaussian distributions, along with lines indicating one, two, and
  three standard deviations.}
  \label{fig:two_gaussians}
\end{figure}

The Gaussian distribution is a member of the elliptical distributions. These
are the distributions with contours of regions of equal density described by
ellipses.  The center of the ellipse is the vector $\mu$, the sizes of its
semiaxis are the diagonal elements from $\Lambda$, and the rotation of the
ellipse with respect to the coordinate system of the Euclidean space is $U$,
where $\Sigma = U\Lambda U^\top$. 
    
The linear dependencies described by Gaussian distributions reduce their
modeling capabilities, but bring computational advantages. First,
affine transformations of Gaussian random vectors are also Gaussian; in
particular,
\begin{equation}
  \label{eq:mvnat}
  \begin{drcases}
   \bm x \equiv \N(\mu,\Sigma)\\
   \bm y \ot A\bm x +b
  \end{drcases} \Rightarrow \bm y \equiv \N(b+A\mu, A\Sigma A^\top).
\end{equation}
One consequence of the multivariate Gaussian affine transform is that any
marginal distribution of a Gaussian distribution is also Gaussian. For example,
to compute the marginal distribution of $(\bm x_1,\bm x_2,\bm x_4)$, set $b=0$ and use
\begin{equation*}
  A =
  \begin{pmatrix}
    1 & 0 & 0 & 0 & 0 & \cdots & 0 \\
    0 & 1 & 0 & 0 & 0 & \cdots & 0 \\
    0 & 0 & 0 & 1 & 0 & \cdots & 0
  \end{pmatrix},
\end{equation*}
that is, dropping the irrelevant terms from $\mu$ and the irrelevant rows
and columns from $\Sigma$. The multivariate Gaussian affine transform
also implies that sums of Gaussian random variables are also Gaussian. Finally,
if $\bm x \equiv \N(\mu,\Sigma)$ with
\begin{equation*}
  \mu =
  \begin{pmatrix}
    \mu_1 \\
    \mu_2 \\
  \end{pmatrix}, \quad
  \Sigma =
  \begin{pmatrix}
    \Sigma_{11} & \Sigma_{12} \\
    \Sigma_{21} & \Sigma_{22} \\
  \end{pmatrix},
\end{equation*}
such that $\mu_1 \in \R^p$, $\mu_2 \in \R^q$, and $\Sigma$ has the appropriate block structure, then
\begin{equation*} 
p(\bm x_1|\bm x_2=a) = \N\pa{\mu_1+\Sigma_{12}\Sigma^{-1}_{22}(a-\mu_2),
\Sigma_{11}-\Sigma_{22}^{-1}\Sigma_{21}}
\end{equation*}

Finally, two important information-theoretic quantities have closed form formulae 
for Gaussian distributions. These are the entropy of a Gaussian
\begin{equation*}
  H(\N(\mu, \Sigma)) = \frac{1}{2} \ln \pa{(2\pi e)^d \cdot |\Sigma|},
\end{equation*}
and the Kullback-Liebler divergence between two Gaussians
\begin{align*}
  D_{\text{KL}}(\N(\mu_0, \Sigma_0) &\| \N(\mu_1,\Sigma_1)) =\\
  &\frac{1}{2}\pa{\text{tr}\pa{\Sigma_1^{-1}\Sigma_0} + (\mu_1-\mu_0)^\top
  \Sigma_1^{-1}(\mu_1-\mu_0)-d+\ln\frac{|\Sigma_1|}{|\Sigma_0|}}.
\end{align*}
For additional identities involving the multiplication, division, integration,
convolution, Fourier transforms, and constrained maximization of Gaussians,
consult \citep{Roweis99}. 

\section{Transformation models}\label{sec:transformation-models}
Transformation (or latent variable) models assume that the random variable
under study $\bm x$ is explained by some simpler latent random
variable $\bm z$. While the distribution of the observed variable $\bm x$ may be
in general complex and high dimensional, it is common to assume that the
distribution of the latent explanatory factors $\bm z$ is low-dimensional and easy to model.
Figure~\ref{fig:transformation-model} illustrates the canonical transformation
generative model.

\begin{figure}
  \begin{center}
  \begin{tikzpicture}[node distance=1cm, auto,]
   \node[punkt] (x11) at (0,0)   {$\bm z$};
   \node[punkt,fill=gray!20] (x12) at (3,0) {$\bm x$};
   \node[above right=-12pt and 25pt of x12] (labelx1)  {observed,};
   \node[below=0.01cm of labelx1] (labelx2)  {complex $p(\bm x)$};
   \node[above left=-12pt and 25pt of x11] (labelz1)  {unobserved,};
   \node[below=0 of labelz1] (labelz2)  {simple $p(\bm z)$};
   \draw[pil] (x11) -- (x12);
   \draw[pil,draw=gray!50] (x12.north) to[out= 135,in= 45] node[text=gray,above] {$p(\bm z \given x)$} (x11.north);
   \draw[pil,draw=gray!50] (x11.south) to[out=-45,in=-135] node[text=gray,below] {$p(\bm x \given z)$} (x12.south);
  \end{tikzpicture}
  \end{center}
  \caption{The canonical transformation model.}
  \label{fig:transformation-model}
\end{figure}
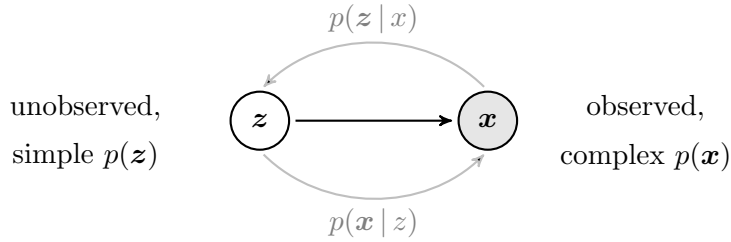

Let us make this definition concrete with one simple example. Consider that we
are designing a generative model of $100\times 100$ pixel images of handwritten
digit images. Instead of directly modeling the dependence structure of the
$10,000$ pixels forming $\bm x$, we could consider instead high level descriptions $\bm
z$ like ``a \emph{thick}, \emph{quite-round} number \emph{six}, which is
\emph{slightly-rotated-to-the-left}''. In such descriptions, italic words
describe the values of the latent explanatory factors ``digit thickness'',
``digit roundness'', ``digit class'', and ``digit rotation'', which incarnate
the intensities of the observed pixels.  Using the latent variable $\bm z$,
modeling the distribution of $\bm x$ translates into modeling i) the
distribution of $\bm z$, and ii) the function $f$ mapping
$\bm z$ to $\bm x$. Therefore, transformation models are
useful to model high-dimensional data when this is described as a function of a
small amount of explanatory factors, and when these explanatory factors follow
a distribution that is easy to model (for instance, when the explanatory
factors are mutually independent).

In the following, we will use the notation $X \in \R^{n \times d}$ to denote
the data matrix constructed by stacking the samples $x_1, \ldots, x_n \sim
P(\bm x)$ as rows.

\subsection{Gaussianization}

Let us start with one of the simplest transformation models.
\emph{Gaussianization} \citep{chen2001gaussianization} computes an invertible
transformation from the input feature matrix $X^{(0)} \in \R^{n \times d}$, which
follows a continuous distribution with strictly positive density, into the output explanatory factor 
matrix $Z=X^{(T)} \in \R^{n \times d}$, which approximately follows a Normal
density function.  Gaussianization computes this transformation by iterating
two computations.  First, it employs the ecdf (Definition~\ref{def:ecdf}) and
the inverse cdf of the Normal distribution to make each column of $X^{(t)}$
follow a Normal distribution.  Let $M^{(t)} \in \R^{n\times d}$ be the matrix
containing the result of these $d$ one-dimensional transformations. Second,
Gaussianization transforms $M^{(t)}$ into $X^{(t+1)}\in\R^{n\times d}$ by applying a
simple transformation.  When this transformation is a random rotation, the
principal component analysis rotation, or the independent component analysis
rotation, the sample $X^{(t+1)}$ follows a distribution closer to the Normal
distribution than the previous iterate $X^{(t)}$ \citep{laparra2011iterative}.

Denote by $f$ the Gaussianization transformation after a sufficiently large
number of iterations, and observe that this function is invertible. Then, the
data $f(X^{(0)})$ approximately follows a $d$-dimensional Normal distribution.
Using $f$, we can obtain a new sample from $P(\bm x)$ by sampling $z \sim \N(0,I_d)$
and returning $x = f^{-1}(z)$. We can also approximate likelihood $p(\bm x=x)$
by evaluating the Normal likelihood $\N(f(x); 0, I_d)$ and renormalizing with
Equation~\ref{eq:transformation-densities}. On the negative side, obtaining
the conditional and marginal distributions of $p(\bm x)$ using the Gaussianization
framework is nontrivial, and the necessary number of iterations to obtain
Gaussianity is often large. Moreover, Gaussianization models obtained from $T$
iterations require storing $O(Tnd+Td^2)$ parameters.

\subsection{Variational inference}\label{sec:variational-inference}
One central computation in Bayesian statistics is posterior inference, 
implemented by applying Bayes' rule on the observed variables $\bm x$ and the
latent variables $\bm z$. That is, to compute quantities
\begin{equation*}
  p_\theta(z \given x) = \frac{p_\theta(x \given z)p_\theta(z)}{p_\theta(x)}.
\end{equation*}
Commonly, the statistician decides the shape of the likelihood $p_\theta(\bm x \given
z)$ and prior $p_\theta(\bm z)$ distributions. However, the marginal likelihood or
data distribution $p(\bm x)$ is often unknown, turning the inference of the
posterior $p(\bm z \given x)$ intractable. One way to circumvent this issue
\citep{jordan98} is to introduce an approximate or \emph{variational} posterior
distribution $q_\phi(\bm z \given x)$, and analyze its Kullback-Liebler divergence
(Equation~\ref{eq:kl}) to the true posterior $p(\bm z \given x)$ using samples:
\begin{align*}
  \textrm{KL}(q_\phi(z \given x) \| p_\theta(z\given x)) &= \E{q}{\log \frac{q_\phi(z\given x)}{p_\theta(z \given x)}}\\
                                    &= \E{q}{\log q_\phi(z\given x)} - \E{q}{\log p_\theta(z \given x)}\\
                                    &= \E{q}{\log q_\phi(z\given x)} - \E{q}{\log p_\theta(z, x)} + \log p_\theta(x).
\end{align*}
The previous manipulation implies that
\begin{equation*}
  \log p_\theta(x) = \textrm{KL}(q_\phi(z\given x) \| p_\theta(z\given x)) +
  \underbrace{\left(\E{q}{\log q_\phi(z\given x)} - \E{q}{\log p_\theta(z, x)}
  \right)}_{\mathcal{L}(\phi,\theta) \, := \, \textrm{ELBO}}.
\end{equation*}
Since $p$ does not depend on $q$, maximizing the ELBO (Evidence Lower BOund)
results in minimizing the Kullback-Liebler divergence between the variational
posterior $q_\phi(\bm z\given x)$ and the target posterior $p_\theta(\bm z\given x)$.
So, if our variational posterior is rich enough, we hope that maximizing the
ELBO will result in a good approximation to the true posterior. The ELBO can be
rewritten as
\begin{equation}\label{eq:elbo}
  \mathcal{L}(\phi,\theta) = -\textrm{KL}(q_\phi(z \given x) \| p_\theta(z)) +
  \E{q}{\log p_\theta(x \given z)},
\end{equation}
an expression in terms of known terms $p_\theta(z)$ and
$p_\theta(x \given z)$, and $q_\phi(z \given x)$.

At this point, we can use gradient descent optimization on \eqref{eq:elbo} to
learn both the variational parameters $\phi$ and the generative parameters
$\theta$. Unfortunately, the expectations in~\eqref{eq:elbo} are in general
intractable, so one approximates them by sampling. Since such sampling depends
on the variables that we are optimizing, the stochastic gradients $\nabla_\phi
\mathcal{L}$ have large variance. To alleviate this issue, we can use the
\emph{reparametrization trick}
\begin{equation}\label{eq:reptrick}
  \E{q_\phi(\bm z\given x)}{f(\bm z)} = \E{p(\bm \epsilon)}{f(g_\phi(\bm
  \epsilon,x))},
\end{equation}
where $g_\phi$ is a deterministic function depending on the variational
parameters $\phi$ \citep{kingma2013auto}. Observe that the right hand side of
\eqref{eq:reptrick} is an expectation with respect to a distribution $p(\bm
\epsilon)$ that no longer depends on the shape of the variational posterior.
For example, Gaussian variational posteriors $\N(\bm z \given \mu, \Sigma)$ can
be reparametrized into Normal posteriors $\N(\bm \epsilon \given 0, I_d)$ and
deterministic functions $g_\phi = \Sigma^\top \bm \epsilon + \mu$.  As a
result, the reparametrization trick reduces the variance of the stochastic
gradients $\nabla_\phi \mathcal{L}$.

In practice \citep{kingma2013auto, kingma2014semi, rezende2014stochastic}, it
is common to set $p(\bm z) = \N(0,I_d)$ and parametrize both $p(\bm x\given z)$
and $q(\bm z \given x)$ using deep neural networks
(Section~\ref{sec:neural-networks}).

\subsection{Adversarial networks}\label{sec:gans}
The Generative Adversarial Network (GAN) framework
\citep{goodfellow2014generative} is a game between two players: a
\emph{generator} $G_\theta$ and a \emph{discriminator} $D_\phi$. In this game,
the generator $G_\theta$ aims at generating samples $x = G_\theta(z)$ that look
as if they were drawn from the data generating distribution $p(\bm x)$, where
$z$ is drawn from some simple noise distribution $q(\bm z)$. On the other hand,
the responsibility of the discriminator $D_\phi$ is to tell if a sample was
drawn from the data generating distribution $p(\bm x)$ or if it was
synthesized by the generator $G_\theta$.  Mathematically, playing this game
is solving the optimization problem 
\begin{equation*}
  \min_\theta \max_\phi \E{}{\log D(\bm x)} + \E{}{\log (1-D(G(\bm z)))},
\end{equation*}
with respect to the generator parameters $\theta$ and the discriminator
parameters $\phi$. In \citep{goodfellow2014generative}, both the generator and
the discriminator are deep neural networks (Section~\ref{sec:neural-networks}).
On the negative side, the GAN framework does not provide an explicit mechanism
to evaluate the probability density function of the obtained generator.

\section{Mixture models}
Let
$p_{\theta_1}, \ldots, p_{\theta_k}$ be a collection of density functions, and
let $\pi_1, \ldots, \pi_k \in [0,1]$ be a collection of numbers summing to one.
Then, the function 
\begin{equation*}
  p(\bm x) = \sum_{i=1}^k \pi_i p_{\theta_i}(\bm x)
\end{equation*}
is also a density function, called a \emph{mixture model}. At a high level,
mixtures implement the ``\texttt{OR}'' operation of the mixed densities, also called \emph{mixture components}.

Sampling from mixture models is as easy as sampling from each the mixture
components: just sample from the $i$-th mixture component, where the index $i
\sim \mathrm{Multinomial}(\pi_1, \ldots, \pi_k)$. Similarly, because of the linearity of
integration, computing the marginal distributions of a mixture model is as easy
as computing the marginal distributions of each the mixture components. At the same
time, since marginalization is feasible, conditional mixture distributions are
easy to compute.

Given data $\{x_1, \ldots, x_n\}$ and number of mixture components $k$, mixture
models are parametric if $k < n$, and are nonparametric if $k = n$. Next, we
briefly review how to estimate both parametric and nonparametric mixture
models.  

\subsection{Parametric mixture models}
Expectation Maximization or EM \citep{dempster1977maximum} is often the tool of
choice to train parametric mixture models. 

Consider the task of modeling the data $\{x_1, \ldots, x_n\}$ using a mixture
of $k$ components, parametrized by the parameter vector $\theta$. To this end,
introduce a set of $n$ latent variables $\{\bm z_1,\ldots, \bm z_n\}$; for all $1 \leq
i \leq n$, the latent variable $\bm z_i \in \{1, \ldots, k\}$, and indicates from
which of the $k$ mixture components the example $x_i$ was drawn. 

Expectation maximization runs for a number of iterations $1 \leq t \leq T$,
and executes two steps at each iteration. First, the \emph{expectation step} computes
the function 
\begin{equation*}
  Q(\theta \given \theta^{(t)}) = \E{\bm z \given x, \theta^{(t)}}{\sum_{i=1}^n \log p(x_i, \bm z_i \given \theta^{(t)})}.
\end{equation*}
Second, the maximization step updates the parameter vector as
\begin{equation*}
  \theta^{(t+1)} = \arg\max_\theta Q(\theta \given \theta^{(t)}).
\end{equation*}
If we treat the parameter $\theta$ as yet another latent variable, the EM
algorithm relates to variational inference
(Section~\ref{sec:variational-inference}). For general mixtures, the EM
algorithm approximates the solution to a nonconvex optimization problem, and
guarantees that the likelihood of the mixture increases per iteration. 

\subsection{Nonparametric mixture models}
Nonparametric mixture models, also known as \emph{Parzen-window estimators} or
\emph{kernel density estimators} \citep{parzen1962estimation}, dedicate one
mixture component per example comprising our $d$-dimensional data, and have
form
\begin{equation*}
  p(\bm x) = \frac{1}{nh^d} \sum_{i=1}^k K\left(\frac{\bm x-x_i}{h}\right),
\end{equation*}
where $K$ is a nonnegative function with mean zero that integrates to one, and
$h > 0$ is a bandwidth parameter proportional to the smoothness of the mixture.
While nonparametric mixture models avoid the need of EM, they do require $n$
mixture components, $O(nd)$ memory requirements. Due to the curse of
dimensionality, modeling a high-dimensional space in a nonparametric manner
requires an exponential amount of data; thus, nonparametric mixture models tend
to overfit in moderate to high dimensions \citep{Wasserman}.

\subsection{Gaussian mixture models}

Gaussian Mixture Models (GMMs) are mixture models where each of the mixture
components is Gaussian with known mean and covariance:
\begin{equation*}
  p(\bm x) = \sum_{i=1}^k \pi_i \N(\bm x; \mu_i, \Sigma_i).
\end{equation*}
One example of a Gaussian mixture model with $10$ components is the one plotted
with contours Figure~\ref{fig:density-discriminative} (a) and (b).  GMMs are
popular generative models because, when given enough components, they are
universal probability density function estimates 
\citep{kostantinos2000gaussian}. Gaussian Parzen-window estimators have form
\begin{equation*}
  p(\bm x) = \frac{1}{n} \sum_{i=1}^n \frac{1}{(h\sqrt{2\pi})^d}
  \exp\left(-\frac{1}{2} \frac{\|\bm x-x_i\|^2}{h^2}\right).
\end{equation*}

\begin{remark}[Transformations in mixture models]
Mixture models can exploit the benefit of transformation models
(Section~\ref{sec:transformation-models}). Simply treat the parameters of the
mixture as a function of the input $x$, instead of fixed quantities, like in
\begin{equation*}
  p(x) = \sum_{i=1}^k \pi_i(x) p(x| \theta_i(x)),
\end{equation*}
and make use of the reparametrization trick (Section~\ref{sec:variational-inference}).
\end{remark}
  
\section{Copula models}\label{sec:copulas}
If you were to measure the speed of a car, would you measure it in kilometers
per hour? Or in miles per hour? Or in meters per second? Or in the logarithm of
yards per minute? Each alternative will shape the distribution of the
measurements differently, and if these measurements are recorded together with
some other variables, also the shape of their joint distribution,
dependence structures, and the results of subsequent learning algorithms.

The previous illustrates that real-world data is composed by variables greatly
different in nature and form.  Even more daunting, all of these variables
interact with each other in heterogeneous and complex patterns. In the language
of statistics,
such depiction of the world calls for the development of flexible multivariate
models, able to separately characterize the marginal distributions of each of
the participating random variables from the way on which they interact with
each other. Copulas offer a flexible framework to model joint distributions by
separately characterizing the marginal distributions of the involved variables,
and the dependence structures joining these random variables together. The richness
of copulas allows to model subtle structures, such as heavy tailed and skewed dependencies,
difficult to capture with transformation and mixture models.

Let $\bm x_1$ and $\bm x_2$ be two continuous random variables with positive density
almost everywhere. Then, if $\bm x_1$ and $\bm x_2$ are independent, their joint cdf is
the product of the two marginal cdfs:
\begin{equation}
  P(\bm x_1,\bm x_2) = P(\bm x_1)P(\bm x_2).
  \label{eq:indepdistr}
\end{equation}
However, when $\bm x_1$ and $\bm x_2$ are not independent this is no longer the case.
Nevertheless, we can correct these differences by using a specific function $C$
to couple the two marginals together into the bivariate model of interest:
\begin{equation}
  P(\bm x_1,\bm x_2) = C(P(\bm x_1),P(\bm x_2)).\label{eq:jointInTermsOfCopula}
\end{equation}
This cdf $C$ is the \emph{copula} of the distribution $P(\bm x_1, \bm x_2)$. Informally
speaking, $C$ links the univariate random variables $\bm x_1$ and $\bm x_2$ into a
bivariate distribution $P(\bm x_1, \bm x_2)$ which exhibits a dependence structure fully
described by $C$. Thus, any continuous bivariate distribution is the product of 
three independent building blocks: the marginal distribution of the first
random variable $\bm x_1$, the marginal distribution of the second random variable
$\bm x_2$, and the copula function $C$ describing how the two variables interact
with each other. 

\begin{remark}[History of copulas]
  The birth of copulas dates back to the pioneering work of \cite{Hoeffding40},
  who unwittingly invented the concept as a byproduct of scale-invariant
  correlation theory. Their explicit discovery is due to \citet{Sklar59}, who
  established the fundamental result that now carries his name. Although
  copulas played an important role in the early development of dependence
  measures \citep{Schweizer81} and probabilistic metric spaces
  \citep{Schweizer83}, their mainstream presence in the statistics literature
  had to wait for four decades, with the appearance of the monographs of
  \cite{Joe97} and \cite{Nelsen06}.  Since then, copulas have enjoyed great
  success in a wide variety of applications such as finance
  \citep{Cherubini04,Trivedi07}, extreme events in natural phenomena
  \citep{Salvadori07}, multivariate survival modeling \citep{Georges01},
  spatial statistics, civil engineering, and random vector generation
  \citep{Jaworski10}.  Copula theory has likewise greatly expanded its
  boundaries, including the development of conditional models \citep{Patton06}
  and nonparametric estimators \citep{Fermanian07}.
  
  Perhaps surprisingly, the machine learning community has until recently
  been ignorant to the potential of copulas as tools to model multivariate
  dependence. To the best of our knowledge, the work of
  \cite{chen2001gaussianization} in Gaussianization and the one of
  \cite{Kirshner07} on averaged copula tree models were the first to appear in
  a major machine learning venue. Since then, the applications of copulas in
  machine learning have extended to 
  scale-invariant component analysis \citep{Ma07,Kirshner08}, measures of
  dependence \citep{Poczos12,dlp-rdc}, semiparametric estimation of
  high-dimensional graph models \citep{Liu09}, nonparametric Bayesian networks
  \citep{Elidan10}, mixture models \citep{Fujimaki11,Tewari11}, clustering
  \citep{Rey12}, Gaussian processes \citep{Wilson10} and financial time series
  modeling \citep{Hernandez-Lobato13}.  \cite{Elidan12a} offers a monograph on
  the ongoing synergy between machine learning and copulas.
\end{remark}

Much of the study of joint distributions is the study of copulas
\citep{Trivedi07}. Copulas offer a clearer view of the
underlying dependence structure between random variables, since they clean 
any spurious patterns generated by the marginal distributions.
Let us start with the formal
definition of a copula:

\begin{definition}[Copula]
  A copula is a function $C : [0,1]^2 \mapsto [0,1]$ s.t.: 
  \begin{itemize}
    \item for all $u,v \in [0,1]$, $C(u,0)=C(0,v)=0$, $C(u,1)=u$, $C(1,v)=v$,
    \item for all $u_1,u_2,v_1,v_2\in [0,1]$ s.t. $u_1 \leq u_2$ and $v_1 \leq
    v_2$, $$C(u_2,v_2)-C(u_2,v_1)-C(u_1,v_2)+C(u_1,v_1) \geq 0.$$
  \end{itemize}
  Alternatively, $C : [0,1]^2 \mapsto [0,1]$ is a copula if $C(u,v)$ is the
  joint cdf of a random vector $(U,V)$ defined on the unit square $[0,1]^2$
  with uniform marginals \cite[Def. 2.2.2.]{Nelsen06}.
\end{definition}

As illustrated by Equation~\ref{eq:jointInTermsOfCopula}, the main practical
advantage of copulas is that they decompose the joint distribution $P$ into its
marginal distributions $P(\bm x_1)$, $P(\bm x_2)$ and its dependence structure
$C$. This means that one can estimate $P(\bm x_1, \bm x_2)$ by separately
estimating $P(\bm x_1)$, $P(\bm x_2)$ and $C$.  Such modus operandi is
supported by a classical result due to Abe Sklar, which establishes the unique
relationship between probability distributions and copulas.

\begin{theorem}[Sklar]
  Let $P(\bm x_1, \bm x_2)$ have continuous marginal cdfs $P(\bm x_1)$ and
  $P(\bm x_2)$. Then, there exists a unique copula $C$ such that for all $x_1,x_2\in
  \R$,
  \begin{equation}\label{eq:copcorr}
    P(\bm x_1, \bm x_2) = C(P(\bm x_1), P(\bm x_2)).
  \end{equation}
  If $P(\bm x_1)$, $P(\bm x_2)$ are not continuous, $C$ is uniquely identified on
  the support of $P(\bm x_1)\times P(\bm x_2)$. Conversely, if $C$ is a copula
  and $P(\bm x_1), P(\bm x_2)$ are some continuous marginals, the function $P(\bm x_1, \bm x_2)$ in
  (\ref{eq:copcorr}) is a valid $2-$dimensional distribution with marginals
  $P(\bm x_1), P(\bm x_2)$ and dependence structure $C$.
  
  In terms of density functions, the relationship (\ref{eq:copcorr}) is
  \begin{equation}
    p(\bm x_1,\bm x_2) = p(\bm x_1)p(\bm x_2)c(P(\bm x_1),P(\bm x_2)),
    \label{eq:copdens}
  \end{equation}
  where $p(\bm x_1, \bm x_2)= \frac{\partial^2 P(\bm x_1, \bm x_2)}{\partial x_1 \partial x_2}$, $p(\bm x_i) =
  \frac{\partial P(\bm x_i)}{\partial x_i}$ and $c = \frac{\partial^2 C}{\partial x_1
  \partial x_2}$.
\end{theorem}
\begin{proof}
  See \cite[Thm. 2.3.3.]{Nelsen06}.
\end{proof}

There is an useful asymmetry in the previous claim. Given a distribution $P(\bm
x_1, \bm x_2)$, we can uniquely identify its underlying dependence structure or
copula $C$.  On the other hand, given a copula $C$, there are infinitely
multiple different bivariate models, each obtained by selecting a different
pair of marginal distributions $P(\bm x_1)$ and $P(\bm x_2)$.  This one-to-many
relationship between copulas and probability distributions is the second most
attractive property of the former: copulas are invariant with respect to
strictly monotone increasing transformations of random variables. 
\begin{lemma}[Scale-invariance of copulas]\label{lemma:scale}
  Let $f_1,f_2 : \R \mapsto \R$ be two strictly monotone
  increasing functions. Then, for any pair of continuous random variables $\bm x_1$ and
  $\bm x_2$, the distributions:
  \begin{align*}
    P(\bm x_1,\bm x_2) & = C(P(\bm x_1),P(\bm x_2)) \text{ and }\\
    P(f_1(\bm x_1),f_2(\bm x_2)) &= C(P(f_1(\bm x_1)), P(f_2(\bm x_2)))
  \end{align*}
  share the same copula function C \cite[Thm. 2.4.3.]{Nelsen06}.
\end{lemma}

Another way to understand the scale invariance of copulas is that they always exhibit
uniformly distributed marginals.  This is due to a classical result of \citet{Rosenblatt52}:
\begin{theorem}[Probability integral transform]
 Let the random variable $\bm x$ have a continuous distribution with cumulative
 distribution function $P$. Then, the random variable $\bm y \equiv P(\bm x)$ is uniformly
 distributed.
\end{theorem}

Scale invariance makes copulas an attractive tool to construct scale-invariant (also known as
weakly equitable) statistics, such as measures of dependence
\citep{Poczos12,dlp-rdc}.  We now turn to this issue, the one of
measuring statistical dependence using copulas.

\subsection{Describing dependence with copulas}
The first use of copulas as explicit models of dependence is due to
\cite{Schweizer81}, as a mean to guarantee the scale invariance 
(Lemma \ref{lemma:scale}) imposed by R\'enyi's axiomatic framework for measures
of dependence \citep{Renyi59}.
Given their interpretation as dependence structures, it is no surprise that
copulas share an intimate relationship with well known dependence statistics, like
Spearman's $\rho$, Kendall's $\tau$, and mutual
information:
\begin{align}
  \rho(\bm x_1,\bm x_2) &= 12 \iint_0^1
  (C(u_1,u_2)-u_1u_2)\mathrm{d}u_1\mathrm{d}u_2,\nonumber\\
  \tau(\bm x_1,\bm x_2) &= 4 \iint_0^1
  C(u_1,u_2)\mathrm{d}C(u_1,u_2)-1,\label{eq:copulakendall}\\
  I(\bm x_1,\bm x_2)    &= \iint_0^1 c(u_1,u_2) \log c(u_1,u_2)
  \mathrm{d}u_1\mathrm{d}u_2\,,\nonumber
\end{align}
where $u_1 = P(\bm x_1 = x_1)$ and similarly for $u_2$.  Kendall's $\tau$
measures correlation between rank statistics; as such, it is invariant under
monotone increasing transformations of random variables.

Copulas are also appropriate to measure dependence between extreme events as
\emph{tail dependencies}.  Formally, we define the \emph{lower and upper
tail dependence coefficients} as the quantities
\begin{align*}
  \lambda_l &= \lim_{u\rightarrow 0} P(\bm x_2 \leq P_2^{-1}(u)|\bm x_1 \leq
  P_1^{-1}(u))= \lim_{u \rightarrow 0} \frac{C(u,u)}{u}, \\
  \lambda_u &= \lim_{u\rightarrow 1} P(\bm x_2 > P_2^{-1}(u)|\bm x_1 > P_1^{-1}(u))=
  \lim_{u \rightarrow 1} \frac{1-2u+C(u,u)}{1-u},
\end{align*}
where $P_1^{-1}$ and $P_2^{-1}$ are the quantile distribution functions of the
random variables $\bm x_1$ and $\bm x_2$, respectively. For instance, the upper
tail dependence coefficient $\lambda_u$ measures the probability of $\bm x_1$
exceeding a very large quantile, conditioned on $\bm x_2$ exceeding that same
very large quantile. Tail dependency differs from the usual notion of
statistical dependence: even a strongly correlated Gaussian distribution
exhibits no tail dependency. 

\subsection{Estimation of copulas from data}\label{sec:copdata}
Having access to $n$ samples $X = \{(x_{1,i}, x_{2,i})\}_{i=1}^n$ drawn
iid from the probability distribution $P(\bm x_1, \bm x_2)$, the question of how to
estimate a model for the copula of $P(\bm x_1, \bm x_2)$ is of immediate practical interest. The
standard way to proceed is
\begin{enumerate}
  \item Estimate the marginal cdfs $P(\bm x_1)$ and $P(\bm x_2)$ as the marginal ecdfs $P_n(\bm x_1)$ and $P_n(\bm x_2)$.
  \item Obtain the copula
  pseudo-sample
    $$U = \{(u_{1,i},u_{2,i})\}_{i=1}^n :=
  \{({P}_n(\bm x_1 = x_{1,i}),{P}_n(\bm x_2 = x_{2,i})\}_{i=1}^n.$$
  \item Choose a parametric copula function $C_{\theta}$ and estimate its
  parameters $\theta$.
\end{enumerate}

First, the transformation from the distribution sample $X$ to the copula sample $U$
involves learning the marginal cdfs $P(\bm x_1)$ and $P(\bm x_2)$ from data.  In practice,
the empirical cdf (Definition~\ref{def:ecdf}) is the tool of choice to obtain
nonparametric estimates of univariate cdfs.

Second, when working with $d-$dimensional samples $\{(x_{1,i},\ldots,
x_{d,i})\}_{i=1}^n$, we need to compute $d$ independent ecdfs to unfold
the underlying copula sample $U$.  The transformation of each of the components
of a random vector to follow an uniform distribution by means of their ecdfs is
the \emph{empirical copula transformation}.

\begin{definition}[Empirical copula transformation]\label{def:coptr}
  Let $\{(x_{1,i}, \ldots, x_{d,i})\}_{i=1}^n$, $x_i\in
  \R^d$, be an iid sample from a probability density over
  $\R^d$ with continuous marginal cdfs $P(\bm x_1), \ldots, P(\bm x_d)$; $P(\bm x_i) :
  \R \mapsto [0,1]$. Let ${P}_{n,1}, \ldots, {P}_{n,d}$ be the
  corresponding ecdfs as in Definition \ref{def:ecdf}. The
  empirical copula transformation of $x$ is 
  \begin{equation*}
    u = T_n(x) = \left[{P}_{n,1}(x_i), \ldots, {P}_{n,d}(
    x_d)\right] \in \R^d.
  \end{equation*}
\end{definition}

Given that the $d$ marginal transformations are independent from each other, we
can straightforwardly use the result from Theorem \ref{thm:massart} to obtain a
guarantee for the fast convergence rate of the empirical copula transformation
to its asymptotic limit as $n \to \infty$.

\begin{corollary}[Convergence of the empirical copula]\label{cor:copconv}
  Let $$\{(x_{i,1}, \ldots, x_{i,d})\}_{i=1}^n,$$ $x_i \in
  \R^d$, be an iid sample from a probability density over
  $\R^d$ with continuous
  marginal cdfs $P(\bm x_1), \ldots, P(\bm x_d)$. Let $T(x)$ be the copula
  transformation obtained using the true marginals cdfs $P(\bm x_1), \ldots,
  P(\bm x_d)$ and let $T_n(x)$ be the empirical copula transformation from
  Definition \ref{def:coptr}. Then, for any $\epsilon > 0$ 
  \begin{equation*}
    \Pr\left[\sup_{x \in \R^d} \|T(x) -T_n(x)\|_2
    > \epsilon \right] \leq 2d\exp\left(-\frac{2n\epsilon^2}{d}\right).
  \end{equation*}
  \begin{proof}
  Use Theorem \ref{thm:massart} taking into account that $\|\cdot\|_2 \leq
  \sqrt{d}\|\cdot\|_\infty$ in $\R^d$. Then apply the union-bound over
  the $d$ dimensions \citep{Poczos12}.
  \end{proof}
\end{corollary}

Third, once we have obtained the copula sample $U$, we may want to fit a
parametric copula model to it.  For instance, we could use maximum likelihood
estimation to tune the parameters of a parametric copula density.  But when
considering multiple candidate parametric copula families, this procedure
becomes computationally prohibitive. Instead, one exploit the fact that most
bivariate copulas with a single scalar parameter share a one-to-one
relationship between Kendall's $\tau$ and their parameter.  This means that
given an estimate of Kendall's $\tau$ built from the copula sample, one can
obtain an estimate of the parameter of a parametric copula by inverting the
relationship (\ref{eq:copulakendall}). This is an efficient procedure, since
the estimation of Kendall's $\tau$ from $n$ data takes $O(n \log n)$ time.
This is \emph{the inversion method} \citep{Dissmann13}.

\begin{example}[Construction of a parametric bivariate copula]
  \begin{figure}
    \includegraphics[width=\textwidth]{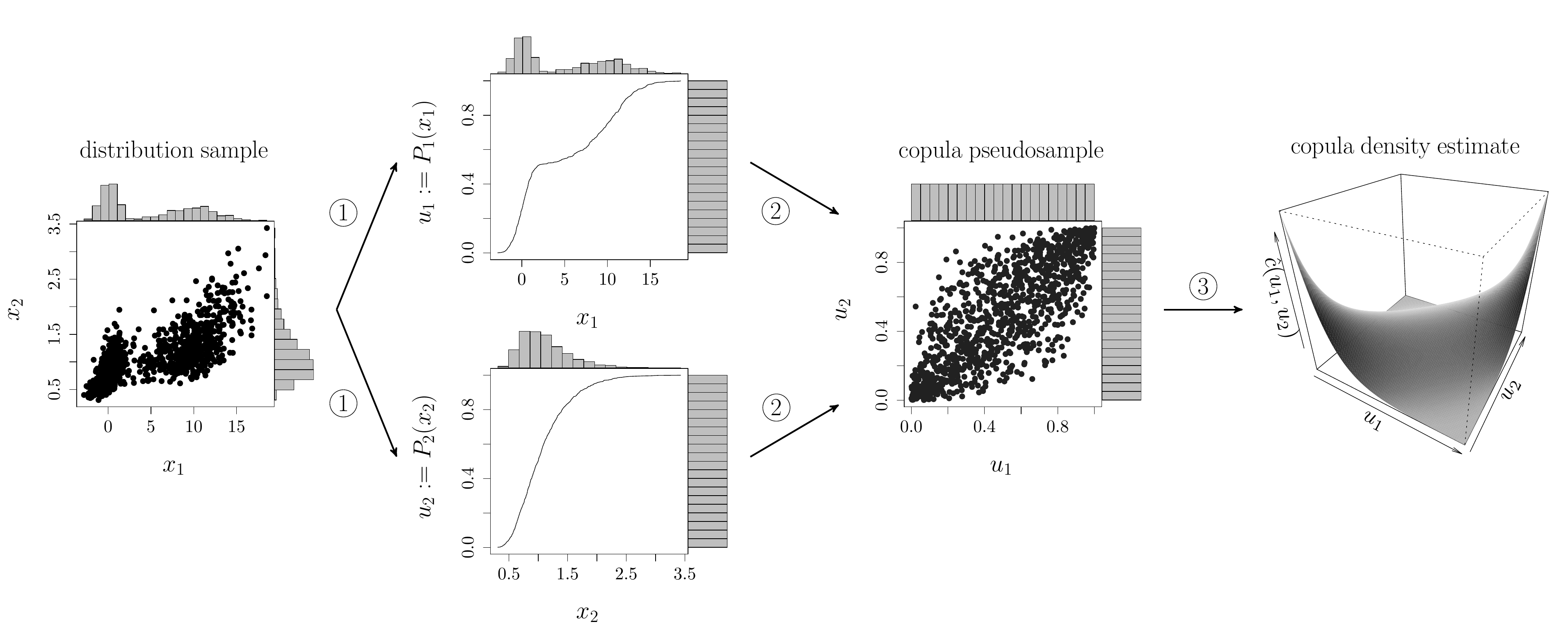}
    \caption{Estimation of a parametric bivariate copula}
    \label{fig:copula-pipeline}
  \end{figure}
  Figure~\ref{fig:copula-pipeline} illustrates the estimation of a parametric
  bivariate copula.  This process involves 1) computing the two marginal ecdfs,
  2) obtaining the  copula sample, and 3) fitting a parametric model using
  Kendall's $\tau$ inversion.  A density estimate of the probability
  distribution of the original data is obtained by multiplying the density
  estimates of the marginals and the density estimate of the copula, as
  illustrated in Figure~\ref{fig:copula-pipeline}.
\end{example}

\begin{remark}[Wonders and worries of copulas]\label{remark:copula-worries}
  Copulas transform each of the variables from our data to follow an uniform
  distribution. Thus, copula data reveals the essence of the dependencies in
  data by disentangling the complex shapes of marginal distributions.  Copula data
  also brings a benefit when dealing with outliers: copulas upper
  bound the influence of outliers by squeezing all data to fit in
  the unit interval.

  Nevertheless, the use of copulas also calls for caution. For instance, copula
  transformations destroy cluster structures in data. This is because the
  uniform margins in copulas flatten regions of high density (indications of
  cluster centers) and fill regions of low density (indications of boundaries
  between clusters).  See for example, in Figure~\ref{fig:copula-pipeline}, how
  the cluster structure of the data sample is no longer present in the
  associated copula sample. Moreover, $d$-dimensional copulas live on the
  $d$-dimensional unit hypercube; this may be a bad parametrization for
  statistical models expecting data with full support on $\Rd$. Lastly, copulas
  may destroy smoothness: for instance, the copula transformation
  of a sinusoidal pattern is a saw-like pattern.
\end{remark}

\subsection{Conditional distributions from copulas}\label{sec:copcond}

As we will see in Section~\ref{sec:vine-copulas}, conditional distributions
play a central role in the construction of multivariate copulas. Using copulas,
formulas for conditional distributions can be obtained by partial differentiation.

\begin{definition}[Copula conditional distributions]
Define the quantity
\begin{align*}
  c_v(u) = P(u | v) = \frac{\partial}{\partial v} C(u,v).
\end{align*}
  For any $u \in [0,1]$, $c_v(u)$ exists almost surely for all $v\in [0,1]$,
  is bounded between $0$ and $1$, well defined, and nondecreasing almost
  everywhere on $[0,1]$. Similar claims follow when conditioning on $u$
  \citep{Salvadori07}.  \citet{Schepsmeier13} provide a collection of
  closed-form expressions for the partial derivatives of common bivariate
  parametric copulas.
\end{definition}
When conditioning to more than one variable, the previous definition extends
recursively:
\begin{equation}\label{eq:recursive_conditional}
  P(u|v) = \frac{\partial\, C(P(u|v_{-j}), P(v_j|v_{-j})|
  v_{-j})}{\partial P(v_j|v_{-j})},
\end{equation}
where the copula $C$ is conditioned to $v_{-j}$, the vector of variable values
$v$ with its $j$-th component removed. 
Conditional distributions are central to copula sampling algorithms. Sampling from
$C(u,v)$ reduces to i) generate $u \sim \mathcal{U}[0,1]$ and ii) set
$v=c^{-1}_v(u)$ \cite[Thm.  2.2.7.]{Nelsen06}. The univariate function $c_v(t)$
is inverted numerically.

\subsection{Parametric copulas}\label{sec:parametric}
There exists a wide catalog of parametric bivariate copulas. For completeness,
we review here the most common families and their properties. 

\paragraph{Elliptical} These are copulas implicitly derived from elliptically
contoured (radially symmetric) probability distributions, and represent linear
dependence structures (correlations). When coupled with arbitrary
marginals (multimodal, heavy-tailed...), they construct a wide-range of
distributions.

\begin{itemize}
  \item The \emph{Gaussian copula} (Table \ref{table:parametriccopulas}, \#1)
  with correlation parameter $\theta \in [-1,1]$ represents the dependence
  structure underlying a bivariate Gaussian distribution of two random
  variables with correlation $\theta$. Gaussian copulas exhibit no tail
  dependence, which makes them a poor choice to model extreme events. In fact,
  this is one of the reasons why the Gaussian copula has been demonized as one
  of the causes of the 2007 financial crisis. The family of distributions with
  Gaussian copula are the \emph{nonparanormals} \citep{Liu09}. 
  The parameter of a multivariate Gaussian copula is a 
  correlation matrix.

  \item The \emph{t-Copula} (Table \ref{table:parametriccopulas}, \#2)
  represents the dependence structure implicit in a bivariate Student-t
  distribution. Thus, t-Copulas are parametrized by their correlation $\theta
  \in [-1,1]$ and the degrees of freedom $\nu \in (0,\infty)$. t-Copulas exhibit
  symmetric lower and upper tail dependencies, of strengths
  \begin{equation}\label{eq:ttails}
    \lambda_l(\theta, \nu) = \lambda_u(\theta, \nu) =
    2\,t_{\nu+1}\left(-(\sqrt{\nu+1}\sqrt{1-\theta})/\sqrt{1+\theta}\right),
  \end{equation}
  where $t_{\nu+1}$ denotes the density of a Student-t distribution with
  $\nu+1$ degrees of freedom. This makes t-Copulas suitable models of
  symmetric extreme events (happening both in the lower and upper quantiles).
  To capture asymmetries in tail dependencies, \citet{Demarta05} proposes a
  variety of skewed t-copulas. The parameters of a multivariate t-Copula are
  one correlation matrix, and the number of degrees of freedom.
\end{itemize}

\paragraph{Archimedean} These copulas are popular due to their ease of
construction. They are not necessarily elliptical, admit explicit constructions
and accommodate asymmetric tails. They originated as an extension of the
triangle inequality for probabilistic metric spaces \citep{Schweizer83}.
Archimedean copulas admit the representation:
\begin{equation*}
  C(u,v) = \psi^{[-1]}(\psi(u|\theta)+\psi(v|\theta)|\theta),
\end{equation*}
where the continuous, strictly decreasing and convex function $\psi : [0,1]
\times \Theta \to [0,\infty)$ is a \emph{generator function}. The
generalized inverse $\psi^{[-1]}$ is 
\begin{equation*}
  \psi^{[-1]}(t|\theta) = \left\{\begin{array}{ll} \psi^{-1}(t|\theta) &
  \mbox{if }0 \leq t \leq \psi(0|\theta) \\ 0 & \mbox{if }\psi(0|\theta) \leq t
  \leq\infty. \end{array}\right.
\end{equation*}
Archimedean copulas are commutative ($C(u,v)=C(v,u)$), associative
($C(C(u,v),w)=C(u,C(v,w))$), partially ordered (for $u_1 \leq u_2$ and $v_1
\leq v_2$, $C(u_1,v_1) \leq C(u_2,v_2)$) and have convex level curves
\citep{Nelsen06}.

Choosing different generator functions $\psi$ yields different
Archimedean copulas: for common examples, refer to Table
\ref{table:parametriccopulas}, \#6-12. 

\paragraph{Extreme-value} These copulas 
are commonly used in risk management and are appropriate to model dependence
between rare events, such as natural disasters or
large drops in stock markets.  They satisfy:
\begin{equation*}
  C(u^t,v^t)=C^t(u,v), \quad C(u,v) = e^{\ln(u,v) A\left(\frac{\ln v}{\ln
  uv}\right)}, \quad t \geq 0,
\end{equation*}
where $A$ is a convex \emph{Pickands dependence function} with $\max(t,1-t)
\leq A(t) \leq 1$ \citep{Nelsen06}. Examples are Gumbel (the only Archimedean
extreme-value copula), Husler-Reiss, Galambos or Tawn (Table
\ref{table:parametriccopulas} \#13-16). 

\paragraph{Perfect (in)dependence} 
 From 
(\ref{eq:indepdistr}) and (\ref{eq:copcorr}), we see that the only copula
describing independence has cdf 
\begin{equation}
  C_\bot(u,v) = uv.\label{eq:indepCop}
\end{equation}
On the other hand, the copulas describing perfect positive
(comonotonicity) or negative (countermonotonicity) dependence are respectively
called the lower and upper Fr\'echet-Hoeffding bounds, and follow the
distributions:
\begin{equation}\label{eq:copbounds}
  C_l(u,v) = \max(u+v-1,0) \quad \text{ and } \quad C_u(u,v) = \min(u,v).
\end{equation}
Any copula $C$ lives inside
this pyramid, i.e., $C_l(u,v) \leq C(u,v) \leq C_u(u,v)$.

The Clayton, Gumbel, Gaussian and t-Copula are some examples of
\emph{comprehensive copulas}: they interpolate the Fr\'echet-Hoeffding bounds
(\ref{eq:copbounds}) as their parameters vary between extremes.

\paragraph{Combinations of copulas} The convex combination or product of
copulas densities is a valid copula density \cite[\S 3.2.4]{Nelsen06}. Furthermore,
if $C(u,v)$ is a copula and $\gamma : [0,1] \mapsto [0,1]$ is a concave,
  continuous and strictly increasing function with $\gamma(0) =0$ and
  $\gamma(1)=1$, then $\gamma^{-1}(C(u,v))$ is also a valid copula \cite[Thm.
  3.3.3]{Nelsen06}.

Table~\ref{table:parametriccopulas} summarizes a variety of bivariate
parametric copulas.  In this table, $\Phi_2$ is the bivariate Gaussian CDF with
correlation $\theta$. $\Phi^{-1}$ is the univariate Normal quantile
distribution function. $t_{2; \nu, \theta}$ is the bivariate Student's t CDF
with $\nu$ degrees of freedom and correlation $\theta$.  $t_\nu^{-1}$ is the
univariate Student's t quantile distribution function with $\nu$ degrees of
freedom. $D$ is a Debye function of the first kind.  $A(w)$ is a Pickands
dependence function.  For contour plots of different parametric copulas, refer
to \citep{Salvadori07}.

\subsection{Gaussian process conditional copulas}\label{sec:conditional}
The extension of the theory of copulas to the case of conditional distributions
is due to \citet{Patton06}.  Let us assume, in addition to $\bm x_1$ and
$\bm x_2$, the existence of a third random variable $\bm x_3$.  This third variable 
influences $\bm x_1$ and $\bm x_2$, in the sense that the joint
distribution for $\bm x_1$ and $\bm x_2$ changes as we condition to different values of
$\bm x_3$.  The conditional cdf for $\bm x_1$ and $\bm x_2$ given $\bm x_3 = x_3$ is
$P(\bm x_1,\bm x_2|x_3)$.

We can apply the copula framework to decompose $P(\bm x_1, \bm x_2|x_3)$ into
its bivariate copula and one-dimensional marginals. The resulting decomposition
is similar to the one shown in (\ref{eq:jointInTermsOfCopula}) for the
unconditional distribution of $\bm x_1$ and $\bm x_2$.  But, since we are now
conditioning to $\bm x_3 = x_3$, both the copula and the marginals of $P(\bm
x_1, \bm x_2|x_3)$ depend on the value $x_3$ taken by the random variable $\bm
x_3$. The copula of $P(\bm x_1, \bm x_2|x_3)$ is the \emph{conditional copula}
of $\bm x_1$ and $\bm x_2$, given $\bm x_3 = x_3$ \citep{Patton06}.

\begin{landscape}
\begin{table}
  \begin{center}
  \resizebox{\linewidth}{!}{
  \begin{tabular}{|r|l|l|l|l|c|c|}\hline
  \bf{\#} & \bf{Name} & \bf{Cumulative Distribution F. $C(u,v) =$} & \bf{Par. Domain} & \bf{Kendall's $\tau = $} & $\lambda_l$ & $\lambda_v$ \\\hline\hline

  1 & Gaussian & $\Phi_2(\Phi^{-1}(u),\Phi^{-1}(v); \theta)$ & $\theta \in [-1,1]$ & \multirow{2}{*}{$\frac{2}{\pi}\arcsin(\theta)$} & 0 & 0 \\ \cline{1-4}\cline{6-7}

  2 & t-Student & $t_{2;\nu,\theta}(t_\nu^{-1}(u), t_\nu^{-1}(v); \theta)$ & $\theta \in [-1,1]$, $\nu > 0$ & & \multicolumn{2}{c|}{Equation (\ref{eq:ttails})} \\ \hline\hline

  3 & Independent & $uv$ & --- & $0$ & $0$ & $0$ \\ \hline

  4 & Upper FH bound & $\min(u,v)$ & --- & $1$ & $1$ & $1$ \\ \hline

  5 & Lower FH bound & $\max(u+v-1,0)$ & --- & $-1$ & $0$ & $0$ \\ \hline\hline

  6 & \bf{Archimedean} & $\psi^{[-1]}\left(\psi(u;\theta)+\psi(v;\theta);\theta\right)$ & $\psi-$dependent & $1+4\int_0^1\frac{\psi(t)}{\psi'(t)}\mathrm{d}t$ & \multicolumn{2}{c|}{$\psi$-dependent} \\ \hline

  7 & Ali-Mikhail-Haq & $\frac{uv}{1-\theta (1-u)(1-v)}$ & $\theta \in [-1,1]$ & $1-\frac{2(\theta+(1-\theta)^2\log(1-\theta))}{3\theta^2}$ & $0$ & $0$ \\ \hline

  8 & Clayton & $\max\left((u^{-\theta}+v^{-\theta}-1),0\right)^{-1/\theta}$ & $\theta \geq -1$ & $\theta/(\theta+2)$ & $2^{-1/\theta}$ & $0$ \\ \hline

  9 & Frank & $\frac{1}{\ln \theta} \ln \left(1+\frac{(\theta^u-1)(\theta^v-1)}{\theta-1}\right)$ & $\theta \geq 0$ & $1+4(D(\theta)-1)/\theta$ & $0$ & $0$ \\ \hline

  10 & Gumbel & $\exp\left({-\left((-\ln u)^\theta+(-\ln v)^\theta\right)^{1/\theta}}\right)$ & $\theta \geq 1$ & $(\theta-1)(\theta)$ & $0$ & $2-2^{1/\theta}$ \\ \hline

  11 & Joe & $1-\left( (1-u)^\theta + (1-v)^\theta - (1-u)^\theta(1-v)^\theta \right)^{1/\theta}$ & $\theta \geq 1$ & $1-\sum_{k=1}^\infty \frac{4}{k(\theta k +2)(\theta(k-1)+2)}$ & $0$ & $2-2^{1/\theta}$ \\ \hline

  12 & Kimeldorf-Sampson & $(u^{-\theta}+v^{-\theta}-1)^{-1/\theta}$ &$\theta \geq 0$ & $\theta/(\theta+2)$ & $2^{-1/\theta}$ & $0$ \\ \hline\hline

  13 & \bf{Extreme-Value} & $\exp\left({\ln u +\ln v }\right)A\left(\frac{\ln v}{\ln u +\ln v }\right)$ & $A-$dependent & \multirow{4}{*}{$\int_0^1 \frac{w(1-w)}{A(w)}A''(w) \mathrm{d} w$} & \multicolumn{2}{c|}{$A$-dependent}\\ \cline{1-4}\cline{6-7}

  14 & Galambos & $uv\exp{\left((-\ln u)^\theta + (-\ln v)^{-\theta}\right)^{-1/\theta}}$ & $\theta \geq 0$ & & $0$ & $2^{-1/\theta}$ \\ \cline{1-4}\cline{6-7}
  
  15 & H\"usler-Reiss & $e^{\ln(u)\Phi\left(\frac{1}{\theta}+\frac{\theta}{2}\ln \frac{\ln u}{\ln v}\right)+\ln(v)\Phi\left(\frac{1}{\theta}+\frac{\theta}{2}\ln\frac{\ln{u}}{\ln{v}}\right)}$ & $\theta \geq 0$ & & $0$ & $2-2\Phi(1/\theta)$ \\ \cline{1-4}\cline{6-7}

  16 & Tawn & $e^{(1-\alpha)\ln u+(1-\beta)\ln v-((-\alpha \ln u)^\gamma + (-\beta \ln v)^\gamma)^{1/\gamma}}$ & $\alpha,\beta \in [0,1]$, $\gamma \geq 0$ & & \multicolumn{2}{c|}{numerical approx.} \\ \hline\hline
  17 & FGM & $uv(1+\theta(1-u)(1-v))$ & $\theta \in [-1,1]$ & $(2\theta)/9$ & $0$ & $0$ \\ \hline
  18 & Marshall-Olkin& $\min\left(u^{1-\alpha}v,uv^{1-\beta}\right)$ & $\alpha, \beta \in [0,1]$ & $(\alpha\beta)/(2\alpha+2\beta-\alpha\beta)$ & $0$ & $\min(\alpha,\beta)$ \\ \hline

  19 & Plackett & $\frac{1+(\theta-1)(u+v)-\sqrt{(1+(\theta-1)(u+v))^2-4\theta(\theta-1)uv}}{2(\theta-1)}$ & $\theta \geq 0$, $\theta \neq 1$ & numerical approx. & $0$ & $0$ \\ \hline
  20 & Raftery & $\#4+\frac{1-\theta}{1+\theta}(uv)^{1/(1-\theta)}(1-\max(u,v)^{-(1+\theta)/(1-\theta)})$ & $\theta \in [0,1]$ & $(2\theta)/(3-\theta)$ & $\frac{2\theta}{\theta+1}$ & $0$ \\ \hline\hline

  21 & {\bf Nonparametric} & Section \ref{sec:nonparametric} & & & & \\ \hline
  22 & {\bf Conditional}   & Section \ref{sec:conditional}   & & & & \\ \hline
  \end{tabular}
  }
  \end{center}
  \caption[A zoo of copulas]{A zoo of copulas} \label{table:parametriccopulas}
\end{table}
\end{landscape}

\begin{definition}[Conditional copula]
The conditional copula of $P(\bm x_1, \bm x_2|x_3)$ is the joint distribution
of $\bm u_{1|3} \equiv P(\bm x_1|x_3)$ and $\bm u_{2|3} \equiv P(\bm x_1|x_3)$, where
$P(\bm x_1|x_3)$ and $P(\bm x_1|x_3)$ are the conditional marginal cdfs of
$P(\bm x_1, \bm x_2|x_3)$.
\end{definition}

\begin{theorem}[Sklar's theorem for conditional distributions]\label{thm:sklarConditional}
Let $P(\bm x_1,\bm x_2|x_3)$ be the conditional joint cdf for $\bm x_1$ and $\bm x_2$ given $\bm x_3 = x_3$ and let
$P(\bm x_1|x_3)$ and $P(\bm x_2|x_3)$ be its continuous conditional marginal cdfs.
Then, there exists a unique conditional copula $C(\bm x_1,\bm x_2|x_3)$ such that 
\begin{equation}
P(x_1,x_2|x_3) = C(P(x_1|x_3), P(x_2|x_3)|x_3)\label{eq:conditionalCopulaConditionalDistribution}
\end{equation}
for any $x_1$, $x_2$ and $x_3$ in the support of $\bm x_1$, $\bm x_2$ and $\bm
x_3$, respectively.  Conversely, if $P(\bm x_1|x_3)$ and $P(\bm x_1|x_3)$ are
the conditional cdfs of $\bm x_1$ and $\bm x_2$ given $\bm x_3 = x_3$ and $C(\bm
x_1,\bm x_2|x_3)$ is a conditional copula, then (\ref{eq:conditionalCopulaConditionalDistribution}) is a valid
conditional joint distribution with marginals $P(\bm x_1|x_3)$ and $P(\bm
x_2|x_3)$ and dependence structure $C(\bm x_1, \bm x_2|x_3)$.
\end{theorem}

\begin{proof}
 See \citep{Patton2002}.
\end{proof}

In the following, we describe a novel method based on Gaussian processes
\citep{Rasmussen06} to estimate conditional copulas \citep{dlp-gp}.

\subsubsection{Semiparametric conditional copulas}
Let $\D_{1,2} = \{x_{1,i},x_{2,i}\}_{i=1}^n$ and
$\D_3=\{x_{3,i}\}_{i=1}^n$ form a dataset corresponding to $n$  paired samples of
$\bm x_1$, $\bm x_2$ and $\bm x_3$ from the joint distribution $P(\bm x_1,\bm
x_2,\bm x_3)$.  We want to learn the conditional copula $C(\bm x_1,\bm
x_2|x_3)$. For this, we first compute estimates ${P}_n(\bm x_1|x_3)$ and
${P}_n(\bm x_2|x_3)$ of the conditional marginal cdfs using the data
available in $\D_{1,2}$ and $\D_3$.  We can obtain a sample $\D'_{1,2} = \{
  u_{1,i}, u_{2,i}\}_{i=1}^n $ from $C(\bm x_1, \bm x_2|x_3)$ by mapping the
  observations for $\bm x_1$ and $\bm x_2$ to their corresponding marginal
  conditional probabilities given the observations for $\bm x_3$:
\begin{equation*}
u_{1,i}={P}_n(\bm x_1 = x_{1,i}|x_{3,i}),\,\,\,
u_{2,i}={P}_n(\bm x_2 = x_{2,i}|x_{3,i}),\,\,\,
\text{for}\,\,i=1,\ldots,n\,.
\end{equation*}
This pair of transformations are computed by i) estimating the marginal cdfs
$P(\bm x_1)$ and $P(\bm x_2)$, ii) estimating two parametric copulas $C(P(\bm
x_1), P(\bm x_3))$ and $C(P(\bm x_2), P(\bm x_3))$, and iii) estimating the
conditional distributions $P(\bm x_1 \given x_3)$ and $P(\bm x_2 \given x_3)$
from such parametric copulas, as explained in Section~\ref{sec:copcond}.

The data in $\D'_{1,2}$ and $\D_3$ can be used to adjust a semiparametric model
for $C(\bm x_1,\bm x_2|x_3)$.  In particular, we assume that $C(\bm x_1, \bm x_2|x_3)$ follows
  the shape of a parametric copula, specified in terms of its Kendall's $\tau$
  statistic, where the value of $\tau$ depends on the value of $x_3$.  The
  parameter $\theta$ of the copula can be easily obtained as a function of
  Kendall's $\tau$.  The connection between $\tau$ and $x_3$ is a latent
  function $g$, such that $\tau = g(x_3)$.  To ease estimation, $g$ is the
  composition of an unconstrained function $f$ and a link function $\sigma$
  mapping the values of $f$ to valid Kendall $\tau$ values.  For example, we
  can fix $\sigma(x) = 2\Phi(x) - 1$, where $\Phi$ is the standard Gaussian
  cdf: this particular choice  maps the real line to the interval
  $[-1,1]$.  After choosing a suitable $\sigma$, our semiparametric model can
  make use of unconstrained nonlinear functions $f$.  We can learn $g$ by
  placing a Gaussian process (GP) prior on $f$ and computing the
  posterior distribution for $f$ given $\D'_{1,2}$ and $\D_3$.

Let ${f}=(f(x_{3,1}),\ldots,f(x_{3,n}))^\top$ be the
$n$-dimensional vector with the evaluation of $f$ at the available observations
from $\bm x_3$. Since $f$ is a sample from a GP, the prior distribution for
  ${f}$ given $\D_3$ is the Gaussian:
  \begin{equation*} p({f}|\D_3) =
  \N({f}|{m}_0,{K})\,, \end{equation*} where
  ${m}_0$ is an $n$-dimensional vector with the evaluation of the mean
  function $m$ at $\D_3$, that is,
  ${m}_0=(m(x_{3,1}),\ldots,m(x_{3,n}))^\top$ and ${K}$
  is an $n\times n$ kernel matrix generated by the evaluation of the
  kernel $k$ at $\D_3$, that is, $k_{i,j}=k(x_{3,i},
  x_{3,j})$.  We select a constant function for $m$ and the Gaussian kernel $k$:
\begin{align}
m(x) & = m_0,\nonumber\\
k(x_i,x_j) & = \sigma^2 \exp \left\{ -(x_i - x_j)^2\lambda^{-2}
\right\} + \sigma^2_0\,,\label{eq:kernel}
\end{align}
where $m_0$, $\sigma^2$, $\lambda$ and $\sigma^2_0$ are hyper-parameters.  The
posterior distribution for ${f}$ given $\D'_{1,2}$ and
$\D_3$ is 
\begin{equation}
p({f}|\D'_{1,2}, \D_3) = \frac{\left[ \prod_{i=1}^n
c(u_{1,i}, u_{2,i}|\sigma(f_i)) \right]
p({f}|\D_3)}
{p(\D'_{1,2}|\D_3)}\,.\label{eq:posteriorDistribution}
\end{equation}
In the equation above, $c(u_{1,i}, u_{2,i}|\sigma(f_i)]$ is
the density function of the parametric copula model with $\tau =
\sigma(f_i)$. Given $\D_{1,2}$, $\D_3$ and a particular 
assignment $\bm x_3 = x_3^\star$, we can make predictions for the conditional
  distribution of $u_1^\star = P(x_1^\star|x_3^\star)$ and $u_2^\star =
  P(x_2^\star|x_3^\star)$, where $x_1^\star$ and $x_2^\star$ are samples from
  $p(x_1,x_2|x_3^\star)$.  In particular, we have that
\begin{align}
p(u_1^\star,u_2^\star|x_3^\star) = \int c(u_1^\star,
u_2^\star|\sigma(f^\star)) p(f^\star|{f})
p({f}|\D'_{1,2},
\D_3)\,d{f}df^\star\,,\label{eq:predictiveDistribution}
\end{align}
where $f^\star = f(x_3^\star)$, $p(f^\star|{f}) =
\N(f^\star|{k}_\star^\top {K}^{-1}{f}, k_{\star,\star} -
{k}_\star^\top {K}^{-1} {k})$, ${k}$ is an $n$-dimensional vector with the
prior covariances between $f(x_3^\star)$ and $\{f(x_3^{(i)})\}_{i=1}^n$ and
$k_{\star,\star}=k(x_3^\star,x_3^\star)$.  Unfortunately, the exact computation
of (\ref{eq:posteriorDistribution}) and (\ref{eq:predictiveDistribution}) is
intractable. To circumvent this issue, we resort to the use of the \emph{expectation
propagation algorithm} \citep{Minka2001}, one alternative to efficiently
compute approximations to (\ref{eq:posteriorDistribution}) and
(\ref{eq:predictiveDistribution}).

\subsubsection{Approximating the posterior with expectation propagation}
We use expectation propagation (EP) \citep{Minka2001} to obtain tractable
approximations to the exact posterior (\ref{eq:posteriorDistribution}) and predictive distributions
(\ref{eq:predictiveDistribution}).  The
posterior $p({f}|\D'_{1,2}, \D_3)$ is, up to a normalization
constant, the product of factors
\begin{equation}
p({f}|\D'_{1,2}, \D_3) \propto \left[ \prod_{i=1}^n
h_i(f_i) \right] h_{n+1}({f})\,,\label{eq:posteriorFactorization} 
\end{equation}
where $h_i(f_i) = c_{x_1,x_2|x_3}[u_{1,i}, u_{2,i}|\sigma(f_i)]$ and
$h_{n+1}({f}) = \N({f}|{m}_0,{K})$.  EP
approximates (\ref{eq:posteriorFactorization}) with a simpler distribution
$q({f}) \propto [\prod_{i=1}^n \tilde{h}_i(f_i)] h_{n+1}({f})$,
obtained by replacing each non-Gaussian factor $h_i$ in
(\ref{eq:posteriorFactorization}) with an approximate factor $\tilde{h}_i$ that
is Gaussian, but unnormalized:
\begin{equation*}
\tilde{h}_i(f_i) = c_i \exp\{ -\frac{1}{2} a_i f_i^2 + b_i f_i\}\,,
\end{equation*}
where $c_i$ is a positive constant and $a_i$ and $b_i$ are the natural
parameters of the Gaussian factor $\tilde{h}_i$. Since $h_{n+1}$ in
(\ref{eq:posteriorFactorization}) is already Gaussian, there is no need for its
approximation.  Since the Gaussian distribution belong to the exponential
family of distributions, they are closed under the product and division
operations, and therefore $q$ is Gaussian with natural parameters equal to the sum of the
natural parameters of the Gaussian factors $\tilde{h}_1,\ldots,\tilde{h}_n$ and
$h_{n+1}$.

Initially all the approximate factors $\tilde{h}_i$ are uninformative or
uniform, that is, $a_i= 0$ and $b_i=0$ for $i=1,\ldots,n$.  EP iteratively
updates each $\tilde{h}_i$ by first computing the \emph{cavity} distribution
$q^{\setminus i}({f}) \propto q({x}) / \tilde{h}_i(f_i)$ and then
minimizing the Kullback-Liebler (KL) divergence between $h_i(f_i)q^{\setminus
i}({f})$ and $\tilde{h}_i(f_i)q^{\setminus i}({f})$
\citep{Minka2001}.  To achieve this, EP matches the first 
two moments of $h_i(f_i)q^{\setminus i}({f})$ and
$\tilde{h}_i(f_i)q^{\setminus i}({f})$, with respect to $f_i$, after marginalizing out all
the other entries in ${f}$ 
\citep{seeger2005expectation}.  In our implementation of EP, we follow
\cite{van2010efficient} and refine all the $\tilde{h}_i$ in parallel. For
this, we first compute the $n$-dimensional vectors
${m}=(m_1,\ldots,m_n)^\top$ and
${v}=(v_1,\ldots,v_n)^\top$ with the marginal means and variances of
$q$, respectively. In particular,
\begin{align}
{v}  &= \text{diag}\left\{ \left({K}^{-1} +
\text{diag}({a}) \right)^{-1}\right\},\nonumber\\
{m}  &= \left({K}^{-1} +
\text{diag}({a})\right)^{-1}\left({b} +
{K}^{-1}{m}_0\right)\,,
\label{eq:marginalsq}
\end{align}
where ${a}=(a_1,\ldots,a_n)^\top$ and
${b}=(b_1,\ldots,b_n)^\top$ are $n$-dimensional vectors with the
natural parameters of the approximate factors $\tilde{h}_1,\ldots,\tilde{h}_n$.
After this, we update all the approximate factors. For this, we
obtain, for $i = 1,\ldots, n$, the marginal mean $m^{\setminus i}$ and the
marginal variance $v^{\setminus i}$ of $f_i$ with respect to the cavity
distribution $q^{\setminus i}$.  This leads to
\begin{align*}
v^{\setminus i} &= (v_i^{-1} - a_i)^{-1},\\
m^{\setminus i} &= v^{\setminus
i}(m_iv_i^{-1} - b_i)\,.
\end{align*}
We then compute, for each approximate factor $\tilde{h}_i$, the new marginal
mean and marginal variance of $q$ with respect to $f_i$ after updating that
factor. In particular, we compute
\begin{align*}
m_i^\text{new} & = \frac{1}{Z_i} \int f_i h_i(f_i)\N(f_i|m^{\setminus
i}, v^{\setminus i})\,df_i\,,\\
v_i^\text{new} & = \frac{1}{Z_i} \int (f_i -
m_i^\text{new})^2h_i(f_i)\N(f_i|m^{\setminus i}, v^{\setminus
i})\,df_i\,.
\end{align*}
where $Z_i = \int h_i(f_i)\N(f_i|m^{\setminus i}, v^{\setminus
i})\,df_i$ is a normalization constant.  These integrals are not analytic,
so we approximate them using numerical integration.
The new values for $a_i$ and $b_i$ are 
\begin{align}
a_i^\text{new} & = [v_i^\text{new}]^{-1} - [v^{\setminus
i}]^{-1}\,,\label{eq:update_a}\\
b_i^\text{new} & = m_i^\text{new} [v_i^\text{new}]^{-1} - m^{\setminus
i}[v^{\setminus i}]^{-1}\,,\label{eq:update_b}
\end{align}
Once we have updated $a_i$ and $b_i$, we can update the marginal mean and the
marginal variance of $f_i$ in $q$, namely,
\begin{align*}
v_i^\text{new}  &= ([v^{\setminus i}]^{-1} + a_i)^{-1},\\
m_i^\text{new}  &=
v_i^\text{new} (m^{\setminus i} [v^{\setminus i}]^{-1}  + b_i)\,.
\end{align*}
Finally, we update $c_i$ to be
\begin{equation}
\log c_i^\text{new} = \log Z_i + \frac{1}{2} \log v^{\setminus i} - \frac{1}{2}
\log v_i^\text{new} + \frac{[m^{\setminus i}]^2}{2 v^{\setminus i}} -
\frac{[m_i^\text{new}]^2}{2 v^\text{new}}\,.\label{eq:update_c}
\end{equation}
This completes the operations required to update all the approximate factors
$\tilde{h}_1,\ldots,\tilde{h}_n$. Once EP has updated all these factors using
(\ref{eq:update_a}), (\ref{eq:update_b}) and (\ref{eq:update_c}), a new
iteration begins. EP stops when the change between two consecutive iterations
in the marginal means and variances of $q$, as given by (\ref{eq:marginalsq}),
is less than $10^{-3}$.  To improve the convergence of EP and avoid numerical
problems related to the parallel updates \citep{van2010efficient}, we damp the
EP update operations. When damping, EP replaces (\ref{eq:update_a}) and
(\ref{eq:update_b}) with
\begin{align*}
a_i^\text{new} & = (1 - \epsilon) a_i^\text{old} + \epsilon \left\{
  [v_i^\text{new}]^{-1} - [v^{\setminus i}]^{-1}\right\}\,,\\
b_i^\text{new} & = (1 - \epsilon) b_i^\text{old} + \epsilon \left\{
  m_i^\text{new} [v_i^\text{new}]^{-1} - m^{\setminus i}[v^{\setminus i}]^{-1}
  \right\}\,,
\end{align*}
where $a_i^\text{old}$ and $b_i^\text{old}$ are the parameters values before
the EP update.  The parameter $\epsilon \in [0,1]$ controls the amount of
damping.  When $\epsilon = 1$, we recover the original EP updates.  When
$\epsilon = 0$, the parameters of the approximate factor $\tilde{h}_i$ are not
modified.  We use an annealed damping scheme: we start with
$\epsilon = 1$ and, after each EP iteration, we scale down $\epsilon$ by
$0.99$.

Some of the parameters $a_i$ may become negative during the execution of EP.
These
negative variances in $\tilde{h}_1,\ldots,\tilde{h}_n$ may result in a
covariance matrix ${V}=\left({K}^{-1} + \text{diag}({a})
\right)^{-1}$ for ${f}$ in $q$ that is not positive definite.  Whenever
this happens, we first restore all the $\tilde{h}_1,\ldots\tilde{h}_n$, to
their previous value, reduce the damping parameter $\epsilon$ by scaling it by
$0.5$ and repeat the update of all the approximate factors with the new
value of $\epsilon$.  We repeat this operation until ${V}$ is positive
definite.

EP can also approximate the normalization constant of the exact posterior
distribution (\ref{eq:posteriorDistribution}), that is,
$p(\D'_{1,2}|\D_3)$. For this, note that
$p(\D'_{1,2}|\D_3)$ is the integral of $[\prod_{i=1}^n
h_i(f_i) ]h_{n+1}({f})$.  We can then approximate
$p(\D'_{1,2}|\D_3)$ as the integral of $[\prod_{i=1}^n
\tilde{h}_i(f_i) ]h_{n+1}({f})$ once all the $\tilde{h}_i$ factors have
been adjusted by EP.  Since all the $\tilde{h}_i$ and $h_{n+1}$ are
Gaussian, this integral can be efficiently computed.  In particular, after
taking logarithms, we obtain
\begin{align}
\log p(\D'_{1,2}|\D_3) \approx\nonumber\\
\sum_{i=1}^n \log c_i
-\frac{1}{2} \log |{K}| -
\frac{1}{2}{m}_0^\top{K}^{-1}{m}_0 + \frac{1}{2} \log
|{V}| +
\frac{1}{2}{m}^\top{V}^{-1}{m}\,,\label{eq:evidenceApprox}
\end{align}
where ${V}$ is the covariance matrix for ${f}$ in $q$ and ${m}$ is the mean
vector for ${f}$ in $q$ as given by (\ref{eq:marginalsq}).  The EP
approximation to $p(\D'_{1,2}|\D_3)$ is also a proxy to
adjust the hyper-parameters $m_0$, $\sigma^2$, $\sigma^2_0$ and $\lambda$ of
the mean function and the covariance function of the GP (\ref{eq:kernel}). In
particular, we can obtain a type-II maximum likelihood estimate of these
hyper-parameters by maximizing $\log p(\D'_{1,2}|\D_3)$
\citep{Bishop06}.  To solve this maximization, descend along the gradient of
$\log p(\D'_{1,2}|\D_3)$ with respect to $m_0$, $\sigma^2$,
$\sigma^2_0$ and $\lambda$.  Fortunately, the right-hand side of
(\ref{eq:evidenceApprox}) approximates this gradient well, if we treat the
parameters of the approximate factors $a_i$, $b_i$ and $c_i$, for
$i=1,\ldots,n$ as constants \citep{seeger2005expectation}.

Finally, the EP solution is also useful to approximate the
predictive distribution (\ref{eq:predictiveDistribution}).  For this, we first
replace $p({f}|\D'_{1,2}, \D_3)$ in
(\ref{eq:predictiveDistribution}) with the EP approximation to this exact
posterior distribution, that is, $q$.  After marginalizing out ${f}$, we
have
\begin{equation*}
\int p(f^\star|{f}) p({f}|\D'_{1,2},
\D_3)\,d{f} \approx \int p(f^\star|{f})
q({f})\,d{f} = \N(f^\star|m^\star,v^\star)
\end{equation*}
where
\begin{align*}
m^\star & = {k}_\star^\top({K} +
\tilde{{V}})^{-1}\tilde{{m}}\,,\\
v^\star & = k_{\star,\star} - {k}_\star^\top({K} +
\tilde{{V}})^{-1}{k}_\star\,,
\end{align*}
${k}_\star$ is an $n$-dimensional vector with the prior covariances
between $f_\star$ and $f_1,\ldots,f_n$, $k_{\star,\star}$ is the prior variance
of $f_\star$, $\tilde{{m}}$ is an $n$-dimensional vector whose $i$-th
entry is $b_i / a_i$ and $\tilde{{V}}$ is an $n\times n$ diagonal matrix
whose $i$-th entry in the diagonal is $1 / a_i$.  Once we have computed
$m^\star$ and $v^\star$, we approximate the integral $\int
c_{x_1,x_2|x_3}[u_1^\star,
u_2^\star|\sigma(f^\star)]\N(f^\star|m^\star,v^\star)\,df^\star$ by
Monte Carlo.  For this, draw $N$ samples
$f^\star_{(1)},\ldots,f^\star_{(N)}$ from
$\N(f^\star|m^\star,v^\star)$ and approximate
(\ref{eq:predictiveDistribution}) by 
\begin{equation*}
p(u_1^\star,u_2^\star|x_3^\star) \approx \frac{1}{N} \sum_{i=1}^N
c_{x_1,x_2|x_3}[u_1^\star,
u_2^\star|\sigma(f_{(i)}^\star)]\,.
\end{equation*}

\subsubsection{Speeding up the computations with GPs}\label{sec:fitc}
The computational cost of the previous EP algorithm is $O(n^3)$, due to the
computation of the inverse of a kernel matrix of size $n\times n$.  To
reduce this cost, we use the FITC approximation for Gaussian processes
described by \cite{Snelson2006}.  The FITC approximation replaces the $n\times
n$ covariance matrix ${K}$ with the low-rank matrix ${K}'= {Q} +
\text{diag}({K} - {Q})$, where ${Q} = {K}_{n,n_0} {K}^{-1}_{n_0,n_0}
{K}^\top_{n,n_0}$ is a low-rank matrix, ${K}_{n_0,n_0}$ is the $n_0 \times
n_0$ covariance matrix generated by evaluating the covariance function $k$ in
(\ref{eq:kernel}) between some $n_0$ training points or pseudo-inputs and
${K}_{n,n_0}$ is the $n\times n_0$ matrix with the covariances between all
training points $x_{3,1},\ldots,x_{3,n}$ and pseudo-inputs. This approximate EP
algorithm has cost ${O}(nn_0^2)$.

\subsection{Nonparametric copulas}\label{sec:nonparametric}
In search for a higher degree of flexibility than the one provided by the
parametric copulas of Section \ref{sec:parametric}, one could try to
perform kernel density estimation to estimate copula densities, for instance by placing a bivariate Gaussian
kernel on each copula sample $(u_i, v_i)$.  However, the resulting
kernel density estimate would have support on $\R^2$, while the support of any
bivariate copula is the unit square. A workaround to this issue is to 1)
transform each copula marginal distribution to have full support and 2)
perform kernel density estimation on such transformed sample.
Following this rationale, this section studies copula estimates of the form
\begin{equation}\label{eq:semicopula}
  \hat{c}(u,v) = \frac{\hat{p}_{ab}(\hat{P}_{a}^{-1}(u),
  \hat{P}_{b}^{-1}(v))}{\hat{p}_{a}(\hat{P}_{a}^{-1}(u))\hat{p}_{b}(\hat{P}_{b}^{-1}(v))},
\end{equation}
where 
\begin{align*}
  \hat{p}_{a}(u) &= \sum_{i=1}^{k_a} m_{a,i} \, \N(u; \mu_{a,i}, \sigma_{a,i}^2),\\
  \hat{p}_{b}(v) &= \sum_{i=1}^{k_b} m_{b,i} \, \N(v; \mu_{b,i}, \sigma_{b,i}^2),
\end{align*}
\noindent and 
\begin{equation*}
 \hat{p}_{ab}(u,v) = \displaystyle\sum_{i=1}^{k_{ab}} m_{ab,i} \, \N(u, v; \mu_{ab,i}, \Sigma_{ab,i}^2),
\end{equation*}
with
\begin{equation*}
  \sum_{i=1}^{k_a} m_{a,i} = \sum_{i=1}^{k_b} m_{b,i} = \sum_{i=1}^{k_{ab}}
  m_{ab,i} = 1.
\end{equation*}

Since Gaussian mixture models are dense in the set of all probability
distributions, the model in (\ref{eq:semicopula}) can model a wide range of distributions.
\citet{Fermanian07} set $p_{a}$ and $p_{b}$ to be Normal distributions,
$k_{ab} = m_{ab,i}^{-1} = n$ and $\mu_{ab,i} = (P_{a}^{-1}(u_i), P_{b}^{-1}(v_i))$
yielding the so-called \emph{nonparametric copula}, equivalent to a Gaussian kernel
density estimate on the transformed sample
$\{(\Phi^{-1}(u_i),\Phi^{-1}(v_i))\}_{i=1}^n$, where $\Phi^{-1}$ is 
the Normal inverse cdf.

\subsubsection{Conditional distributions}
The conditional distribution $P(u|v)$ for the copula model in
(\ref{eq:semicopula}) is
\begin{align}\label{eq:semicopulaint1}
  \hat{P}(u|v) &= \int_0^u \hat{c}(x,v)\mathrm{d}x
  = \int_0^u \frac{\hat{p}_{ab}(\hat{P}^{-1}_{a}(x),
  \hat{P}^{-1}_{b}(v))}{\hat{p}_{a}(\hat{P}^{-1}_{a}(x))\hat{p}_{b}(\hat{P}^{-1}_{b}(v))}
  \mathrm{d}x\nonumber\\
  &= \sum_{i=1}^{k_{ab}} \frac{m_{ab,i}}{\hat{p}_{b}(\hat{P}^{-1}_{b}(v))} \int_0^u
  \frac{\N(\hat{P}^{-1}_{a}(x), \hat{P}^{-1}_{b}(v); \mu_{ab,i},
  \Sigma_{ab,i})}{\sum_{i=1}^{k_a} m_{a,i}\, \N(\hat{P}_{a}^{-1}(x);
  \mu_{a,i}, \sigma_{a,i}^2)}\mathrm{d}x.
\end{align}
Unfortunately, the integral in (\ref{eq:semicopulaint1}) has no analytical
solution for arbitrary mixtures $(p_{a},p_{b})$. Let us instead restrict ourselves
to the case where $p_{a}(x) = p_{b}(x) :=
\N(x; 0, 1)$. By denoting $z_i := \mu_{ab,i}^{(1)}$, $w_i
:= \mu_{ab,i}^{(2)}$, \citet{dlp-ssl} derives
\begin{align*}
  \hat{P}(u|v) &= \int_0^u \hat{c}(x,v)\mathrm{d}x
  = \int_0^u \frac{\hat{p}_{ab}(\Phi^{-1}(x),
  \Phi^{-1}(v))}{\phi(\Phi^{-1}(x))\phi(\Phi^{-1}(v))} \mathrm{d}x\nonumber\\
  &= \sum_{i=1}^{k_{ab}} \frac{m_{ab,i}}{\phi(\Phi^{-1}(v))} \int_0^u
  \frac{\N(\Phi^{-1}(x), \Phi^{-1}(v); \mu_{ab,i}, \bm
  \Sigma_{ab,i})}{\phi(\Phi^{-1}(u))}\mathrm{d}x\nonumber\\
  &= \sum_{i=1}^{k_{ab}} \frac{m_{ab,i}}{\phi(\Phi^{-1}(v))}
  \N(\Phi^{-1}(v); \mu_{ab,i}^{(2)},
  \sigma_{wi}^{2})\,\Phi\left(\frac{\Phi^{-1}(u) -
  \mu_{z_i|w_i}}{\sigma^2_{z_i|w_i}}\right),
\end{align*}
where $\Sigma_{ab,i} = \left(\begin{array}{cc} \sigma^2_{z_i} & \gamma_i \\
\gamma_i &\sigma^2_{w_i}\end{array}\right)$, $\mu_{z_i|w_i} = z_i +
\frac{\sigma_{z_i}}{\sigma_{w_i}}\gamma_i (w - w_i)$ and $\sigma^2_{z_i|w_i} =
\sigma_{z_i}^2 (1 -\gamma_i^2)$, for some correlations $-1 \leq \gamma_i \leq
1$ and $1 \leq i \leq k_{ab}$. Setting $k_{ab} = m_{ab,i}^{-1} = n$,
$\mu_{ab,i} = (\Phi^{-1}(u_i), \Phi^{-1}(v_i))$ and $\Sigma_{ab,i} =
\left(\begin{array}{cc} \sigma^2_{z} & \gamma_i \\ \gamma_i
&\sigma^2_{w}\end{array}\right)$ produces similar expressions for the
nonparametric copula.

The previous are closed-form expressions for our nonparametric copula model and
their exact conditional distributions.  Therefore, these formulas can be used
to construct vine copulas (presented in Section~\ref{sec:vine-copulas}) in a
consistent manner. 

\section{Product models}
Product models exploit the conditional probability rule
\begin{equation*}
  p(\bm x = x, \bm y = y) = p(\bm y = y \given \bm x = x)p(\bm x = x)
\end{equation*}
and the conditional independence rule
\begin{equation*}
  p(\bm x = x, \bm y = y \given \bm z = z) = p(\bm x = x \given \bm z = z) p(\bm y = y \given \bm z= z)
\end{equation*}
to express high-dimensional joint probability density function as the product
of low-dimensional conditional probability density functions. As opposed to
mixture models, which implement the ``\texttt{OR}'' operation
between their components, product models implement the ``\texttt{AND}''
operation between their factors. The most prominent example of product
models are Bayesian networks.

\subsection{Bayesian networks}
Bayesian networks \citep{pearl1985bayesian} are probabilistic graphical models
that represent the joint probability distribution of a set of random variables as a
Directed Acyclic Graph (DAG). In this DAG, each node represents a random
variable, and each edge represents a conditional dependence between two
variables. Using this graphical representation, Bayesian networks factorize
probability distributions in one factor per node, equal to the conditional
distribution of the variable associated with that node, when conditioned on all
its parents in the graph. Figure~\ref{fig:bayesian-network}
illustrates a Bayesian network on six random variables, and the resulting
factorization of the six-dimensional density $p(\bm x = x)$. 

\begin{figure}
  \begin{center}
  \begin{tikzpicture}[node distance=1cm, auto,]
   \node[punkt] (x1) at (0,0) {$\bm x_1$};
   \node[punkt] (x2) at (2,0) {$\bm x_2$};
   \node[punkt] (x3) at (4,0) {$\bm x_3$};
   \node[punkt] (x4) at (2,-2) {$\bm x_5$};
   \node[punkt] (x5) at (0,-2) {$\bm x_4$};
   \node[punkt] (x6) at (4,-2) {$\bm x_6$};
   \draw[pil] (x1) -- (x2);
   \draw[pil] (x4) -- (x2);
   \draw[pil] (x5) -- (x2);
   \draw[pil] (x4) -- (x5);
   \draw[pil] (x2) -- (x3);
   \draw[pil] (x5) -- (x1);
  \end{tikzpicture}
  \begin{center}
  $p(\bm x = x) = p(x_6)p(x_5)p(x_4 \given x_5)p(x_1 \given x_4)p(x_2 \given x_1,x_4,x_5)p(x_3\given x_2)$
  \end{center}
  \end{center}
  \caption{A Bayesian network and its factorization.}
  \label{fig:bayesian-network}
\end{figure}

The arrows in the DAG of a Bayesian network are a mathematical representation
of conditional dependence: this notion has nothing to do with causation between
variables. We will devote Chapter~\ref{chapter:language-causality} to extend 
Bayesian networks to the language of causation.

\begin{remark}[Other product models]
  Other product models include Markov networks and
  factor graphs. In contrast to Bayesian networks, Markov networks and factor
  graphs rely on undirected, possibly cyclic graphs.  Bayesian networks and
  Markov networks are complimentary, in the sense that each of them can
  represent dependencies that the other can not.  The Hammersley-Clifford
  theorem establishes that factor graphs can represent both Bayesian Networks
  and Markov networks.
\end{remark}

\subsection{Vine copulas}\label{sec:vine-copulas}

We now extend the framework of copulas (Section~\ref{sec:copulas}) to model
$d-$dimensional probability density functions $p(\bm x)$ as the product of its one-dimensional
marginal densities $p(\bm x_i)$ and its dependence structure or copula $c$:
\begin{equation*}
  p(\bm x) = \underbrace{\left[\prod_{i=1}^d p(\bm x_i)\right]}_{\text{marginals}}
  \cdot \, \underbrace{c(P(\bm x_1),\ldots,P(\bm x_d))}_{\text{dependence
  structure}},
\end{equation*}
where $P(\bm x_i)$ denotes the marginal cdf of $\bm x_i$, for all $1 \leq i
\leq n$. To learn $p$, we first estimate each of the $d$ marginals $p_i$
independently, and then estimate the multivariate copula $c$. However, due to
the curse of dimensionality, directly learning $c$ from data is a challenging
task.  One successful approach to deal with this issue is to further factorize
$c$ into a product of bivariate, \emph{parametric}, \emph{unconditional}
copulas. This is the approach of \emph{vine decompositions} \citep{Bedford01}.

{Vine copulas} \citep{Bedford01} are hierarchical graphical models that
factorize a $d$-dimensional copula into the product of $d(d-1)/2$ bivariate
copulas. Vines are flexible models, since each of the bivariate copulas in the
factorization can belong to a different parametric family (like the ones described in
Section \ref{sec:parametric}). Multiple types of vines populate the literature;
we here focus on regular vine copula distributions, since they are the most
general kind \citep{Aas09,Kurowicka11}.

\begin{remark}[History of vine copulas]
  Vines are due to \cite{Joe96} and \cite{Bedford01}.  \cite{Aas09} and
  \cite{Kurowicka11} offer two monographs on vines.  Vines enjoy a
  mature theory, including results in sampling \citep{Bedford02},
  characterization of assumptions \citep{Haff10,Acar12,dlp-gp}, model selection
  \citep{Kurowicka11,Dissmann13},
  extensions to discrete distributions \citep{Panagiotelis12} and
  identification of equivalences to other well known models, such as Bayesian
  belief networks and factor models \citep{Kurowicka11}.

  Vines have inherited the wide range of applications that copulas have
  enjoyed, including time series prediction, modeling of financial returns,
  comorbidity analysis, and spatial statistics \citep{Kurowicka11}. Initial
  applications in machine learning include semisupervised domain adaptation
  \citep{dlp-ssl} and Gaussian process conditional distribution estimation
  \citep{dlp-gp}.
\end{remark}

A vine $\mathcal{V}$ is a hierarchical collection of $d-1$ undirected trees
$T_1, \ldots, T_{d-1}$.
Each tree $T_i$ owns a set of nodes $N_i$ and a set of edges $E_i$, and each
edge in each tree will later correspond to a different bivariate copula in the
vine factorization.  The copulas derived from the edges of the first tree are
unconditional. On the other hand, the copulas derived from the edges of the trees $T_2,
\ldots, T_{d-1}$ will be conditioned to some variables.

Therefore, the edges of a vine $\mathcal{V}$ specify the factorization of a
d-dimensional copula density $c(u_1, \ldots, u_d)$ into the product of
bivariate copula densities, that we write using the notation 
\begin{equation*}
  c(u_1, \ldots, u_d) = \prod_{T_i \in \mathcal{V}} \prod_{e_{ij} \in E_i}
  c_{ij|D_{ij}}(P_{u_{ij}|D_{ij}}(u_{ij}|D_{ij}),P_{v_{ij}|D_{ij}}(v_{ij}|D_{ij})|D_{ij}),
\end{equation*}
where $e_{ij} \in E_i$ is the $j$-th edge from the $i$-th tree $T_i$,
corresponding to a bivariate copula linking the two variables $u_{ij}$ and
$v_{ij}$ when conditioned to the set of variables $D_{ij}$. The set of variables
$\{u_{ij}, v_{ij}\}$ is \emph{the conditioned set} of $e_{ij}$, and the set
$D_{ij}$ is \emph{the conditioning set} of $e_{ij}$. The elements of
these sets for each edge $e_{ij}$ are constant during the construction of the
vine, and detailed in Definition~\ref{def:sets}.

Three rules establish the hierarchical relationships between the trees forming a vine.

\begin{definition}[Regular vine structure]\label{def:rvine}
The structure of a $d-$dimensional regular vine $\mathcal{V}$ is a sequence of $d-1$
trees $T_1, \ldots, T_{d-1}$ satisfying:
\begin{enumerate}
  \item $T_1$ has node set $N_1 = \{1, \ldots, d\}$ and edge set $E_1$.
  \item $T_i$ has node set $N_i = E_{i-1}$ and edge set $E_i$, for $2 \leq i \leq d-1$.
  \item For $\{a,b\} \in E_i$, with $a = \{a_1, a_2\}$ and $b = \{b_1, b_2\}$,
  it must hold that $\#(a \cap b) =1$ (proximity condition). That is, the edges $a$ and $b$ must share a common node.
\end{enumerate}
\end{definition}

Define by $C(e_{ij}) := \{u_{ij}, v_{ij}\}$ and $D(e_{ij}) := D_{ij}$ the
conditioned and conditioning sets of the edge $e_{ij}$, respectively. These two
sets, for each edge, specify each bivariate copula in the
vine factorization. To construct these sets a third and auxiliary set, 
the \emph{constraint set}, is necessary.  In the following definition we show how to
obtain the conditioned and conditioning sets in terms of the constraint sets.
\begin{definition}[Constraint, conditioning and conditioned
vine sets]\label{def:sets} An edge $e = \{a,b\} \in E_i$, with $a,b \in E_{i-1}$
owns:
\begin{enumerate}
  \item its constraint set 
    \begin{align*}
      N(e) = \{&n \in N_1 : \exists\,\,e_j \in E_j, j = 1, \ldots, d-1,\\
      &\text{with } n \in e_1 \in e_2 \in \ldots \in e \}\subset N_1.
    \end{align*}
  \item its conditioning set $D(e) = N(a) \cap N(b)$.
  \item its conditioned set $C(e) = \{N(a)\setminus D(e), N(b)\setminus D(e)$\}.
\end{enumerate}
\end{definition}
The constraint set $N(e)$ contains all the nodes in $N_1$ reachable 
from nested structure of edges contained in $e$.  For example, consider the
edge $e = \{\{1,2\},\{2,3\}\}\in E_2$, with $1,2,3 \in N_1$.  Then, $e$ has
constraint set $N(e) = \{1,2,3\}$, conditioned set $C(e) = \{1,3\}$ and
conditioning set $D(e) = \{2\}$. Therefore, the edge $e$ will later
correspond to the bivariate copula
$c_{1,3|2}(P_{1|2}(u_{1|2}),P_{3|2}(u_{3|2})|u_2)$ in the resulting vine
factorization.

\subsubsection{Estimation of vines from data}
The structure of a vine is determined by the particular spanning trees chosen at
each level of the hierarchy.  There exists $\frac{d!}{2} 2^{{d-2 \choose 2}}$
different vine structures to model a $d-$dimensional copula function
\citep{Kurowicka11}. Therefore, to estimate a vine decomposition from data, one
must first decide on a particular structure for its trees. One 
common alternative is to use the greedy algorithm of \citet{Dissmann13}. This
algorithm selects maximum spanning trees after giving each edge $e$ a weight
corresponding to the empirical estimate of Kendall's $\tau$ between each
variable in $C(e)$ when conditioned to $D(e)$.

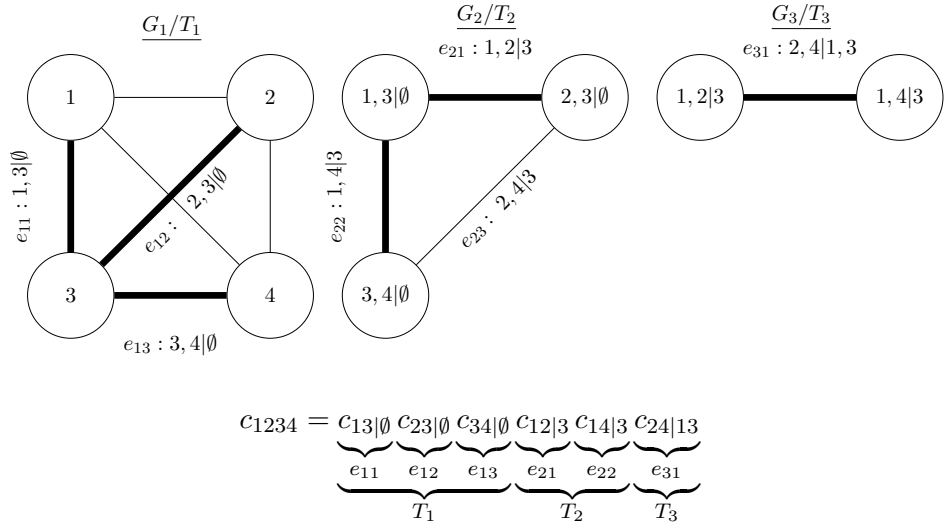
\begin{figure}[h!]
  \begin{center}
\resizebox{\linewidth}{!}
{
  \tikzset{
    select/.style={draw=black,line width=1.2mm},
    punkt/.style={circle,draw=black,minimum height=4em,text centered}
  }
  \begin{tikzpicture}[node distance=2cm]
    \node[punkt]             (1) {$1$};
    \node[punkt, right=of 1] (2) {$2$};
    \node[punkt, below=of 1] (3) {$3$};
    \node[punkt, below=of 2] (4) {$4$};
    \path[select] (1) edge node[rotate=90,anchor=south,shift={(0mm,5mm)}]{$e_{11}:
    1,3|\emptyset$} (3);
    \path[select] (2) edge node[rotate=45,anchor=north,
    shift={(-0mm,0mm)}]{$e_{12} : \,\,\,\,\,2,3|\emptyset$} (3);
    \path (1) edge (4);
    \path (1) edge (2);
    \path (2) edge node[anchor=west]{} (4);
    \path[select] (3) edge node[anchor=north, shift={(0,-5mm)}]{$e_{13} :
    3,4|\emptyset$} (4);
    \node[punkt, right=0.5 of 2] (13) {$1,3|\emptyset$};
    \node[punkt, right=of 13] (23)    {$2,3|\emptyset$};
    \node[punkt, below= of 13] (34)   {$3,4|\emptyset$};

    \path[select] (13) edge node[anchor=south,above, shift={(0,5mm)}]{$e_{21} :
    1,2|3$} (23);
    \path[select] (13) edge node[rotate=90,anchor=south,shift={(0,5mm)}]{$e_{22} :
    1,4|3$} (34);
    \path[] (34) edge node[rotate=45,anchor=north,
    shift={(-0mm,0mm)}]{$e_{23} : \,\,2,4|3$} (23);
    \node[punkt, right=0.5 of 23] (231) {$1,2|3$};
    \node[punkt, right=of 231,] (341) {$1,4|3$};
    \path[select] (231) edge node[anchor=south,above, shift={(0,5mm)}]{$e_{31} :
    2,4|1,3$} (341);
    \node[above=0.1 of 1,   shift={(18mm,0)}]{  \underline{$G_1/{T_1}$}};
    \node[above=0.3 of 13,  shift={(18mm,0)}]{ \underline{$G_2/{T_2}$}};
    \node[above=0.3 of 231, shift={(18mm,0)}]{\underline{$G_3/{T_3}$}};
  \end{tikzpicture}
}
\begin{equation*}
    c_{1234} = 
    \underbrace{
    \underbrace{c_{13|\emptyset}}_{e_{11}}
    \underbrace{c_{23|\emptyset}}_{e_{12}}
    \underbrace{c_{34|\emptyset}}_{e_{13}}}_{T_1}
    \underbrace{
    \underbrace{c_{12|3}}_{e_{21}}
    \underbrace{c_{14|3}}_{e_{22}}}_{T_2}
    \underbrace{
      \underbrace{c_{24|13}}_{e_{31}}}_{T_3}
\end{equation*}
  \end{center}
\caption[Vine factorization example]{Example of the hierarchical construction
of a vine factorization of a copula density $c(u_1,u_2,u_3,u_4)$. The edges
selected to form each tree are highlighted in bold. Conditioned and
conditioning sets for each node and edge are shown as $C(e) | D(e)$.}
\label{fig:vine_construction}
\end{figure}

\begin{example}[Construction of a four-dimensional regular vine]

  Assume access to a sample $U = \{(u_{1,i}, u_{2,i}, u_{3,i},
  u_{4,i})\}_{i=1}^n$, drawn iid from some copula $c_{1234}(u_1, u_2, u_3, u_4)$.
  
  \begin{enumerate}
    \item Our starting point is the four-dimensional complete graph, denoted by
    $G_1$. The graph $G_1$ has set of nodes $N_1$, one node per random variable
    $\bm u_i$, and one edge per bivariate unconditional copula $c_{ij}(u_i,u_j)$. See
    Figure~\ref{fig:vine_construction}, left.
    \item To construct the first tree in the vine, $T_1$, we give each edge in
    $G_1$ a weight equal to an empirical estimate of Kendall's $\tau$ between
    the variables connected by the edge. For example, we assign the edge
    $e_{11}$ in $G_1$ a weight of $\hat{\tau}(u_1,u_3)$. Using the edge
    weights, we infer the maximum spanning tree $T_1$. Assume that $E_1 = \{
      e_{11}, e_{12}, e_{13}\}$ are the edges of the maximum spanning tree
      $T_1$. Then
    \begin{itemize}
      \item $e_{11} = \{1,3\}$ owns $N(e_{11}) = \{1,3\}$, $D(e_{11}) =
      \emptyset$ and $C(e_{11}) = \{1,3\}$, and produces the copula
      $c_{13}(u_1,u_3)$ in the vine factorization.
      \item $e_{12} = \{2,3\}$ owns $N(e_{12}) = \{2,3\}$, $D(e_{12}) =
      \emptyset$ and $C(e_{12}) = \{2,3\}$, and produces the copula
      $c_{23}(u_2,u_3)$ in the vine factorization.
      \item $e_{13} = \{3,4\}$ owns $N(e_{13}) = \{3,4\}$, $D(e_{13}) =
      \emptyset$ and $C(e_{13}) = \{3,4\}$, and produces the copula
      $c_{34}(u_3,u_4)$ in the vine factorization.
    \end{itemize}
    The edges in $E_1$ are highlighted in bold in the left-hand side of
    Figure~\ref{fig:vine_construction}. The parametric copulas $c_{13}$,
    $c_{23}$ and $c_{34}$ can belong to any of the families presented in
    Section \ref{sec:parametric}, and their parameters can be chosen via
    maximum likelihood or Kendall's $\tau$ inversion on the available data.
  
    \item The next tree $T_2$ is be the maximum spanning tree of a graph $G_2$,
    constructed by following the rules in Definition~\ref{def:rvine}.
    That is, $G_2$ has node set $N_1 := E_1$ and set of edges formed by pairs of
    edges in $E_1$ sharing a common node from $N_1$.
    
    \item To assign a weight to the edges $\{e_{21}, e_{22}, e_{23}\}$ in $G_2$, we
    need samples of the conditional variables $\{\bm u_{1|3}, \bm u_{2|3}, u_{4|3}\}$,
    where $\bm u_{i|j} = P(\bm u_i|\bm u_j)$. These samples can be obtained using the original
    sample $U$ and the recursive equation (\ref{eq:recursive_conditional}).
    For instance, to obtain the samples for $\bm u_{1|3}$, use 
    \begin{equation*} u_{1|3,i} = \frac{\partial
    C_{1,3}(u_{1,i},u_{3,i})}{\partial u_{3,i}}.
    \end{equation*}
  
    Once we have computed the empirical estimate of Kendall's $\tau$ on these
    new conditioned samples, we can assign a weight to each of the edges in
    $G_2$ and infer a second maximum spanning tree $T_2$. Let us assume that
    $E_2 = \{e_{21}, e_{22}\}$ are the edges forming the maximum $T_2$
    (Figure~\ref{fig:vine_construction}, middle).
  
    \item The edges of $T_2$ represent bivariate \emph{conditional} copulas.
    Using Definition~\ref{def:sets}, we obtain
    \begin{itemize}
      \item $e_{21} = \{e_{11},e_{12}\}$ owns $N(e_{21}) = \{1,2,3\}$, $D(e_{21})
      = \{3\}$ and $C(e_{21})=\{1,2\}$ and produces the copula
      $c_{1,2|3}(u_{1|3},u_{2|3})$ in the vine factorization.
      \item $e_{22} = \{e_{11},e_{13}\}$ owns $N(e_{22}) = \{1,3,4\}$, $D(e_{22})
      = \{3\}$ and $C(e_{22})=\{1,4\}$ and produces the copula
      $c_{1,4|3}(u_{1|3},u_{4|3})$ in the vine factorization.
    \end{itemize}
    \item We repeat this procedure until we have built $d-1$ trees. In our
    example, we compute a third and last graph $G_3$, from which we estimate a
    third and last maximum spanning tree $T_3$
    (Figure~\ref{fig:vine_construction}, right-hand side). The corresponding
    conditional copula ($c_{24|13}$) is the final factor of the overall vine
    factorization, as depicted in the bottom part of
    Figure~\ref{fig:vine_construction}.
  \end{enumerate}%
\end{example}

\subsubsection{Model truncation}
In the presence of high-dimensional data, it may be computationally prohibitive
to build the $d-1$ trees and $d(d-1)/2$ bivariate copulas that form a complete
vine decomposition and specify the full copula density.  Similarly, when using
finite samples, the curse of dimensionality calls for a large amount of data to efficiently
model the higher-order dependencies described in the copulas from the last trees of
the factorization.

We can address both of these issues by truncating the vine structure, that is,
stopping the construction process after building $d'< d-1$ trees. A truncated
vine with $d'$ trees assumes independence in the conditional interactions
described by the ignored trees $T_{d'+1}, \ldots T_{d-1}$. A truncated vine has
a valid density function because the density of the independent copula is
constant and equal to one (Equation \ref{eq:indepCop}).  This allows to control
the complexity of vine density estimates given a computational budget,
dimensionality of the modeled random variable, and size of its sample. This is an attractive property of product
models that contrasts mixture models: by structure, mixture models necessarily model all
the dependencies at once.

When should we truncate a vine? This is yet another model selection task, which
can be addressed by monitoring the log-likelihood on some validation data as we
add more trees to the vine hierarchy.  For example, we can discard the last
built tree if the validation log-likelihood does not improve when adding the
corresponding copulas to the vine factorization.

\subsubsection{Model limitations and extensions}\label{sec:vine-limitations}

We now identify two major limitations of vine models, and propose novel
solutions to address them based on the material introduced in Sections
\ref{sec:conditional} and \ref{sec:nonparametric}. 

\paragraph{Simplification of conditional dependencies}

Because of the challenges involved in estimating conditional copulas, the
literature on vines has systematically ignored the effect of the conditioning
variables on the bivariate copulas participating in the vine factorization
\citep{Bedford01, Bedford02,Aas09,Kurowicka11,Dissmann13}. This means that the
influence of the variables in the conditioning set $D_{ij}$ on each copula
$c_{ij|D_{ij}}$ is only incorporated through the conditional cdfs
$P_{u_{ij}|D_{ij}}$ and $P_{v_{ij}|D_{ij}}$.  That is, the dependence of the
copula function $c_{ij|D_{ij}}$ on $D_{ij}$ is ignored.  This results in the
simplified densities 
  \begin{align*}
    c_{ij|D_{ij}}&(P_{u_{ij}|D_{ij}}(u_{ij}|D_{ij}),P_{v_{ij}|D_{ij}}(v_{ij}|D_{ij})|D_{ij})
    \approx\\
    &c_{ij}(P_{u_{ij}|D_{ij}}(u_{ij}|D_{ij}),P_{v_{ij}|D_{ij}}(v_{ij}|D_{ij}))\,.
  \end{align*}
This approximation is the \emph{vine simplifying
assumption} \citep{Haff10}.  \citet{Acar12} argues that this approximation may
be too crude when modeling real-world phenomena, and proposes a solution to
incorporate the conditioning influence of scalar random variables in the second
tree of a vine. However, the question of how to generally describe conditional
dependencies across all the trees of a vine remains open. 

To address this issue, we propose to model vine conditional dependencies
using the novel Gaussian process conditional copulas described in Section
\ref{sec:conditional}. The same ideas apply to the construction of
\emph{conditional vine models}, that is, vines conditioned to some set of exogenous variables.
In Section
\ref{sec:conditional_experiments} we conduct a variety of experiments that
demonstrate the advantages of modeling the previously ignored conditional dependencies in vine
decompositions.

\paragraph{Strong parametric assumptions}
Throughout the literature, vines restrict their bivariate copulas to belong to
a parametric family. This has the negative consequence that vines are not
universal density estimators like, for example, Gaussian mixture models. To
address this issue, we propose to use the described nonparametric bivariate
copulas and its novel conditional distribution rules from Section
\ref{sec:nonparametric} to construct more flexible vine distributions. Section
\ref{sec:adaptation} uses the proposed nonparametric vines to address
semisupervised domain adaptation problems on a variety of real-world data.

\section{Numerical simulations}\label{sec:copexps}
We present two series of numerical experiments. First, we
evaluate the improvements obtained by incorporating conditional dependencies
into the copulas forming a vine. For this we use the extension
proposed in Section~\ref{sec:conditional}. 
In the second series of experiments, we analyze vine density estimates when we
allow nonparametric copulas to participate in the factorization, built as in
Section~\ref{sec:nonparametric}. We illustrate this by using 
nonparametric vines to address the problem of semisupervised domain adaptation.

\subsection{Conditional density estimation}
\label{sec:conditional_experiments}

We evaluate the performance of the proposed method for the estimation of vine
copula densities with full conditional dependencies, as described in Section
\ref{sec:conditional}. Because our method relies on Gaussian processes, we
call it GPRV.  We compare with two other methods: SRV, a vine model
based on the simplifying assumption which ignores conditional
dependencies in the bivariate copulas, and NNRV, a vine model based on the
nearest-neighbour method of \citet{Acar12}.  This latter model can
only handle conditional dependencies with respect to a single scalar variable.
Therefore, we can only evaluate the performance of NNRV in vine
models with two trees, since additional trees would require to account for
multivariate conditional dependencies.

In all the experiments, we use 20 pseudo-inputs in the sparse Gaussian process
approximation described in Section~\ref{sec:fitc}. The Gaussian processes
kernel parameters and pseudo-input locations are tuned using approximate
Bayesian model selection, that is, by maximizing the EP estimate of the
marginal likelihood.  The mean of the GP prior is set to be constant and equal
to $\Phi^{-1}((\hat{\tau}_{MLE} + 1) / 2)$, where $\hat{\tau}_{MLE}$ is the
maximum likelihood estimate of $\tau$ given the training data.  In NNRV, the
bandwidth of the Epanechnikov kernel is selected by running a leave-one-out
cross validation search using a 30-dimensional log-spaced grid ranging from
$0.05$ to $10$.  To simplify the experimental setup, we focus on regular vines
formed by bivariate Gaussian copulas.  The extension of
the proposed approach to incorporate different parametric families of bivariate
copulas is straightforward \citep{Hernandez-Lobato13}. We use the empirical
copula transformation to obtain data with uniform marginal distributions, as
described in Section \ref{sec:copdata}.

\subsubsection{Synthetic data}
We sample synthetic scalar variables $\bm x$, $\bm y$ and $\bm z$ according to the
following generative process.  First, we sample $\bm z$ uniformly from the interval
$[-6,6]$ and second, we sample $\bm x$ and $\bm y$ given $\bm z$ from a bivariate Gaussian
distribution with zero mean and covariance matrix given by $\text{Var}(\bm x) =
\text{Var}(\bm y) = 1$ and $\text{Cov}(\bm x, \bm y \given \bm z) = 3 / 4 \sin(\bm z)$.  We sample a total
of 1000 data points and choose 50 subsamples of size 100 to infer a vine model
for the data using SRV, NNRV and GPRV.  The first row of the left-hand
  side of Figure~\ref{fig:little_exps} shows the average test log-likelihoods
  on the remaining data points. In these experiments, GPRV shows the best
  performance.

Figure \ref{fig:syn} displays the true value of the function $g$ that maps
$u_3$ to the Kendall's $\tau$ value of the conditional copula
$c_{12|3}(P(u_1|u_3),P(u_2|u_3)|u_3)$, where $u_1$, $u_2$ and $u_3$ are the
empirical cumulative probability levels of the samples generated for $\bm x$,
$\bm y$ and $\bm z$, respectively. We also show the approximations of $g$
generated by GPRV and NNRV. In this case, GPRV does a better job than NNRV at
approximating the true $g$.

\begin{figure}[t]
  \begin{center}
    \includegraphics[width=0.7\linewidth]{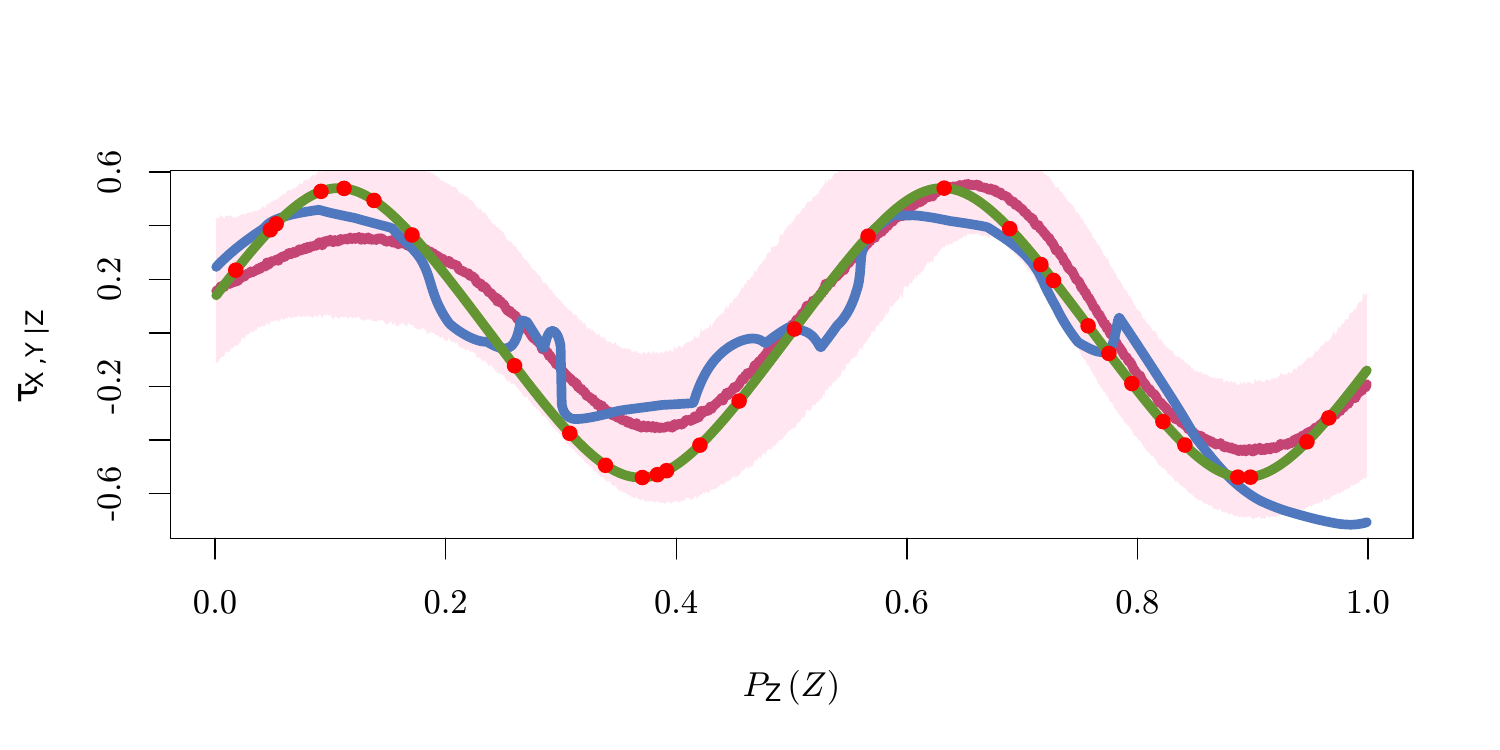}
    \vspace{-0.8 cm}
  \end{center}
  \caption[Results for GPRV synthetic experiments]{In green, the true
  function $g$ that maps $u_3$ to $\tau$. In red, the GPRV approximation. In
  blue, the NNRV approximation.  In red, the uncertainty of the GPRV
  prediction, plus-minus one standard deviation.  In red dots, the training
  samples of $u_3$.}
  \label{fig:syn}
\end{figure}

\subsubsection{Real data}
We further compare the performance of SRV, NNRV and GPRV on an array of
real-world datasets. For a detailed description of the datasets, consult
\citet{dlp-gp}.  For each dataset, we generate 50 random partitions of the data
into training and test sets, each containing half of the available data. Here,
each method learns from each training set, and evaluate its log-likelihood on
the corresponding test set (higher is better). Table \ref{table:realexps} shows
the test log-likelihood for SRV and GPRV, when using up to $T$ trees in the
vine, $1 \leq T \leq d-1$, where $d$ is the number of variables in the data.
In general, taking into account conditional dependencies in
bivariate copulas leads to superior predictive performance.  Also, we often
find that improvements get larger as we increase the number of trees in
the vines.  However, in the ``stocks'' and ``jura'' datasets the simplifying
assumption seems valid.  The left-hand side of Figure~\ref{fig:little_exps}
shows a comparison between NNRV and GPRV, when restricted to vines 
of two trees.  In these experiments, NNRV is most of the times outperformed by
GPRV.  Figure~\ref{fig:little_exps} shows the use of GPRV to discover
scientifically interesting features, revealed by learning spatially varying
correlations.  In this case, the blue region in the plot corresponds to the
Pyrenees mountains. Furthermore, one could examine the learned Gaussian process
models to interpret the shape and importance of each of the estimated
conditional dependencies.

\begin{table}
  \begin{center}
  \resizebox{\linewidth}{!}{
  \begin{tabular}{|l|c|c|c|}
  \hline
  \textbf{data} & \textbf{T} & \textbf{SRV} & \textbf{GPRV}\\
  \hline
    \multirow{9}{*}{cloud}   & 1  & $\bf{7.860 \pm 0.346}$ & $\bf{ 7.860 \pm 0.346}$\\
                             & 2  & $   8.899 \pm 0.334 $ & $\bf{ 9.335 \pm 0.348}$\\
                             & 3  & $   9.426 \pm 0.363 $ & $\bf{10.053 \pm 0.397}$\\
                             & 4  & $   9.570 \pm 0.361 $ & $\bf{10.207 \pm 0.415}$\\
                             & 5  & $   9.644 \pm 0.357 $ & $\bf{10.332 \pm 0.440}$\\
                             & 6  & $   9.716 \pm 0.354 $ & $\bf{10.389 \pm 0.459}$\\
                             & 7  & $   9.783 \pm 0.361 $ & $\bf{10.423 \pm 0.463}$\\
                             & 8  & $   9.790 \pm 0.371 $ & $\bf{10.416 \pm 0.459}$\\
                             & 9  & $   9.788 \pm 0.373 $ & $\bf{10.408 \pm 0.460}$\\\hline
    \multirow{8}{*}{glass}   & 1  & $\bf{0.827 \pm 0.150}$ & $\bf{ 0.827 \pm 0.150}$\\
                             & 2  & $  {1.206 \pm 0.259}$ & $\bf{ 1.264 \pm 0.303}$\\
                             & 3  & $  {1.281 \pm 0.251}$ & $\bf{ 1.496 \pm 0.289}$\\
                             & 4  & $  {1.417 \pm 0.251}$ & $\bf{ 1.740 \pm 0.308}$\\
                             & 5  & $  {1.493 \pm 0.291}$ & $\bf{ 1.853 \pm 0.318}$\\
                             & 6  & $  {1.591 \pm 0.301}$ & $\bf{ 1.936 \pm 0.325}$\\
                             & 7  & $  {1.740 \pm 0.282}$ & $\bf{ 2.000 \pm 0.345}$\\
                             & 8  & $  {1.818 \pm 0.243}$ & $\bf{ 2.034 \pm 0.343}$\\\hline
    \multirow{6}{*}{jura}    & 1  & $\bf{1.887 \pm 0.153}$ & $\bf{ 1.887 \pm 0.153}$\\
                             & 2  & $  {2.134 \pm 0.164}$ & $\bf{ 2.151 \pm 0.173}$\\
                             & 3  & $  {2.199 \pm 0.151}$ & $\bf{ 2.222 \pm 0.173}$\\
                             & 4*  & $  {2.213 \pm 0.153}$ &$\bf{ 2.233 \pm 0.181}$\\
                             & 5*  & $  {2.209 \pm 0.153}$ &$\bf{ 2.215 \pm 0.185}$\\
                             & 6*  & $\bf{2.213 \pm 0.155}$ & $  { 2.197 \pm 0.189}$\\\hline
    \multirow{9}{*}{shuttle} & 1  & $\bf{1.487 \pm 0.256}$ & $\bf{ 1.487 \pm 0.256}$\\
                             & 2  & $  {2.188 \pm 0.314}$ & $\bf{ 2.646 \pm 0.349}$\\
                             & 3  & $  {2.552 \pm 0.273}$ & $\bf{ 3.645 \pm 0.427}$\\
                             & 4  & $  {2.782 \pm 0.284}$ & $\bf{ 4.204 \pm 0.551}$\\
                             & 5  & $  {3.092 \pm 0.353}$ & $\bf{ 4.572 \pm 0.567}$\\
                             & 6  & $  {3.284 \pm 0.325}$ & $\bf{ 4.703 \pm 0.492}$\\
                             & 7  & $  {3.378 \pm 0.288}$ & $\bf{ 4.763 \pm 0.408}$\\
                             & 8  & $  {3.417 \pm 0.257}$ & $\bf{ 4.761 \pm 0.393}$\\
                             & 9  & $  {3.426 \pm 0.252}$ & $\bf{ 4.755 \pm 0.389}$\\\hline
  \end{tabular}
  \hspace{0.3cm}
  \begin{tabular}{|l|c|c|c|}
  \hline
  \textbf{data} & \textbf{T} & \textbf{SRV} & \textbf{GPRV}\\
  \hline
    \multirow{8}{*}{weather} & 1  & $\bf{0.684 \pm 0.128}$ & $\bf{ 0.684 \pm 0.128}$\\
                             & 2  & $  {0.789 \pm 0.159}$ & $\bf{ 1.312 \pm 0.227}$\\
                             & 3  & $  {0.911 \pm 0.178}$ & $\bf{ 2.081 \pm 0.341}$\\
                             & 4  & $  {1.017 \pm 0.184}$ & $\bf{ 2.689 \pm 0.368}$\\
                             & 5  & $  {1.089 \pm 0.188}$ & $\bf{ 3.078 \pm 0.423}$\\
                             & 6  & $  {1.138 \pm 0.181}$ & $\bf{ 3.326 \pm 0.477}$\\
                             & 7  & $  {1.170 \pm 0.169}$ & $\bf{ 3.473 \pm 0.467}$\\
                             & 8  & $  {1.177 \pm 0.170}$ & $\bf{ 3.517 \pm 0.465}$\\\hline
    \multirow{5}{*}{stocks}  & 1  & $\bf{2.776 \pm 0.142}$ & $  { 2.776 \pm 0.142}$\\
                             & 2*  & $\bf{2.799 \pm 0.142}$ & $  { 2.785 \pm 0.146}$\\
                             & 3  & $\bf{2.801 \pm 0.142}$ & $  { 2.764 \pm 0.151}$\\
                             & 4  & $\bf{2.802 \pm 0.143}$ & $  { 2.742 \pm 0.158}$\\
                             & 5  & $\bf{2.802 \pm 0.141}$ & $  { 2.721 \pm 0.159}$\\\hline
    \multirow{11}{*}{housing} & 1  & $\bf{3.409 \pm 0.354}$ & $\bf{ 3.409 \pm 0.354}$\\
                             & 2  & $  {3.975 \pm 0.342}$ & $\bf{ 4.487 \pm 0.386}$\\
                             & 3  & $  {4.128 \pm 0.363}$ & $\bf{ 4.953 \pm 0.425}$\\
                             & 4  & $  {4.250 \pm 0.376}$ & $\bf{ 5.307 \pm 0.458}$\\
                             & 5  & $  {4.386 \pm 0.380}$ & $\bf{ 5.541 \pm 0.498}$\\
                             & 6  & $  {4.481 \pm 0.399}$ & $\bf{ 5.691 \pm 0.516}$\\
                             & 7  & $  {4.576 \pm 0.422}$ & $\bf{ 5.831 \pm 0.529}$\\
                             & 8  & $  {4.666 \pm 0.412}$ & $\bf{ 5.934 \pm 0.536}$\\
                             & 9  & $  {4.768 \pm 0.399}$ & $\bf{ 6.009 \pm 0.516}$\\
                             & 10 & $  {4.838 \pm 0.382}$ & $\bf{ 6.084 \pm 0.520}$\\
                             & 11 & $  {4.949 \pm 0.362}$ & $\bf{ 6.113 \pm 0.525}$\\\hline
  \end{tabular}
  }
  \vskip 1cm
  \includegraphics[width=\textwidth]{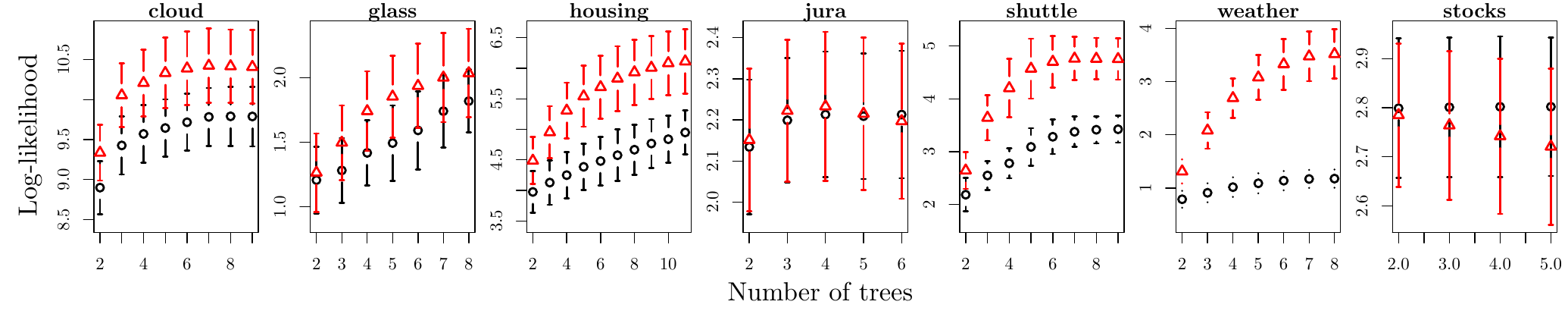}
  \end{center}
  \caption[Results for SRV and GPRV]{Top: Average test log-likelihood and
  standard deviations for SRV and GPRV on real-world datasets (higher is
  better). Asterisks denote results not statistically significant with respect to a paired
  Wilcoxon test with $\text{p--value} =10^{-3}$. Same information depicted as a
  plot; red triangles correspond to GPRV, black circles to SRV. For results comparing GPRV to NNRV, see Table~\ref{fig:little_exps}.}
  \label{table:realexps}
\end{table}

\begin{figure}
  \centering
  \begin{minipage}[c]{0.6\linewidth}
    \resizebox{\linewidth}{!}{
    \begin{tabular}{|l|c|c|c|}
    \hline
    \textbf{data} & \textbf{SRV} & \textbf{NNRV} & \textbf{GPRV}\\ \hline
    synthetic   & $-0.005 \pm 0.012$ & $0.101 \pm 0.162 $ & $  \bm{ 0.298 \pm 0.031}$ \\
    uranium     & $ 0.006 \pm 0.006$ & $0.016 \pm 0.026$   & $ \bm{ 0.022 \pm 0.012}$ \\
    cloud       & $ 8.899 \pm 0.334$ & $9.013 \pm 0.600$   & $ \bm{ 9.335 \pm 0.348}$ \\
    glass       & $ 1.206 \pm 0.259$ & $0.460 \pm 1.996$   & $ \bm{ 1.264 \pm 0.303}$ \\
    housing     & $ 3.975 \pm 0.342$ & $4.246 \pm 0.480$   & $ \bm{ 4.487 \pm 0.386}$ \\
    jura        & $ 2.134 \pm 0.164$ & $2.125 \pm 0.177$   & $ \bm{ 2.151 \pm 0.173}$ \\
    shuttle     & $ 2.552 \pm 0.273$ & $2.256 \pm 0.612$   & $ \bm{ 3.645 \pm 0.427}$ \\
    weather     & $ 0.789 \pm 0.159$ & $0.771 \pm 0.890$   & $ \bm{ 1.312 \pm 0.227}$ \\
    stocks      & $ \bm{2.802 \pm 0.141}$ & $2.739 \pm 0.155$   & $ { 2.785 \pm 0.146}$ \\
    \hline
    \end{tabular}
  }
  \end{minipage}
  \hfill
  \begin{minipage}[c]{0.35\linewidth}%
    \includegraphics[width=\linewidth]{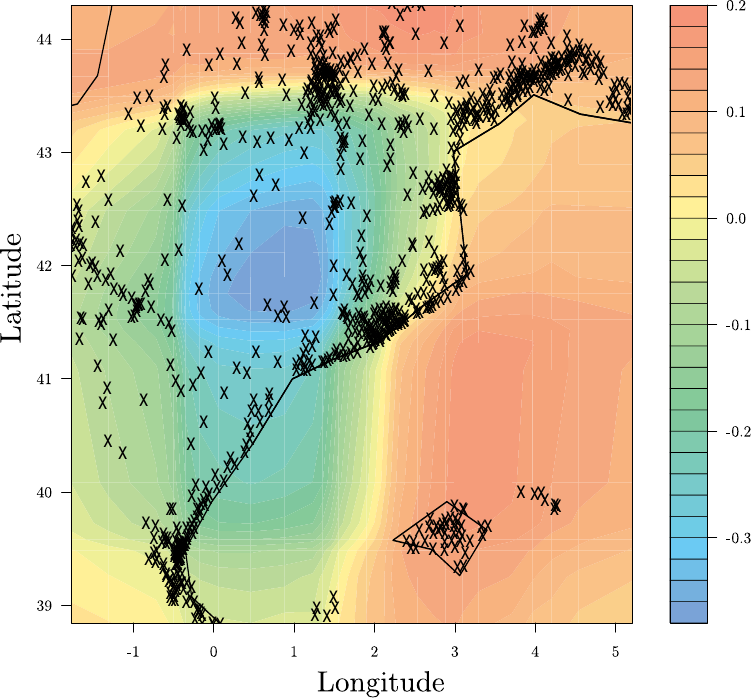}
  \end{minipage}%
  \caption[Result for GPRV spatial conditioning experiment]{Left: Average
  test log-likelihood and standard deviations for all methods and datasets when
  limited to 2 trees in the vine (higher is better).  Right: Kendall's $\tau$
  correlation between \emph{atmospheric pressure} and \emph{cloud percentage
  cover} (color scale) when conditioned to \emph{longitude} and
  \emph{latitude}.}
  \label{fig:little_exps}
\end{figure}

\subsection{Vines for semisupervised domain adaptation}\label{sec:adaptation}
We study the use of nonparametric bivariate
copulas in single-tree regular vines, using their novel conditional distributions from 
Section~\ref{sec:nonparametric}. We call this model Non-Parametric Regular Vine
(\textsc{NPRV}). The density estimates in this
section are the product of the one-dimensional marginal densities
and a vine copula decomposition
\begin{equation}\label{eq:vine2}
  p(\bm x) = \prod_{i=1}^d p_i(x_i) \prod_{T_i \in \mathcal{V}} \prod_{e_{ij}
  \in E_i}
  c_{ij}(P_{u_{ij}|D_{ij}}(u_{ij}|D_{ij}),P_{v_{ij}|D_{ij}}(v_{ij}|D_{ij})).
\end{equation}

\begin{remark}[Domain adaptation problems]
\emph{Domain adaptation} \citep{shai} aims at transferring knowledge
between different but related learning tasks.  Generally, the goal of domain
adaptation is to improve the learning performance on a \emph{target task}, by using
knowledge obtained when solving a different but related \emph{source task}.
\end{remark}

In the following, we assume access to large amounts of data sampled from some
source distribution $p_s$. However, a much scarcer sample is available to
estimate the target density $p_t$.  Given the data available for both tasks,
our objective is to build a good estimate for the density $p_t$.  To do so, we
assume that $p_t$ is a modified version of $p_s$.  In particular, we assume
that the transformation from $p_s$ to $p_t$ takes two steps.  First, $p_s$ follows a vine
factorization, as in Equation \ref{eq:vine2}. Second, we modify a small 
subset of the factors in $p_s$, either marginals or bivariate copulas, to
obtain $p_t$. This is equivalent to assuming that only a small amount of
marginal distributions or dependencies in the joint distribution change across
domains, while the structure of the trees forming the vine remains constant.

All we need to address the adaptation across domains is to reconstruct the vine
representation of $p_s$ using data from the source task, and then identify
which of the factors forming $p_s$ changed to produce $p_t$.  These factors are
re-estimated using data from the target task, when available.  Note that this
is a general domain adaptation strategy (subsuming \emph{covariate shift},
among others), and works in the unsupervised or semisupervised scenario (where
target data has missing variables, do not update factors related to those
variables).  To decide whether to re-estimate a given univariate marginal or
bivariate copula when adapting $p_s$ to $p_t$, we use the \emph{Maximum Mean
Discrepancy} test, or MMD \citep{mmd}. 

We analyze \textsc{NPRV} in a series of domain adaptation nonlinear regression
problems on real data. For a more detailed description about the experimental
protocol and  datasets, consult the supplementary material of
\citet{dlp-ssl}.  During our experiments, we compare \textsc{NPRV} with
different benchmark methods.  The first two methods, \textsc{GP-Source} and
\textsc{GP-All}, are baselines.  They are two Gaussian Process (GP) methods,
the first one trained only with data from the source task, and the second one
trained with the normalized union of data from both source and target problems.
The other five methods are state-of-the-art domain adaptation techniques:
including \textsc{Daume} \citep{daume}, \textsc{SSL-Daume}  \citep{ssldaume},
\textsc{ATGP} \citep{bin}, Kernel Mean Matching or \textsc{KMM} \citep{kmm},
and Kernel unconstrained Least-Squares Importance Fitting or \textsc{KuLSIF}
\citep{kulsif}. Besides \textsc{NPRV}, we also include in the experiments its
unsupervised variant, \textsc{UNPRV}, which ignores any labeled data from the
target task and adapts only vine factors depending on the input features. 
For training, we randomly sample 1000 data points for both source and target
tasks, where all the data in the source task and 5\% of the data in the target
task have labels.  The test set contains 1000 points from the target task.
Table \ref{table:da_exps} summarizes the average test normalized mean square
error (NMSE) and corresponding standard deviations for each method in each
dataset across 30 random repetitions of the experiment.  The proposed methods
obtain the best results in 5 out of 6 cases. The two last two rows in
Table \ref{table:da_exps} show the average number of factors (marginals or bivariate
copulas) updated from source to target task, according to the MMD test.
\begin{table}
  \begin{center}
  \resizebox{0.85\linewidth}{!}
  {
  \begin{tabular}{|l|c|c|c|c|c|c|}
    \hline
    \textbf{data} & \textbf{wine}& \textbf{sarcos}& \textbf{rocks-mines}& \textbf{hill-valleys}&\textbf{axis-slice}&\textbf{isolet}  \\
    \hline
    No. of variables & \textbf{12}& \textbf{21}& \textbf{60}& \textbf{100}&\textbf{386}&\textbf{617}  \\ \hline
    \bf{GP-Source} &    0.86 $\pm$ 0.02  &    1.80 $\pm$ 0.04  &    0.90 $\pm$ 0.01  &    1.00 $\pm$ 0.00  &    1.52 $\pm$ 0.02  &    1.59 $\pm$ 0.02 \\ 
    \bf{GP-All   } &    0.83 $\pm$ 0.03  &    1.69 $\pm$ 0.04  &    1.10 $\pm$ 0.08  &    0.87 $\pm$ 0.06  &    1.27 $\pm$ 0.07  &    1.58 $\pm$ 0.02 \\ 
    \bf{Daume    } &    0.97 $\pm$ 0.03  &    0.88 $\pm$ 0.02  &    0.72 $\pm$ 0.09  &    0.99 $\pm$ 0.03  &    0.95 $\pm$ 0.02  &    0.99 $\pm$ 0.00 \\ 
    \bf{SSL-Daume} &    0.82 $\pm$ 0.05  &    0.74 $\pm$ 0.08  &    0.59 $\pm$ 0.07  &    0.82 $\pm$ 0.07  &    0.65 $\pm$ 0.04  &    0.64 $\pm$ 0.02 \\ 
    \bf{ATGP     } &    0.86 $\pm$ 0.08  &    0.79 $\pm$ 0.07  & \bf{0.56 $\pm$ 0.10} &    0.15 $\pm$ 0.07  &    1.00 $\pm$ 0.01  &    1.00 $\pm$ 0.00 \\ 
    \bf{KMM      } &    1.03 $\pm$ 0.01  &    1.00 $\pm$ 0.00  &    1.00 $\pm$ 0.00  &    1.00 $\pm$ 0.00  &    1.00 $\pm$ 0.00  &    1.00 $\pm$ 0.00 \\ 
    \bf{KuLSIF   } &    0.91 $\pm$ 0.08  &    1.67 $\pm$ 0.06  &    0.65 $\pm$ 0.10  &    0.80 $\pm$ 0.11  &    0.98 $\pm$ 0.07  &    0.58 $\pm$ 0.02 \\ 
    \bf{NPRV     } & \bf{0.73 $\pm$ 0.07}& \bf{0.61 $\pm$ 0.10} &    0.72 $\pm$ 0.13  & \bf{0.15 $\pm$ 0.07} &   {0.38 $\pm$ 0.07} &    0.46 $\pm$ 0.09 \\ 
    \bf{UNPRV    } &    0.76 $\pm$ 0.06  &    0.62 $\pm$ 0.13  &    0.72 $\pm$ 0.15  &    0.19 $\pm$ 0.09  & \bf{0.37 $\pm$ 0.07} & \bf{0.42 $\pm$ 0.04} \\
    \hline
    Av. Ch. Mar. & 10 & 1 & 38 & 100 & 226 & 89\\
    Av. Ch. Cop. & 5  & 8 & 49 & 34 & 155 & 474\\
    \hline
  \end{tabular}
  }
  \caption[Domain adaptation experiments]{NMSE for all domain adaptation algorithms and datasets.}
  \label{table:da_exps}
\end{center}
\end{table}

  \chapter{Discriminative dependence}\label{chapter:discriminative-dependence}

\vspace{-1.25cm}
  \emph{This chapter contains novel material.
  Section~\ref{sec:component-analysis} presents a framework for nonlinear
  component analysis based on random features \citep{dlp-rca}, called
  Randomized Component Analysis (RCA). We exemplify RCA by proposing 
  Randomized Principal Component Analysis (RPCA, Section~\ref{sec:rpca}), and
  Randomized Canonical Correlation Analysis (RCCA, Section~\ref{sec:rcca}).
  Based on RCA and the theory of copulas, we introduce a measure of dependence
  termed the Randomized Dependence Coefficient (RDC, Section~\ref{sec:rdc},
  \citet{dlp-rdc}).
  We give theoretical guarantees for RPCA, RCCA, and RDC by using recent matrix
  concentration inequalities.
  We illustrate the effectiveness of the proposed methods in
  a variety of numerical simulations (Section~\ref{sec:rca-experiments}).}
\vspace{1.25cm}

\noindent The previous chapter studied {generative} models of dependence: those that
estimate the entire dependence structure of some multivariate data, and are
able to synthesize new samples from the data generating distribution.  We have
seen that generative modeling is intimately linked to density estimation, which
is a general but challenging learning problem.  General, because one can solve
many other tasks of interest, such as regression and classification, as a byproduct of 
density estimation. Challenging, because it requires the estimation of all the
information contained in data.

Nevertheless, in most situations we are not interested in describing the whole
dependence structure cementing the random variables under study, but in
summarizing some particular aspects of it, which we believe useful for
subsequent learning tasks. Let us give three examples.  First, \emph{component
analysis} studies how to boil down the variables in some high-dimensional data
to a small number of explanatory components.  These explanatory components
throw away some of the information from the original data, but retain
directions containing most of the variation in data.  Second, \emph{dependence
measurement}, given two random variables, quantifies to what degree they depend
on each other. Third, \emph{two-sample-testing} asks: given two random samples,
were they drawn from the same distribution? These three tasks do not require
estimating the density of the data in its entirety.  Instead, these
\emph{discriminative dependence methods} summarize the dependence structure of
a multivariate data set into a low-dimensional statistic that answers the
question at hand.

Let us examine the state-of-the-art more concretely. Two of the most popular
discriminative dependence methods are Principal Component Analysis (PCA) by
\citet{Pearson01} and Canonical Correlation Analysis (CCA) by
\citet{Hotelling36}. Both have played a crucial role in multiple applications
since their conception over a century ago.  Despite their great successes, an
impediment of these classical discriminative methods for modern data science is
that they only reveal linear relationships between the variables under study.
But linear component analysis methods, such as PCA and CCA, operate in terms of
inner products between the examples contained in the data at hand.  This makes
kernels one elegant way to extend these algorithms to capture nonlinear
dependencies. Examples of these extensions are Kernel PCA or KPCA
\citep{Schoelkopf99}, and Kernel CCA or KCCA \citep{Lai00,Bach02}.
Unfortunately, when working on $n$ data, kernelized discriminative methods
require the construction and inversion of $n \times n$ kernel matrices.
Performing these operations takes $O(n^3)$ time, a prohibitive computational
requirement when analyzing large data.

In this chapter, we propose the use of random features
(Section~\ref{sec:random-mercer-features}) to overcome the limits of linear
component analysis algorithms and the computational burdens of their kernelized
extensions. We exemplify the use of random features in three discriminative
dependence tasks: component analysis (Section~\ref{sec:component-analysis}),
dependence measurement (Section~\ref{sec:dependence-measures}), and two-sample
testing (Section~\ref{sec:two-sample-tests}). The algorithms presented in this
chapter come with learning rates and consistency guarantees
(Section~\ref{sec:rca-proofs}), as well as a performance evaluation on multiple
applications and real-world data (Section~\ref{sec:rca-experiments}).  Since
our framework and its extensions are based on random features, we call it 
Randomized Component Analysis, or RCA.

Before we start, let us introduce the main actors of this chapter in the
following definition.

\begin{definition}[Assumptions on discriminative dependence]\label{remark:rca}
  As usual, we consider data $\{x_1, \ldots, x_n\} \sim P^n(\bm x)$,
  where $x_i \in \X$ for all $1 \leq i \leq n$.  Using this data, the central
  object of study throughout this chapter is the spectral norm $\|\hat{K}-K\|$,
  where $K \in \R^{n\times n}$ is a full-rank kernel matrix, and $\hat{K} \in
  \R^{n\times n}$ is a rank-$m$ approximation of $K$ 
  (Section~\ref{sec:matrices}).  The full-rank kernel matrix $K$ has entries
  $K_{i,j} = k(x_i,x_j)$, where $k$ is a real-valued, shift-invariant $(k(x,x')
  = k(x-x',0))$, and $L_k$-Lipschitz kernel,
  \begin{equation*}
    |k(\delta,0)-k(\delta',0)| \leq L_k\|\delta-\delta'\|,
  \end{equation*} 
  also satisfying the \emph{boundedness} condition $|k(x,x')| \leq 1$ for all $x,x'\in\X$.  On the
  other hand, the approximate kernel matrix $\hat{K}$ is 
  \begin{align}
    z_i &= \sqrt{\frac{2}{m}} \left({\cos}(\dot{w_i}{x_1}+b_i), \ldots,
    {\cos}(\dot{w_i}{x_n}+b_i)\right)^\top \in \R^n,\label{eq:theaug}\\
    \hat{K}_i &= z_iz_i^\top,\nonumber\\
    \hat{K} &= \frac{1}{m} \sum_{i=1}^m  z_i  z_i^\top\nonumber,
  \end{align}
  where $z_i \in \R^{n \times 1}$, with $\|z_i\|^2 \leq B$, is the $i$-th random feature of the $n$
  examples contained in our training data 
  (Section~\ref{sec:random-mercer-features}). We call 
  $Z\in\R^{n\times m}$ the matrix with rows $z_1, \ldots, z_m$. 
  
  Finally, some parts of this chapter will consider data from two
  random variables
  \begin{equation*}
    \{(x_1, y_1), \ldots, (x_n, y_n)\} \sim P^n(\bm x, \bm y).
  \end{equation*}
  When this is the case, we will build full rank kernel matrices $K_x$ and
  $K_y$ for each of the two random variables, and their respective rank-$m$
  approximations $\hat{K}_x$ and $\hat{K}_y$.
\end{definition}

\section{Randomized Component analysis}\label{sec:component-analysis}
Component analysis relates to the idea of \emph{dimensionality reduction}:
summarizing a large set of variables into a small set of factors able to
explain key properties about the original variables. Some examples of
component analysis algorithms include ``principal component analysis, factor
analysis, linear multidimensional scaling, Fisher’s linear discriminant
analysis, canonical correlations analysis, maximum autocorrelation factors,
slow feature analysis, sufficient dimensionality reduction, undercomplete
independent component analysis, linear regression, and distance metric
learning'', all of these eloquently reviewed in  \citep{Cunningham15}.
Component analysis is tightly related to transformation generative models, in the
sense that both methods aim at extracting a set of explanatory factors from data.
Dimensionality reduction algorithms differ in what they call 
``the important information to retain about the original data''. During the
remainder of this section, we review two of the most widely used linear
dimensionality reduction methods, PCA and CCA, and extend them to model nonlinear
dependencies in a theoretically and computationally sustained way. 

\begin{remark}[Prior work on randomized component analysis]
  \citet{Achlio02} pioneered the use of randomized techniques to approximate
  kernelized component analysis, by suggesting three sub-sampling strategies to
  speed up KPCA. \citet{Avron13} used randomized Walsh-Hadamard transforms to
  adapt linear CCA to large datasets.  \citet{McWilliams13} applied the
  Nystr\"om method to CCA on the problem of semisupervised learning. 
\end{remark}

\subsection{Principal component analysis}\label{sec:rpca}
Principal Component Analysis or PCA \citep{Pearson01} is the orthogonal
transformation of a set of $n$ observations of $d$ variables $ X\in \Rnd$ into
a set of $n$ observations of $d$ uncorrelated \emph{principal components} $XF$
(also known as factors or latent variables). Principal components owe their
name to the following property: the first principal component captures the
maximum amount of variations due to linear relations in the data; successive
components account for the maximum amount of remaining variance in dimensions
orthogonal to the preceding ones. PCA is commonly used for dimensionality
reduction, assuming that the $d' < d$ principal components capture the core
properties of the data under study.  For a centered matrix of $n$ samples and
$d$ dimensions $ X \in \R^{n \times d}$, PCA requires computing the singular
value decomposition $ X = U \Sigma  F'$ (Section~\ref{sec:matrices}).  The top
$d'$ principal components are $XF_{1:d',:}$, where $ F_{1:d',:}$ denotes the
first $d'$ rows of $ F$.  PCA seeks to retain linear variations in the data, in
the sense that it minimizes the reconstruction error of the linear
transformation from the $d'$ principal components back to the original data.

\begin{remark}[History of PCA]
  Principal component analysis was first formulated over a century ago by
  \citet{Pearson01}. The method was independently discovered and advanced to
  its current form by \citet{Hotelling33}, who is also responsible for coining
  the term \emph{principal components}.  PCA has found a wide range of
  successful applications, including finance, chemistry, computer vision,
  neural networks, and biology, to name some. \citet{Jolliffe02} offers a
  modern account on PCA and its applications.
\end{remark}

One of the limitations of the PCA algorithm is that the recovered principal
components can only account for linear variations in data. This is a limiting
factor, as it may be the case that there exists interesting nonlinear patterns
hidden in the data, not contributing to the linear variance that
PCA seeks to retain. To address these limitations, \citet{Schoelkopf99}
introduced Kernel PCA or KPCA, an algorithm that leverages the kernel trick
(Section~\ref{sec:kernels}) to extract linear components in some
high-dimensional and nonlinear representation of the data. Computationally speaking,
KPCA performs the eigendecomposition of a $n \times n$ kernel matrix when
analyzing data sets of $n$ examples. Unfortunately, these operations
require $O(n^3)$ computation, a prohibitive complexity for
large data. But there is hope: in words of Joel Tropp, ``large data sets
tend to be redundant, so the kernel matrix also tends to be redundant. This
manifests in the kernel matrix being close to a low-rank matrix''. This quote
summarizes the motivation of RPCA, the first example of the RCA framework,
proposed next.

To extend the PCA algorithm to discover nonlinear principal components, while
avoiding the computational burden of KPCA, we propose Randomized PCA or RPCA
\citep{dlp-rca}. In particular, RPCA proceeds by
\begin{enumerate}
  \item maps the original data $ X := ( x_1, \ldots,  x_n) \in \R^{n\times d}$,
  into the random feature data $Z := (z_1,\ldots,z_n) \in \R^{n\times m}$. 
  \item performs PCA on the data $Z$.
\end{enumerate}
Therefore, RPCA approximates KPCA when the random features used by the former
approximate the kernel function used by the latter.  The principal components
obtained with RPCA are no longer linear transformations of the data, but
approximations to nonlinear transformations of the data living in the reproducing kernel
Hilbert Space $\H$. Computationally, approximating the covariance
matrix of $Z\in\R^{n\times m}$ dominates the time complexity of RPCA. This
operation has a time complexity $O(m^2 n)$ in the typical regime $n \gg m$,
which is competitive with the linear PCA complexity $O(d^2 n)$.

Since RPCA approximates KPCA, and the solution of KPCA relates to the spectrum
of $K$, we will study the convergence rate of 
$\hat{K}$ to $K$ in \emph{operator norm} as $m$ grows (see
Definition~\ref{remark:rca}). A bound about $\| \hat{K} -
K\|$ is quite valuable to our purposes: such bound simultaneously
controls the error in every linear projection of the approximation, that is:
\begin{equation*}
  \| \hat{K} - K \| \leq \varepsilon \Rightarrow | \mathrm{tr}(\hat{K}X) - \mathrm{tr}(KX) | \leq \varepsilon,
\end{equation*}
where $\|X\|_{S_1} \leq 1$ and $\|\cdot\|_{S_1}$ is the Schatten 1-norm
\citep{tropp}. Such bound also controls the whole spectrum of singular values
of our approximation $\hat{K}$, that is:
\begin{equation*}
  \| \hat{K} - K \| \leq \varepsilon \Rightarrow | \sigma_j(\hat{K}) - \sigma_j(K) | \leq
  \varepsilon,
\end{equation*}
for all $j = 1, \ldots, n$ \citep{tropp}.

\begin{theorem}[Convergence of RPCA]\label{thm:pca}
  Consider the assumptions from Definition~\ref{remark:rca}. Then,
  \begin{equation}\label{eq:pcaconc}
      \E{}{\| {\hat{\bm K}} -  K\|} \leq
      \sqrt{\frac{3B\|K\|\log
      n}{m}} + \frac{2B\log n}{m}.
  \end{equation}
\end{theorem}
\begin{proof}
  See Section~\ref{proof:thm:pca}.
\end{proof}
Theorem~\ref{thm:pca} manifests that RPCA approximates KPCA with a
small amount of random features whenever the intrinsic dimensionality $n/\|K\|$
of the exact kernel matrix $K$ is small.

\begin{remark}[Similar algorithms to RPCA]
  {Spectral clustering} uses the spectrum of $K$ to perform dimensionality
  reduction before applying $k$-means \citep{Luxburg07}. Therefore, the
  analysis of RPCA inspires a randomized and nonlinear variant of spectral
  clustering.
\end{remark}

\begin{remark}[Compression and intelligence]
  Dimensionality reduction, and more generally unsupervised learning, relates
  to data \emph{compression}. In fact, compression is possible because of the
  existence of patterns in data, as these patterns allow to recover some
  variables from others. In the absence of patterns, data would be independent
  noise, and the best compression would be the data itself.
  Compression amounts to finding the simplest descriptions of objects, a task
  considered intimate to intelligence.
\end{remark}

\subsection{Canonical correlation analysis}\label{sec:rcca}
Canonical Correlation Analysis or CCA \citep{Hotelling36} estimates the
correlation between two multidimensional random variables.  Given two paired
samples $ X \in \R^{n\times p}$ and $ Y \in \R^{n\times q}$, CCA computes pairs
of \emph{canonical bases} $ f_i \in \R^{p}$ and $ g_i \in \R^{q}$ such that
they maximize the correlation between the transformed samples $Xf_i$
and $Yg_i$, for all $1\leq i\leq \min(p,q)$. Graphically, CCA finds a
pair of linear transformations, $XF$ from $X$ and $YG$ from $Y$, such that the
dimensions of $XF$ and $YG$ (also known as canonical variables) are maximally
correlated. This is an useful manipulation when we learn from two different
views of the same data. Consider for instance of document translation
\citep{Vinokourov02}, where the training data is a collection of documents in
two different languages, let us say English and Spanish. In this task, we could
use CCA to transform the documents into a representation that correlates their
English version and their Spanish version, and then exploit these correlations
to predict translations.

More formally, let $ C_{xy}$ be the empirical covariance matrix
between $ X$ and $ Y$. Thus CCA maximizes 
\begin{equation*}
  \rho^2_i := \rho^2(Xf_i,  Yg_i) = \frac{ f_i^\top  C_{xy}
   g_i}{\sqrt{ f_i^\top  C_{xx}  f_i}\sqrt{ g_i^\top C_{yy} g_i}},
\end{equation*}
for $1 \leq i \leq r =\text{min}(\text{rank}( X),
\text{rank}( Y))$, subject to 
\begin{equation*}
\rho^2(Xf_i,  Yg_j) = \rho^2(Xf_i, Xf_j)
= \rho^2( Yg_i, Yg_j) = 0,
\end{equation*}
for all $i \neq j$ and $1 \leq j \leq r$.  We call the quantities $\rho^2_i$ the
  \emph{canonical correlations}.  Analogous to principal
  components, we order the \emph{canonical variables} $(Xf_i,  Yg_i)$
  with respect to their cross-correlation, that is, $\rho_1^2 \geq \cdots \geq \rho^2_r$.
  The canonical correlations $\rho^2_1, \ldots, \rho^2_r$ and canonical bases $
  f_1,\ldots,  f_r \in \R^p$, $ g_1, \ldots,  g_r \in \R^q$ are
  the solutions of the generalized eigenvalue problem \citep[Equation
  (2)]{Bach02}:
\begin{align*}
  \left(
      \begin{array}{cc}
        0 &   C_{xy}\\
        C_{yx} &  0
      \end{array}
      \right)
  \left(
      \begin{array}{c}
        f\\
        g 
      \end{array}
      \right)
  =
  \rho^2
  \left(
      \begin{array}{cc}
        C_{xx} &  0\\
        0 &  C_{yy} 
      \end{array}
      \right)
  \left(
      \begin{array}{c}
        f\\
        g 
      \end{array}
      \right)
  ,\nonumber
\end{align*}

Said differently, CCA processes two different views of the same data
(speech audio signals and paired speaker video frames) and returns their
maximally correlated linear transformations. This is particularly useful when
the two views of the data are available at training time, but only one of them
is available at test time \citep{Kakade07,Chaudhuri09,Vapnik09}. 

\begin{remark}[History of CCA]
  \citet{Hotelling36} introduced CCA to
  measure the correlation between two multidimensional random variables.  CCA
  has likewise found numerous applications, including multi-view statistics
  \citep{Kakade07} and learning with missing features
  \citep{Chaudhuri09,dlp-rca}.  \citet{Hardoon04} offers a monograph on CCA
  with a review on applications.
\end{remark}

The main limitation of the CCA algorithm is that the recovered canonical
variables only extract linear patterns in the analyzed pairs of data. To address these
limitations, \citet{Lai00,Bach02} introduce Kernel CCA or KCCA, an algorithm
that leverages the kernel trick (Section~\ref{sec:kernels}) to extract linear
components in some high-dimensional nonlinear representation of the data.
Computationally speaking, KCCA performs the eigendecomposition of the
matrix 
\begin{align}
  \label{eq:ccapop}
  M := \begin{pmatrix}
    I & M_{12}\\
    M_{21} & I
  \end{pmatrix},
\end{align}
where
\begin{align*}
  M_{12} &= (K_x+n\lambda I)^{-1} K_x K_y (K_y+n\lambda I)^{-1},\\
  M_{21} &= (K_y+n\lambda I)^{-1}K_yK_x(K_x+n\lambda I)^{-1},
\end{align*}
and $\lambda > 0$ is a regularization parameter necessary to avoid
spurious perfect correlations \citep[Equation (16)]{Bach02}.  Unfortunately, these
operations require $O(n^3)$ computations, a prohibitive 
complexity for large data. 

To extend the CCA algorithm to extract nonlinear canonical variables while
avoiding the computational burdens of KCCA, we propose RCCA \citep{dlp-rca},
the second example of our framework RCA. In particular, RCCA
\begin{enumerate}
  \item maps the original data $ X := ( x_1, \ldots,  x_n) \in
  \R^{n\times p}$, into the random feature data $Z_x := (z^{(x)}_1, \ldots, z^{(x)}_n) \in
  \R^{n \times m_x}$, 
  \item maps the original data $ Y := ( y_1, \ldots,  y_n) \in
  \R^{n\times q}$, into the randomized feature vectors $Z_y := (z^{(y)}_1, \ldots, z^{(y)}_{n}) \in
  \R^{n \times m_y}$ and 
  \item performs CCA on the pair of datasets $Z_x$ and $Z_y$.
\end{enumerate}

Thus, RCCA approximates KCCA when the random features of the former approximate
the kernel function of the latter.  The canonical variables in RCCA are no
longer linear transformations of the original data; they are approximations of
nonlinear transformations living in the Hilbert Spaces $(\H_x,
\H_y)$, induced by the kernels $k_x$ and $k_y$. The
computational complexity of RCCA is $O((m_x^2+m_y^2)n)$, which is competitive
when compared to the computational complexity $O((p^2+q^2)n)$ of linear CCA for
a moderate number of random features.

As with PCA, we will study the convergence rate of RCCA
to KCCA in operator norm, as $m_x$ and $m_y$ grow.  Let $\hat{K}_x$
and $\hat{K}_y$ be the approximations to the kernel matrices $K_x$ and $K_y$,
obtained by using $m_x$ and $m_y$ random features on the data $X$ and $Y$ as
in Definition~\ref{remark:rca}. Then, RCCA approximates the KCCA matrix
\eqref{eq:ccapop} with 
\begin{align}
  \label{eq:ccaemp}
  \hat{M} := \begin{pmatrix}
     I & \hat{M}_{12}\\
     \hat{M}_{21} & I
   \end{pmatrix}.
\end{align}
where
\begin{align*}
  \hat{M}_{12} &= (\hat{K}_x+n\lambda I)^{-1} \hat{K}_x \hat{K}_y (\hat{K}_y+n\lambda I)^{-1},\\
  \hat{M}_{21} &= (\hat{K}_y+n\lambda I)^{-1}\hat{K}_y\hat{K}_x(\hat{K}_x+n\lambda I)^{-1}.
\end{align*}
The solution of RCCA is the eigendecomposition of
\eqref{eq:ccaemp}. The following theorem allows to phrase the convergence rate
of RCCA to KCCA, as a function of the convergence rate of RPCA to KPCA.

\begin{theorem}[Norm of kernel matrices bound norm of CCA]\label{thm:cca}
Let $M$ and $\hat{M}$ be as in Equations \eqref{eq:ccapop} and \eqref{eq:ccaemp}, respectively. Then,
\begin{align*}
   \| \hat{M} - M \|
  \leq\left\lbrace \frac{3}{n}\left(\frac{1}{\lambda^2} + \frac{1}{\lambda} \right) \right\rbrace\left(
  \|\hat{K}_x-K_x\|+\|{K}_y-\hat{K}_y\|\right) \leq 1.
\end{align*}
\end{theorem}
\begin{proof}
  See Section~\ref{proof:thm:cca}.
\end{proof}

The following result characterizes the convergence rate of RCCA to KCCA, in terms of the number of random features $m$. 

\begin{corollary}[Convergence of RCCA]\label{cop:cca}
  Consider the assumptions from Definition~\ref{remark:rca}, and Equations
  \ref{eq:ccapop}-\ref{eq:ccaemp}. Then,
  \begin{equation}\label{eq:ccaconc}
      \E{}{\| {\hat{\bm M}} -  M\|} \leq\left\lbrace \frac{6}{n}\left(\frac{1}{\lambda^2} + \frac{1}{\lambda} \right) \right\rbrace 
      \left(\sqrt{\frac{3B\|K\|\log
      n}{m}} + \frac{2B\log n}{m}\right) \leq 1,
  \end{equation}
  where $\|K\| = \max(\|K_x\|, \|K_y\|)$, and $m = \min(m_x,m_y)$.
  \begin{proof}
  Combine Theorem \ref{thm:pca} and Theorem \ref{thm:cca}. 
  \end{proof}
\end{corollary}

\begin{remark}[Similar algorithms to RCCA]
  \emph{Linear discriminant analysis} seeks a linear combination
  of the features of the data $ X \in \R^{n\times d}$ such that the
  samples become maximally separable with respect to a paired labeling $ y$ with $y_i
  \in \{1,\ldots,c\}$. LDA solves $\mathrm{CCA}( X, T)$, where $ T_{ij} =
  \I\{y_i = j\}$ \citep{Bie05}.  Therefore, a similar analysis to the
  one of RCCA applies to study randomized and nonlinear variants of LDA.
\end{remark}

\begin{remark}[Extensions and improvements to component analysis]
  Component analysis reduces our data into a number of explanatory factors
  smaller than the number of original variables. However, in some situations it
  may be the case that the observed variables are the summary of a larger
  number of explanatory factors.  In this case, component analysis would aim at
  estimating a number of explanatory factors larger than the number of observed
  variables. We refer to this kind of component analysis as \emph{overcomplete
  component analysis}.  Random features allow overcomplete component analysis,
  by setting $m \gg d$, and regularizing properly.
  
  There are two straightforward ways to improve the speed of the algorithms
  presented in this chapter. First, distributing the computation of the random
  feature covariance matrices over multiple processing units. Second, using
  computing only the top $k$ singular values of the random feature matrix $Z$,
  if we are only interested in extracting the top $k \ll m$ components.
\end{remark}

\section{Measures of dependence}\label{sec:dependence-measures}
Measuring the extent to which two random variables depend on each other is a
fundamental question in statistics and applied sciences (\emph{What genes are
responsible for a particular phenotype?}). Mathematically, given the sample 
\begin{equation*}
\Z = \{(x_i,y_i)\}_{i=1}^n \sim P^n(\bm x, \bm y),
\end{equation*}
the question of whether the two random variables $\bm x$ and $\bm y$ are
independent is to estimate whether the equality 
\begin{equation*}
  p(x,y) = p(x)p(y)
\end{equation*}
holds for all $x \in \X$ and $y \in \Y$. One way to prove that two random
variables are independent is to check if their mutual information
\begin{equation*}
  I(\bm x, \bm y) = \int_\Y \int_\X p(x,y) \log \frac{p(x,y)}{p(x)p(y)} \d x \d y.
\end{equation*}
is zero. Unfortunately, estimating the mutual information is a challenging task,
since it requires the estimation of the densities $p(\bm x)$, $p(\bm y)$, and $p(\bm x, \bm y)$.

Let us take one step back, and simplify the problem by asking if the random
variables $\bm x$ and $\bm y$ are \emph{correlated}, that is, related by a
\emph{linear} dependence. Answering this question is much simpler; two random
variables are not correlated if and only if their Pearson's correlation
coefficient 
\begin{equation*}
  \hat{\rho}(\{(x_i,y_i)\}_{i=1}^n) = \frac{\sum_{i=1}^n
  (x_i-\hat{\mu}_x)(y_i-\hat{\mu}_y)}{\sqrt{\sum_{i=1}^n
  (x_i-\hat{\mu}_x)}\sqrt{\sum_{i=1}^n (y_i-\hat{\mu}_y)}}
\end{equation*}
converges to zero, as $n \to \infty$, where $\hat{\mu}_x =
n^{-1} \sum_i x_i$, and similarly for $\hat{\mu}_y$. 

Correlation measures to what extent the relationship between two
variables is a \emph{straight line}. There are multiple ways to extend the
concept of correlation to slightly more general situations. For example,
Spearman's $\rho$ and Kendall's $\tau$ measure to what extent we can express
the relationship between two random variables as a \emph{monotone}
function.
But what should a general measure of dependence satisfy?

\subsection{Renyi's axiomatic framework}
Half a century ago, Alfr\'ed \citet{Renyi59} argued that a general measure of dependence
$\rho^* : \X \times \Y \rightarrow [0,1]$ between two nonconstant random
variables $\bm x\in \X$ and $\bm y\in\Y$ should satisfy seven
fundamental properties:
\begin{enumerate}
  \item $\rho^*(\bm x, \bm y)$ is defined for any pair of random variables 
  $\bm x$ and $\bm y$.
  \item $\rho^*(\bm x,\bm y) = \rho^*(\bm y,\bm x)$
  \item $0 \leq \rho^*(\bm x,\bm y) \leq 1$
  \item $\rho^*(\bm x,\bm y) = 0$ iff $\bm x$ and $\bm y$ are statistically independent.
  \item For bijective  Borel-measurable $f,g : \R
  \rightarrow \R$, $\rho^*(\bm x,\bm y) = \rho^*(f(\bm x),g(\bm y))$.
  \item $\rho^*(\bm x,\bm y) = 1$ if for Borel-measurable $f$ or $g$, $\bm y =
  f(\bm x)$ or $\bm x = g(\bm y)$.
  \item If $(\bm x,\bm y) \equiv \N( \mu,  \Sigma)$, then $\rho^*(\bm x,\bm y) =
  |\rho(\bm x,\bm y)|$, where $\rho$ is the correlation coefficient.
\end{enumerate}
In the same work, R\'enyi also showed that the \emph{Hirschfeld-Gebelein-R\'enyi
Maximum Correlation Coefficient} (HGR) satisfies all these properties. HGR, 
introduced by \cite{Gebelein41} is the suprema of Pearson's
correlation coefficient $\rho$ over all Borel-measurable functions $f,g$ of
finite variance:
\begin{equation}\label{eq:hgr}
  \text{HGR}(\bm x,\bm y) = \sup_{f,g} \rho(f(\bm x),g(\bm y)).
\end{equation}
Unfortunately, the suprema in \eqref{eq:hgr} is NP-hard to compute. In the
following, we review different computable alternatives to approximate the HGR statistic.

\subsection{Kernel measures of dependence}
\emph{Kernel measures of dependence} measure the dependence between two random
variables as their correlation when mapped to some RKHS. This is just another
clever use of the kernel trick: in Chapter~\ref{chapter:representing-data},
kernels allowed us to phrase nonlinear regression as linear regression in RKHS
(recall Example~\ref{example:two-rings}). In this section, they will allow us
to phrase dependence as correlation in RKHS. 

In the following, consider the kernel matrix $K \in \R^{n \times n}$ with
entries $K_{i,j} = k_x(x_i,x_j)$, and the kernel matrix $L \in \R^{n\times n}$
with entries $L_{i,j} = k_y(y_i,y_j)$, for all $1 \leq i \leq n$ and $1 \leq j
\leq n$. Also, consider the centered kernel matrices $\tilde{K} = HKH$ and
$\tilde{L} = HLH$, built using the centering matrix $H = I_n - n^{-1} 1_n
1_n^\top$. We review three important measures of dependence based on kernels.

First, the COnstrained COvariance or COCO \citep{GreHerSmoBouetal05} is the
largest singular value of the cross-covariance operator associated with the
reproducing kernel Hilbert spaces $\H_{k_x}$ and $\H_{k_y}$: 
\begin{equation*}
  \textrm{COCO}(\Z,k_x,k_y) = \frac{1}{n} \sqrt{\|\tilde{K}\tilde{L}\|_2} \in [0,\infty).
\end{equation*}
As a covariance, the COCO statistic is nonnegative and unbounded.

Second, the Hilbert-Schmidt Independence Criterion or HSIC \citep{Gretton05}
borrows the same ideas from COCO, but uses the entire spectrum of the
cross-covariance operator instead of only its largest singular value, that is,
\begin{equation*}
  \textrm{HSIC}(\Z,k_x,k_y) = \frac{1}{n^2} \textrm{tr}({\tilde{K}\tilde{L}}) \in [0,\infty).
\end{equation*}
HSIC relates to COCO, but was shown
superior on several benchmarks \citep{Gretton05}.

Third, the Kernel Canonical Correlation or KCC \citep{Bach02} is the largest
kernel canonical correlation, that is,
\begin{equation*}
  \textrm{KCC}(\Z,k_x,k_y,\lambda_x,\lambda_y) = \sqrt{\|\hat{M}\|_2} \in [0,1],
\end{equation*}
where $\hat{M}$ is the matrix in Equation~\ref{eq:ccaemp}. As an absolute correlation, the KCC
statistic is bounded between zero and one. Since KCC relies on KCCA, we
require the use of two regularization parameters $\lambda_x, \lambda_y > 0$ to avoid
spurious perfect correlations.

When the kernels $k_x, k_y$ are characteristic kernels
\citep{sriperumbudur2011universality} the COCO, HSIC, and KCC statistics
converge to zero as $n \to \infty$ if and only if the random variables $\bm x$
and $\bm y$ are independent. 

From a computational point of view, these three kernel measures of dependence
rely on the computation and eigendecomposition of $n\times n$ matrices.
These are operations taking $O(n^3)$ computations, a prohibitive running time
for large data. In the following, we propose the Randomized Dependence
Coefficient or RDC \citep{dlp-rdc}, the third example within our framework RCA,
which approximates HGR in $O(n\log n)$ time, while being invariant with respect
to changes in marginal distributions. 

\subsection{The randomized dependence coefficient}\label{sec:rdc}
Estimating the dependence between two random samples $ X \in \R^{n
\times p}$ and $ Y \in \R^{n \times q}$ with RDC involves three steps. 
First, RDC maps the two input samples to their respective
\emph{empirical copula transformations}
\begin{align*}
   X = ( x_1, \ldots,  x_n)^\top &\mapsto T_{\bm x, n}( X) := (T_{\bm x, n}(
  x_1), \ldots, T_{\bm x, n}( x_n))^\top,\nonumber\\
   Y = ( y_1, \ldots,  y_n)^\top  &\mapsto T_{\bm y, n}(Y) := (T_{\bm y, n}(
  y_1), \ldots, T_{\bm y,n}( y_n))^\top,
\end{align*}
which have uniformly distributed marginals. 

Working with copulas makes 
RDC invariant with respect to transformations on the marginal distributions, as
requested by R\'enyi's fifth property. Second, RDC maps the copula data 
to a randomized feature representation
${\phi}_{\bm x, m} : \R^{n\times p} \mapsto \R^{n\times m_x}$ and
${\phi}_{\bm y, m} : \R^{n\times q} \mapsto \R^{n\times m_y}$,
constructed as in~\eqref{eq:theaug}. For simplicity, let $m_x = m_y = m$. That is, we compute:
\begin{align*}
  {T}_{\bm x,n}( X) &\mapsto 
  {\phi}_{\bm x,m}({T}_{\bm x,n}( X))^\top := ({\phi}_{\bm x, m}({T}_{\bm x,n}( x_1)), \ldots,
  {\phi}_{\bm x,m}({T}_{\bm x,n}( x_n)))^\top\nonumber,\\
  {T}_{\bm y,n}( Y) &\mapsto 
  {\phi}_{\bm y,m}({T}_{\bm y,n}(Y))^\top := ({\phi}_{\bm y,m}({T}_{\bm y,n}( y_1)), \ldots,
  {\phi}_{\bm y,m}({T}_{\bm y,n}(y_n)))^\top.
\end{align*}
Third, RDC is the largest canonical correlation between
the previous two maps
\begin{equation}\label{eq:rdceq}
  \text{RDC}( X,  Y) = \sup_{ \alpha,  \beta} \rho(
  \dot{{\phi}_{\bm x,m}({T}_{\bm x, n}( X))}{\alpha}, \dot{{\phi}_{\bm y, m}({T}_{\bm y,n}( Y))}{\beta}),
\end{equation}
where $ \alpha, \beta \in \R^{m \times 1}$.
Figure~\ref{fig:rdcsteps} offers a sketch of this process.
\begin{figure}
  \centering
  \includegraphics[width=\textwidth]{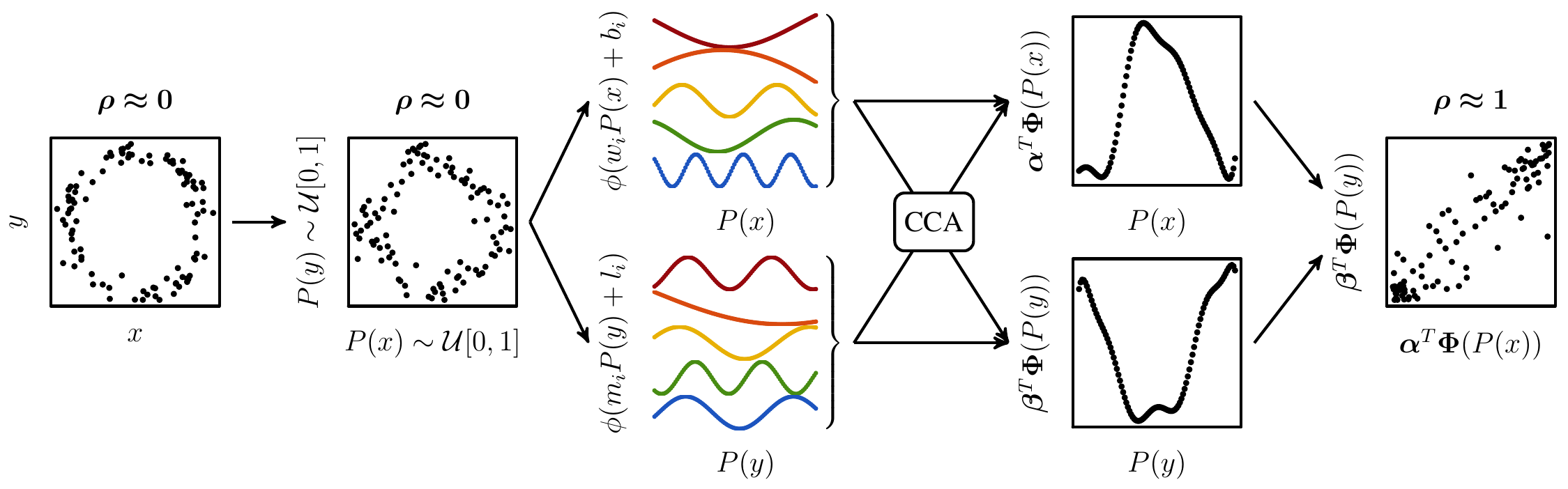}
  \caption[RDC computation example]{RDC computation for the sample $\{(x_i,y_i)\}_{i=1}^{100}$ drawn
  from a noisy circular pattern.}
  \label{fig:rdcsteps}
\end{figure}

In another words, RDC is the largest canonical correlation as computed by RCCA
on random features of the copula transformations of two random samples.

\subsubsection{Properties of RDC}
RDC enjoys some attractive properties.  First, its computational complexity is
$O((p+q) n \log n + m^2n)$, that is, log-linear with respect to the sample
size.  This cost is due to the estimation of two copula transformations and the
largest RCCA eigenvalue.  Second, RDC is easy to implement.  Third, RDC
compares well with the state-of-the-art. Table \ref{table:comparison}
summarizes, for a selection of well-known measures of dependence, whether they
allow for general nonlinear dependence estimation, handle multidimensional
random variables, are invariant with respect to changes in the one-dimensional
marginal distributions of the variables under analysis, return a statistic in
$[0,1]$, satisfy R\'enyi's properties, and their number of parameters. As
parameters, we here count the kernel function for kernel methods, the basis
function and number of random features for RDC, the stopping tolerance for ACE
\citep{Breiman85} and the grid size for MIC.  The table lists
computational complexities with respect to sample size.

RDC can prescind from the use of copulas. In that case, RDC would no longer be
scale-invariant, but would avoid potential pitfalls related to the misuse of
copulas (see Remark~\ref{remark:copula-worries}).

\begin{table}
  \resizebox{\textwidth}{!}
  {
    \begin{tabular}{lccccccl}
    \hline
    \head{2cm}{\textbf{Dependence coefficient}}     &
    \head{2cm}{\textbf{Nonlinear measure}}         &
    \head{2cm}{\textbf{Multidim. inputs}}      &
    \head{2cm}{\textbf{Marginal invariant}} & 
    \head{1.4cm}{\textbf{Renyi's axioms}} &
    \head{1.2cm}{Coeff. \textbf{$\in [0,1]$}} &
    \head{1cm}{\# \textbf{Par.}}         &
    \head{1cm}{\textbf{Comp. Cost}}\\\hline\hline
    Pearson's $\rho$      & $\times$     & $\times$     & $\times$     & $\times$     & $\checkmark$ & 0 & $n$        \\ \hline
    Spearman's $\rho$     & $\times$     & $\times$     & $\checkmark$ & $\times$     & $\checkmark$ & 0 & $n \log n$ \\ \hline
    Kendall's $\tau$      & $\times$     & $\times$     & $\checkmark$ & $\times$     & $\checkmark$ & 0 & $n \log n$ \\ \hline
    CCA                   & $\times$     & $\checkmark$ & $\times$     & $\times$     & $\checkmark$ & 0 & $n$     \\ \hline
    KCCA  & $\checkmark$ & $\checkmark$ & $\times$     & $\times$     & $\checkmark$ & 1 & $n^3$      \\ \hline
    ACE  & $\checkmark$ & $\times$     & $\times$     & $\checkmark$ & $\checkmark$ & 1 & $n$   \\ \hline
    MIC   & $\checkmark$ & $\times$     & $\times$     & $\times$     & $\checkmark$ & 1 & $n^{1.2}$      \\ \hline
    dCor & $\checkmark$ & $\checkmark$ & $\times$     & $\times$     & $\checkmark$ & 1 & $n^2$      \\ \hline
    HSIC  & $\checkmark$ & $\checkmark$ & $\times$     & $\times$     & $\times$     & 1 & $n^2$      \\ \hline
    CHSIC & $\checkmark$ & $\checkmark$ & $\checkmark$ & $\times$     & $\times$     & 1 & $n^2$      \\ \hline
    \textbf{RDC} & $\checkmark$ & $\checkmark$ & $\checkmark$ & $\checkmark$ & $\checkmark$ & 2 & $n \log n$ \\ \hline\hline
  \end{tabular}
  }
  \caption{Comparison of measures of dependence.}
  \label{table:comparison}
\end{table}

Fourth, RDC is consistent with respect to KCCA. In particular, we are
interested in how quickly does RDC converge to KCCA when the latter is
performed on the true copula transformations of the pair of random variables
under study. For that, consider the matrix
\begin{align}
  \label{eq:ccacop}
  \tilde{Q} := \begin{pmatrix}
     I & \tilde{Q}_{12}\\
     \tilde{Q}_{21} & I
   \end{pmatrix}.
\end{align}
with blocks
\begin{align*}
  \tilde{Q}_{12} &= (\tilde{K}_x+n\lambda I)^{-1} \tilde{K}_x \tilde{K}_y (\tilde{K}_y+n\lambda I)^{-1},\\
  \tilde{Q}_{21} &= (\tilde{K}_y+n\lambda I)^{-1}\tilde{K}_y\tilde{K}_x(\tilde{K}_x+n\lambda I)^{-1},
\end{align*}
where $\tilde{K}_{x,i,j} = k(T_n(x_i), T_n(x_j))$ are true kernel evaluations
on the empirical copula of the data. The matrix $\hat{Q}$ has the same
structure, but operates on the empirical copula of the data and random
features. The matrix $Q$ has the same structure, but operates on the true
copula of the data and the true kernel.  The following theorem provides an
specific rate on the convergence of RDC to the largest copula kernel canonical
correlation.

\begin{theorem}[Convergence of RDC]\label{thm:rdc}
  Consider the definitions from the previous paragraph, and the assumptions
  from Definition~\ref{remark:rca}. Then,
  \begin{align*}
    \E{}{\|\hat{\bm Q} -  Q\|} &\leq \left\lbrace \frac{6}{n} \left(
    \frac{1}{\lambda^2} + \frac{1}{\lambda}\right)\right\rbrace\\
    &\times\left(
    \sqrt{\frac{3B\|K\|\log
      n}{m}} + \frac{2B\log n}{m}
    +{2L_k}
    \sqrt{nd}\left(\sqrt{\pi}+\sqrt{\log 2d}\right)\right)\\
    &\leq 1,
  \end{align*}
  where $\|K\| = \max(\|K_x\|, \|K_y\|)$, and $m = \min(m_x,m_y)$.
\end{theorem}
\begin{proof}
  See Secton~\ref{proof:thm:rdc}.
\end{proof}

As it happened with KCCA, regularization is necessary in RDC to avoid spurious
perfect correlations.  However, \eqref{eq:rdceq} lacks regularization.
This is because, as we will see in our numerical simulations, using a small
number of random features (smaller than the number of samples) provides an
implicit regularization that suffices for good empirical performance.

\subsection{Conditional RDC}
In some situations, including the causal inference problems studied in the
second part of this thesis, we will study the statistical
dependence of two random variables $\bm x$ and $\bm y$ when conditioned to the
effects of a third random variable $\bm z$. Mathematically, $\bm x$ and $\bm y$ are
conditionally independent given $\bm z$ if the equality
\begin{equation*}
  p(x,y\given z) = p(x\given z)p(y\given z)
\end{equation*}

holds for all $x \in \X$, $y \in \Y$, and $z \in \Z$. Measuring conditional dependence using RDC relies on partial CCA
\citep{Rao69}, a variant of CCA designed to measure the correlation between two
multidimensional random samples $ X$ and $ Y$ after eliminating the effects of a
third sample $Z$. Partial canonical correlations 
are the solutions of the following generalized eigenvalue problem:
\begin{align*}
   C_{xy|z}  C_{yy|z}^{-1}  C_{yx|z}  f = \rho^2  C_{xx|z}  f\\
   C_{yx|z}  C_{xx|z}^{-1}  C_{xy|z}  g = \rho^2  C_{yy|z}  g,
\end{align*}
where $ C_{ij|z} =  C_{iz}  C_{zz}^{-1}  C_{zj}$, for $i,j \in \{x,y\}$. In
this case, computing the conditional RDC is as follows.  First, we map the
three random samples $\bm x$, $\bm y$, and $\bm z$ to a randomized nonlinear
representation of their copula transformations.  Second, we compute the
conditional RDC as the largest partial canonical correlation between these
three random feature maps. 

\begin{remark}[A general recipe for measures of conditional dependence]
  There is a common recipe to measure conditional dependence using
  unconditional measures of dependence an nonlinear regression methods:
  \begin{enumerate}
    \item Estimate the regression residuals $\bm r_x = \bm x - \E{}{\bm x \given z}$.
    \item Estimate the regression residuals $\bm r_y = \bm y - \E{}{\bm y \given z}$.
    \item Estimate the dependence between $\bm r_x$ and $\bm r_y$.
  \end{enumerate}
\end{remark}

\subsubsection{Hypothesis testing with RDC} 
Consider the hypothesis ``the two sets of nonlinear projections are mutually
uncorrelated''.  Under normality assumptions and large sample sizes,
Bartlett's approximation \citep{Mardia79} approximates the null-distribution of RCCA as 
$$
  \left(\frac{2k+3}{2}-n\right) \log \prod_{i=1}^k (1-\rho_i^2) \sim
  \chi^2_{k^2},
$$
which can be easily adapted for approximate RDC hypothesis testing. 

Alternatively, we could use bootstrapping to obtain nonparametric estimates of
the null-distribution of RDC. Figure~\ref{fig:rdc_null} shows the
null-distribution of RDC for unidimensional random samples and different sample
sizes $n$, as estimated from $100,000$ pairs of independent random samples. The Beta distribution (dashed lines
in the figure) is a good approximation to the empirical null-distribution of
RDC (solid lines). The parameters of the Beta distribution vary smoothly as the sample size $n$
increases. For scalar random variables, the marginal distributions of the random samples under measurement
do not have any effect on the null-distribution, thanks to the scale invariance
provided by the empirical copula transformation. Therefore, tables for the
null-distribution of RDC can be efficiently pre-computed for one-dimensional
random variables.

\begin{figure}
  \begin{center}
    \includegraphics[width=\linewidth]{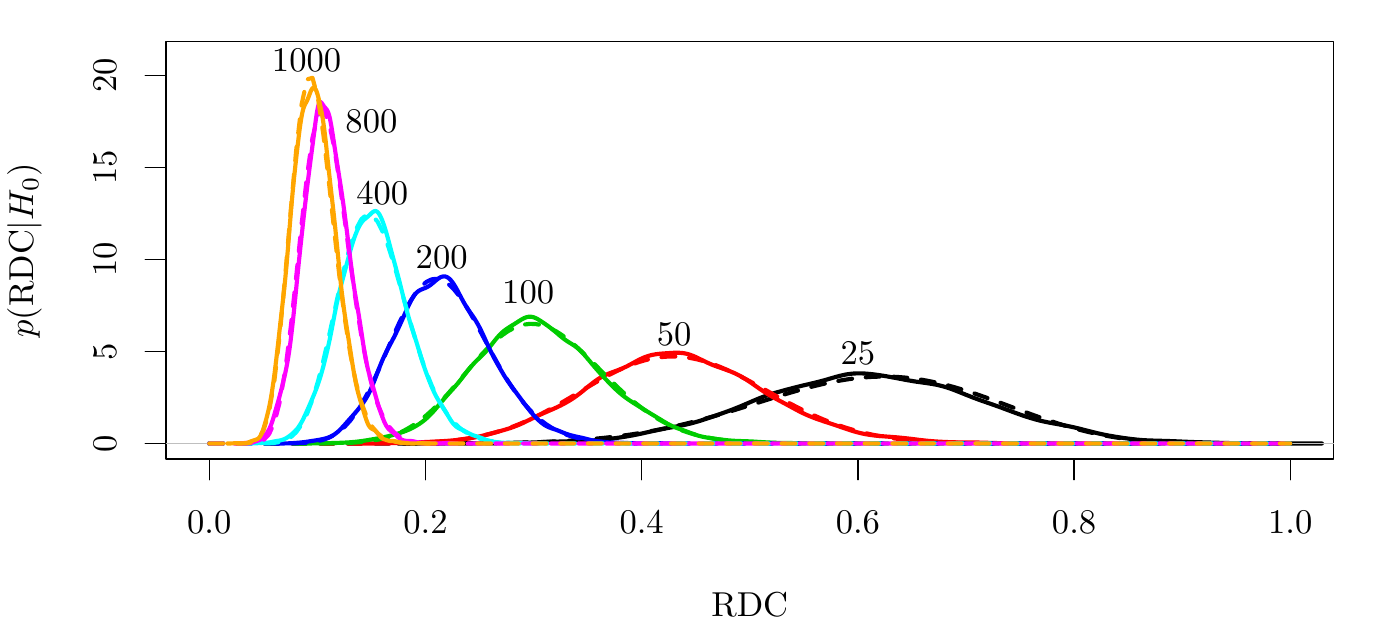}
  \end{center}
  \vskip -0.5 cm
  \caption[Approximations to the null-distribution of RDC]{Empirical
  null-distribution of RDC for unidimensional random variables (solid lines),
  and corresponding Beta-approximations (dashed lines).  Sample sizes 
  on the top of each associated curve.} \label{fig:rdc_null}
\end{figure}

\subsection{Model selection and hypothesis testing}\label{sec:model-selection-2}
All the measures of dependence presented in this section have tunable
parameters: their regularizers, kernel functions, parameters of these kernel
functions, and so forth. This is not a novel nuisance for us, as tunable
parameters populated our discussions in previous chapters about data
representation and density estimation. All these parameters were tuned by
monitoring the objective function of the problem at hand in some \emph{held out}
validation data.

To some extent, the measures of dependence from this section follow the same
techniques for model selection \citep{sugiyama2012density}. After all, these
algorithms aim at extracting the largest amount of patterns from data, as long
as those patterns are not hallucinated from noise. Therefore, cross-validation
is of use to avoid overfitting. If possible, such cross-validation should aim
at directly maximizing the power\footnote{The power of a dependence test
is the probability that the test rejects the independence hypothesis when
analyzing dependent random variables.} of the dependence statistic
\citep{gretton2012optimal}.

Model selection is more subtle when performed for hypothesis testing.
In dependence testing, our null hypothesis $H_0$ means ``the random variables
$\bm x$ and $\bm y$ are independent''.  Therefore, a type-I error (false
positive) is to conclude that a pair of independent random variables is
dependent, and a type-II error (false negative) is to conclude that a pair of
dependent random variables is independent.  Regarding parameters,
simultaneously avoiding type-I and type-II errors are two conflicting
interests: parameters providing flexible measures of dependence (for instance,
RDC with large number of random features) will tend to make type-I errors
(overfit, low bias, high variance), and rigid measures of dependence will tend
to make type-II errors (underfit, high bias, low
variance). Nevertheless, in some applications, one of the two errors is more
severe. For instance, it is worse to tell a patient suffering from cancer that he is
healthy, rather than diagnosing a healthy patient with
cancer.  Mathematically, given a measure of dependence with
parameters $\theta$, model selection could maximize the objective 
\begin{equation*}
  \argmin_\theta \underbrace{(1-\lambda) \textrm{dep}(\Z,\theta)}_{\text{avoids
  type-II error}} + 
  \underbrace{\lambda\textrm{dep}(\Z_\pi,\theta)}_{\text{avoids type-I error}},
\end{equation*}
where $\Z_\pi$ is a copy of $\Z$ where the samples of the second random
variable have been randomly permuted, and $\lambda \in [0,1]$ is a parameter that
balances the importance between type-I and type-II errors, and depends on the
problem at hand.

\section{Two-sample tests}\label{sec:two-sample-tests}
The problem of two-sample testing addresses the following question:
\begin{center}
\emph{Given two samples
$\{x_i\}_{i=1}^n \sim P^n$ and $\{y_i\}_{i=1}^n \sim Q^n$, is $P=Q$?}
\end{center}

One popular nonparametric two-sample test is
the \emph{Maximum Mean Discrepancy} or MMD \citep{mmd}. Given a kernel function 
$k$, the empirical MMD statistic is  
\begin{align}
  \text{MMD}^2(\Z, k) 
  &=\frac{1}{n_x^2} \sum_{i,j=1}^{n_x^2} k( x_i,  x_j) -
  \frac{2}{n_xn_y}\sum_{i,j=1}^{n_x,n_y} k( x_i,  y_j) + \frac{1}{n_y^2}
  \sum_{i,j=1}^{n_y} k( y_i,  y_j),\label{eq:mmdpop}
\end{align}
When $k$ is a characteristic kernel and $n \to \infty$, the MMD
statistic is zero if and only if $P=Q$. For simplicity, let $n_x = n_y = n$; then, computing
the MMD statistic takes $O(n^2)$ operations. By making use of the random
features introduced in Section~\ref{sec:random-mercer-features}, we can define
an approximate, randomized version of MMD 
\begin{align}
  \text{RMMD}^2(\Z, k) &= \left\| \frac{1}{n_x} \sum_{i=1}^{n_x}
  \hat{\phi}_k( x_i) - \frac{1}{n_y} \sum_{i=1}^{n_y} \hat{\phi}_k( y_i)
\right\|^2_{\R^m}\nonumber\\
&=\frac{1}{n_x^2} \sum_{i,j=1}^{n_x^2} \hat{k}( x_i,  x_j) -
\frac{2}{n_xn_y}\sum_{i,j=1}^{n_x,n_y} \hat{k}( x_i,  y_j) + \frac{1}{n_y^2}
\sum_{i,j=1}^{n_y} \hat{k}( y_i,  y_j)\label{eq:mmdemp},
\end{align}
where $\hat{\phi}$ is a random feature map and $\hat{k}$ is the induced
approximate kernel. RMMD can be computed in $O(nm)$ operations, and is the
fourth example of our framework RCA.

\begin{theorem}[Convergence of RMMD]\label{thm:rmmd}
  Let the $\mathrm{MMD}$ and $\mathrm{RMMD}$ be as in \eqref{eq:mmdpop} and
  \eqref{eq:mmdemp}, respectively. Let the data be $d$-dimensional and live in
  a compact set $\mathcal{S}$ of diameter $|\mathcal{S}|$, let $\mathrm{RMMD}$
  use $m$ random features corresponding to a shift invariant kernel, and let
  $n=n_x=n_y$. Then, 
  \begin{equation*}
    \Pr\pa{\left|\mathrm{MMD}(\Z, k)- \mathrm{RMMD}(\Z,k)\right| \geq
    \frac{4(h(d,|\mathcal{S}|,c_k)+\sqrt{2t})}{\sqrt{m}}} \leq
    \exp\pa{{-t}},
  \end{equation*}
  where the function $h$ is defined as in \eqref{eq:szabo2}.
\end{theorem}
\begin{proof}
  See Section~\ref{proof:thm:rmmd}.
\end{proof}

\section{Numerical simulations}\label{sec:rca-experiments}
We evaluate the performance of a selection of the RCA methods introduced in this 
chapter throughout a variety of experiments, on both synthetic and real-world data.  In
particular, we organize our numerical simulations as follows.
Section~\ref{sec:expbounds} validates the Bernstein bounds from
Theorem~\ref{thm:pca} and Corollary \ref{cop:cca}.
Section~\ref{sec:ccaexps} evaluates the performance of RCCA on the task of
learning shared representations between related datasets.
Section~\ref{sec:lupi} explores the use of RCCA in Vapnik's \emph{learning
using privileged information} setup.  Section~\ref{sec:autoencoders-exp}
exemplifies the use of RPCA as an scalable, randomized strategy to train
autoencoder neural networks.  Finally, Section~\ref{sec:rdcexps} offers a
variety of experiments to study the capabilities of RDC to measure
statistical dependence between multivariate random variables. The Gaussian
random features used throughout these experiments are like the ones from
Equation~\ref{eq:gauss-map}.  The bandwidth parameter $\gamma$ is adjusted
using the median heuristic, unless stated otherwise.

\subsection{Validation of Bernstein bounds}
  \label{sec:expbounds} We now validate empirically the Bernstein bounds
  obtained in Theorem \ref{thm:pca} and Corollary \ref{cop:cca}.  To do so, we perform
  simulations in which we separately vary the values of the two tunable
  parameters in RPCA and RCCA: the number of random projections $m$, and the
  regularization parameter $\lambda$. We use synthetic data matrices ${X}\in
  \R^{1000\times 10}$ and $ {Y}\in \R^{1000\times 10}$, formed by iid normal
  entries. When not varying, the parameters are fixed to $m=1000$ and $\lambda
  = 10^{-3}$.

  \begin{figure}
    \begin{center}
    \includegraphics[width=0.3\linewidth]{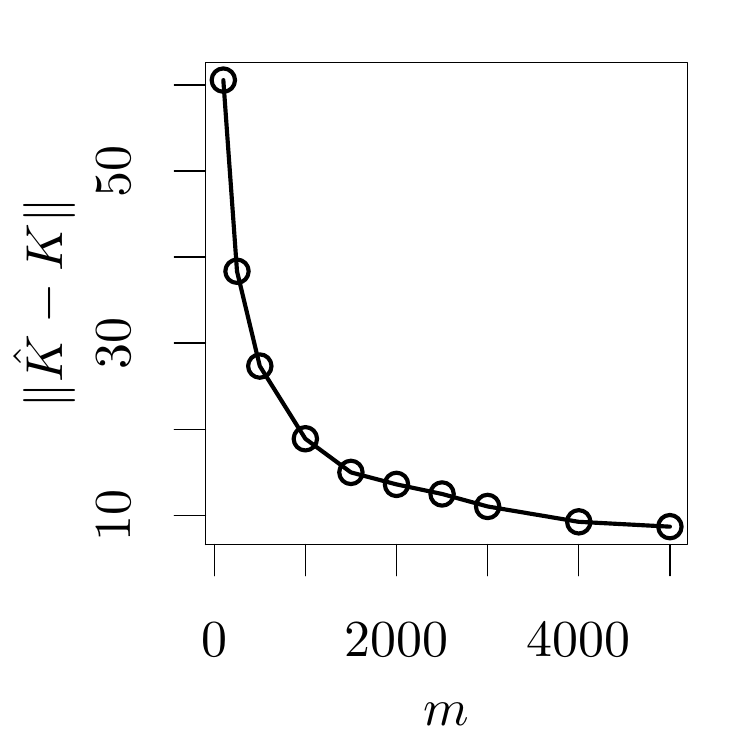}%
    \includegraphics[width=0.3\linewidth]{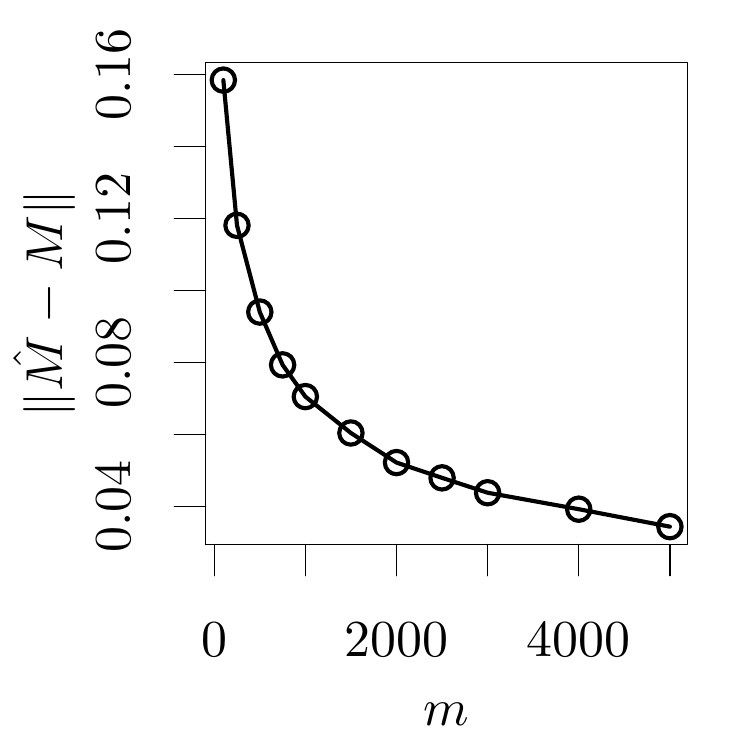}%
    \includegraphics[width=0.3\linewidth]{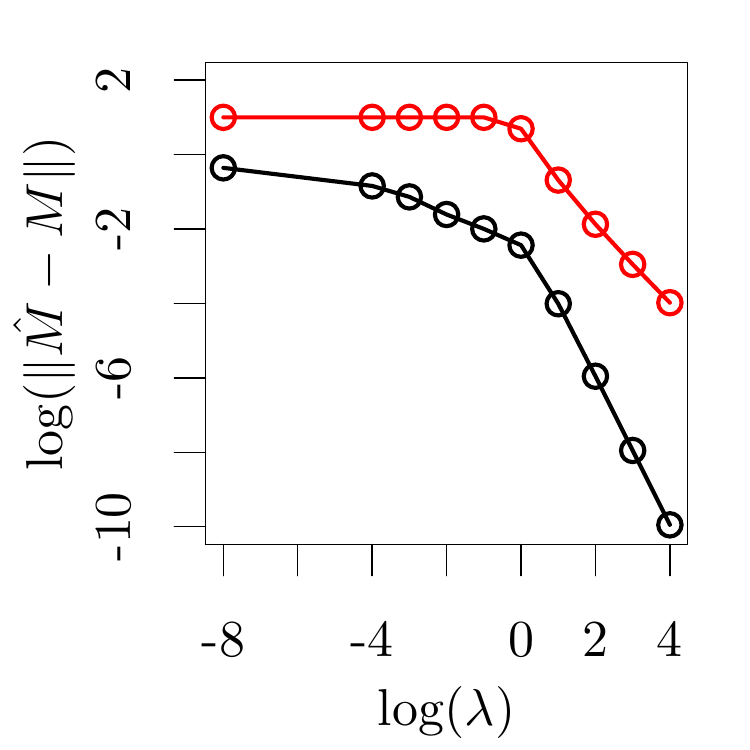}%
    \end{center}
    \vskip -0.5 cm
    \caption[Matrix Bernstein inequality error norms]{Matrix Bernstein inequality error norms.}
    \label{fig:bounds}
  \end{figure}
  
  Figure \ref{fig:bounds} depicts the value of the norms from equations
  (\ref{eq:pcaconc}, \ref{eq:ccaconc}), as the parameters $\{m,\lambda\}$ vary,
  when averaged over a total of $100$ random data matrices $X$ and $Y$. The
  simulations agree with the presented theoretical analysis: the number of
  random features $m$ has an inverse square root effect in both RPCA and RCCA,
  and the effect of the regularization parameter is upper bounded by the
  theoretical bound $\min(1,\lambda^{-1}+\lambda^{-2})$ (depicted in red) in
  RCCA.

\subsection{Principal component analysis}\label{sec:autoencoders-exp}
  \begin{figure}
    \begin{center}
    \includegraphics[width=0.7\linewidth]{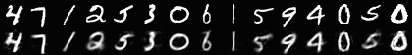}\\
    \includegraphics[width=0.7\linewidth]{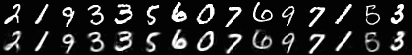}\\
    \vskip 0.3 cm
    \includegraphics[width=0.7\linewidth]{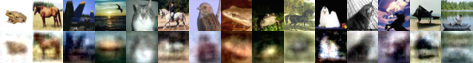}\\
    \includegraphics[width=0.7\linewidth]{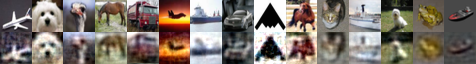}
    \vskip -0.1 cm
    \end{center}
    \vskip -0.2 cm
    \caption[Results for the randomized autoencoder experiments]{Autoencoder
    reconstructions of unseen test images for the MNIST (top) and CIFAR-10
    (bottom) datasets.} 
    \label{fig:pcaauto}
  \end{figure}
  
  One use for RPCA is the scalable training of nonlinear autoencoders (for a
  review on autoencoders, see Section~\ref{sec:autoencoders}). The process
  involves i) mapping the observed data $X\in\R^{n\times D}$ into the latent
  factors $Z \in\R^{n\times d}$ using the top $d$ nonlinear principal
  components from RPCA, and ii) reconstructing $X$ from $Z$ using $D$ nonlinear
  regressors.  Figure \ref{fig:pcaauto} shows the reconstruction of
  \emph{unseen} MNIST and CIFAR-10 images from the RPCA compressions.  Here,
  the number random projections is $m=2000$, the number of latent dimensions is
  $d=20$ for MNIST, and $d=40$ (first row) or $d=100$ (second row) for
  CIFAR-10.  Training took under 200 seconds on a 1.8GhZ processor for each
  dataset of $50000$ samples.
  
\subsection{Canonical correlation analysis}\label{sec:ccaexps}
  We compare three variants of CCA on the task of learning correlated features
  from two modalities of the same data: linear CCA, Deep CCA \citep{Galen13}
  and the proposed RCCA.  Deep CCA (DCCA), the current state-of-the-art, feeds
  the pair of input samples through a deep neural network, and learns its
  weights by solving a nonconvex optimization problem with gradient descent.
  We were unable to run exact KCCA on the proposed datasets due to its cubic
  complexity. Instead, we offer a comparison to a low-rank approximation based
  on the Nystr\"om method (see Section~\ref{sec:nystrom}). We replicate the two
  experiments from \citet{Galen13}.  The task is to measure the test
  correlation between canonical variables computed on some training data.  The
  participating datasets are MNIST and XRMB \citep{Galen13}.

  For the MNIST dataset, we learn correlated representations between the left
  and right halves of the MNIST images \citep{MNIST}.  Each image has a width
  and height of 28 pixels; therefore, each of the two views of CCA consists on
  392 features. We use 54000 random samples for training, 10000 for testing
  and 6000 to cross-validate the parameters of CCA and DCCA.  For the X-Ray
  Microbeam Speech (XRMB) dataset, we learn correlated representations of
  simultaneous acoustic and articulatory speech measurements
  \citep{Galen13}.  The articulatory measurements describe the position of
  the speaker's lips, tongue and jaws for seven consecutive frames, yielding a
  112-dimensional vector at each point in time; the acoustic measurements are
  the MFCCs for the same frames, producing a 273-dimensional vector for each
  point in time. We use 30000 random samples for training, 10000 for testing
  and 10000 to cross-validate the parameters of CCA and DCCA.

  \begin{table}
  \caption[Results for CCA, DCCA, and RCCA]{Sum of largest test canonical correlations and running times by
  all CCA variants in the MNIST and XRMB datasets.}
  \begin{center}
    \resizebox{0.49\textwidth}{!}{
    \begin{tabular}{|p{1.1cm}|p{1.25cm}|p{1.25cm}|p{1.25cm}|p{1.25cm}|}
      \multicolumn{5}{c}{RCCA on \textbf{MNIST} (50 largest canonical correlations)}\\
      \hline
      \multirow{2}{*}{$m_x,m_y$}  & \multicolumn{2}{|c|}{Fourier} &
      \multicolumn{2}{|c|}{Nystr\"om}\\\cline{2-5}
           & corr. & minutes& corr. & minutes \\\hline
      1000 & 36.31 & 5.55   & 41.68 & 5.29   \\\hline
      2000 & 39.56 & 19.45  & 43.15 & 18.57  \\\hline
      3000 & 40.95 & 41.98  & 43.76 & 41.25  \\\hline
      4000 & 41.65 & 73.80  & 44.12 & 75.00  \\\hline
      5000 & 41.89 & 112.80 & 44.36 & 115.20 \\\hline
      6000 & 42.06 & 153.48 & \textbf{44.49} & 156.07 \\\hline
    \end{tabular}
    }
    \hfill
    \resizebox{0.49\textwidth}{!}{
      \begin{tabular}{|p{1.1cm}|p{1.25cm}|p{1.25cm}|p{1.25cm}|p{1.25cm}|}
        \multicolumn{5}{c}{RCCA on \textbf{XRMB} (112 largest canonical correlations)}\\
        \hline
        \multirow{2}{*}{$m_x,m_y$}  & \multicolumn{2}{|c|}{Fourier} & \multicolumn{2}{|c|}{Nystr\"om}\\\cline{2-5}
             & corr. & minutes& corr.  & minutes\\\hline
        1000 & 68.79 & 2.95   & 81.82  & 3.07 \\\hline
        2000 & 82.62 & 11.45  & 93.21  & 12.05 \\\hline
        3000 & 89.35 & 26.31  & 98.04  & 26.07 \\\hline
        4000 & 93.69 & 48.89  & 100.97 & 50.07 \\\hline
        5000 & 96.49 & 79.20  & 103.03 & 81.6 \\\hline
        6000 & 98.61 & 120.00 & \textbf{104.47} & 119.4 \\\hline
      \end{tabular}
    }
    \vskip 0.4 cm
    \resizebox{0.49\textwidth}{!}{
      \begin{tabular}{|p{1.5cm}|p{1.25cm}|p{1.25cm}|p{1.25cm}|p{1.25cm}|}
      \cline{2-5}
      \multicolumn{1}{c|}{} & \multicolumn{2}{|c|}{linear CCA} & \multicolumn{2}{|c|}{DCCA} \\\cline{2-5}
      \multicolumn{1}{c|}{} & corr. & minutes & corr. & minutes\\\hline
      \bf MNIST & 28.0 & 0.57 & 39.7 & 787.38  \\\hline
      \bf XRMB  & 16.9 & 0.11 & 92.9 & 4338.32 \\\hline
      \end{tabular}
    }
  \end{center}
    \label{table:real}
  \end{table}

  Table~\ref{table:real} shows the sum of the largest canonical correlations
  (corr.) in the test sets of both MNIST and XRMB, obtained by each CCA
  variant, as well as their running times (minutes, single 1.8GHz core).  For
  RCCA, we use representations based on Nystr\"om and Mercer random features
  (see Section~\ref{sec:random-features}).\footnote{The theorems
  presented in this chapter only apply to Mercer random features.} Given enough
  random projections ($m=m_x=m_y$), RCCA is able to extract the most test
  correlation while running drastically faster than DCCA.  Moreover, when using
  random Mercer features, the number of parameters of the RCCA model is up to
  two orders of magnitude lower than for DCCA.
  
  We tune no parameters for RCCA: the kernel widths were set using the median
  heuristic (see Section~\ref{sec:kernelex}), and CCA regularization is
  implicitly provided by the use of random features (and thus set to
  $10^{-8}$).  On the contrary,  DCCA has ten parameters (two autoencoder
  parameters for pre training, number of hidden layers, number of hidden units
  and CCA regularizers for each view), which were cross-validated using the
  grids described in \citet{Galen13}.  Cross-validating RCCA parameters did not
  improve our results.

\subsection{Learning using privileged information}\label{sec:lupi}
  In Vapnik's \emph{Learning Using Privileged Information} (LUPI) paradigm
  \citep{Vapnik09} the learner has access to a set of \emph{privileged}
  features or information $ X_\star$, exclusive of training time, that he would
  like to exploit to obtain a better classifier for test time.
  Although we will discuss the problem of learning using privileged
  information in Section~\ref{sec:iclr}, we now test the capabilities of RCCA
  to address this problem. To this end, we propose the use RCCA to construct a
  highly correlated subspace between the regular features $ X$ and the
  privileged features $ X_\star$, accessible at test time through a nonlinear
  transformation of $X$. 
  
  We experiment with the \emph{Animals-with-Attributes}
  dataset\footnote{\url{http://attributes.kyb.tuebingen.mpg.de/}}.  In this
  dataset, the regular features $ X$ are the SURF descriptors of $30000$
  pictures of $35$ different animals; the privileged features $ X_\star$ are
  $85$ high-level binary attributes associated with each picture (such as
  \emph{eats-fish} or \emph{can-fly}). To extract information from $ X_\star$
  at training time, we build a feature space formed by the concatenation of the
  $85$, five-dimensional top canonical variables
  associated with $\mathrm{RCCA}( X, [ X_\star^{(i)},  y])$, $i \in \{ 1,
  \ldots, 85\}$.  The vector $ y$ denotes the training labels.

  \begin{figure}
    \begin{center}
    \includegraphics[width=0.7\linewidth]{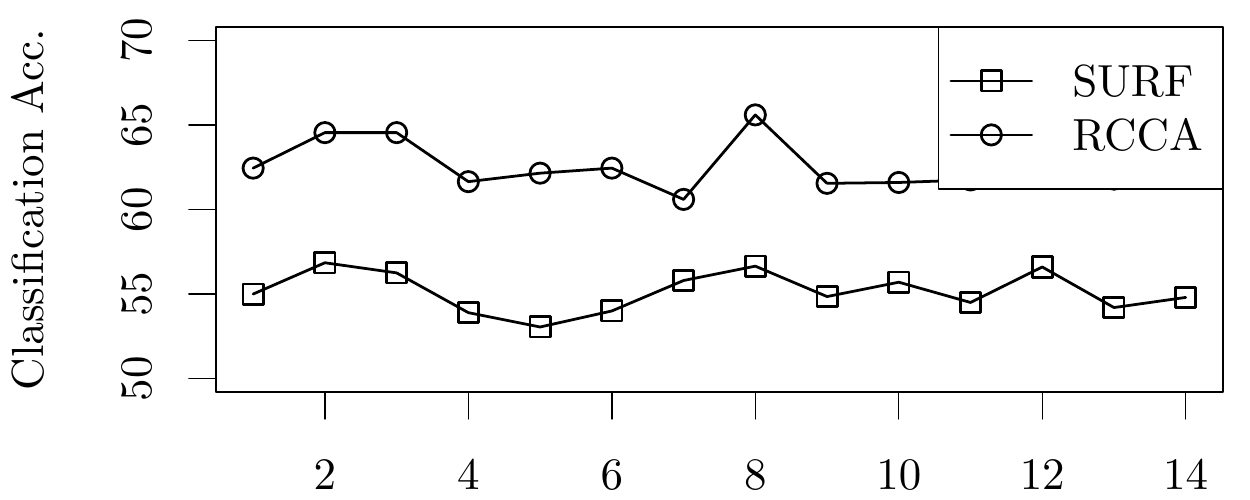}
    \end{center}
    \vskip -0.4 cm
    \caption{Results for the LUPI experiments.}
    \label{fig:animals}
  \end{figure}

  We perform 14 random training/test partitions of $1000$ samples each. Each
  partition groups a random subset of $10$ animals as class ``$0$'' and a
  second random subset of $10$ animals as class ``$1$''. Hence, each experiment 
  is a different, challenging binary classification problem. Figure
  \ref{fig:animals} shows the test classification accuracy of a linear SVM when
  using as features the images' SURF descriptors or the RCCA ``semiprivileged''
  features. As a side note, directly using the high-level attributes yields
  $100\%$ accuracy.  The cost parameter of the linear SVM is cross-validated on
  the grid $[10^{-4}, \ldots, 10^4]$. We observe an average improvement of
  $14\%$ in classification when using the RCCA basis instead of the image
  features alone.  Results are statistically significant respect to a paired
  Wilcoxon test on a $95\%$ confidence interval. The SVM+ algorithm
  \citep{Vapnik09} did not improve our results when compared to regular SVM
  using SURF descriptors.

\subsection{The randomized dependence coefficient}\label{sec:rdcexps}
  We perform experiments on both synthetic and real-world data to validate the
  empirical performance of RDC as a measure of statistical dependence.
  
  Concerning parameter selection, for RDC we set the number of random features
  to $k=20$ for both random samples, and observed no significant improvements
  for larger values.  The random feature bandwidth $\gamma$ is set to a linear
  scaling of the input variable dimensionality $d$. Note that the stability
  of RDC can be improved by allowing a larger amount of random features, and
  regularizing the RCCA step using cross-validation
  (Section~\ref{sec:model-selection-2}). In all our experiments $\gamma =
  \frac{1}{6d}$ worked well. On the other hand, HSIC and
  CHSIC (HSIC on copula) use Gaussian kernels $k( z,  z') = exp(-\gamma\|  z-
  z'\|_2^2)$ with $\gamma$ set using the median heuristic.  For MIC, the
  search-grid size is $B(n) = n^{0.6}$, as recommended in \citep{Reshef11}.
  The tolerance of ACE is $\epsilon = 0.01$, the default value in the R
  package \texttt{acepack}.

  \subsubsection{Resistance to additive noise}
  We define the \emph{power} of a measure of dependence as its ability to discern
  between dependent and independent samples that share equal marginal distributions. We follow the experiments of 
  Simon and
  Tibshirani\footnote{\url{http://www-stat.stanford.edu/~tibs/reshef/comment.pdf}},
  and choose 8 bivariate association patterns, depicted
  inside boxes in Figure \ref{fig:power}. For each of the 8 association
  patterns, we generate 500 repetitions of 500 samples, in which the input
  sample is uniformly distributed on the unit interval. Next, we regenerated
  the input sample randomly, to generate independent versions of each sample
  with equal marginals. Figure \ref{fig:power} shows the power for the discussed
  nonlinear measures of dependence as the variance of some zero-mean Gaussian
  additive noise increases from $1/30$ to $3$.  RDC shows worse performance in
  the linear association pattern due to overfitting, and in the
  step-function due to the smoothness prior induced by the Gaussian random features.
  On the other hand, RDC shows good performance in nonfunctional patterns. As a
  future research direction, it would be interesting to analyze the separate
  impact of copulas and CCA in the performance of RDC (see a similar discussion
  by \citet{Gretton05}), and to cross-validate the parameters of all the
  competing measures of dependence.

  \begin{figure}
  \begin{center}
  \includegraphics[width=\textwidth]{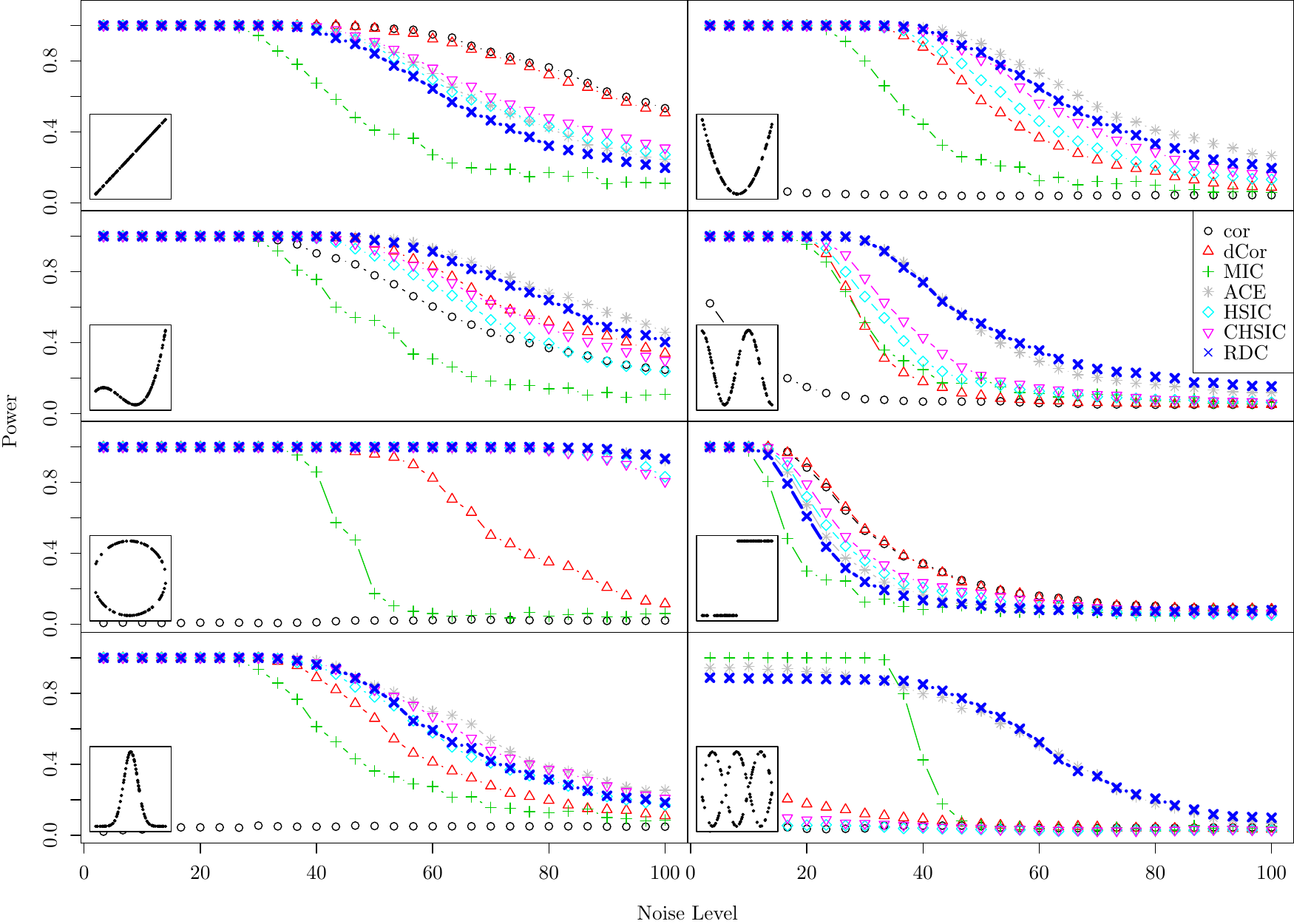}
  \end{center}
  \caption[Results of power for different measures of dependence.]{Power of
  discussed measures on example bivariate association patterns as noise
  increases. Insets show the noise-free form of each association pattern.}
  \label{fig:power}
  \end{figure}
  
  \subsubsection{Statistic semantics} 
  Figure \ref{fig:pairs} shows RDC, ACE, dCor, MIC, Pearson's $\rho$,
  Spearman's rank and Kendall's $\tau$ dependence estimates for 14 different
  associations of two scalar random samples.  RDC is close to one on
  all the proposed dependent associations, and is close to zero
  for the independent association, depicted last.  When the associations are
  Gaussian (first row), RDC is close to the absolute value Pearson's correlation
  coefficient, as requested by the seventh property of R\'enyi.
  
  \begin{figure}
  \centering
  \vskip .3 cm
  \includegraphics[width=\textwidth]{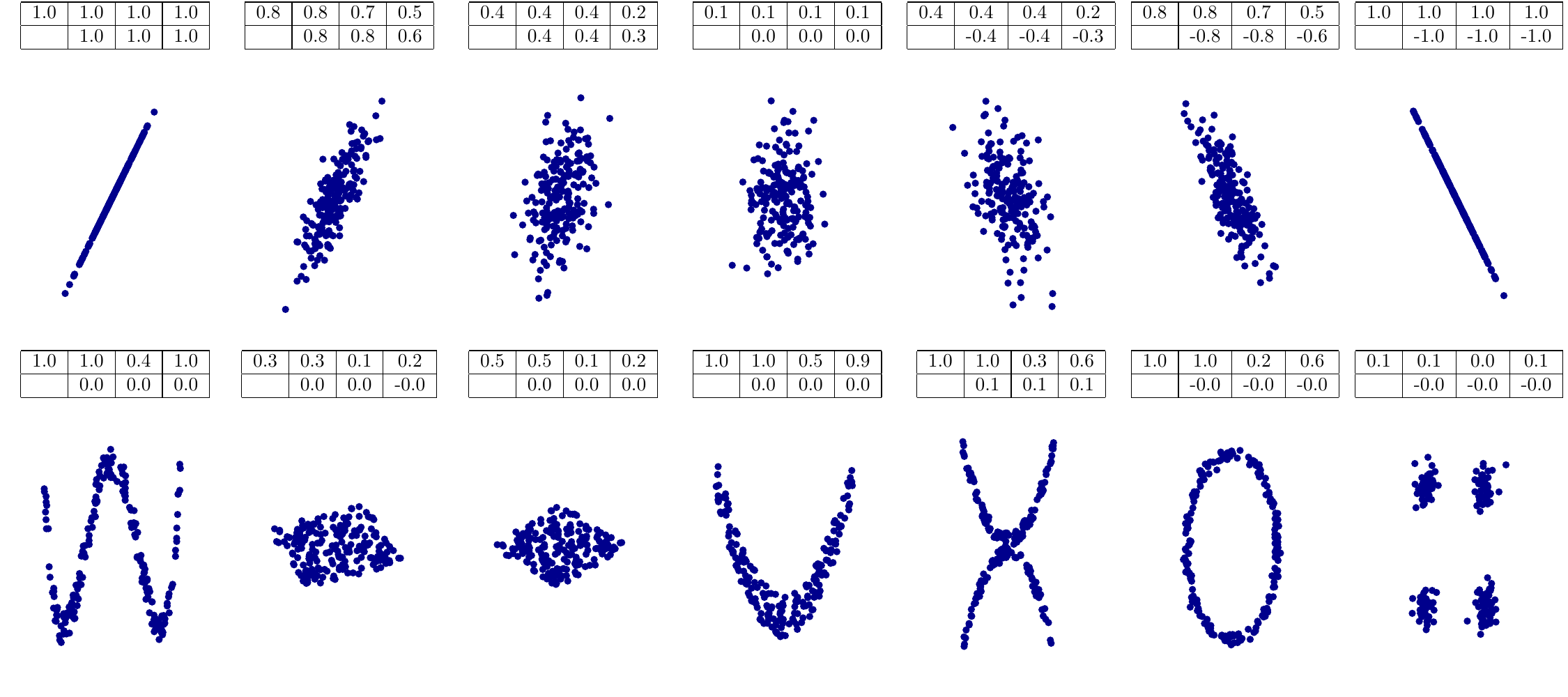}
  \vskip -.3 cm
  \caption[Dependence measures scores on different data]{RDC, ACE, dCor, MIC, Pearson's $\rho$, Spearman's rank and Kendall's
  $\tau$ estimates (numbers in tables above plots, in that order).}
  \label{fig:pairs}
  \end{figure}
 
  \subsubsection{Computational complexity}
  Table \ref{fig:times} shows running times for the considered nonlinear
  measures of dependence on scalar, uniformly distributed, independent samples
  of sizes $\{10^3, \ldots, 10^6\}$, when averaged over 100 runs. We cancelled
  all simulations running over ten minutes.  The implementation
  of  Pearson's $\rho$, ACE, dCor \citep{Szekely07}, KCCA \citep{Bach02} and
  MIC is in C, and the one of RDC, HSIC and CHSIC is in R.
  
  \begin{table}
    \caption[Running time results for measures of dependence.]{Average running
    times (in seconds) for measures of dependence on vs sample sizes.}
    \vskip 0.3 cm
      \resizebox{\textwidth}{!}{
      \begin{tabular}{l|cccccccc}
      \hline
     \bf sample size& \bf Pearson's $\rho$ & \bf RDC       & \bf ACE       & \bf KCCA & \bf  dCor     & \bf HSIC     & \bf CHSIC    & \bf MIC      \\\hline\hline
      1,000     & 0.0001  & 0.0047   &  0.0080   &  0.402  &  0.3417  & 0.3103  &  0.3501 & 1.0983 \\\hline
      10,000    & 0.0002  & 0.0557   &  0.0782   &  3.247  &  59.587  & 27.630  &  29.522 &  ---   \\\hline
      100,000   & 0.0071  & 0.3991   &  0.5101   &  43.801 &  ---     &  ---    &  ---    &  ---   \\\hline
      1,000,000 & 0.0914  & 4.6253   &  5.3830   &  ---    &  ---     &  ---    &  ---    &  ---    \\\hline
      \hline
      \end{tabular}
    }
  \label{fig:times}
  \end{table}
  
  \subsubsection{Feature selection in real-world data.}
  We performed greedy feature selection via dependence maximization
  \citep{Song12} on real-world datasets. More specifically, we 
  aim at constructing the subset of features $\G \subset \X$ that
  minimizes the Normalized Mean Squared Error (NMSE) of a Gaussian
  process. We do so by selecting the feature $x_{:,i}$ maximizing
  dependence between the feature set $\G_{i} = \{\G_{i-1} ,
  x_{:,i}\}$ and the target variable $y$ at each iteration $i \in \{1, \ldots
  10\}$, such that $\G_0 = \{ \emptyset \}$ and $x_{:,i} \notin
  \G_{i-1}$.
  
  We considered 12 heterogeneous datasets, obtained from the UCI dataset
  repository\footnote{\url{http://www.ics.uci.edu/~mlearn}}, the Gaussian process
  web site Data\footnote{\url{http://www.gaussianprocess.org/gpml/data/}} and the
  Machine Learning data set repository\footnote{\url{http://www.mldata.org}}.
  All random training and test partitions are disjoint and of equal size.
  
  Since $\G$ can be multi-dimensional, we compare RDC to the multivariate 
  methods dCor, HSIC and CHSIC. Given their quadratic computational demands, dCor,
  HSIC and CHSIC use up to $1000$ points when measuring dependence. This
  constraint only applied on the \texttt{sarcos} and \texttt{abalone} datasets.
  Results are averages over $20$ random training/test partitions.
  
  \begin{figure}
    \includegraphics[width=\textwidth]{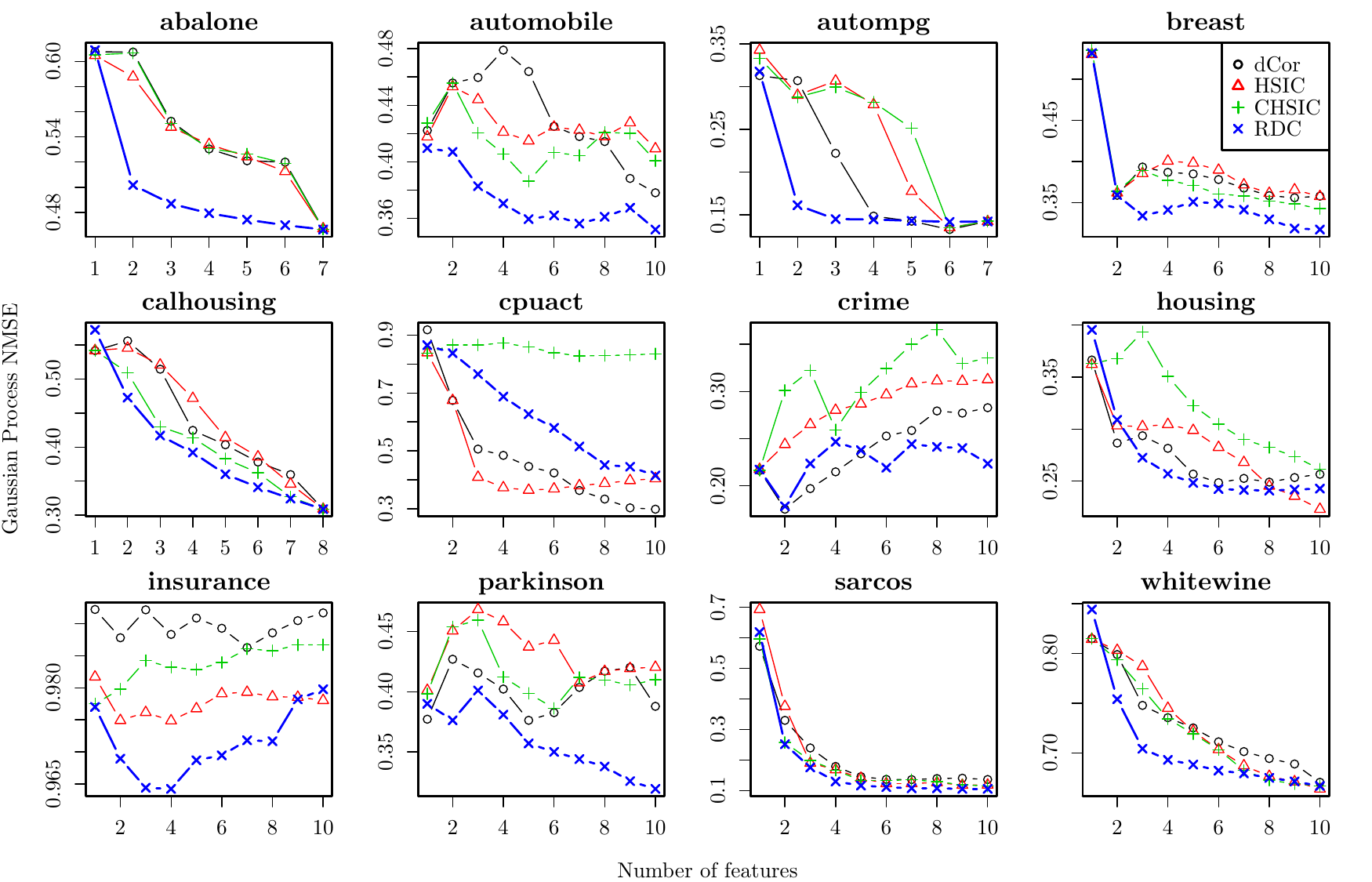}
    \caption{Feature selection experiments on real-world datasets.}
    \label{fig:featsel}
  \end{figure}
  
  Figure \ref{fig:featsel} summarizes the results for all datasets and
  algorithms as the number of selected features increases. RDC performs best in
  most datasets, using a much lower running time than its contenders. 
  In some cases, adding more features damages test accuracy. This is
  because the added features may be irrelevant, and the corresponding increase
  in dimensionality turns the learning problem harder.

\section{Proofs}\label{sec:rca-proofs}

\subsection{Theorem \ref{thm:pca}}\label{proof:thm:pca}
\begin{proof} 
  Observe that $\E{}{\hat{\bm K}} =  K$, and that ${\hat{ K}}$ is the sum of the
  $m$ independent matrices ${\hat{ K}_i}$, where the randomness is over random
  feature sampling. This is because the random features are independently and
  identically distributed, and the data matrix $ X$ is constant.  Consider the
  error matrix
  \begin{align*}
    E   = {\hat{ K}} -  K = \sum_{i=1}^m  E_i, \text{ where }
    E_i = \frac{1}{m} ({\hat{ K}^{(i)}}-  K),
  \end{align*}
  and $\E{}{\bm E_i} =  0$ for all $1 \leq i \leq m$. Since we are using bounded
  kernels and features (Definition~\ref{remark:rca}), it follows that there exists
  a constant $B$ such that $\| z\|^2 \leq B$. Thus, 
  \begin{align*}
    \|  E_i \| &= \frac{1}{m} \| z_i z_i^\top - \E{}{z z^\top}\|\\
    &\leq \frac{1}{m}(\| z_i\|^2 + \E{}{\| \bm z\|^2})
    \leq \frac{2B}{m},
  \end{align*}
  because of the triangle inequality on the norm and Jensen's inequality on the
  expected value.  To bound the variance of $ E$, bound  first the variance
  of each of its summands $ E_i$ and observe that 
  $\E{}{{\bm z}_i{\bm z}_i^\top}={K}$:
  \begin{align*}
    \E{}{\bm E_i^2} &= \frac{1}{m^2} \E{}{(\bm  z_i \bm z_i^\top
      -  K)^2}\\
    &=\frac{1}{m^2} \E{}{\| \bm z_i\|^2 \bm z_i \bm z_i^\top - 
      \bm z_i \bm z_i^\top K - {K}{\bm z}_i{\bm z}_i^\top
      +  K^2}\\ 
    &\preceq \frac{1}{m^2} \left[
      B K - 2 K^2 +  K^2\right] \preceq \frac{B
      K}{m^2}.
  \end{align*}
  Next, taking all summands $ E_i$ together we obtain
  \begin{equation*}
    \|\E{}{\bm E^2}\| \leq \left\|\sum_{i=1}^m \E{}{
      \bm E_i^2}\right\| \leq \frac{1}{m}{B\| K\|},
  \end{equation*}
  where the first inequality follows by Jensen. We can now invoke the matrix
  Bernstein inequality (Theorem \ref{thm:matrix-bernstein}) on $ E - \E{}{\bm E}$
  and obtain the bound:
  \begin{equation*}
    \E{}{\| {\hat{\bm K}} -  K\|}\leq \sqrt{\frac{3B\| K\|\log n}{m}} +
    \frac{2B\log n}{m}.
  \end{equation*}
\end{proof}

\subsection{Theorem \ref{thm:cca}}\label{proof:thm:cca}
\begin{proof}
We are looking after an upper bound on the norm of the matrix
\begin{align}
   &\begin{pmatrix}
     0 & (\hat{K}_x+n\lambda I)^{-1} \hat{K}_x \hat{K}_y (\hat{K}_y+n\lambda I)^{-1}\\
     (\hat{K}_y+n\lambda I)^{-1}\hat{K}_y\hat{K}_x(\hat{K}_x+n\lambda I)^{-1} & 0
   \end{pmatrix}-\nonumber\\
   &\begin{pmatrix}
     0 & (K_x+n\lambda I)^{-1} K_x K_y (K_y+n\lambda I)^{-1}\\
     (K_y+n\lambda I)^{-1}K_yK_x(K_x+n\lambda I)^{-1} & 0
   \end{pmatrix},\label{eq:normtarget}
\end{align}
where the identity matrices have canceled out. The norm of this matrix is upper
bounded by the sum of the norms of each block, due to the triangle inequality.
Therefore, we first bound the norm of 
\begin{equation}
    (\hat{K}_y+n\lambda I)^{-1} \hat{K}_y \hat{K}_x (\hat{K}_x+n\lambda I)^{-1}-
    (K_y+n\lambda I)^{-1} K_y K_x (K_x+n\lambda I)^{-1}\label{eq:normblock1}.
\end{equation}
The other block is bounded analogously. We follow a similar argument to
\citet{fukumizu2007statistical}. Start by observing that \eqref{eq:normblock1} equals
\begin{align}
&\left[ (\hat{K}_y+n\lambda I)^{-1}-(K_y+n\lambda I)^{-1}\right] \hat{K}_y \hat{K}_x (\hat{K}_x + n\lambda I)^{-1}\label{eq:norm1}\\
&+ (K_y + n \lambda I)^{-1} (\hat{K}_y\hat{K}_x-K_yK_x) (\hat{K}_x+n\lambda I)^{-1}\label{eq:norm2}\\
&+ (K_y + n\lambda I)^{-1} K_y K_x \left[ (\hat{K}_x+ n\lambda I)^{-1}-(K_x + n \lambda I)^{-1}\right]\label{eq:norm3}.
\end{align}
Next, use the identity 
\begin{equation*}
  A^{-1}-B^{-1} = \left[B^{-1}(B^2-A^2) + (A-B)\right] A^{-2}
\end{equation*}
to develop \eqref{eq:norm1} as
\begin{align}
  &\left\lbrace(K_y+n\lambda I)^{-1}\left[(K_y+n\lambda I)^2-(\hat{K}_y+n\lambda I)^2\right]+ (\hat{K}_y-K_y)\right\rbrace(\hat{K}_y+n\lambda I)^{-1} \times\label{eq:norm112}\\
  &(\hat{K}_y+n\lambda I)^{-1}\hat{K}_y\hat{K}_x(\hat{K}_x+n\lambda
  I)^{-1}.\label{eq:norm12}
\end{align}
The norm of \eqref{eq:norm112} can be upper-bounded using the fact that
\begin{align*}
  \|A^2-B^2\| &= \frac{\|(A+B)(A-B)+(A-B)(A+B)\|}{2} \\
              & \leq \|(A+B)(A-B)\|\\
              & \leq \|A+B\|\|A-B\|\\
              & \leq \left(\|A\|+\|B\|\right)\|A-B\|,
\end{align*}
to obtain
\begin{align}
 &\frac{1}{n^2 \lambda^2} \left(\left\|\hat{K}_y+n\lambda I\right\| +\left\|K_y + n\lambda I\right\|\right) \left\|\hat{K}_y-K_y\right\|
 + \frac{1}{n \lambda} \|\hat{K}_y - K_y\|\nonumber\\
 &\leq 
 \left(\frac{2}{n\lambda^2} + \frac{3}{n\lambda}\right) \left\|\hat{K}_y-K_y\right\|\nonumber\\
 &\leq
 \frac{3}{n}\left(\frac{1}{\lambda} + \frac{1}{\lambda^2}\right) \left\|\hat{K}_y-K_y\right\|
 \label{eq:normsol1},
\end{align}
In the previous, the second line uses the triangle inequalities $\|K_y+n\lambda
I\| \leq \|K_y\| + \|n\lambda I\|$ and $\|\hat{K}_y+n\lambda I\| \leq
\|\hat{K}_y\| + \|n\lambda I\|$, the boundedness of our kernel function and random
features (Definition~\ref{remark:rca}) to obtain $\|K_y\| \leq n$ and
$\|\hat{K}_y\| \leq n$, and the fact that $\|n \lambda I\| \leq n\lambda$.

Since the norm of \eqref{eq:norm12} is upper-bounded by $1$,
Equation~\ref{eq:normsol1} is also an upper-bound for \eqref{eq:norm1}. Similarly,
upper-bound the norm of \eqref{eq:norm3} by
\begin{equation}
   \frac{3}{n}\left(\frac{1}{\lambda} + \frac{1}{\lambda^2}\right)\left\|\hat{K}_x-K_x\right\|\label{eq:normsol3}.
\end{equation}
Finally, an upper-bound for \eqref{eq:norm2} is 
\begin{align}
\Big\|(K_y &+ n \lambda I)^{-1} (\hat{K}_y\hat{K}_x-K_yK_x) (\hat{K}_x+n\lambda I)^{-1}\Big\|\nonumber\\
&\leq \frac{1}{n^2\lambda^2} \left\|\hat{K}_y\hat{K}_x-K_yK_x\right\|\nonumber\\
&= \frac{1}{n^2\lambda^2} \left\|\hat{K}_y\hat{K}_x-K_yK_x + \hat{K}_yK_x-\hat{K}_yK_x\right\|\nonumber\\
&= \frac{1}{n^2\lambda^2} \left\|\hat{K}_y(\hat{K}_x-K_x)-({K}_y-\hat{K}_y)K_x\right\|\nonumber\\
&\leq \frac{1}{n^2\lambda^2}\left( \left\|\hat{K}_y(\hat{K}_x-K_x)\right\|+\left\|({K}_y-\hat{K}_y)K_x\right\|\right)\nonumber\\
&\leq \frac{1}{n\lambda^2}\left( \left\|\hat{K}_x-K_x\right\|+\left\|{K}_y-\hat{K}_y\right\|\right)\label{eq:normsol2}.
\end{align}
Equations \eqref{eq:normsol1}, \eqref{eq:normsol3}, and
\eqref{eq:normsol2} upper-bound the norm of \eqref{eq:normblock1} as
\begin{equation*}
  \left\lbrace \frac{3}{n}\left(\frac{1}{\lambda^2} + \frac{1}{\lambda} \right) \right\rbrace\left(
  \left\|\hat{K}_x-K_x\right\|+\left\|{K}_y-\hat{K}_y\right\|\right).
\end{equation*}
Observing that this same quantity upper-bounds the norm of the upper-right
block of \eqref{eq:normtarget} produces the claimed result.
\end{proof}

\subsection{Theorem \ref{thm:rdc}}\label{proof:thm:rdc}
\begin{proof}
We bound the two approximations (kernel and copula) separately, using the
triangle inequality:
\begin{align*}
 \|\hat{Q} - Q\| 
 &\leq
 \|\hat{Q}-\tilde{Q}\| +
 \|\tilde{Q} - Q\|,
\end{align*}
where
\begin{itemize}
  \item $Q$ operates on the true kernel and the true copula,
  \item $\tilde{Q}$ operates on the true kernel and the empirical copula,
  \item $\hat{Q}$ operates on random features and the empirical copula.
\end{itemize}
Therefore, the overall bound will be
  \begin{equation*}
    \E{}{\|\hat{\bm Q}-Q \|}\leq \left\lbrace
    \frac{3}{n}\left(\frac{1}{\lambda^2} + \frac{1}{\lambda} \right)
    \right\rbrace \left(\E{}{\|\hat{\bm K}_a - \tilde{K}_a\|} + \E{}{\|\tilde{\bm K}_b - {K}_b\|}\right),
  \end{equation*}
  where $a,b \in \{x,y\}$ are chosen to produce the worst-case upper-bound.
  The term $\E{}{\| \hat{\bm K}_a-\tilde{K}_a\|}$ is bounded as in 
  Theorem~\ref{thm:cca}. To bound $\E{}{\|\tilde{\bm K}_b - {K}_b\|}$, follow
  \begin{align*}
    \E{}{\| \tilde{\bm K}_b - {K}_b \|}
    &\leq \E{}{\sqrt{\sum_{i=1}^n \sum_{j=1}^n \left(k(\tilde{
          \bm u}_i,\tilde{\bm u}_j)-k( u_i, u_j)\right)^2}}\nonumber\\
    &\leq \E{}{\sqrt{\sum_{i=1}^n \sum_{j=1}^n \left(L_k\|\tilde{
          \bm u}_i- u_i\|+L_k\|\tilde{\bm u}_j -  u_j\|\right)^2}}\nonumber\\
    &\leq {2L_kn}\,\E{}{\sup_{1 \leq i \leq n} \|\tilde{\bm u}_i - 
    u_i\|}\nonumber\\
    &\leq {2L_kn}\,\E{}{\sup_{x\in \X} \|{T}_n(x) - 
    T(x)\|}\nonumber\\
    &\leq {2L_k} \sqrt{nd}\left(\sqrt{\pi}+\sqrt{\log 2d}\right).
  \end{align*}
  where the inequalities follow from the Frobenius norm dominating
  the operator norm, the $L_k$-Lipschitzness of the kernel function, the
  analysis of the worst difference, the generalization of the worst difference
  to the whole input domain, and applying the expectation of Bernstein's
  inequality (Theorem~\ref{thm:bernstein}) to Corollary~\ref{cor:copconv}. 
\end{proof}

\subsection{Theorem \ref{thm:rmmd}}\label{proof:thm:rmmd}
\begin{proof}
  Unfold the definitions of MMD and RMMD as
  \begin{align*}
    \left|\mathrm{MMD}(\Z, k)- \mathrm{RMMD}(\Z, k)\right| 
    &\leq
    \frac{1}{n^2} \sum_{i,j=1}^{n^2} |k( x_i,  x_j)-\hat{k}( x_i,  x_j)|\nonumber\\
    &+\frac{2}{n^2}  \sum_{i,j=1}^{n^2} |k( x_i,  y_j)-\hat{k}( x_i,  y_j)|\nonumber\\
    &+\frac{1}{n^2} \sum_{i,j=1}^{n^2}   |k( y_i,  y_j)-\hat{k}( y_i,  y_j)|\nonumber\\
  \end{align*}
  where the upper bound follows by applying the triangle inequality. The
  claim follows by noticing that our data lives in a compact set $\mathcal{S}$ of diameter $|\mathcal{S}|$, and by applying Equations \ref{eq:szabo1} and \ref{eq:szabo2}.
\end{proof}

\newpage
\section{Appendix: Autoencoders and heteroencoders}\label{sec:autoencoders}

Autoencoders \citep{Baldi89,Kramer91,Hinton06} are neural networks that learn
to produce their own input. Autoencoders are the extension of component
analysis methods to the language and tools of neural networks.  Autoencoders
are the composition of two functions: one encoder $f_e$, which maps the observed
variables into the latent explanatory factors, and one decoder $f_d$, which maps
the latent explanatory factors back into the observed variables. To learn the
encoder and the decoder functions, one minimizes the reconstruction
error
\begin{equation*}
  L(f_e,f_d; x) = \frac{1}{n} \sum_{i=1}^n \|x_i - f_d(f_e(x_i))\|,
\end{equation*}
where $f_e$ and $f_d$ are often parametrized as deep fully-connected neural
networks. If the weights of the encoder neural network are equal
to the transpose of the weights of the decoder neural network, we say that the
encoder and the decoder have \emph{tied} weights. Figure~\ref{fig:autoencoder}
illustrates an autoencoder neural network of one hidden layer, which reduces
five variables into three explanatory components. 

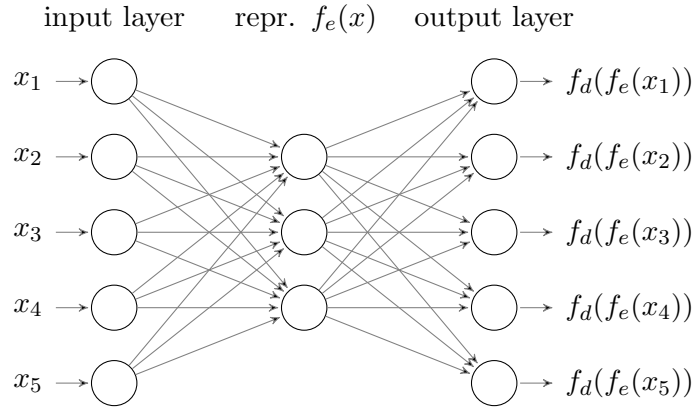
\begin{figure}
  \begin{center}
    \begin{tikzpicture}[shorten >=1pt,->,draw=black!50, node distance=2.5cm]
        \tikzstyle{every pin edge}=[<-,shorten <=1pt]
        \tikzstyle{neuron}=[circle,fill=black!25,minimum size=17pt,inner sep=0pt]
        \tikzstyle{input neuron}=[neuron, draw=black, fill=white];
        \tikzstyle{output neuron}=[neuron, draw=black, fill=white]; 
        \tikzstyle{hidden neuron}=[neuron, draw=black, fill=white];
        \tikzstyle{annot} = [text width=6em, text centered]
    
        \foreach \name / \y in {1,...,5}
          \path[yshift=0.5cm]
            node[input neuron, pin=left:$x_\y$] (I-\name) at (0,-\y) {};
    
        \foreach \name / \y in {1,...,3}
            \path[yshift=-.5cm]
                node[hidden neuron] (H-\name) at (2.5cm,-\y cm) {};
        
        \foreach \name / \y in {1,...,5}
          \path[yshift=0.5cm]
            node[output neuron, pin={[pin edge={->}]right:$f_d(f_e({x}_\y))$}] (O-\name) at (5cm,-\y) {};
        
        \foreach \source in {1,...,5}
            \foreach \dest in {1,...,3}
                \path (I-\source) edge (H-\dest);
        
        \foreach \source in {1,...,3}
            \foreach \dest in {1,...,5}
                \path (H-\source) edge (O-\dest);

        \node[annot,above=1.2 of H-1, node distance=1cm] (hl1) {repr. $f_e(x)$};
        \node[annot,left of=hl1] {input layer};
        \node[annot,right of=hl1] {output layer};
    \end{tikzpicture}
  \end{center}
  \caption[An autoencoder]{An autoencoder.}
  \label{fig:autoencoder}
\end{figure}

If unconstrained, autoencoders may learn to reconstruct their inputs by
implementing the trivial identity map $f_d(f_e(x))=x$. The following are some
alternatives to avoid this trivial case, each of them favouring one kind of
representation over another.
\begin{enumerate} 
 \item \emph{Bottleneck} autoencoders have representations $f_e(x)$ of lower dimensionality than the
 one of the inputs $x$. The transformation computed by a linear
 autoencoder with a bottleneck of size $r < d$ is the projection into the
 subspace spanned by the first $r$ principal components of the training data
 \citep{Baldi89}.
 \item \emph{Sparse} autoencoders promote sparse representations $f_e(x)$ for all $x$.
 \item \emph{Denoising} autoencoders \citep{Vincent08} corrupt the data $x$ before
 passing it to the encoder, but force the decoder to reconstruct the original,
 clean data $x$. Linear denoising autoencoders are one special case of
 heteroencoders, which solve the CCA problem \citep{Roweis99b}. 
 \item \emph{Contractive} autoencoders \citep{rifai2011contractive} penalize the norm
 of the Jacobian of the encoding transformation. This forces the encoder to be
 contractive in the neighborhood of the data, resulting into a focused
 representation that better captures the directions of variation of data and
 ignores all others.
 \item \emph{Variational} autoencoders \citep{kingma2013auto} use variational
 inference (Section~\ref{sec:variational-inference}) to learn probabilistic
 encoder and decoder functions. Variational autoencoders are also 
 generative models, as they allow the estimation of new samples from the data
 generating distribution.
\end{enumerate}

All the previous autoencoder regularization schemes allow for overcomplete
component analysis, except for bottleneck autoencoders.

  \part{Causation}
  \chapter{The language of causation}\label{chapter:language-causality}
\vspace{-1.25cm}
  \emph{This chapter is a review of well-known results.}
\vspace{1.25cm}

\noindent Chapters~\ref{chapter:generative-dependence} and
\ref{chapter:discriminative-dependence} studied the concept of
\emph{statistical dependence}.  There, we learned that when two random
variables $\bm x$ and $\bm y$ are \emph{statistically dependent}, we may 
predict expected values for $\bm y$ given values for $\bm x$ 
using the \emph{conditional
expectation}
\begin{equation*}
  \E{}{\bm y \given \bm x = x}.
\end{equation*}
Using the same statistical dependence, we may predict expected values for $\bm
x$ given values for $\bm y$ using the opposite conditional expectation
\begin{equation*}
  \E{}{\bm x \given \bm y = y}.
\end{equation*}
So, statistical dependence is a \emph{symmetric} concept: if $\bm x$ is
dependent to $\bm y$, then $\bm y$ is also dependent to $\bm x$.  Like the
tides in the sea and the orbit of the Moon, the luminosity and the warmth of a
star, the area and radius of a circle, and the price of butter and cheese. 

Yet, statistical dependences often arise due to a most fundamental
\emph{asymmetric} relationship between entities. To see this, consider the positive
dependence between high levels of blood cholesterol and heart disease. This
dependence arises because higher levels of blood cholesterol lead to higher
chances of suffering from heart disease, but not vice versa. In everyday
language, we say that ``blood cholesterol \emph{causes} heart disease''. In
causal relations, variations in the cause lead to variations in the effect, but
variations in the effect do not lead to variations in the cause.
Thus, causal relations are \emph{asymmetric}, but all we observe in statistics
are symmetric dependencies.  How can we tell the difference between dependence
and causation? And the difference between cause and effect?

\begin{remark}[Dependence does not imply causation!]
  When facing two dependent random variables, it is tempting to conclude that
  one causes the other. The scientific literature is full of statistical
  dependencies misinterpreted as causal relationships.

  \citet{Messerli} observed a strong positive correlation between the chocolate
  consumption and the amount of Nobel laureates from a given country. When
  explaining his finding, Messerli claimed that chocolate consumption causes
  the sprouting of Nobel laureates. A more reasonable explanation is due to the
  existence of a common cause, responsible for the increase in both chocolate
  consumption and research budget in a given country. For instance, the
  socioeconomic status of the said country.

  In \emph{Nature}, \citet{quinn1999myopia} claimed that sleeping with intense
  ambient light causes the development of myopia in children.  This is in fact
  not a causal relationship. On the contrary, a common cause, the parents of
  the children having myopia, is responsible for the observed
  association.  If the parents of the child have myopia, they tend to leave the
  lights on at night and, at the same time, their child tends to inherit myopia.

  More generally, spurious correlations occur between any two monotonically
  increasing or decreasing time series. One famous example is the positive
  association between the price of British bread and the level of Venetian seas
  \citep{sober2001venetian}. A dependence that, when conditioned on time, would
  most likely vanish.
\end{remark}

\section{Seeing versus doing}
The conditional expectation $\E{}{\bm y \given \bm x = x}$
is a summary of the conditional probability distribution $P(\bm y \given \bm x
= x)$. We estimate this conditional expectation in two steps. First, we observe
samples $S = \{(x_i, y_i)\}^n_{i=1}$ drawn from the joint probability
distribution $P(\bm x, \bm y)$. Second, we select or smooth the samples $S_x
\subseteq S$ compatible with the assignment $\bm x = x$, and use $S_x$ to
compute the empirical average of $\bm y$. An analogous procedure applies 
to compute $\E{}{\bm x \given \bm y = y}$. In both
cases, the procedure is \emph{observational}: as a \emph{passive} agent, we
\emph{see} and filter data, from which we compute statistics.

But, there is a difference between \emph{seeing} and \emph{doing}. To
illustrate this difference, let us now consider the case where, instead of
observing one system and summarizing its behaviour whenever $\bm x = x$
happens, we \emph{intervene} on the system and \emph{actively} force $\bm x =
x$.  We denote this \emph{intervention} by the \emph{interventional
distribution}
\begin{equation}\label{eq:intervention}
  P(\bm y \given \text{do}(\bm x = x)).
\end{equation}
The interventional distribution \eqref{eq:intervention} is in general different
from the \emph{observational distribution} $P(\bm y \given \bm x = x)$.
Intuitively, the passive filtering used to compute the observational
distribution does not control for the values that the common causes of $\bm x$
and $\bm y$ take. The distribution of this uncontrolled values will in turn
induce a bias, which translates into differences between observational and
interventional distributions. However, these biases vanish when we actively
intervene on the system.

In principle, the differences between interventional and observational
distributions can be arbitrarily large, even under arbitrarily small
interventions.  The bridge between observational and
interventional distributions will be a set of assumptions about the
causal structure between the random variables under study. These assumptions
will, in some cases, allow us to infer properties about
interventional distributions from observational distributions. This is the
power of causal inference.  Reasoning, just by
\emph{seeing}, the consequences of \emph{doing}. In another words, causation
allows to estimate the behaviors of a system under
varying or unseen environments. We will do so by placing causal assumptions
that will allow us to use observational distributions to access aspects of interventional distributions. 

\begin{example}[The difference between seeing and doing] Consider
\begin{align*}
  z_i & \sim \N(0,1),\\
  x_i & \leftarrow 5z_i,\\
  y_i & \leftarrow x_i+5z_i.
\end{align*}
If we draw $10^6$ samples from this model, we can estimate that 
\begin{equation*}
  \E{}{\bm y \given \bm x = 1} \approx 2.
\end{equation*}
This an observational expectation. We have passively observed samples drawn  from
the model, and used a regression method to estimate the
mean of $\bm y$.  On the contrary, we now put our finger in the system, and
perform the intervention $\text{do}(\bm x = 1)$.  Then, the intervened
generative model is 
\begin{align*}
  z_i &\sim \N(0,1),\\
  x_i & \leftarrow 1,\\
  y_i & \leftarrow x_i+5z_i,
\end{align*}
If we draw again $10^6$ samples, we can estimate that
\begin{equation*}
  \E{}{\bm y \given \text{do}(\bm x = 1)} \approx 1.
\end{equation*}
The interventional and observational conclusions differ!
\end{example}

\begin{remark}[Counterfactual reasoning]
  We can read interventions like \eqref{eq:intervention} as 
  \emph{contrary-to-fact} or \emph{counterfactual} questions: 
 
  \begin{center}
  ``What would have been the distribution of $\bm y$ had $\bm x = x$?''
  \end{center}
  
  \citet{lewis1973} introduced the concept of counterfactuals.
  Philosophically, counterfactuals assume the existence of a
  parallel world where everything is the same, except for the hypothetical
  intervention and its effects.  For example, the counterfactual ``had I called
  Paula, I would be dating her'' describes an alternative world,
  where everything is the same as in ours, except that I called Paula, and the
  effects of that call unfolded. By definition, counterfactuals are never
  observed, so their validity is never verified. This is a source of criticism
  \citep{dawid2000causal}. In any case, counterfactuals are one 
  concise way to state causal hypothesis.
\end{remark}

\citet{pearl2009causal} does a great job at summarizing the distinction
between statistics and causal analysis:
\begin{quote}
``... causal analysis goes one step further; its aim is to infer not only beliefs or
probabilities under static conditions, but also the dynamics of beliefs under
changing conditions, for example, changes induced by treatments or external
interventions. [..] An associational concept is any relationship that can be
defined in terms of a joint distribution of observed variables, and a causal
concept is any relationship that cannot be defined from the distribution alone.
Examples of associational concepts are: correlation, regression, dependence,
conditional independence, likelihood. [...] Examples of causal concepts are
randomization, influence, effect, confounding, ``holding constant'',
disturbance, spurious correlation, intervention, explanation, attribution.''
\end{quote}

In a nutshell, causation is one tool to describe the statistical behaviour of
a system in changing environments, where we do not necessarily observe data
from all possible environments.  The question is, how can we formalize,
identify, and exploit causation in learning? The answer, presented throughout
the rest of this chapter, will come as a extension of the theory of
probability.

\begin{remark}[Philosophy of causation] In
\emph{Metaphysics}, Aristotle (384-322 BC) categorizes the causes of phenomena
into material causes (what something is made of), formal causes (the form or
archetype of something), efficient causes (the source of change and rest in
something), and final causes (the reason why something is done).  In
\emph{Novum Organum}, Francis Bacon (1606-1625) rejects the Aristotelian view,
regarding it as nonscientific. Instead, the Baconian scientific method
searches for conditions in which the phenomena under study occurs, does not
occur, and occur in different degrees. Then, the method strips down
these conditions to necessary and sufficient causes for the phenomena.

David Hume (1711-1776) had an skeptic view on causal knowledge, as described
in his \emph{A Treatise of Human Nature}. For Hume, causal relations are one form of induction from the experience of constant conjunction of events (nearby events of type
A are usually followed by events of type B). But 
induction, from a Humean perspective, is not logically justified.  Immanuel
Kant (1724-1804) challenges Hume by considering causation a synthetic,
objective, a priori knowledge not acquired by experience. For Kant, this a
priori type of knowledge, which includes causal knowledge, is intrinsically true
and shapes the world to be what it is.

Francis Galton (1822-1911) and his student Karl Pearson (1857-1936) hinted
the relation between dependence and causation. When studying the relationship
between the size of the human forearm and head, Galton wrote that
``co-relation must be the consequence of the variations of the two organs
being partly due to common causes''.  Hans Reichenbach (1891-1953) sharpened
the relation between dependence and causation in his \emph{Principle of
Common Cause}, described in the next section.

To learn more about the philosophy of causation, we recommend the reader to
consult the monograph \citep{beebee2009oxford}.
\end{remark}

\section{Probabilistic causation} Fortunately, not all people with high levels
of cholesterol suffer from heart disease. Although high levels of cholesterol
increase the risk of heart disease, a number of other factors such as smoking,
diet, genetics, and so forth determine experiencing a cardiovascular failure or
not.  This situation is easily described using a probabilistic account of
causation: causes modify the \emph{probability} of their effects happening.

The main proposition of probabilistic causation is due to 
\citet{reichenbach56}. The cornerstone of his theory is the \emph{Principle of
Common Cause} (PCC), which states that, when two random variables $\bm x$ and
$\bm y$ are dependent, this is because either
\begin{enumerate}
  \item $\bm x$ causes $\bm y$,
  \item $\bm y$ causes $\bm x$, 
  \item there exists a third random variable $\bm z$ which is a
common cause of $\bm x$ and $\bm y$, or
  \item there exists a third random variable $\bm z$ which is a common effect
  of $\bm x$ and $\bm y$, upon which the observations are conditioned.
\end{enumerate}
Figure~\ref{fig:reichenbach} illustrates the four cases of the PCC. We refer to
common causes as \emph{confounders}. When confounders are unobserved, we call
them \emph{unobserved confounders}. Often, spurious correlations
are due to the existence of unobserved confounders. Even worse, if the
functions mapping confounders to their common effects are rich enough,
hidden confounding can reproduce any observed dependence pattern.

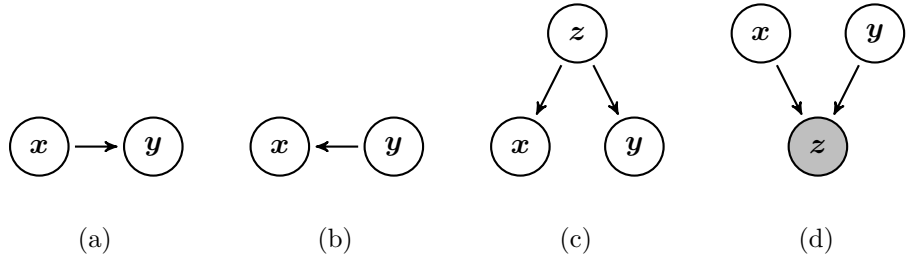
\begin{figure}
  \begin{subfigure}{0.24\textwidth}
    \begin{center}
    \begin{tikzpicture}[node distance=1cm, auto,]
     \node[punkt] (x11) at (0,0)   {$\bm x$};
     \node[punkt] (x12) at (1.5,0) {$\bm y$};
     \node[punkt,draw=white] (x13) at (1,1.5) {};
     \draw[pil] (x11) -- (x12);
    \end{tikzpicture}
    \end{center}
    \caption{}
    \label{fig:causal}
  \end{subfigure}
  \begin{subfigure}{0.24\textwidth}
    \begin{center}
    \begin{tikzpicture}[node distance=1cm, auto,]
     \node[punkt] (x11) at (0,0)   {$\bm x$};
     \node[punkt] (x12) at (1.5,0) {$\bm y$};
     \node[punkt,draw=white] (x13) at (1,1.5) {};
     \draw[pil] (x12) -- (x11);
    \end{tikzpicture}
    \end{center}
    \caption{}
    \label{fig:anticausal}
  \end{subfigure}
  \begin{subfigure}{0.24\textwidth}
    \begin{center}
    \begin{tikzpicture}[node distance=1cm, auto,]
     \node[punkt] (x11) at (0,0)   {$\bm x$};
     \node[punkt] (x12) at (1.5,0) {$\bm y$};
     \node[punkt] (x13) at (0.75,1.5) {$\bm z$};
     \draw[pil] (x13) -- (x11);
     \draw[pil] (x13) -- (x12);
    \end{tikzpicture}
    \end{center}
    \caption{} 
    \label{fig:commoncause}
  \end{subfigure}
  \begin{subfigure}{0.24\textwidth}
    \begin{center}
    \begin{tikzpicture}[node distance=1cm, auto,]
     \node[punkt] (x11) at (0,1.5)   {$\bm x$};
     \node[punkt] (x12) at (1.5,1.5) {$\bm y$};
     \node[punkt,fill=gray!50] (x13) at (0.75,0) {$\bm z$};
     \draw[pil] (x11) -- (x13);
     \draw[pil] (x12) -- (x13);
    \end{tikzpicture}
    \end{center}
    \caption{} 
    \label{fig:commoneffect}
  \end{subfigure}
  \caption[Reichenbach's Principle of Common Cause (PCC)]{According to Reichenbach's principle,
  dependencies between random variables $\bm x$ and $\bm y$ arise because
  either (a) $\bm x$ causes $\bm y$, (b) $\bm y$ causes $\bm x$, (c) $\bm x$
  and $\bm y$ share a common cause $\bm z$, (d) $\bm x$ and $\bm y$ share a
  common effect $\bm z$ on which the observations are conditioned.}
  \label{fig:reichenbach}
\end{figure}

\begin{remark}[Other interpretations of causation]
  Probabilistic causation is not free from criticism. In \citep[Chapter
  9]{beebee2009oxford}, Jon Williamson is reluctant to model logical
  relationships between variables as probabilistic cause-effect relations. For
  example, in $\bm z = \text{XOR}(\bm x, \bm y)$, the random variables $\bm x$,
  $\bm y$, and $\bm z$ are jointly independent, although both $\bm x$ and $\bm
  y$ are causes of $\bm z$. This complicates the application of the PCC.
  In opposition, Williamson
  offers an \emph{epistemic} account of causation: causal relations are how we
  interpret the world, and have nothing to do with a world free from
  interpretation.
  For other interpretations of causation (and questions on the primitivism,
  pluralism, and dispositionalism of causation), we refer the reader to the
  accessible and short introduction \citep{mumford2013causation}.
\end{remark}

In the following, we extend language of probability theory to describe causal
structures underlying high-dimensional dependence structures.

\section{Structural equation models}\label{sec:sems}
This section introduces the use of structural equation models to describe
causal relationships \citep{pearl2009causality}.

The following is a bottom-up exposition of these concepts, divided in five
parts.  First, we introduce the necessary notations to describe the structure
of directed graphs.  Second, we enumerate assumptions to link directed graphs and
probability distributions defined on their nodes, to form graphical models.
Third, we introduce a generalization of graphical models, termed structural
equation models.  Fourth, we describe the necessary assumptions to link
structural equation models and the causal relationships in the real
world.  Fifth and last, we describe how to manipulate structural equation
models to reason about the outcome of interventions and answer
counterfactual questions.

\subsection{Graphs}
We borrow some of the following from \citep[Definition
2.1]{peters2012restricted}.
\begin{enumerate}
  \item A \emph{directed graph} $G = (\mathcal{V}, \Ex)$ is a set of
  nodes $\mathcal{V} = \{v_1, \ldots, v_d\}$ and a set of edges $\Ex
  \subseteq \mathcal{V}^2$.
  \item For all $v_i, v_j \in \mathcal{V}$, $v_i \neq v_j$, we say that $v_i$
  is a \emph{parent} of $v_j$ if $(i,j) \in \Ex$, and we write $v_i \to
  v_j$.  A pair of nodes $(v_i, v_j)$ are \emph{adjacent} if either $v_i \to
  v_j$ or $v_j \to v_i$, and we write $v_i - v_j$.
  \item For all $v_j \in \mathcal{V}$, $\Pa(v_j) = \{ v_i \given
  v_i \to v_j \}$ is the set of all parents of $v_j$.
  \item The \emph{skeleton} of $G$ is the set of all edges $(i,j)$ such that
  $v_i \to v_j$ or $v_j \to v_i$.
  \item Three nodes form a \emph{v-structure} or \emph{immorality} if one of
  them is the child of the two others, which themselves are not adjacent.
  \item A \emph{path} in $G$ is a sequence $v_{i_1}, \ldots, v_{i_n}$ such that
  $v_{i_k} \to v_{i_{k+1}}$ or $v_{i_{k+1}} \to v_{i_k}$ for all $1 \leq k \leq
  n-1$ and $n \geq 2$.
  \item A path $v_{i_1}, \ldots, v_{i_n}$ in $G$ is a \emph{directed path} if
  $v_{i_k} \to v_{i_{k+1}}$ for all $1 \leq k \leq n-1$.
  \item $G$ is a \emph{Directed Acyclic Graph} (DAG) if it contains no directed
  path from $v_i$ to itself, for all $v_i \in \mathcal{V}$.
  \item A path between $v_{i_1}$ and $v_{i_n}$ is \emph{blocked} by $\Z
  \subseteq \mathcal{V} \setminus \{ v_{i_1}, v_{i_n} \}$ if
  \begin{itemize}
    \item $v_{i_k} \in \Z$ and 
    \begin{itemize}
      \item $v_{i_{k-1}} \to v_{i_k} \to v_{i_{k+1}}$ or
      \item $v_{i_{k-1}} \ot v_{i_k} \ot v_{i_{k+1}}$ or
      \item $v_{i_{k-1}} \ot v_{i_k} \to v_{i_{k+1}}$.
    \end{itemize}
    \item $v_{i_{k-1}} \to v_{i_k} \ot v_{i_{k+1}}$ and 
    $v_{i_k}$ and its descendants are not in $\Z$.
  \end{itemize}
  \item Given three disjoint subsets $\mathcal{A}, \mathcal{B}, \Z
  \subseteq \mathcal{V}$, we say that $\mathcal{A}$ and $\mathcal{B}$ are
  \emph{d-separated} by $\Z$ if all the paths between the nodes of
  $\mathcal{A}$ and the nodes of $\mathcal{B}$ are blocked by $\Z$.
  If so, we write $\mathcal{A} \dsep \mathcal{B} \given \Z$.
\end{enumerate}

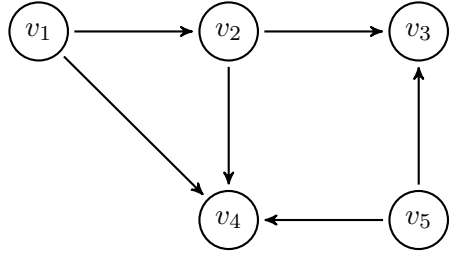
\begin{figure}
\begin{center}
\begin{tikzpicture}[node distance=1cm, auto,]
 \node[punkt] (x1) at (0,0)  {$v_1$};
 \node[punkt] (x2) at (2.5,0)  {$v_2$};
 \node[punkt] (x3) at (5,0)  {$v_3$};
 \node[punkt] (x4) at (2.5,-2.5) {$v_4$};
 \node[punkt] (x5) at (5,-2.5) {$v_5$};
 \draw[pil] (x1) -- (x2);
 \draw[pil] (x2) -- (x3);
 \draw[pil] (x1) -- (x4);
 \draw[pil] (x2) -- (x4);
 \draw[pil] (x5) -- (x4);
 \draw[pil] (x5) -- (x3);
\end{tikzpicture}
\end{center}
\caption{A directed acyclic graph.}
\label{fig:dag}
\end{figure}

Figure~\ref{fig:dag} shows a graph with $5$ nodes and $6$ edges.  The graph
contains a node $v_4$ with three parents $\Pa(v_4) = \{v_1, v_2, v_5\}$.  The
graph contains a directed path from $v_1$ to $v_3$, which is blocked by
$\Z = \{v_2\}$. The graph contains a blocked path from $v_1$ to $v_5$,
which is unblocked under $\Z = \{v_4\}$.  The node sets $\mathcal{A} =
\{v_1\}$ and $\mathcal{B} = \{v_5\}$ are d-separated by $\Z =
\{v_3\}$.  The graph is acyclic, since there is no directed path starting and
ending in the same node.

\subsection{From graphs to graphical models}
Let $G = (\mathcal{V}, \Ex)$ be a DAG, and denote by $\bm x = (\bm x_1,
\ldots, \bm x_d)$ a vector-valued random variable with joint probability
distribution $P(\bm x)$.  For all $1 \leq i \leq d$, we associate the random
variable $\bm x_i$ to the node $v_i \in \mathcal{V}$.  Then, 
\begin{enumerate}
  \item $P$ is \emph{Markov} with respect to $G$ if
  \begin{equation*}
    \mathcal{A} \dsep \mathcal{B} \given \Z \Rightarrow
    \mathcal{A} \indep \mathcal{B} \given \Z,
  \end{equation*}
  for all disjoint sets $\mathcal{A}, \mathcal{B}, \Z \subseteq
    \mathcal{V}$.  The Markov condition states that the probability
    distribution $P$ embodies all the conditional independences read from the
    $d$-separations in $G$.  The \emph{Markov condition} enables the
    factorization 
  \begin{equation}\label{eq:markov-fac}
    p(\bm x) = \prod_{i=1}^d p(\bm x_i \given \Pa(\bm x_i)),
  \end{equation}
  where $p$ is the density function of $\bm x$, and $\Pa(\bm x_i)$ is the set
  of parents of $v_i \in \mathcal{V}$.  For example, the DAG from
  Figure~\ref{fig:dag}, when associated to a random variable $\bm x = (\bm x_1,
  \ldots, \bm x_5)$, produces the Markov factorization
  \begin{equation*}
    p(\bm x = x) = p(\bm x_1) \, p(\bm x_2 \given \bm x_1, \bm x_4) \, p(\bm x_3 \given \bm
    x_2, \bm x_5) \, p(\bm x_4 \given \bm x_1, \bm x_5) \, p(\bm x_5).
  \end{equation*}
  
  Nevertheless, the probability distribution $P$ may contain further conditional
  independences not depicted in the d-separations from $G$. This nuance
  is taken care by the faithfulness condition, stated next.

  \item $P$ is \emph{faithful} to $G$ if 
  \begin{equation*}
    \mathcal{A} \dsep \mathcal{B} \given \Z \Leftarrow
    \mathcal{A} \indep \mathcal{B} \given \Z,
  \end{equation*}
  for all disjoint sets $\mathcal{A}, \mathcal{B}, \Z \subseteq
  \mathcal{V}$.  The \emph{faithfulness condition} forces the probability
  distribution $P$ to not embody any further conditional independences other
  than those encoded by the d-separations associated with the graphical
  structure of $G$. For example, the distribution 
  \begin{equation*}
    p(\bm x_1) \, p(\bm x_2 \given \bm x_1) \, p(\bm x_3 \given \bm
    x_2, \bm x_5) \, p(\bm x_4 \given \bm x_1, \bm x_5) \, p(\bm x_5)
  \end{equation*}
  is unfaithful to the DAG in Figure~\ref{fig:dag}, since the conditional
  independence $x_2 \indep x_4 \given x_1$ does not follow from the structure
  of the graph.  This conditional independence, not depicted in the graph $G$,
  may be due to the cancellation between the effect of $x_1$ on $x_2$ and the
  effect of $x_4$ on $x_2$.  Faithfulness is in charge of protecting us from
  the existence of such spurious independences.

  \item The pair $(G,P)$ satisfies the \emph{minimality condition} if it
  satisfies the Markov condition, but any pair $(G',P)$, where $G'$ is a graph
  obtained by removing edges from $G$, does not satisfy the Markov condition.
  Faithfulness implies minimality, but not vice versa.

  \item We denote by
  \begin{equation*}
    \text{Markov}(G) = \{ P \given P \text{ is Markov with respect to } G \}
  \end{equation*}
  the \emph{Markov equivalence class} of $G$. We say that two DAGs $G_1$ are
  $G_2$ are \emph{Markov equivalent} if $\text{Markov}(G_1) =
  \text{Markov}(G_2)$. Two graphs are Markov equivalent if they
  have the same skeleton and set of immoralities \citep{Verma91}.

  Figure~\ref{fig:markov-equivalent} illustrates three different but Markov
  equivalent DAGs. These three graphs entail the same d-separations, or
  equivalently, share the same skeleton and set of v-structures.

  \item If $P$ is Markov with respect to $G$, we call the tuple $(G, P)$ a
  \emph{graphical model}.
\end{enumerate}

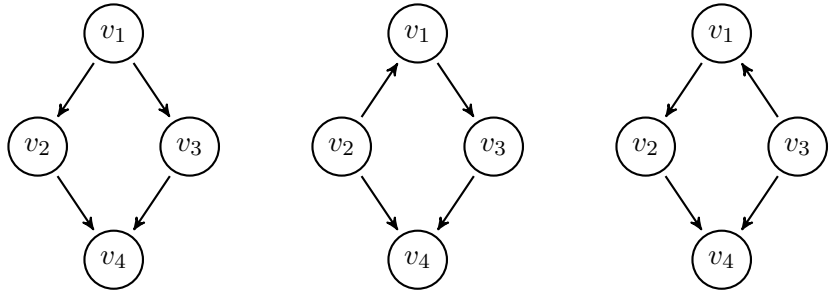
\begin{figure}
\begin{center}
\begin{tikzpicture}[node distance=1cm, auto,]
 \node[punkt] (x11) at (0,0)   {$v_1$};
 \node[punkt] (x12) at (-1,-1.5) {$v_2$};
 \node[punkt] (x13) at (1,-1.5)  {$v_3$};
 \node[punkt] (x14) at (0,-3)  {$v_4$};
 \draw[pil] (x11) -- (x12);
 \draw[pil] (x11) -- (x13);
 \draw[pil] (x12) -- (x14);
 \draw[pil] (x13) -- (x14);
 
 \node[punkt] (x21) at (4,0)  {$v_1$};
 \node[punkt] (x22) at (3,-1.5) {$v_2$};
 \node[punkt] (x23) at (5,-1.5) {$v_3$};
 \node[punkt] (x24) at (4,-3) {$v_4$};
 \draw[pil] (x22) -- (x21);
 \draw[pil] (x21) -- (x23);
 \draw[pil] (x22) -- (x24);
 \draw[pil] (x23) -- (x24);

 \node[punkt] (x31) at (8,0)  {$v_1$};
 \node[punkt] (x32) at (7,-1.5) {$v_2$};
 \node[punkt] (x33) at (9,-1.5) {$v_3$};
 \node[punkt] (x34) at (8,-3) {$v_4$};
 \draw[pil] (x31) -- (x32);
 \draw[pil] (x33) -- (x31);
 \draw[pil] (x32) -- (x34);
 \draw[pil] (x33) -- (x34);
\end{tikzpicture}
\end{center}
\caption{Three Markov equivalent DAGs.}
\label{fig:markov-equivalent}
\end{figure}

In short, the Markov condition says that every conditional independence
described by the DAG is present in the probability distribution.  Since
different DAGs can entail the same set of conditional independences, the
Markov condition is insufficient to distinguish between Markov equivalent DAGs.
The faithfulness condition assumes more to resolve this issue, saying that no
conditional independence other than the ones described by the graph is 
present in the probability distribution.  In situations where the
faithfulness condition is too restrictive, we may use the
minimality condition instead.

\subsection{From graphical models to structural equation models}

A \emph{Structural Equation Model} or SEM \citep{wright1921} is a pair
$(\mathcal{S},Q(\bm n))$, or simply $(\mathcal{S}, Q)$, where $\mathcal{S} =
\{S_1, \ldots, S_d\}$ is a set of equations 
  \begin{equation*}
    S_i : \bm x_i = f_i(\Pa(\bm x_i), \bm n_i),
  \end{equation*}
and $\bm n = (\bm n_1, \ldots, \bm n_d)$ is a vector of $d$ independent
\emph{noise} or \emph{exogenous} random variables, following the 
probability distribution $Q(\bm n)$.  If the functions $f_i$ 
are free form, call the SEM a \emph{nonparametric structural equation
model}. On the contrary, if we place assumptions on the shape of these
functions, call the SEM a \emph{restricted structural equation model}.
\citet{wright1921} introduced structural equation models to describe biological
systems, and restricted the functions $f_i$ to be linear. 

One can map structural equation models $(\mathcal{S}, Q)$ to graphical models
$(G,P)$ as follows. First, construct the graph $G = (\mathcal{V},
\Ex)$ by associating the output $\bm x_i$ in equation $S_i \in
\mathcal{S}$ to the node $v_i \in \mathcal{V}$, and drawing an edge $(j,i) \in
\Ex$ from each $v_j \in \Pa(\bm x_i)$ to $v_i$. Second, construct the
probability distribution $P(\bm x)$ by choosing the distributions of each of
the exogenous variables $\bm n_1, \ldots, \bm n_d$. Propagating these
distributions using the equations $\mathcal{S}$ produces the distributions of
each of the random variables $\bm x_i$, jointly described by $P$.  The mapping
induces a distribution $P$ Markov with respect to the graph $G$ \citep[theorem
1.4.1]{pearl2009causality}. Different structural equation models can map 
to the same graphical model or, the mapping from structural equation
models to graphical models is surjective. Simply put, structural equation
models contain strictly more information than graphical models
\citep{peters2012restricted}.

\subsection{From structural equation models to causation}\label{sec:from-sem-to-causality}

Up to know, we have described the abstract concepts of directed acyclic graph and
probability distribution, how to merge them together into a graphical model,
and how graphical models relate to structural equation models.  Yet, none
of these have causal meaning, let alone model causal relationships shaping the
real world. 

Given a graphical model $(G,P)$, the DAG $G$ describes the conditional
independences embodied in the probability distribution $P$, and allows the
factorization \eqref{eq:markov-fac}.  Although tempting, the directed edges
$\bm x_i \to \bm x_j$ do not always bear the causal interpretation ``$\bm x_i$
causes $\bm x_j$''. Graphs are just abstract tools that, together with the
Markov assumption, talk about conditional independences in
distributions.  Different Markov equivalent DAGs state the same conditional
independences, but the orientation of some of their edges can differ. This
discrepancy may lead to wrong causal claims, under a premature causal
interpretation of the edges in the graph.

The causal relationships between a collection of random variables $\bm x_1,
\ldots, \bm x_d$ are formalized as a DAG by placing two assumptions
\citep{Dawid10}.
\begin{enumerate}
  \item The \emph{representational assumption} or, the causal structure of $\bm
  x$ indeed admits an \emph{causal DAG} $G_0$. The representational assumption
  discards the consideration of cyclic graphs.
  \item The \emph{causal Markov condition} or, the d-separations in $G_0$ are
  embodied as conditional independences in the distribution $P(\bm x)$.
\end{enumerate}
So, when the DAG $G_0$ turns out to be the true causal structure of $P$, we rename the
Markov condition as the \emph{causal Markov condition}. This new
condition establishes the causal meaning of the arrows in the graph $G_0$, and
allows to draw causal inferences from properties of conditional independence.
The causal Markov condition states that the edge $\bm x_i \to
\bm x_j$ means ``$\bm x_i$ causes $\bm x_j$'', or that ``$\Pa(\bm x_i)$ are the
direct causes of $\bm x_i$''. Furthermore, the factorization
\eqref{eq:markov-fac} carries the semantics ``variables are independent when
conditioned to their direct causes''.

Armed with the causal Markov condition, we can also define \emph{causal}
structural equation models $(\mathcal{S}, Q)$, with $\mathcal{S} = \{S_1,
\ldots, S_d\}$, where the equations
\begin{equation*}
  S_i : \bm x_i = f_i(\Pa(\bm x_i), \bm n_i)
\end{equation*}
are now endowed with the causal interpretation ``the causes of $\bm x_i$ are
$\Pa(\bm x_j)$''. This is the most important distinction between a regular
graphical model, like a Bayesian network, and a causal graphical model.
While Bayesian networks are abstract descriptions of the conditional independences
embodied in a probability distribution, causal graphical models are explicit descriptions
of real-world processes, and their arrows describe the causal effects of
performing real-world interventions or experiments on their variables. 

As it happened with conditional independence, we can further ease causal
inference by placing additional, stronger assumptions
\citep{pearl2009causality}.
\begin{enumerate}
  \item The \emph{causal faithfulness condition} or, the causal DAG $G_0$ is
  faithful to the distribution $P(\bm x)$.
  \item The \emph{causal minimality condition} or, the pair $(G_0,P)$ satisfies
  the minimality condition. Causal faithfulness implies causal minimality.
  \item The \emph{causal sufficiency assumption} or, the inexistence of
  unmeasured variables $\bm x_0$ causing any of the measured
  variables $\bm x_1, \ldots, \bm x_d$.
\end{enumerate}

Although we have made some progress in the formalization of causation, we have
not yet formalized what we mean by ``$\bm x$ causes $\bm y$'' or, 
what properties does the true causal DAG $G_0$ must satisfy in
relation with the real world causal relations. The next section resolves this issue in
terms of \emph{interventions}.

\subsection{From causation to the real world}

We set two assumptions about how the world will react with
respect to interventions \citep{pearl2009causality,Dawid10}.
\begin{enumerate}
  \item The \emph{locality condition} or, under any intervention over the set
  of variables $\mathcal{A} \subseteq \mathcal{V}$, the distribution of the
  variables $\mathcal{B} = \mathcal{V}\setminus\mathcal{A}$ depends only on
  $\Pa(\mathcal{B})$, as given by the causal DAG $G_0 = (\mathcal{V},
  \Ex)$.
  \item The \emph{modularity condition} or, for all $1 \leq i \leq d$, the
  conditional distribution $p(\bm x_i \given \Pa(\bm x_i))$ is invariant with
  respect to any interventions made on the variables $\bm x\setminus
  \bm x_i$.
\end{enumerate}

Let us see what these assumptions entail. As usual, denote by $\bm x = (\bm
x_1, \ldots, \bm x)$ be a random variable with probability distribution $P$,
and density or mass function $p$. Let $\mathcal{K} \subseteq \{1, \ldots, d\}$
be the subset of the random variables over which we perform the interventions $\{
  \text{do}(\bm x_k = q_k(\bm x_k)) \}_{k \in \mathcal{K}}$, using some set of
  probability density or mass functions $\{q_k\}_{k\in\mathcal{K}}$.  Then,
  using the locality and modularity conditions, we obtain the \emph{truncated
  factorization}
\begin{align*}
  p(\bm x \given \{\text{do}(\bm x_k = q_k(\bm x_k))\}_{k \in \mathcal{K}}) &=
  \prod_{k \notin \mathcal{K}}^d p(\bm x_k \given \Pa(\bm x_k)) \prod_{k \in
  \mathcal{K}} q_k(\bm x_k).
\end{align*}
In this equation, we are forcing the random variables in $\mathcal{K}$ to
follow the interventional distributions $q_k(\bm x_k)$. The rest of the
variables and their conditional probability distributions remain unchanged, due
to the locality and modularity conditions. When we intervene on a
variable $\bm x_k$, the effects from $\Pa(\bm x_k)$ into $\bm x_k$ are no
longer present in the truncated factorization. A corollary of this is that
intervening on variables without parents is the same as conditioning on those
variables, in the observational sense. 

In the most common type of intervention, where we set the random variable $\bm
x_k = x_k$, the density or mass function $q_k = \delta_k$
\citep{pearl2009causality,peters2012restricted}. Using the Markov and
minimality conditions, together with the concept of truncated factorizations,
we are now ready to define the \emph{true causal DAG} associated with a
probability distribution $P(\bm x)$.

\begin{definition}[True causal DAG]\label{def:true-causal-dag}
  The DAG $G_0$ is the \emph{true causal DAG} of the probability distribution
  $P(\bm x)$ if $G_0$ satisfies the Markov and minimality conditions, and
  produces a truncated factorization that coincides with $p(\bm x|
  \{\text{do}(\bm x_k = q_k(\bm x_k)\}_{k\in\mathcal{K})})$ for all
  interventions $\{\text{do}(\bm x_k = q_k(\bm x_k)\}_{k\in\mathcal{K}}$
  possible in the real-world system described by $P$ \citep[Def.
  1.3]{peters2012restricted}.
\end{definition}

We now describe how to perform interventions in structural equation models
$(\mathcal{S},Q)$.  The intervened structural equation model
$(\tilde{\mathcal{S}}, Q)$ associated with the set of interventions
$\{\text{do}(\bm x_k = q(\bm x_k))\}_{k \in \mathcal{K}}$ is constructed by
replacing the equations $S_k \in \mathcal{S}$ with the equations $\tilde{S}_k :
\bm x_k = q(\bm x_k)$ in $\tilde{\mathcal{S}}$, for all $k \in \mathcal{K}$.
Thus, intervening the variable $\bm x_k$ in a structural equation
model amounts to setting such variable to be exogenous, and distributed
according to the probability density or mass function $q_k$. The intervened SEM
induces an intervened graphical model $(\tilde{G}, \tilde{P})$, Moreover, if
$(\tilde{G}, \tilde{P})$ satisfies the conditions from
Definition~\ref{def:true-causal-dag} for all possible interventions, then the
graph associated with the SEM $(\mathcal{S},Q)$ is the true causal DAG $G_0$.

As emphasized in the introduction of this chapter, intervening and observing a
system are disparate things. While intervening modifies the mechanisms of the
underlying causal graph and generates a new different probability distribution,
observing amounts to passively filtering samples from the joint distribution
and then computing statistics using those filtered samples. We finally have the
tools to illustrate this difference formally. In the following example, we
intervene in a SEM to analyze its responses, or equivalently, answer
counterfactual questions. We borrow the example from \citet[example
3.1.1]{jonasscript}.

\begin{table}
\begin{center}
\begin{tabular}{lccc}
treatment & all stones & small stones & large stones\\\hline
A & 78\% (273/350) & 93\% (81/87) & 73\% (192/263)\\
B & 83\% (289/350) & 87\% (234/270) & 69\% (55/80)\\\hline
\end{tabular}
\end{center}
\caption{Data for the kidney stones example.}
\label{table:kidney}
\end{table}

\begin{example}[Kidney stones]
  Table~\ref{table:kidney} summarizes the success rates of two different
  treatments for two different sizes of kidney stones, when tested on 350
  patients each \citep{jonasscript}. In the following, let the binary random
  variables $\bm s$, $\bm t$ and $\bm r$ mean ``kidney stone size'',
  ``treatment received'', and ``patient recovered''. 
  Overall, treatment B seems to be more successful, since
  \begin{align}
    \Pr(\bm r = 1 \given \bm t = A) = 0.78\nonumber,\\
    \Pr(\bm r = 1 \given \bm t = B) = 0.83.\label{eq:kidney-observational}
  \end{align}
  Nevertheless, treatment A is more successful than treatment B for patients with
  both small and large kidney stones, when examined separately. This is a
  classic example of Simpson's paradox: a result appearing in different groups
  of data reverses when analyzing the groups combined. This is confusing, so, in
  the unfortunate event of having kidney stones of \emph{unknown} size, what
  treatment should I prefer? 

  To answer this question, assume the causal graph in Figure~\ref{fig:kidney-dag-original}.
  \begin{figure}
  \begin{subfigure}{0.49\textwidth}
  \begin{center}
  \begin{tikzpicture}[node distance=1cm, auto,]
   \node[punkt,minimum height=2.5em] (x11) at (0,0)   {$\bm t$};
   \node[punkt,minimum height=2.5em] (x12) at (2,0)   {$\bm r$};
   \node[punkt,minimum height=2.5em] (x13) at (1,1.5) {$\bm s$};
   \draw[pil] (x11) -- (x12);
   \draw[pil] (x13) -- (x11);
   \draw[pil] (x13) -- (x12);
  \end{tikzpicture}
  \end{center}
  \caption{Assumed causal graph.}
  \label{fig:kidney-dag-original}
  \end{subfigure}
  \begin{subfigure}{0.49\textwidth}
  \begin{center}
  \begin{tikzpicture}[node distance=1cm, auto,]
   \node[punkt,minimum height=2.5em] (x11) at (0,0)   {$\bm t=t$};
   \node[punkt,minimum height=2.5em] (x12) at (2,0)   {$\bm r$};
   \node[punkt,minimum height=2.5em] (x13) at (1,1.5) {$\bm s$};
   \draw[pil] (x11) -- (x12);
   \draw[pil] (x13) -- (x12);
  \end{tikzpicture}
  \end{center}
  \caption{Intervened causal graph.}
  \label{fig:kidney-dag-intervened}
  \end{subfigure}
  \caption{Causal graph for the kidney stones example.}
  \label{fig:kidney-dag}
  \end{figure}
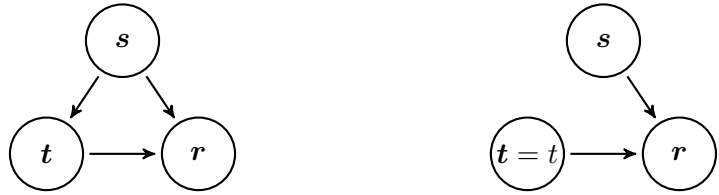
  We want to characterize the interventional
  probability mass function 
  \begin{equation*}
    p_t(\bm r) := p(\bm r \given \text{do}(\bm t = t)),
  \end{equation*}
  for $t \in \{0,1\}$. To do so, we amputate the causal graph from Figure~\ref{fig:kidney-dag-original} by
  removing the edge $\bm s \to \bm t$, and construct two new causal graphs
  $G_A$ and $G_B$ corresponding to hold $\bm t = A$ and $\bm t = B$ constant.
  Figure~\ref{fig:kidney-dag-intervened} shows the intervened graph. The two new
  probability distributions induced by these two different interventions are
  $P_A$ and $P_B$. Then, 
  \begin{align*}
    p_A(\bm r = 1) &= \sum_{s} p_A(\bm r = 1, \bm t = A, \bm s = s)\\
                   &= \sum_{s} p_A(\bm r = 1 \given \bm t = A, \bm s = s)
                   p_A(\bm t = A, \bm s = s)\\
                   &= \sum_{s} p_A(\bm r = 1 \given \bm t = A, \bm s = s)
                   p(\bm s = s)\\
                   &= \sum_{s} p(\bm r = 1 \given \bm t = A, \bm s = s)
                   p(\bm s = s)
  \end{align*}
  where the last two steps follow by the relation $\bm t \indep \bm s$ in the
  amputated graph, and the truncated factorization rule. Using
  analogous computations for $\bm t = B$ and the data in
  Table~\ref{table:kidney}, we estimate:
  \begin{align}
    p_A(\bm r = 1) &= \sum_{s} p(\bm r = 1 \given \bm t = B, \bm s = s)
    p(\bm s = s) \approx 0.832,\label{eq:kidney-interventional}\\
    p_B(\bm r = 1) &= \sum_{s} p(\bm r = 1 \given \bm t = B, \bm s = s)
    p(\bm s = s) \approx 0.782.\nonumber
  \end{align}

  Therefore, we should prefer to receive treatment A. The opposite
  decision (Equations \ref{eq:kidney-interventional} and
  \ref{eq:kidney-observational}) from the one taken by 
  just observing the data!
\end{example}

Given the true causal DAG of a system, we use truncated factorizations to
express interventional distributions as observational distributions. This
avoids the need of intervening on a system, which is often impractical. 

Counterfactual reasoning is also tool to design interventions to maximize
particular statistics of the intervened distribution, such as recovery rates
with respect to patient treatments, or revenue with respect to business
decisions.

\begin{remark}[Criticism on DAGs]
  The use of DAGs to describe the causal structure of multivariate
  systems is not free of criticism. \citet{Dawid10} surveys some of the
  shortcomings of DAGs for causal modeling, emphasizing the amount and strength
  of assumptions necessary to guarantee the correctness of the
  counterfactual answers produced from them. As an alternative, \citet{Dawid10} suggests
  a generalization of DAGs termed \emph{augmented DAGs}, where interventions are
  additional nodes in the DAG, and the concept of conditional independence
  generalizes to deal with these new types of nodes.

  A second criticism on the use of DAGs is their inherent incapacity to model
  dynamical systems with causal cycles, such as feedback loops. Those cycles
  exist, for instance, in protein interaction networks. We describe two
  solutions to the problem.  First, to sample the dynamical system
  over time, and unroll the causal cycles into duplicate graph nodes
  corresponding to the same variable at different points in time. Second, 
  to assume that data follows the equilibrium
  distribution of the dynamical system. For more details,
  consult \citep{mooij2011causal}.
\end{remark}

\section{Observational causal inference}
The previous section assumed the knowledge of the true causal DAG $G_0$, the
graph governing the causal mechanics of the system under study, the graph
giving rise to the data generating distribution $P(\bm x)$. If we know the true
causal DAG $G_0$, we can answer counterfactual questions about the potential outcome of
interventions by using truncated factorizations. But, what if we do not
know $G_0$?

The gold standard to infer $G_0$ is to perform Randomized Controlled Trials
(RCTs). Consider the question ``Does
aspirin cause relief from headache?''.  To answer such causal question using an
RCT, we first gather a large number of patients suffering from headaches, but 
equal in all their other characteristics.  Second, we divide the patients
into two groups, the treatment group and the control group.  Next, to every
person in the treatment group, we supply with an aspirin pill.  To every person
in the control group, we supply with a placebo.  Finally, we study the relief
rate in each of the two groups, and determine if
the difference between the recovery rate within the two groups is statistically significant. If it is, we
conclude that the aspirin has an effect on relieving headaches.

Unfortunately, RCTs are often expensive, unethical, or impossible to
perform: it is expensive to perform RCTs that extend over years,
it is unethical to supply experimental drugs to humans, and it is
impossible to reverse the rotation of the Earth.  Therefore, in these 
situations, we face the need of inferring causal relationships from an
observational position, by seeing but not doing.

\emph{Observational causal inference} is the problem of
recovering the true causal DAG $G_0$ associated with the probability
distribution $P$, given only samples from $P$. In the rest of this section, we
review assumptions and algorithms used for observational causal inference, as
well as their limitations.

\begin{remark}[Causal inference as a missing data problem]
  In our example RCT, we record each patient under one of the two possible
  \emph{potential outcomes}: either they took aspirin or placebo, but never
  both.  Instead, we could imagine that for each patient, we have two records:
  the observed record associated with the assigned treatment, and the
  counterfactual record associated to the treatment that was not assigned to
  the patient.  Thus, the problem of causal inference is to some extent a
  problem of missing data, where we must complete the counterfactual records.  The
  Neyman-Rubin causal model builds on this idea to develop causal inference
  techniques, such as \emph{propensity score matching} \citep{rosenbaum1983central},
  to perform causal inference in both interventional and observational data. 
\end{remark} 

\subsection{Assumptions}

The problem of observational causal inference is impossible without restricting
the class of structural equation models under study. Even when enforcing the
causal Markov condition, any distribution $P$ is Markov with respect to a large
number of different graphs. Therefore, all we can hope for is to recover a
Markov equivalence class ---the skeleton and the immoralities of the true
causal graph, lacking the orientation of some arrows--- even when using an
infinite amount of data.  In these situations, we say that the true underlying
causal graph is \emph{not identifiable}. But, we may be able to recover
a set of causal graphs which agrees with the observed data, and contains the true
causal graph.  As investigated by the different algorithms reviewed
below, placing further assumptions on $P$ reduces the size of the
equivalence class of graphs identifiable from data. The problem
is \emph{identifiable} if our assumptions allow us to uniquely recover the
true causal graph uniquely.

In a nutshell, the precision of the recovery of the true underlying causal
graph is inversely proportional to the number and strength of assumptions
that we are able to encode in the causal inference problem at hand.

\subsubsection{Independence of cause and mechanism}
Section~\ref{sec:sems} described how to exploit conditional independences to
infer causal properties about the data under study.  But, conditional
independence is not always applicable. For example, consider observational
causal inference in a system formed by two random variables, $\bm x$ and $\bm
y$. Here, our goal is to decide whether $\bm x \to \bm y$ or $\bm x \ot \bm y$.
Unfortunately, the absence of a third random variable prevents us from
measuring conditional independences, as prescribed in Section~\ref{sec:sems}.
Because of this, the research community has developed
principles for causal inference not based on conditional independence.
In the following, we present a widely used principle, the Independence between
Cause and Mechanism (ICM) assumption, useful to perform observational
cause effect inference in the two-variable case. 

To motivate the ICM assumption, recall that the joint probability distribution
of two random variables $\bm x$ and $\bm y$ admits the two conditional
decompositions
\begin{align*}
  p(\bm x, \bm y) &= p(\bm y \given \bm x) p(\bm x)\\
         &= p(\bm x \given \bm y) p(\bm y),
\end{align*}
where we may interpret the conditional distribution $p(\bm y \given \bm x)$ as a
causal mechanism mapping the cause $\bm x$ to its effect $\bm y$, and the
conditional distribution $p(\bm x \given \bm y)$ as a causal mechanism mapping
the cause $\bm y$ to its effect $\bm x$. Which of the two conditional
distributions should we prefer as the true causal mechanism?

In the spirit of Occam's razor, we prefer the conditional decomposition that
provides with the shortest description of the causal structure contained in the joint distribution $p(\bm x, \bm
y)$ \citep{lemeire2006causal}.  In terms of algorithmic
information theory, the conditional decomposition with algorithmically
independent factors has a lower Kolmogorov complexity
\citep{Janzing10}.  For instance, if the two distributions $p(\bm y
\given \bm x)$ and $p(\bm x)$ are ``independent'', then the shortest
description of $p(\bm x, \bm y)$ is the conditional decomposition $p(\bm
y \given \bm x)p(\bm x)$, and we should prefer the causal explanation $\bm x \to
\bm y$ to describe the joint distribution $p(\bm x, \bm y)$. In short, 
\begin{align}
&\text{we prefer
the causal direction under which the distribution of the cause}\nonumber\\
&\text{is independent
from the mechanism mapping the cause to the effect.}\tag{ICM}\label{ass:icm}
\end{align}
The \ref{ass:icm} assumption is often violated in the incorrect causal
direction. In our example, this means that if the factors $p(\bm x)$ and $p(\bm
y \given \bm x)$ are ``independent'', then this will be not the case for the
factors $p(\bm y)$ and $p(\bm x \given \bm y)$ \citep{scholkopf12anti}. This
asymmetry renders the observational causal inference possible.

In the previous paragraph, the word \emph{independence} appears in scare
quotes. This is because it is not obvious how to measure dependence between
distributions and functions in full generality. Nevertheless, the next section reviews
some algorithms where, thanks to parametric assumptions on the conditional
decompositions, the \ref{ass:icm} assumption becomes statistically testable.

\begin{example}[Limits of the \ref{ass:icm} assumption]
  The intuition behind the \ref{ass:icm} assumption is that laws in Nature are 
  fixed and therefore independent to what we feed into them. Although the
  \ref{ass:icm} assumption enjoys this natural interpretation, it does not hold whenever
  $\bm x \to \bm y$ but $\bm y = f(p(\bm x))$ is some statistic of the
  distribution $p(\bm x)$. For example, the spatial probability distribution of
  precious stones causes their price, the probability of a poker hand causes
  its expected reward, and the probability of genetic mutations cause the
  average phenotype expression of a population. 
\end{example}

\subsection{Algorithms}\label{sec:causal-inference-algorithms}

In the following, we review a collection of algorithms for observational causal
inference. The algorithms differ on how they operate, and the
assumptions that they place to guarantee their correctness.

\subsubsection{Conditional independence methods}

The Spirtes-Glymour-Scheines (SGS) algorithm \citep{spirtes2000causation}
assumes the representational, causal Markov, sufficiency, and faithfulness
conditions, but does not place any assumption on the 
relationships between variables.  Furthermore, SGS assumes the faithfulness
condition between the data generating distribution $P(\bm x_1, \ldots, \bm
x_d)$ and the true causal graph $G$. The SGS algorithm works as follows:
\begin{enumerate}
  \item Build $K = (\mathcal{V},
  \Ex)$, with $\mathcal{V} = \{\bm x_1, \ldots, \bm x_d\}$ and $(i,j),
  (j,i) \in \Ex$, for all $1 \leq i,j \leq d$.

  \item For each pair $(\bm x_i, \bm x_j)$, if $\exists \Z
  \subseteq \mathcal{V} \setminus \{\bm x_i, \bm x_j\}$ such that $\bm x_i
  \indep_d \bm x_j \given \Z$, remove $(i,j)$ from $\Ex$.
  \item For each structure $\bm x_i - \bm x_j - \bm x_k$ with $(i,k)
  \notin \Ex$ and no $\Z \subseteq \bm x_j \cup \mathcal{V}
  \setminus \{\bm x_i, \bm x_k\}$ such that $\bm x_i \indep_d \bm x_k \given
  \Z$, remove $(j,i)$ and $(j,k)$ from $\Ex$.
  \item Until no more edges get removed from $\Ex$, repeat
  \begin{enumerate}
    \item if $\bm x_i \to \bm x_j - \bm x_k$, $\bm x_i \not\to \bm x_k$, and $\bm
    x_i \not\ot \bm x_k$, then remove
    $(k,j)$ from $\Ex$.
    \item if there is a directed path from $\bm x_i$ to $\bm x_k$, and $\bm x_i
    \to \bm x_k$, remove $(k,j)$ from $\Ex$.
  \end{enumerate}
\end{enumerate}

The second step of the SGS algorithm performs a conditional independence test
for all possible conditioning sets $\Z \subseteq \mathcal{V} \setminus
\{\bm x_i, \bm x_j\}$ and pair of distinct nodes $(\bm x_i, \bm x_j)$. Thus,
for a node set $\mathcal{V}$ of $d$ nodes, SGS performs $2^{d-2}$ conditional
independence tests for each pair of distinct nodes. For large $d$, this
exponential amount of conditional independence tests is prohibitive,
both computationally and statistically. Computationally, because each
conditional independence test takes a nontrivial amount of computation.
Statistically, because conditional independence tests with limited
data and high-dimensional conditioning sets suffer from the curse of
dimensionality.

Because of these reasons, the SGS algorithm evolved into the PC algorithm,
which exploits a clever sorting of the variables to reduce the amount of
necessary conditional independence tests. For some problems, the PC
algorithm can not improve the computational complexity of the SGS algorithm.
The FCI algorithm is an extension of the SGS/PC algorithm to deal with
\emph{insufficiency}: causal inference on the presence 
unobserved confounders \citep{spirtes2000causation}.

The SGS is \emph{universally} consistent ---able to recover the Markov
equivalence class containing the true causal DAG for all $P$--- but not
\emph{uniformly} consistent ---there exists no upper bound on how fast SGS
recovers such result as the amount of available data increases. In fact, no
causal inference algorithm can be both universally and uniformly consistent. To
achieve uniform (but not universal) consistency, it is necessary to strengthen
the faithfulness assumption (for further discussion and references, see
\citet{peters2012restricted}).

\subsubsection{Score methods} Score methods \citep{Heckerman97} construct a
mapping from parameter vectors $\theta \in \Rm$ to the set of DAGs on $d$
nodes, and evaluate the score of each candidate by using the posterior
distribution 
\begin{equation}\label{eq:dag-posterior}
  p(\bm \theta = \theta | x_1, \ldots, x_n) = \frac{p(x_1,\ldots,
  x_n|\bm \theta = \theta)p(\bm \theta = \theta)}{p(x_1, \ldots, x_n)},
\end{equation}
where the prior distribution $p(\bm \theta)$ incorporates the available prior
knowledge to favour some DAG structures over others, and the likelihood
distribution $p(x_1, \ldots, x_n | \theta)$ measures how well does a given DAG,
parametrized by the parameter vector $\theta$, explain the data $x_1, \ldots,
x_n$, where $x_i \in \Rd$ for all $1 \leq i \leq n$. For instance, the prior
distribution can favor sparse DAGs, simple conditional
distributions, and known independences 
in the factorization of the data distribution. Score methods return the DAG
corresponding to the parameter vector maximizing the posterior distribution
\eqref{eq:dag-posterior} as the true causal DAG generating the data. One must
choose prior and likelihood distributions that allow for efficient posterior
inference; this restriction, in turn, translates into additional assumptions
about the true causal graph under search. 

\subsubsection{Additive noise models}\label{sec:anm}
The family of Additive Noise Models (ANM) 
assumes structural equation models $(\mathcal{S},Q)$ with a
set of equations $\mathcal{S} = (S_1, \ldots, S_d)$ of form 
\begin{equation*}
  S_i : \bm x_i = f_i(\Pa(\bm x_i)) + \bm n_i,
\end{equation*}
where the exogenous or noise variables $\bm n_i$ and functions $f_i$ are
absolutely continuous with respect to the Lebesgue measure for all $1 \leq i
\leq d$. 

The identifiability of additive noise models calls for additional assumptions,
either on the shape of the functions $f_i$, or the distribution of the
independent noise variables $\bm n_i$. Additive noise models are identifiable
when the representational, sufficiency, and causal Markov assumptions hold, and 
\begin{enumerate}
  \item the functions $f_i$ are linear with nonzero coefficients, and the
  noise variables $\bm n_i$ are non-Gaussian 
  \citep{shimizu2006linear}, or
  \item the functions $f_j$ are smooth and nonlinear, and the densities of both
  the equation outputs $\bm x_i$ and noise variables $\bm n_i$ are strictly
  positive and smooth \citep[condition 19]{peters2014causal}.
\end{enumerate}

On the one hand, the identifiability of the first point above is due to
Independent Component Analysis (ICA), proved using the Darmois-Skitovi\u{c} theorem,
and does not require the faithfulness condition \citep{shimizu2006linear}.  On
the other hand, the identifiability of the second point above requires a mild
technical assumption \citep[Theorem 1]{hoyer2009nonlinear}, and the causal
minimality condition. These results do not rely on conditional dependencies, so
they apply to the case where the causal DAG has only two variables. The
identifiability result in both cases full: we can not only recover the
Markov equivalence class containing the true causal DAG, but the true causal
DAG itself. 

The statistical footprint revealing the direction of causation in additive
noise models is the dependence structure between the cause and noise 
variables.  More specifically, given two random variables $\bm x$ and $\bm y$
with causal relation $\bm x \to \bm y$,
if we assume the previous conditions there exists an additive noise model 
\begin{equation*}
  \bm y = f(\bm x)+ \bm n,
\end{equation*}
in the correct causal direction, but there exists no additive noise model
\begin{equation*}
  \bm x = g(\bm y)+ \bm n',
\end{equation*}
in the anticausal direction. Due to the definition of the additive noise model,
this means that $\bm x \indep \bm n$, but it cannot be the case that $\bm y \indep \bm n'$.

Given a consistent nonparametric regression method and a consistent
nonparametric independence test (such as the ones reviewed in
Section~\ref{sec:dependence-measures}), it is possible to decide whether $\bm x
\to \bm y$ or $\bm x \ot \bm y$ on the basis of empirical data
$\{(x_i,y_i)\}_{i=1}^n \sim P^n(\bm x, \bm y)$, as $n$ tends to infinity. Under each of the two possible causal directions, proceed by 
computing a regression function from one variable to the other, and
then testing for independence between the input variable and the obtained
regression residuals. The
independence tests can be replaced with Gaussianity tests, 
to discover both linear and nonlinear causal
relationships \citep{dlp-gauss}.

\begin{figure}
  \includegraphics[width=\textwidth]{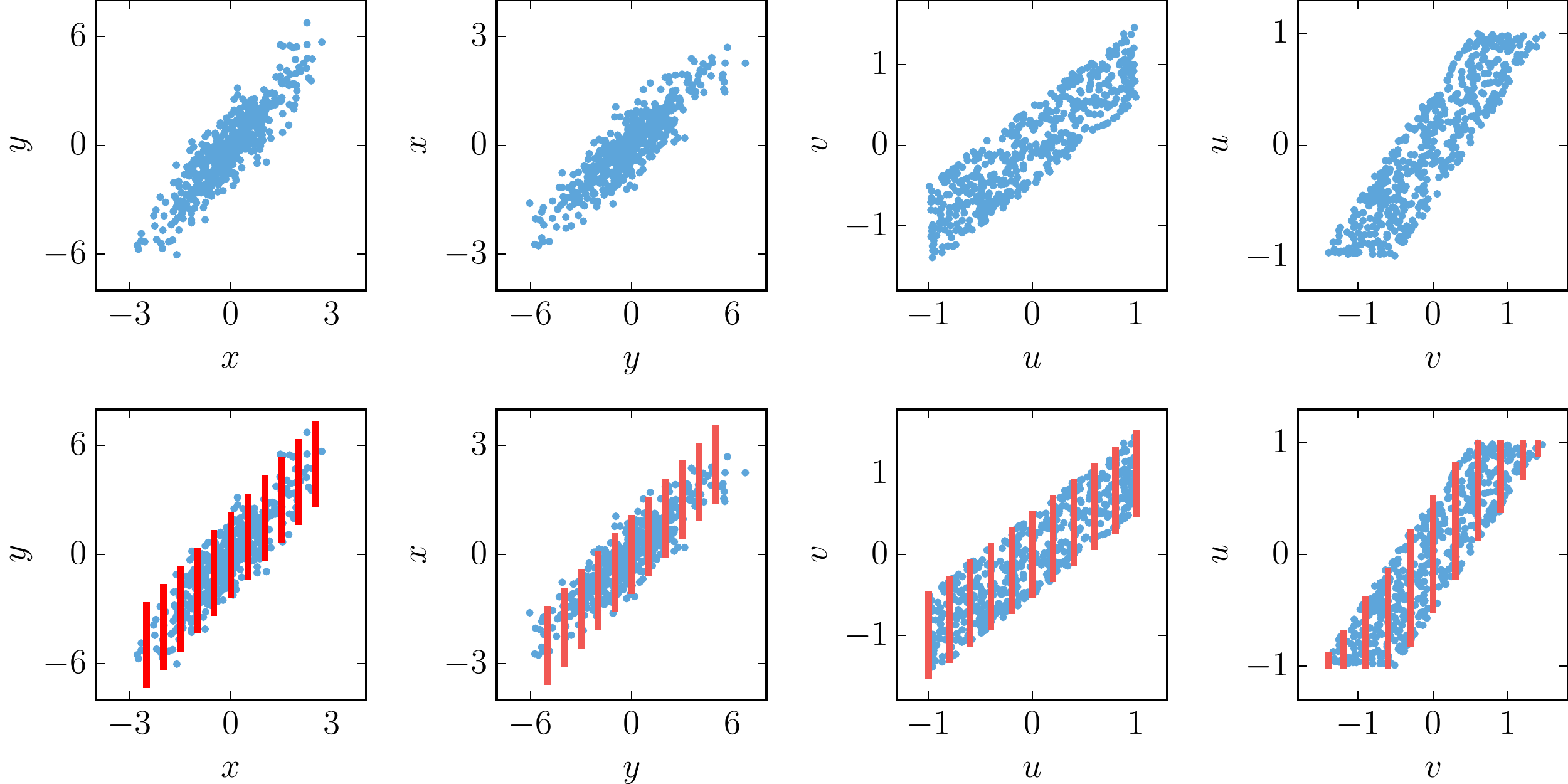}
  \caption{Examples of linear additive noise models.}
  \label{fig:anm}
\end{figure}

The additive noise model is not identifiable for structural equations 
with linear functions and Gaussian exogenous variables.  We now
exemplify this phenomena in the case of two random variables:
\begin{itemize}
  \item In the four plots from the left half of Figure~\ref{fig:anm}, we have a
  cause variable $\bm x \equiv \N$, a noise variable $\bm n \equiv
  \N(0,1)$, and a effect variable $\bm y \leftarrow 2\bm x + \bm n$. In
  this setup, the joint distribution $P(\bm x, \bm y)$ is also Gaussian. This
  means that the joint distribution is elliptical, and that there exists no
  asymmetry that we could exploit to infer the direction of causation between
  $\bm x$ and $\bm y$. Thus, the data admits an additive noise model in
  both directions, since the regression noise (depicted as red bars) is
  always independent from the alleged cause. 
  \item In the four plots from the right half of Figure~\ref{fig:anm}, we have
  a cause variable $\bm u \equiv \U[-1,+1]$, a noise variable $\bm e \equiv
  \U[-1,+1]$, and an effect variable $\bm v \leftarrow \bm u + 0.5 \bm e$.
  Therefore, this setup falls under the identifiability conditions of
  \citet{shimizu2006linear}, since the data does not admit an additive
  noise model in the incorrect causal direction $\bm u \ot \bm v$. We see this
  because the regression noise (depicted as red bars) is dependent from the
  alleged cause $\bm v$: its variance peaks at $\bm v = 0$, and shrinks as the
  absolute value of $\bm v$ increases. This asymmetry renders causal inference
  possible from observing the statistics of the data.
\end{itemize}

Additive noise models are consistent \citep{samory14}, and there exists
extensions to discrete variables \citep{peters2011causal}, latent variables
\citep{stegle2010probabilistic}, cyclic graphs
\citep{mooij2011causal,lacerda2012discovering}, and postnonlinear equations
$\bm x_i = g_i(f_i(\Pa(\bm x_i))+\bm n_i)$, where $g_i : \R \to \R$
is an additional monotone function \citep{zhang2009identifiability}. Some of
these extensions, however, sacrifice the identifiability of the problem up to
the true causal DAG, and return a equivalence class of graphs instead.

\subsubsection{Information geometric casual inference}\label{sec:igci}
Additive noise models rely on the independences between the cause variable and
the exogenous noise variable. Therefore, they are not applicable to 
discover cause-effect relationships
\begin{equation*}
  \bm y = f(\bm x),
\end{equation*}
where $f$ is an invertible function, and no noise is present.

Let us exploit the Independence between Cause and Mechanism (ICM) assumption to
achieve the identifiability of deterministic causal relations.  This is the
strategy followed by the Information Geometric Causal Inference (IGCI) method
\citep{Daniusis10,Janzing12}, which prefers the causal
direction under which the distribution of the cause is independent from the
derivative of the mechanism mapping the cause to the effect.

\begin{figure}
  \begin{center}
  \includegraphics[scale=0.9]{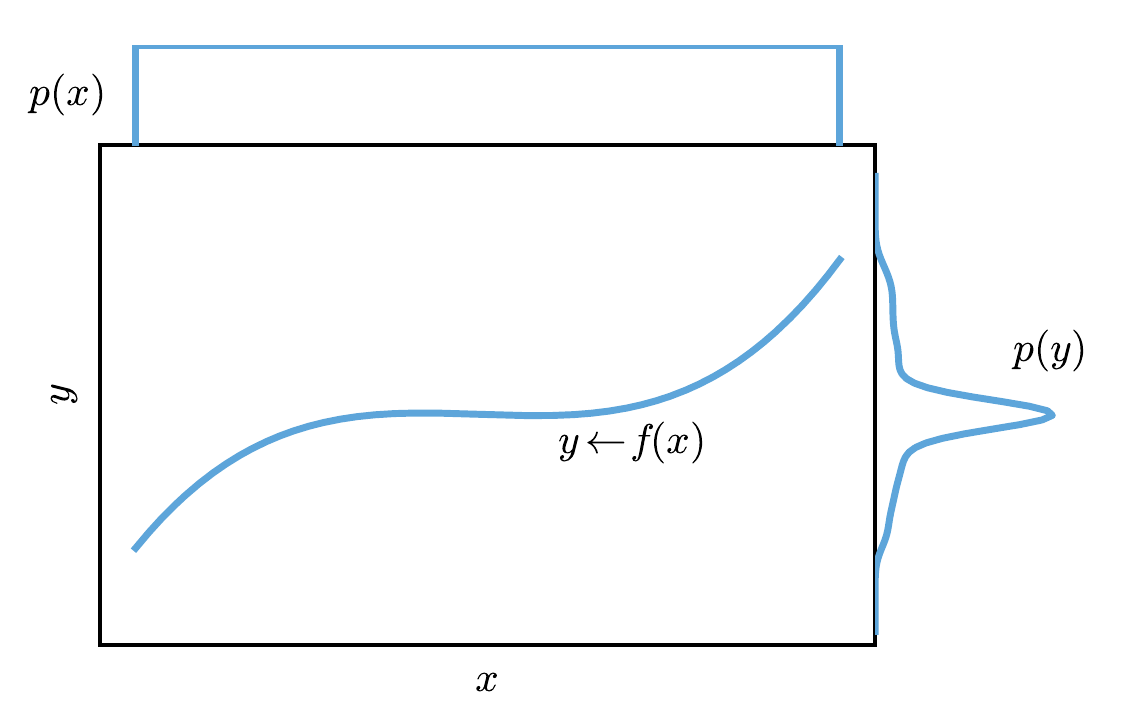}
  \end{center}
  \vskip -0.5 cm
  \caption{Example of information geometric causal inference.}
  \label{fig:igci}
\end{figure}

Figure~\ref{fig:igci} illustrates the IGCI method. Here, $\bm x \to \bm y$,
$\bm x \equiv \U[a,b]$, and $\bm y \leftarrow f(\bm x)$, where $f$ is a smooth,
invertible function. The probability density
function of the effect variable carries a footprint of the derivative of the
function $f$: regions of large density in $p(\bm y)$ correlate with
regions of small derivative $f'(\bm x)$. Therefore, if we believe in the ICM 
assumption, these correlations should look suspicious to us, and we should 
prefer the model $\bm x \to \bm y$, since the density $p(\bm x)$ carries no
footprint (is independent) from the inverse function $f'^{-1}(\bm y)$. On the
contrary, if we insist to believe in the incorrect causal relation $\bm
y \to \bm x$, we have to also believe that the correlations
between $p(\bm y)$ and $f'(\bm x)$ are spurious.

\subsubsection{Time series algorithms}
Time series data are collections of samples measured from a system over time,
presented as 
\begin{equation*}
  x_\mathcal{T} = (x_1, \ldots, x_T)
\end{equation*}
where $\mathcal{T} = \{1,\ldots, T\}$, and $x_t \in \Rd$ is the value of the
system at time $t$, for all $t \in \mathcal{T}$. The major challenge in time
series analysis is that samples $x_t$ and $x_{t'}$ measured at nearby times
depend on each other. Therefore, we can not assume that time
series data is identically and independently distributed according to some
fixed probability distribution, a condition required by all the
algorithms reviewed so far in this thesis. 

One classic way to measure causal relationships between time series is Granger
causation \citep{granger1969investigating}.  The key idea behind Granger
causation is simple. Let
\begin{align*}
  x_\mathcal{T} &= (x_1, \ldots, x_T),\\
  y_\mathcal{T} &= (y_1, \ldots, y_T),
\end{align*}
be two time series forming one isolated system. Then, $x_\mathcal{T}$
causes $y_\mathcal{T}$ if the prediction of $y_{t+1}$ given $(x_{\mathcal{T'}},
y_{\mathcal{T'}})$ is significantly better than the
prediction of $y_{t+1}$ given $(y_{\mathcal{T'}})$ for all
$\mathcal{T'} = \{1, \ldots, T-1\}$.

Granger causation was first developed in the context of linear time series, and
then extended to model causal effects between nonlinear time series (see the
references in \citet{peters2012restricted}).  Granger causation does not
account for instantaneous effects between time series, that is, when the value $x_t$ has
an effect on the value $y_t$, and it is prone to failure in the presence of
unmeasured, confounding time series. To address some of these issues,
\citet[Chapter 8]{peters2012restricted} extends the framework of structural
equation models, reviewed in Section~\ref{sec:sems}, to the analysis of time
series data.

\begin{remark}[Causality and time]
  In most natural situations, causes precede their effects in time. What is the
  exact relation between causation, space, and time? Is causal order defined in
  terms of time order, or vice versa? 

  These are challenging questions. One can define causal order to follow time order. In turn, time order can be
  described in terms of the Second Law of Thermodynamics, which states that the
  entropy of an isolated system increases over time with high probability. The
  direction of time is then established in two steps. First, we assume 
  a ``boundary condition'': the universe started in an
  configuration of extremely low entropy (See Remark~\ref{remark:boltzmann}).
  Second, we define the direction of time as the most common direction of
  increasing entropy among most isolated systems in the universe. For example,
  coffee mixing with milk or eggs turning into omelettes are examples of
  processes of increasing entropy. If we were to play a reversed video of these
  processes, it would look highly unnatural or ``anticausal'' to us.

  Alternatively, we can adopt a causal theory of time, as put forward by
  Leibniz, and define time order in terms of causal order. In modern terms,
  follow Reichenbach's principle of common cause: if a random variable $\bm z$
  is a common cause of two other random variables $\bm x$ and $\bm y$, we
  conclude that $\bm z$ happened before $\bm x$ and $\bm y$.
\end{remark}

\subsection{Limitations of existing algorithms}
This section reviewed a variety of observational causal inference
algorithms. Each of these algorithms works in a different way, under a
different set of assumptions such as the causal Markov, faithfulness,
sufficiency, minimality, acyclicity, linearity, or non-Gaussianity conditions.
Unfortunately, these conditions are difficult or impossible to test in
practice, and when assumed but violated, causal inferences will be erroneous.

The next chapter presents a different point of view on observational causal
inference. There, we pose the problem of deciding the direction of
a cause-effect relationship as the problem of classifying probability
distributions \citep{dlp-clt,dlp-jmlr}. This interpretation allow us to
transfer all the theoretical guarantees and practical advances of machine
learning to the problem of observational causal inference, as well as
implementing arbitrarily complex prior knowledge about causation 
as training data.

\section{Causality and learning}\label{sec:causation_learning}
The \ref{ass:icm} assumption has remarkable implications in learning
\citep{scholkopf12anti}. Consider the common scenario where using data 
$\{(x_i, y_i)\}_{i=1}^n \sim P(\bm x, \bm y)$, we want to learn the
function $\E{}{\bm y \given \bm x=x}$.  From a causal point of view, here we face
one of two scenarios: either $\bm x$ causes $\bm y$, or $\bm y$ causes $\bm x$.
We call the former a \emph{causal learning problem}, since we want to 
learn a function $\E{}{\bm y \given \bm x=x}$ mapping one cause to its effect.
We call the latter an \emph{anticausal learning problem}, since we want to
learn a function $\E{}{\bm x \given \bm y=y}$ mapping one effect
to its cause. 

This asymmetry, together with the \ref{ass:icm}, entails some distinctions
between learning a causal or an anticausal problem. When  learning a causal
learning problem, further amounts of unlabeled input data
$\{(x_i)\}_{i=n+1}^{n+m} \sim P^m(\bm x)$ are unhelpful.  This is because the
\ref{ass:icm} assumption tells us that the \emph{cause} distribution $P(\bm x)$
contains no information about the function of interest $\E{}{\bm y \given \bm
x=x}$. This negative result holds for regular semisupervised learning, or more
complicated variants such as unsupervised, semisupervised, transfer, and domain
adaptation learning problems. On the contrary, if we are dealing with an
anticausal learning problem, additional unlabeled input data can be of help,
since now $P(\bm y)$ is the \emph{effect} distribution, which possibly contains
information about the function $\E{}{\bm x \given \bm y =y}$ that we are trying
to learn. This distinction is not unique to semisupervised learning, but
extend to unsupervised learning, domain adaptation, and multitask learning
problems \citep{scholkopf12anti}. 

  \chapter{Learning causal relations}\label{chapter:learning-causal-relations}
\vspace{-1.25cm}
  \emph{This chapter contains novel material. In particular, we pose the
  problem of observational cause-effect inference as a binary classification
  task \citep{dlp-clt}. To this end, Section~\ref{sec:theory} extends the
  theory of surrogate risk minimization for binary classification to the
  problem of learning from samples of probability distributions.
  Section~\ref{sec:exps} instantiates an algorithm built on top of this theory,
  termed the Randomized Causation Coefficient (RCC), and shows state-of-the-art
  causal inference on a variety of simulations on real-world data.
  Finally, Section~\ref{sec:visual-causation} proposes a variant of RCC based
  on neural networks, the Neural Causation Coefficient (NCC), and illustrates
  its use to reveal causal signals in collections of static images, when
  described by convolutional neural network features \citep{dlp-img}.
  }
\vspace{1.25cm}

\noindent A quick look to Figure~\ref{fig:tuebingen-pairs} summarizes the
central question of this chapter:
\begin{center}
\emph{given samples from two random variables $\bm x$ and $\bm y$, does
$\bm x \to \bm y$ or $\bm y \to \bm x$?}
\end{center}
The same figure highlights the challenge of answering this question: even for
our human eyes, telling between cause and effect from data is a complex task.
As opposed to statistical dependence, sharply defined in terms of the
difference between joint and marginal distributions, causation lacks a closed
mathematical expression, and reveals itself in many forms. This inspires the
use of different algorithms in different situations.

\begin{figure}
  \begin{center}
  \includegraphics[width=\textwidth]{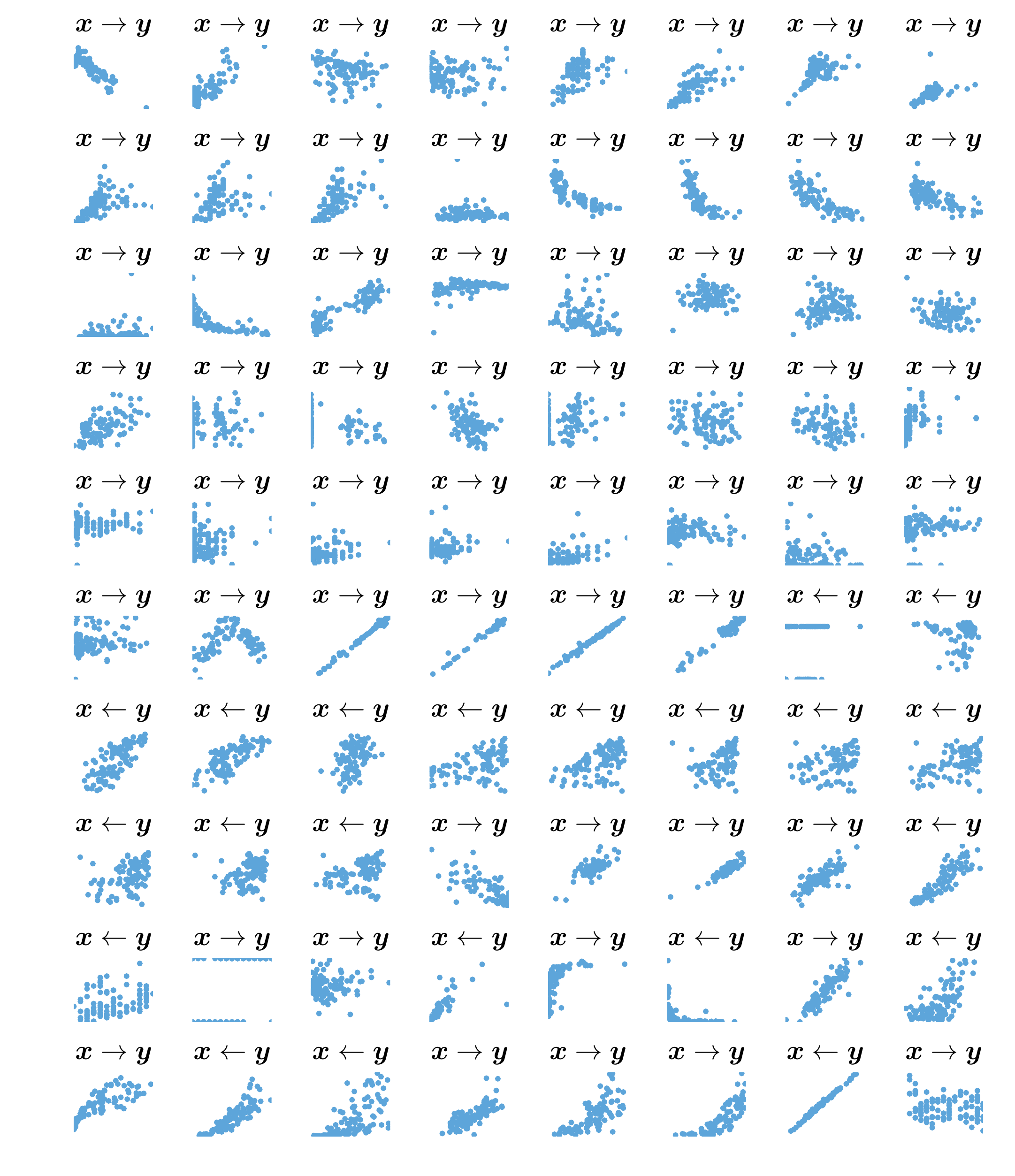}
  \end{center}
  \caption[Eighty T\"ubingen pairs of real-world samples with known causal
  structure.]{Eighty T\"ubingen pairs of real-world samples with known causal
  structure. In each plot, the variable $\bm x$ lays on the horizontal axis,
  and the variable $\bm y$ lays on the vertical axis.}
  \label{fig:tuebingen-pairs}
\end{figure}

In principle, we could tackle the problem of observational causal inference
using any of the algorithms reviewed in
Section~\ref{sec:causal-inference-algorithms}: conditional dependence based
algorithms, information geometric methods, additive noise models, and so forth.
But which one should we use? In the end, each of these algorithms work under a
different and specialized set of assumptions, which are difficult to verify in
practice. Each of them exploit a particular \emph{observable causal footprint},
and construct a suitable statistical test to verify its presence in data. But
is that particular footprint in our data, or is it another one? What if we want to
consider a new footprint? Developing a new causal inference algorithm is adding
one new item to the catalog of causal footprints, together with its
corresponding statistical test.

Engineering and maintaining a catalog of causal footprints is a tedious task.
Moreover, any such catalog will most likely be incomplete.  To amend this
issue, this chapter proposes to \emph{learn} such catalog and how to perform
causal inference from a corpus of data with labeled causal structure. Such a
``data driven'' approach moves forward by allowing complex causal assumptions
and data generating processes, and removes the need of characterizing new
causal footprints and their identifiability conditions.

More specifically, this chapter poses causal inference as the problem of
learning to classify probability distributions.  To this end, we setup a
learning task on the collection of input-output pairs 
\begin{equation*}
  \{(S_i,l_i)\}_{i=1}^n,
\end{equation*}
where each input sample 
\begin{equation*}
  S_i = \{(x_{i,j},y_{i,j})\}_{j=1}^{n_i} \sim P^{n_i}(\bm x_i, \bm y_i)
\end{equation*}
and each output binary label
$l_i$ indicates whether ``$\bm x_i \to \bm y_i$'' or ``$\bm x_i \leftarrow \bm
y_i$''.  Given these data, we build a causal inference rule in two steps.
First, we featurize each variable-length input sample $S_i$ into a
fixed-dimensional vector representation $\mu_k(S_i)$. Second, we train a binary
classifier on the data $\{(\mu_k(S_i),l_i)\}_{i=1}^n$ to distinguish between causal
directions.

We organize the exposition as follows. We start by introducing the concept of
\emph{kernel mean embeddings} in Section~\ref{sec:embeddings}. These will be
the tool of choice to featurize variable-length input samples into
fixed-dimensional vector representations.  Using kernel mean embeddings,
Section~\ref{sec:theory} poses the problem of bivariate causal inference as the
task of classifying probability distributions. In that same section, we provide
a theoretical analysis on the consistency, learning rates, and large-scale
approximations of our setup.  In Section~\ref{sec:theory}, we extend our ideas
from bivariate to multivariate causal inference.  Section~\ref{sec:exps}
provides a collection of numerical simulations, illustrating that a simple
implementation of our framework achieves state-of-the-art causal inference
performance in a variety of real world datasets. 
Finally, Section~\ref{sec:visual-causation} closes this chapter by proposing
a variant of RCC based on neural networks, and applying it to the discovery
of causal signals in collections of static images.

\begin{example}[Prior work on learning from distributions]
The competitions organized by \citet{Kaggle13,Codalab14} pioneered the view of
causal inference as a learning problem.  These competitions provided the
participants with a large collection of \emph{cause-effect samples}
$\{(S_i,l_i)\}_{i=1}^n$, where we sample $S_i =
\{(x_{i,j},y_{i,j})\}_{j=1}^{n_i}$ from the probability distribution $P^{n_i}(\bm
x_i,\bm y_i)$, and $l_i$ is a binary label indicating whether ``$\bm x_i \to
\bm y_i$'' or ``$\bm y_i \to \bm x_i$''. Given these data, most participants
adopted the strategy of i) crafting a vector of features from each $S_i$, and
ii) training a binary classifier on top of the constructed features and paired
labels. Although these ``data-driven'' methods achieved state-of-the-art
performance \citep{Kaggle13}, their hand-crafted features
render the theoretical analysis of the algorithms impossible.

In a separate strand of research, there has been multiple proposals to learn
from probability distributions
\citep{Jebara04,Hein04,Cuturi05,Martins09,Muandet12,Poczos13}.  \citet{Szabo14}
presented the first theoretical analysis of distributional learning based on
kernel mean embeddings, with a focus on kernel ridge regression. Similarly,
\citet{Muandet12} studied the problem of classifying kernel mean embeddings of
distributions, but provided no guarantees regarding consistency or learning
rates.
\end{example}

\section{Kernel mean embeddings}
\label{sec:embeddings}

The recurring idea in this chapter is the 
classification of probability distributions according to their causal structure.
Therefore, we first need a way to featurize probability distributions into a
vector of features.  To this end, we will use \emph{kernel mean embeddings}
\citep{Smola07Hilbert,Muandet2015}.  Kernel mean embeddings are tools based on
kernel methods: this may be a good time to revisit the introduction about
kernels provided in Section~\ref{sec:kernels}.

In particular, let $P \in \P$ be the probability distribution of some random variable
$\bm z$ taking values in the separable topological space $(\Z,\tau_z)$.  Then, the
\emph{kernel mean embedding} of $P$ associated with the continuous, bounded,
and positive-definite kernel function $k : \Z \times \Z \to \R$ is 
\begin{equation}
  \label{eq:meank}
  \mu_k(P) := \int_{\Z} k( z, \cdot) \, \d P(z),
\end{equation}
which is an element in $\H_k$, the Reproducing Kernel Hilbert Space (RKHS)
associated with $k$ \citep{Scholkopf01}.  A key fact is that the mapping $\mu_k : \P \to \H_k$ is
injective if $k$ is a \emph{characteristic} kernel \citep{Sriperumbudur10:Metrics}.
Thus, characteristic kernel mean embeddings satisfy
\begin{equation*}
  \|\mu_k(P)-\mu_k(Q)\|_{\H_k}=0 \Leftrightarrow P=Q.
\end{equation*}
The previous implication means that, when using a characteristic kernel, we do
not lose any information by embedding distributions. An example of
characteristic kernel is the Gaussian kernel, reviewed in
Section~\ref{sec:kernelex}, and with form
\begin{equation}\label{eq:gauss}
  k(z,z') = \exp\left(-\gamma \|z-z'\|_2^2\right), \,\, \gamma > 0.
\end{equation}
We will work with the Gaussian kernel during the remainder of this chapter.

In practice, it is unrealistic to assume access to the distributions $P$ that
we wish to embed, and consequently to their exact embeddings $\mu_k(P)$.
Instead, we often have access to a sample $S = \{z_i\}_{i=1}^n \sim P^n$, which
we can use to construct the empirical distribution
\begin{equation*}
P_S := \frac{1}{n} \sum_{z_i\in S} \delta_{(z_i)},
\end{equation*}
where $\delta_{(z)}$ is the Dirac distribution centered at $z$.  Using the
empirical distribution $P_S$, we can approximate \eqref{eq:meank} by the
\emph{empirical kernel mean embedding}
\begin{equation}
  \label{eq:meank2} 
  \mu_k(P_S) := \frac{1}{n} \sum_{i=1}^n k( z_i, \cdot) \in \H_k.
\end{equation}
Figure~\ref{fig:embedding} illustrates the transformation of a sample $S =
\{z_1, \ldots, z_n\} \sim P^n$ into the empirical kernel mean embedding
$\mu_k(P_S) = n^{-1}\sum\nolimits_{i=1}^n k(\cdot, z_i)$, depicted as a red dot in the Hilbert
space $\H_k$.

The following result, slightly improved from \citep[Theorem 27]{S06},
characterizes the convergence of the empirical embedding $\mu_k(P_S)$ to the
true embedding $\mu_k(P)$ as the sample size $n$ grows.
\begin{theorem}[Convergence of empirical kernel mean embedding]
\label{thm:LeSong}
Assume that $\|f\|_{\infty}\leq 1$ for all $f\in \H_k$ with $\|f\|_{\H_k}\leq
1$.  Then with probability at least $1-\delta$ we have
\[
\|\mu_k(P)-\mu_k(P_S)\|_{\H_k}
\leq
2\sqrt{\frac{\E{z\sim P}{k(z,z)}}{n}} + \sqrt{\frac{2\log\frac{1}{\delta}}{n}}.
\]
\end{theorem}
\begin{proof}
  See Section~\ref{proof:LeSong}.
\end{proof}

At this point, we have the necessary machinery to summarize sets of samples $S$
drawn from distributions $P$ as vectors $\mu_k(P_S)$, which
live in the RKHS $\H_k$ associated with some kernel function $k$. Let's apply
these tools to the problem of causal inference.

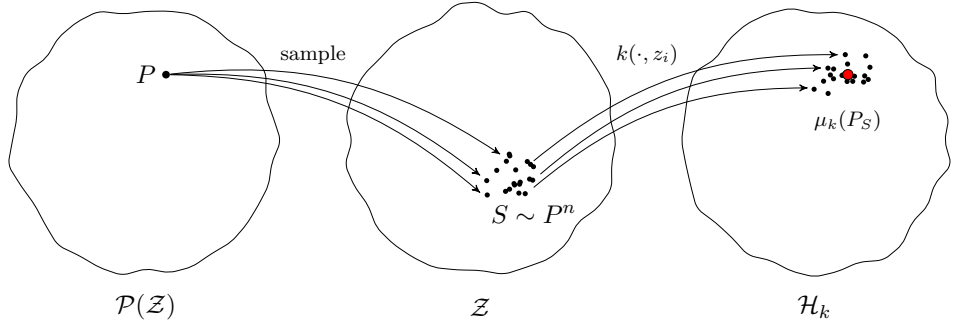
\begin{figure}
\begin{center}
\resizebox{\textwidth}{!}{
\begin{tikzpicture}
  \coordinate (R1) at (0,0);
  \coordinate (R2) at (5,0);
  \coordinate (R3) at (10,0);

  \draw[rounded corners=1mm] (R1) \irregularcircle{2cm}{1mm};
  \draw[rounded corners=1mm] (R2) \irregularcircle{2cm}{1mm};
  \draw[rounded corners=1mm] (R3) \irregularcircle{2cm}{1mm};

  \node (L1) at (0,-2.5) {$\P(\Z)$};
  \node (L2) at (5,-2.5) {$\Z$};
  \node (L3) at (10,-2.5) {$\H_k$};
  
  \node (P) at (0,1) {$P$};
  \node[shape=circle,fill=black,inner sep=0pt, minimum size=0.3em] (Pdot) at (0.3,1) {};

  \node (S) at (5.8,-1.1) {$S \sim P^n$};
  \node[shape=circle,fill=black,inner sep=0pt, minimum size=0.2em] (Sx) at (0.2* 1.02204836+5.5,0.2* 0.9501906-0.5) {};
  \node[shape=circle,fill=black,inner sep=0pt, minimum size=0.2em] (Sx) at (0.2* 0.86632266+5.5,0.2*-1.4706109-0.5) {};
  \node[shape=circle,fill=black,inner sep=0pt, minimum size=0.2em] (Sx) at (0.2* 0.39348144+5.5,0.2*-1.4291111-0.5) {};
  \node[shape=circle,fill=black,inner sep=0pt, minimum size=0.2em] (Sx) at (0.2*-0.32743748+5.5,0.2* 1.5077663-0.5) {};
  \node[shape=circle,fill=black,inner sep=0pt, minimum size=0.2em] (Sx) at (0.2*-1.26399052+5.5,0.2* 0.2855397-0.5) {};
  \node[shape=circle,fill=black,inner sep=0pt, minimum size=0.2em] (Sx) at (0.2* 1.42570902+5.5,0.2*-0.4520762-0.5) {};
  \node[shape=circle,fill=black,inner sep=0pt, minimum size=0.2em] (Sx) at (0.2* 0.49335914+5.5,0.2*-0.6017368-0.5) {};
  \node[shape=circle,fill=black,inner sep=0pt, minimum size=0.2em] (Sx) at (0.2*-0.28676032+5.5,0.2* 1.3773745-0.5) {};
  \node[shape=circle,fill=black,inner sep=0pt, minimum size=0.2em] (Sx) at (0.2* 1.48347412+5.5,0.2* 0.5634053-0.5) {};
  \node[shape=circle,fill=black,inner sep=0pt, minimum size=0.2em] (Sx) at (0.2* 0.15545080+5.5,0.2* 0.3207756-0.5) {};
  \node[shape=circle,fill=black,inner sep=0pt, minimum size=0.2em] (Sx) at (0.2* 1.19371318+5.5,0.2*-0.3707745-0.5) {};
  \node[shape=circle,fill=black,inner sep=0pt, minimum size=0.2em] (Sx) at (0.2* 0.44281335+5.5,0.2*-0.7561343-0.5) {};
  \node[shape=circle,fill=black,inner sep=0pt, minimum size=0.2em] (Sx) at (0.2* 1.26395740+5.5,0.2* 0.7365613-0.5) {};
  \node[shape=circle,fill=black,inner sep=0pt, minimum size=0.2em] (Sx) at (0.2* 0.02682692+5.5,0.2*-0.7123388-0.5) {};
  \node[shape=circle,fill=black,inner sep=0pt, minimum size=0.2em] (Sx) at (0.2*-0.30476011+5.5,0.2*-1.1311914-0.5) {};
  \node[shape=circle,fill=black,inner sep=0pt, minimum size=0.2em] (Sx) at (0.2*-2.01254986+5.5,0.2*-0.4871593-0.5) {};
  \node[shape=circle,fill=black,inner sep=0pt, minimum size=0.2em] (Sx) at (0.2*-0.60520287+5.5,0.2*-1.2949941-0.5) {};
  \node[shape=circle,fill=black,inner sep=0pt, minimum size=0.2em] (Sx) at (0.2*-1.97285682+5.5,0.2*-1.5700367-0.5) {};
  \node[shape=circle,fill=black,inner sep=0pt, minimum size=0.2em] (Sx) at (0.2*-0.54233790+5.5,0.2* 0.9721010-0.5) {};
  \node[shape=circle,fill=black,inner sep=0pt, minimum size=0.2em] (Sx) at (0.2* 0.07281261+5.5,0.2*-0.8147764-0.5) {};

  \draw[->] (Pdot) edge[bend left=20,->] (5.3,-0.2);
  \draw[->] (Pdot) edge[bend left=20,->] (5,-0.52);
  \draw[->] (Pdot) edge[bend left=20,->] (5,-0.78);
  
  \node[shape=circle,fill=black,inner sep=0pt, minimum size=0.2em] (Sx) at (0.2*-1.07481042+10.5, 0.2*-0.31078400+1) {};
  \node[shape=circle,fill=black,inner sep=0pt, minimum size=0.2em] (Sx) at (0.2*-2.52580057+10.5, 0.2*-1.09863049+1) {};
  \node[shape=circle,fill=black,inner sep=0pt, minimum size=0.2em] (Sx) at (0.2* 0.01005841+10.5, 0.2* 0.16984097+1) {};
  \node[shape=circle,fill=black,inner sep=0pt, minimum size=0.2em] (Sx) at (0.2* 1.52271558+10.5, 0.2*-0.36448383+1) {};
  \node[shape=circle,fill=black,inner sep=0pt, minimum size=0.2em] (Sx) at (0.2*-1.47875334+10.5, 0.2*-1.41814855+1) {};
  \node[shape=circle,fill=black,inner sep=0pt, minimum size=0.2em] (Sx) at (0.2* 0.49551206+10.5, 0.2*-0.12068299+1) {};
  \node[shape=circle,fill=black,inner sep=0pt, minimum size=0.2em] (Sx) at (0.2*-0.02330213+10.5, 0.2* 0.80095016+1) {};
  \node[shape=circle,fill=black,inner sep=0pt, minimum size=0.2em] (Sx) at (0.2*-1.48988742+10.5, 0.2* 0.49178386+1) {};
  \node[shape=circle,fill=black,inner sep=0pt, minimum size=0.2em] (Sx) at (0.2*-0.05691220+10.5, 0.2*-0.51488153+1) {};
  \node[shape=circle,fill=black,inner sep=0pt, minimum size=0.2em] (Sx) at (0.2* 1.65068532+10.5, 0.2* 0.55707099+1) {};
  \node[shape=circle,fill=black,inner sep=0pt, minimum size=0.2em] (Sx) at (0.2*-1.06700032+10.5, 0.2* 0.42685217+1) {};
  \node[shape=circle,fill=black,inner sep=0pt, minimum size=0.2em] (Sx) at (0.2*-1.81752821+10.5, 0.2*-0.49342955+1) {};
  \node[shape=circle,fill=black,inner sep=0pt, minimum size=0.2em] (Sx) at (0.2*-1.30805489+10.5, 0.2*-0.02466303+1) {};
  \node[shape=circle,fill=black,inner sep=0pt, minimum size=0.2em] (Sx) at (0.2* 1.25682373+10.5, 0.2*-0.09324799+1) {};
  \node[shape=circle,fill=black,inner sep=0pt, minimum size=0.2em] (Sx) at (0.2*-0.43993068+10.5, 0.2*-0.05645059+1) {};
  \node[shape=circle,fill=black,inner sep=0pt, minimum size=0.2em] (Sx) at (0.2* 1.33833385+10.5, 0.2* 1.41396902+1) {};
  \node[shape=circle,fill=black,inner sep=0pt, minimum size=0.2em] (Sx) at (0.2*-0.19366189+10.5, 0.2* 1.50859715+1) {};
  \node[shape=circle,fill=black,inner sep=0pt, minimum size=0.2em] (Sx) at (0.2* 0.36328739+10.5, 0.2*-0.50237003+1) {};
  \node[shape=circle,fill=black,inner sep=0pt, minimum size=0.2em] (Sx) at (0.2* 1.17612265+10.5, 0.2*-0.55784483+1) {};
  \node[shape=circle,fill=black,inner sep=0pt, minimum size=0.2em] (Sx) at (0.2*-0.36177821+10.5, 0.2*-0.01790047+1) {};
  
  \node[draw,shape=circle,fill=red,inner sep=0pt, minimum size=0.4em] (Sx) at (10.5, 1) {};
  
  \draw (5.8,-0.3) edge[bend left=20,->] (10.35,1.3);
  \draw (5.9,-0.5) edge[bend left=20,->] (10.1,1.1);
  \draw (5.8,-0.7) edge[bend left=20,->] (9.9,0.8);
  
  \node (L2) at (2.5,1.3) {\footnotesize sample};
  \node (L2) at (7.5,1.3) {\footnotesize $k(\cdot, z_i)$};
  \node (L2) at (10.5,0.3) {\footnotesize $\mu_k(P_S)$};
\end{tikzpicture}
}
\end{center}
\caption{Transforming a sample $S$ drawn from a distribution $P$ into the empirical mean embedding $\mu_k(P_S)$.}
\label{fig:embedding}
\end{figure}

\section{Causal inference as distribution classification}\label{sec:theory}
This section poses causal inference as the classification of kernel mean
embeddings associated to probability distributions with known causal structure, and
analyzes the learning rates, consistency, and approximations of such approach.
To make things concrete, we encapsulate the setup of our learning problem
in the following definition.

\begin{definition}[Distributional learning setup]\label{def:learning-setup}
Throughout this chapter, our learning setup is as follows:
\begin{enumerate}
  \item Assume the existence of some \emph{Mother distribution} $\M$,
  defined on $\P \times \L$, where $\P$ is the set of all Borel probability
  measures on the space $\Z$ of two causally related random variables, and $\L
  = \{-1,+1\}$. 

  \item A set $\{(P_i, l_i)\}_{i=1}^n$ is sampled from $\M^n$. Each
  measure $P_i \in \P$ is the joint distribution of the causally related random
  variables $\bm z_i = (\bm x_i, \bm y_i)$, and the label $l_i \in \L$ indicates whether
  ``$\bm x_i \to \bm y_i$'' or ``$\bm x_i \leftarrow \bm y_i$''.
  
  \item In practice, we do not have access to the measures $\{P_i\}_{i=1}^n$.
  Instead, we observe samples $S_i = \{(x_{i,j}, y_{i,j})\}_{j=1}^{n_i} \sim
  P_i^{n_i}$, for all $1 \leq i \leq n$.  

  \item We featurize every sample $S_i$ into the empirical kernel mean
  embedding $\mu_k(P_{S_i})$ associated with some kernel function $k$
  (Equation~\ref{eq:meank2}). If $k$ is a characteristic kernel, we incur no
  loss of information in this step.

  \item For computational considerations, we approximate each high-dimensional
  embedding $\mu_k(P_{S_i})$ into  the $m$-dimensional embedding
  $\mu_{k,m}(P_{S_i})$. The data $\{(\mu_{k,m}(P_{S_i}),l_i)\}_{i=1}^n
  \subseteq \Rm \times \L$ is provided to the classifier.
\end{enumerate}
Figure~\ref{fig:causalsketch} summarizes this learning setup. 
\end{definition}

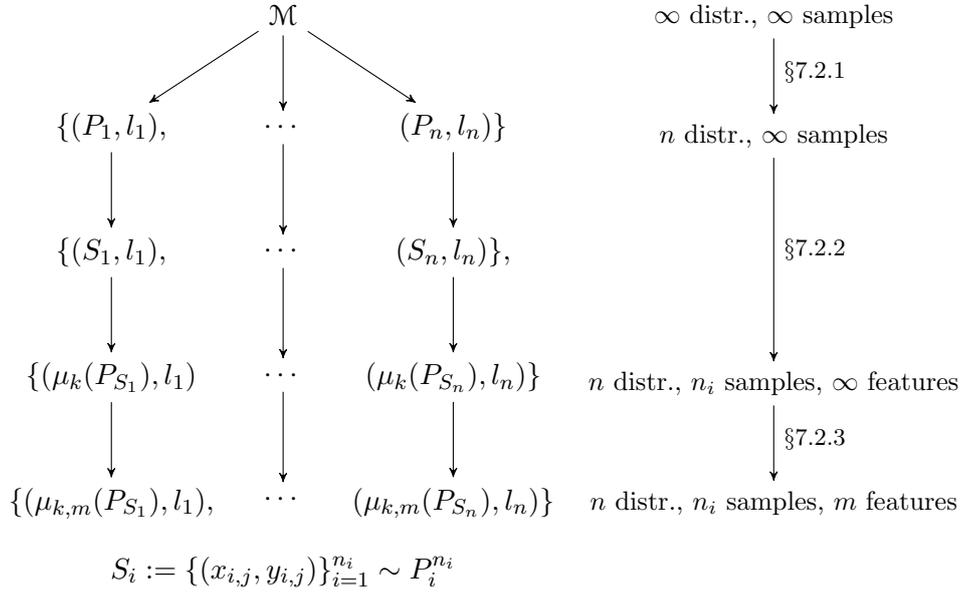
\begin{figure}[t]
  \begin{center}
    \begin{tikzpicture}[node distance=1cm, auto,]
     \node (M) {$\M$};
     \node (Pd) [below=of M] {$\cdots$};
     \node (P1) [left =1of Pd] {$\{(P_1,l_1),$};
     \node (Pn) [right=1of Pd] {$(P_n,l_n)\}$};

     \path (M) edge[->] (P1);
     \path (M) edge[->] (Pd);
     \path (M) edge[->] (Pn);

     \node (S1) [below=    of P1] {$\{({S_1},l_1),$};
     \node (Sd) [below=1.2 of Pd] {$\cdots$};
     \node (Sn) [below=    of Pn] {$({S_n},l_n)\},$};

     \path (Pd) edge[->] (Sd);
     \path (P1) edge[->] (S1);
     \path (Pn) edge[->] (Sn);

     \node (M1) [below=    of S1] {$\{(\mu_k(P_{S_1}),l_1)$};
     \node (Md) [below=1.2 of Sd] {$\cdots$};
     \node (Mn) [below=    of Sn] {$(\mu_k(P_{S_n}),l_n)\}$};

     \path (S1) edge[->] (M1);
     \path (Sd) edge[->] (Md);
     \path (Sn) edge[->] (Mn);

     \node (F1) [below=    of M1] {$\{(\mu_{k,m}(P_{S_1}),l_1),$};
     \node (Fd) [below=1.2 of Md] {$\cdots$};
     \node (Fn) [below=    of Mn] {$(\mu_{k,m}(P_{S_n}),l_n)\}$};

     \path (M1) edge[->] (F1);
     \path (Md) edge[->] (Fd);
     \path (Mn) edge[->] (Fn);

     \node (SS) [below=0.4 of Fd] {$S_i := \{(x_{i,j},y_{i,j})\}_{i=1}^{n_i} \sim P_i^{n_i}$};

     \node (label5) [right=0.2of Fn] {\small $n$ distr., $n_i$ samples, $m$ features};
     \node (label4) [above=of label5] {\small $n$ distr., $n_i$ samples, $\infty$ features};
     \node (label3) [above=of label4] {};
     \node (label2) [above=2.7 of label4] {\small $n$ distr., $\infty$ samples};
     \node (label1) [above=of label2] {\small $\infty$ distr., $\infty$ samples};
     
     \draw[->] (label1) -- (label2) node [pos=0.45,right] {\footnotesize \S\ref{sec:classic}};
     \draw[->] (label2) -- (label4) node [pos=0.45,right] {\footnotesize \S\ref{sec:distrib}};
     \draw[->] (label4) -- (label5) node [pos=0.45,right] {\footnotesize \S\ref{sec:random}};
     
    \end{tikzpicture}
    \vspace{-0.5cm}
  \end{center}
  \caption[Generative process of the causal learning setup]{Generative process of our learning setup.}
  \label{fig:causalsketch}
\end{figure}

Using Definition~\ref{def:learning-setup}, we will use the data set
$\{(\mu_k(P_{S_i}),l_i)\}_{i=1}^n$ to train a binary classifier from $\H_k$ to
$\L$, which we will use to unveil the causal directions of new, unseen
probability measures drawn from $\M$.  This framework can be straightforwardly
extended to also infer the ``confounding ($\bm x \leftarrow \bm z \rightarrow \bm y$)'' and
``independent ($\bm x \indep \bm y$)'' cases by adding two extra labels
to~$\mathcal{L}$, as we will exemplify in our numerical simulations. 

Given the two nested levels of sampling (being the first one from the Mother
distribution $\M$, and the second one from each of the drawn cause-effect measures
$P_i$), it is not trivial to conclude whether this learning procedure is consistent, or how
its learning rates depend on the sample sizes $n$ and $\{n_i\}_{i=1}^n$.  In
the following, we will answer these questions by studying the generalization performance of empirical risk
minimization over this learning setup.  Specifically, our goal is to 
upper bound the {excess risk} between the empirical risk minimizer and
the best classifier from our hypothesis class, with respect to the Mother
distribution $\M$.

We divide our analysis in three parts. Each part will analyze the impact of
each of the finite samplings described in Definition~\ref{def:learning-setup}
and depicted in Figure~\ref{fig:causalsketch}. First, Section~\ref{sec:classic}
reviews standard learning theory for surrogate risk
minimization.  Second, Section~\ref{sec:distrib} adapts these standard results to the
case of empirical kernel mean embedding classification.  Third,
Section~\ref{sec:random} considers embedding approximations suited to deal with big
data, and analyses their impact on learning rates.

\begin{remark}[Philosophical considerations]
  Reducing causal inference to a learning problem is reducing identifiability
  assumptions to learnability assumptions. For example, we know from
  Section~\ref{sec:anm} that additive noise models with linear functions and
  additive Gaussian noise are not identifiable. In the language
  of learning, this means that the kernel mean embeddings of causal
  distributions and anticausal distributions fully overlap. Under this framework, a
  family of distributions $\P$ is causally identifiable if and only if the
  conditional Mother distributions $\M(\mu_k(\P) \given l = +1)$ and
  $\M(\mu_k(\P) \given l = -1)$ are separable. 

  Learning to tell cause from effect on the basis of empirical data relates to other 
  philosophical questions. Paraphrasing 
  \citet{goodman2011learning}, is the human sense of causation innate or
  learned? Or invoking David Hume, is the human sense of causation a
  generalization from the observation of constant association of events?
  Perhaps the most troubling fact of our framework from a philosophical perspective is that the
  training data from our learning setup is \emph{labeled} so the machine, as
  opposed to learning humans, gets an explicit peek at the true causal
  structure governing the example distributions. One can mitigate this
  discrepancy by recalling another: unlike observational causal inference
  machines, humans obtain causal labels by interacting with the
  world.
\end{remark}

\subsection{Theory of surrogate risk minimization}
\label{sec:classic}
Let $P$ be some unknown probability measure defined on $\Z \times \L$, where
we call $\Z$ the \emph{input space}, and $\L = \{-1,+1\}$ 
the \emph{output space}. As introduced in
Section~\ref{sec:learning-theory}, one of the main goals of statistical
learning theory is to find a classifier $h\colon \Z\to\L$ that minimizes the
\emph{expected risk}
\begin{equation*}
  R(h) = \E{}{\ell\bigl(h(\bm z),\bm l\bigr)}
\end{equation*}
for a suitable \emph{loss function} $\ell\colon\L \times \L\to\R^+$, which
penalizes departures between predictions $h(z)$ and true labels $l$.  For
classification, one common choice of loss function is the \emph{0-1 loss}
$\ell_{01}(l,l')=|l-l'|$, for which the expected risk measures the
probability of misclassification.  Since~$P$ is unknown in natural
situations, one usually resorts to the minimization of the \emph{empirical
risk} $\frac{1}{n}\sum_{i=1}^n\ell\bigl(h(z_i),l_i\bigr)$ over some fixed
hypothesis class $\H$, for \emph{the training set} $\{(z_i,l_i)\}_{i=1}^n
\sim P^n$.  It is well known that this procedure~is consistent under
mild~assumptions \citep{BBL05}.

The exposition is so far parallel to the introduction of learning theory in
Section~\ref{sec:learning-theory}. Unfortunately, the 0-1 loss function is
nonconvex, which turns empirical risk minimization intractable.  Instead, we
will focus on the minimization of surrogate risk functions \citep{BJM06}.  We
proceed by considering the set of classifiers with form $\Hyp = \{{\sig
f}\colon {f\in\F}\}$, where $\F$ is some fixed set of real-valued functions
${f\colon \Z\to\R}$.  Introduce a nonnegative \emph{cost function}
$\varphi\colon \R\to\R^+$ which is surrogate to the 0-1 loss, that is,
$\varphi(\epsilon) \geq \I_{\epsilon>0}$.  For any $f\in \F$, we
define its expected and empirical $\varphi$-risks as
\begin{equation}\label{eq:phirisk1}
  \Rp(f) = \E{(z,l)\sim P}{\varphi\bigl(-f(\bm z) \bm l\bigr)}
\end{equation}
and
\begin{equation}\label{eq:phirisk2}
  \Rpn(f) = \frac{1}{n}\sum_{i=1}^n \varphi\bigl(-f(z_i) l_i\bigr).
\end{equation}
Some natural choices of $\varphi$ lead to tractable empirical risk
minimization.  Common examples of cost functions include the \emph{hinge loss}
$\varphi(\epsilon)=\max(0,1+\epsilon)$ used in SVM, the \emph{exponential loss}
$\varphi(\epsilon)=\exp(\epsilon)$ used in Adaboost, the \emph{logistic loss}
$\varphi(\epsilon)=\log_2\bigl(1+e^\epsilon)$ used in logistic regression, and
the squared loss $\varphi(\epsilon)=(1+\epsilon)^2$ used in least-squares
regression. Figure~\ref{fig:surrogate-losses} depicts these losses together
with the intractable 0-1 loss.

\begin{figure}
  \begin{center}
  \includegraphics[width=0.6\textwidth]{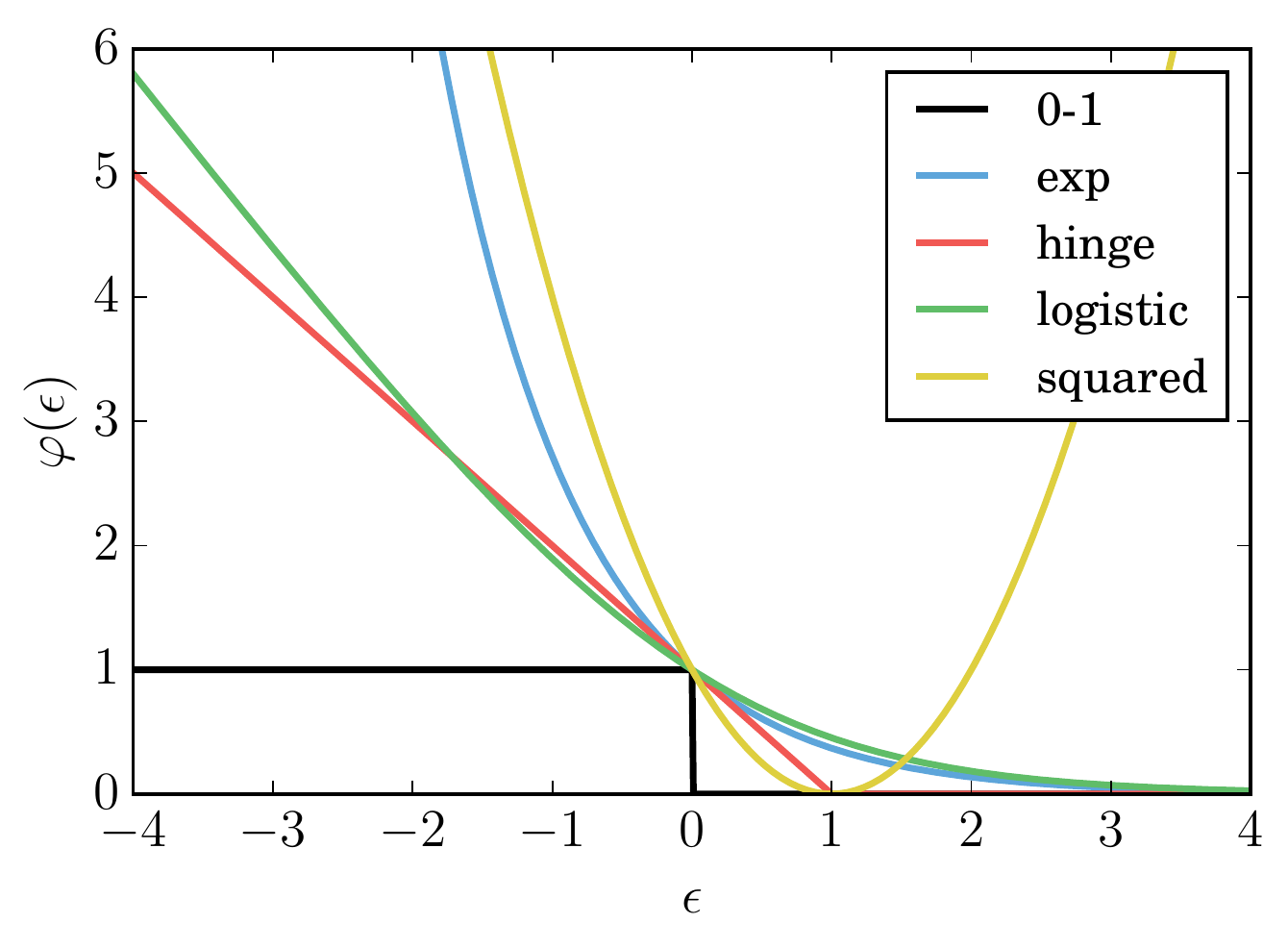}
  \end{center}
  \caption[Surrogate loss functions for margin-based learning.]{Surrogate loss
  functions for margin-based learning. All tractable surrogate losses
  upper-bound the intractable 0-1 loss.}
  \label{fig:surrogate-losses}
\end{figure}

The misclassification error of $\sig f$ is always upper bounded by $\Rp(f)$.
The relationship between functions minimizing $\Rp(f)$ and functions minimizing
${R(\sig f)}$ has been intensively studied in the literature \citep[Chapter
3]{Steinwart08}.  Given the high uncertainty associated with causal inferences,
we argue that it is more natural to predict class-probabilities instead of hard
labels (see Section~\ref{sec:logreg}), a fact that makes the study of
margin-based classifiers well suited for our problem.

We now focus on the estimation of $f^*\in\F$, the function minimizing
\eqref{eq:phirisk1}.  But, since the distribution $P$ is unknown, we can
only hope to estimate $\hat{f}_n\in\F$, the function minimizing
\eqref{eq:phirisk2}.  Therefore, our goal is to develop high-probability upper
bounds on the
\emph{excess $\varphi$-risk}
\begin{equation}
\label{eq:excess}
\Ex_{\F}(\fn) =
\Rp(\fn) - \Rp(\f),
\end{equation}
with respect to the random training sample $\{(z_i,l_i)\}_{i=1}^n \sim P^n$.
As we did back in Equation~\ref{eq:emprisk1}, we can upper bound the
excess risk \eqref{eq:excess} as:
\begin{align}
\Ex_{\F}(\fn)
\notag
&\leq
\Rp(\fn) - \Rpn(\fn)
+
\Rpn(\f) - \Rp(\f)\\
\label{eq:SLT-Bad}
&\leq
2\sup_{f\in\F}|\Rp(f) - \Rpn(f)|.
\end{align}

The following result --- in spirit of \citet{KP99,BM01} and found in
\citet[Theorem 4.1]{BBL05} --- extends Theorem~\ref{thm:fundamental} to
surrogate risk minimization. 
\begin{theorem}[Excess risk of empirical risk minimization]
\label{thm:classic}
Consider a class $\F$ of functions mapping $\Z$ to~$\R$.  Let $\varphi\colon
\R\to\R^+$ be a $L_\varphi$-Lipschitz function such that $\varphi(\epsilon)
\geq \I_{\epsilon>0}$.  Let $B$ be a uniform upper bound on
$\varphi\bigl(-f(\epsilon) l\bigr)$.  Let $D = \{(z_i,l_i)\}_{i=1}^n \sim P$ and
$\{\sigma_i\}_{i=1}^n$ be iid Rademacher random variables.  Then, with
probability at least $1-\delta$,
\begin{equation*}
\sup_{f\in\F}|\Rp(f) - \Rpn(f)|
\leq
2 L_\varphi \Rad_n(\F) + B\sqrt{\frac{\log(1/\delta)}{2n}},
\end{equation*}
where $\Rad_n(\F)$ is the Rademacher complexity of $\F$, see
Definition~\ref{def:rademacher}. 
\end{theorem}

\subsection{Distributional learning theory}\label{sec:distrib}

The empirical risk minimization bounds from Theorem~\ref{thm:classic} do not directly apply to our
causal learning setup from Definition~\ref{def:learning-setup}.  This is
because instead of learning a classifier on some sample
$\{\mu_k(P_i),l_i\}_{i=1}^n$, we are learning over the set
$\{\mu_k(P_{S_i}),l_i\}_{i=1}^n$, where $S_i \sim P_i^{n_i}$.  Thus, our input
vectors $\mu_k(P_{S_i})$ are ``noisy'': they exhibit an additional source of
variation, just like two any different random samples $S_i,S_i'\sim P_i^{n_i}$ do.
This is because the \emph{empirical} mean embedding of two different samples
$S_1, S_2 \sim P^n$, depicted as a red dot in Figure~\ref{fig:embedding}, will
differ due to the randomness of the embedded samples $S_1$ and $S_2$. In the
following, we study how to incorporate these nested sampling effects into an
argument similar to Theorem~\ref{thm:classic}.

To this end, let us frame the actors playing in
Definition~\ref{def:learning-setup} within the language of standard learning
theory laid out in the previous section. Recall that our learning setup
initially considers some \emph{Mother distribution} $\M$ over $\P \times \L$.
Let $\mu_k(\P) = \{\mu_k(P) : P \in \P \} \subseteq \H_k$, $\L = \{-1,+1\}$,
and $\M_k$ be a measure on $\mu_k(\P) \times \L$ induced by $\M$. Although this
is an intricate technical condition, we prove the existence of the measure $\M_k$
in Lemma~\ref{lemma:measurability}.  Under this measure, we will consider
$\muP\subseteq \H_k$ and $\L$ to be the input and output spaces of our learning
problem.  Let $\bigl\{\bigl(\mu_k(P_i),l_i\bigr)\bigr\}_{i=1}^n\sim \M_k^n$ be
our training set.  We will now work with the set of classifiers $\{\sig f\colon
f\in \F_k\}$ for some fixed class $\F_k$ of functionals mapping from the RKHS
$\H_k$ to $\R$.

As pointed out in the description of our learning setup, we do not have 
access to the distributions $\{P_i\}_{i=1}^n$, but to samples $S_i \sim
P_i^{n_i}$, for all $1 \leq i \leq n$.  Because of this reason, we define the
\emph{sample-based empirical $\varphi$-risk}
\[
\Rpnt(f) = \frac{1}{n}\sum_{i=1}^n \varphi\bigl(-l_i f\bigl(\mu_k(P_{S_i})\bigr)\bigr),
\]
which is the approximation to the empirical $\varphi$-risk $\Rpn(f)$ that
results from substituting the embeddings $\mu_k(P_i)$ with their empirical
counterparts $\mu_k(P_{S_i})$.

Our goal is again to find the function $\f\in\F_k$ minimizing expected
$\varphi$-risk $\Rp(f)$.  Since $\M_k$ is unknown to us, and we have no access
to the embeddings $\{\mu_k(P_i)\}_{i=1}^n$, we will instead use the minimizer
of $\Rpnt(f)$ in $\F_k$:
\begin{equation}
\label{eq:sample-based-erm}
\fnt \in \arg\min_{f\in \F_k} \Rpnt(f).
\end{equation}
To sum up, the excess risk \eqref{eq:excess} is equal to
\begin{equation}
\label{eq:excess2}
\Rp(\fnt) - \Rp(\f).
\end{equation}
Note that the estimation of $\f$ drinks from two nested sources of error, which
are i) having only $n$ training samples from the distribution $\M_k$, and ii)
having only $n_i$ samples from each measure $P_i$.  Using a similar technique
to \eqref{eq:SLT-Bad}, we can upper bound \eqref{eq:excess2} as
\begin{align}
\label{eq:excess-term1}
\Rp(\fnt) - \Rp(\f)
&\leq
\sup_{f\in\F_k}|\Rp(f) - \Rpn(f)|\\
\label{eq:excess-term2}
&+
\sup_{f\in\F_k}|\Rpn(f) - \Rpnt(f)|.
\end{align}
The term \eqref{eq:excess-term1} is upper bounded by Theorem~\ref{thm:classic}.
On the other hand, to deal with \eqref{eq:excess-term2}, we will need to upper
bound the deviations
$\bigl|f\bigl(\mu_k(P_i)\bigr)-f\bigl(\mu_k(P_{S_i})\bigr)\bigr|$ in terms of
the distances $\|\mu_k(P_i) - \mu_k(P_{S_i})\|_{\H_k}$, which are in turn upper
bounded using Theorem \ref{thm:LeSong}.  To this end, we will have to assume
that the class $\F_k$ consists of functionals with uniformly bounded Lipschitz
constants.  One natural example of such a class is the set of linear
functionals with uniformly bounded operator norm \citep{Maurer06}.

We now present the main result of this section, which provides a
high-probability bound on the excess risk \eqref{eq:excess2}. 

\begin{theorem}[Excess risk of ERM on empirical kernel mean embeddings]
\label{thm:risk-bound}
Consider the RKHS $\H_k$ associated with some bounded, continuous,
characteristic kernel function $k$, such that $\sup_{z \in \Z} k(z,z)\leq 1$.
Consider a class $\F_k$ of functionals mapping $\H_k$ to~$\R$ with Lipschitz
constants uniformly bounded by $L_{\F}$.  Let $\varphi\colon \R\to\R^+$ be a
$L_{\varphi}$-Lipschitz function such that $\varphi(z) \geq \I_{z>0}$.
Let $\varphi\bigl(-f(h) l\bigr) \leq B$ for every $f\in\F_k$, $h \in \H_k$, and
$l\in\L$.  Then, with probability not less than $1-\delta$ (over all sources of
randomness) 
\begin{align*}
&\Rp(\fnt) - \Rp(\f)
\leq
4 L_\varphi R_n(\F_k) + 2B\sqrt{\frac{\log(2/\delta)}{2n}}\\
&+
\frac{4L_\varphi L_{\F}}{n}
\sum_{i=1}^n
\left(
\sqrt{\frac{\E{z\sim P_i}{k(z,z)}}{n_i}} + \sqrt{\frac{\log\frac{2n}{\delta}}{2n_i}}
\right).
\end{align*}
\end{theorem}
\begin{proof}
See Section \ref{sect:ProofRiskBound}.
\end{proof}
As mentioned in Section~\ref{sec:classic}, the typical order of $R_n(\F_k)$ is
$O(n^{-1/2})$.  In such cases, the upper
bound in Theorem~\ref{thm:risk-bound} converges to zero (meaning that our
procedure is consistent) as both $n$ and $n_i$ tend to infinity, as long
as $\log n/n_i = o(1)$.  The rate of convergence with respect to $n$ can
improve to $O(n^{-1})$ if we place additional assumptions on $\M$
\citep{BBM05}. On the contrary, the rate with respect to $n_i$ is not
improvable in general. Namely, the convergence rate $O(n^{-1/2})$ presented in
the upper bound of Theorem~\ref{thm:LeSong} is tight, as shown in the following
novel result.
\begin{theorem}[Lower bound on empirical kernel mean embedding]
\label{thm:lower-bound}
Under the assumptions of Theorem~\ref{thm:LeSong} denote
\[
\sigma^2_{\H_k} = \sup_{\|f\|_{\H_k}\leq 1} \V{}{f(\bm z)}.
\]
Then there exist universal constants $c,C$ such that for every integer $n\geq
1/\sigma^2_{\H_k}$, and with probability at least $c$
\[
\|\mu_k(P) - \mu_k(P_S)\|_{\H_k}
\geq
C\frac{\sigma_{\H_k}}{\sqrt{n}}.
\]  
\end{theorem}
\begin{proof}
See Section \ref{sect:ProofLowerBound}.
\end{proof}
For a minimax lower bound on the set of all kernel mean embedding estimators, see
\citep{ilyabarath}.

It is instructive to relate the notion of ``identifiability'' often considered
in the causal inference community \citep{pearl2009causality} to the properties
of the Mother distribution.  Saying that the model is \emph{identifiable} means
that $\M$ labels each $P\in\P$ deterministically.  In this case, learning rates
can become as fast as $O(n^{-1})$.  On the other hand, as $\M(l|P)$ becomes
nondeterministic, the problem degrades to unidentifiable, and learning rates slow
down (for example, in the extreme case of cause-effect pairs related by linear
functions polluted with additive Gaussian noise, $\M(l=+1|P) = \M(l=-1|P)$
almost surely). The Mother distribution is an useful tool to characterize the
difficulty of causal inference problems, as well as a convenient language to
place assumptions over the distributions that we want to classify.

\subsection{Low dimensional embeddings}\label{sec:random}

The embeddings $\mu_k(P_S)\in\H_k$ are nonparametric. As we saw in
Remark~\ref{remark:nonparametric}, nonparametric representations require
solving dual optimization problems, which often involve the construction and
inversion of big kernel matrices.  Since these are prohibitive operations for
large $n$, in this section we provide $m$-dimensional approximations to these
embeddings based on the random Mercer features, introduced in
Section~\ref{sec:random-mercer-features}.

We now show that, for any probability measure $Q$ on $\Z$~and
$z\in\Z$, we can approximate $k(z,\cdot)\in\H_k\subseteq L^2(Q)$ by a linear
combination of randomly chosen elements from the Hilbert space $L^2(Q)$, where
$L^2(Q)$ is the set of functions $f : \X \to \R$ satisfying
\begin{equation*}
  \sqrt{\int_\X f^2(x) \mathrm{d}Q(x)} < \infty.
\end{equation*}
Namely, consider the functions parametrised by $w,z\in\Z$ and $b\in[0,2\pi]$:
\begin{equation*}
g_{w,b}^z (\cdot) = 2c_k\cos(\langle w,z\rangle + b)\cos(\langle w,\cdot\rangle + b),
\end{equation*}
which belong to $L^2(Q)$, since they are bounded.  If we sample
$\{(w_j,b_j)\}_{j=1}^m$ iid, as discussed above, the average 
\[
\hat{g}_m^z(\cdot) = \frac{1}{m}\sum_{i=1}^m g_{w_i,b_i}^z(\cdot)
\]
is an $L^2(Q)$-valued random variable.  Moreover,
Section~\ref{sec:random-mercer-features} showed that $\E{\bm w, \bm b}{\hat{g}_m^z(\cdot)} =
k(z,\cdot)$.  This enables us to invoke concentration inequalities for Hilbert
spaces \citep{LT91}, to show the following result. For simplicity, the
following lemma uses \citet[Lemma 1]{Rahimi08}, although a tighter bound could
be achieved using recent results from \citep{Sriperumbudur15}.

\begin{lemma}[Convergence of random features to $L^2(Q)$ functions]
\label{lemma:approx}
Let $\Z=\R^d$.  For any shift-invariant kernel~$k$, such that 
$\sup_{z\in\Z}k(z,z)\leq 1$, any fixed $S=\{z_i\}_{i=1}^n\subset \Z$, any
probability distribution $Q$ on $\Z$, and any $\delta > 0$, we have
\[
\left\|\mu_k(P_S) - \frac{1}{n}\sum_{i=1}^n\hat{g}_m^{z_i}(\cdot)\right\|_{L^2(Q)}
\!\!\!\!\!\leq 
\frac{2c_k}{\sqrt{m}}\left(1 + \sqrt{{2\log(n/\delta)}}\right)
\]
with probability larger than $1-\delta$ over $\{(w_i,b_i)\}_{i=1}^m$.
\end{lemma}
\begin{proof}
See Section \ref{proof:ApproxLemma}.
\end{proof}

Once sampled, the parameters $\{(w_i,b_i)\}_{i=1}^m$ allow us to approximate
the empirical kernel mean embeddings $\{\mu_k(P_{S_i})\}_{i=1}^n$ using
elements from $\text{span}({\{\cos(\langle w_i, \cdot\rangle +
b_i)\}_{i=1}^m})$, which is a finite-dimensional subspace of $L^2(Q)$.  
Therefore, we propose to
use $\{(\mu_{k,m}(P_{S_i}),l_i)\}_{i=1}^n$ as the training sample for our final
empirical risk minimization problem, where
\begin{equation}\label{eq:meank3}
  \mu_{k,m}(P_S) = \frac{2c_k}{|S|} \sum_{z\in S} \bigl( \cos(\langle w_j,
  z\rangle + b_j)\bigr)_{j=1}^m \in
  \R^m.
\end{equation}

These $m$-dimensional embeddings require $O(m)$ computation time and $O(1)$
memory storage; Moreover, these finite dimensional embeddings are compatible
with most off-the-shelf learning algorithms.  For the precise excess risk
bounds that take into account the use of these low-dimensional approximations,
see Theorem~\ref{thm:new-theorem} in Section~\ref{sec:new-theorem}.

\section{Extensions to multivariate causal inference}\label{sec:dags}

Although we have focused so far on causal inference between two variables, it
is possible to extend our framework to infer causal relationships between $d
\geq 2$ variables $\bm x=(\bm x_1,\ldots,\bm x_d)$. To this end, as introduced in
Section~\ref{sec:from-sem-to-causality}, assume the existence of a causal directed acyclic
graph $G$ which underlies the dependencies in the probability
distribution $P(\bm x)$. Therefore, our task is to recover $G$ from $S \sim P^n$.

Na\"ively, one could extend the framework presented in Section~\ref{sec:theory}
from the binary classification of $2$-dimensional distributions to the
multiclass classification of $d$-dimensional distributions. Unfortunately, the
number of possible DAGs, which equals the number of labels in the planned
multiclass classification problem, grows super-exponentially in $d$. As an
example, attacking causal inference over ten variables using this strategy
requires solving a classification problem with $4175098976430598143$ different
labels.

An alternative approach is to consider the probabilities of the three labels
``$\bm x_i \to \bm x_j$'', ``$\bm x_i \leftarrow \bm x_j$'', and ``$\bm x_i \indep \bm x_j$'' for each
pair of variables $\{\bm x_i,\bm x_j\}\subseteq X$, when embedded along with every
possible \emph{context} $\bm x_k \subseteq \bm x \setminus \{\bm x_i,\bm x_j\}$. The intuition
here is the same as in the PC algorithm described in
Section~\ref{sec:causal-inference-algorithms}: in order to decide the (absence
of a) causal relationship between $\bm x_i$ and $\bm x_j$, one must analyze the
confounding effects of every $\bm x_k \subseteq \bm x \setminus \{\bm x_i,\bm x_j\}$. 

\section{Numerical simulations}\label{sec:exps}

We conduct an array of experiments to test the effectiveness of a simple
implementation of the causal learning framework described in
Definition~\ref{def:learning-setup}, illustrated in
Figure~\ref{fig:causalsketch}, and analyzed in Section~\ref{sec:theory}. Since
we will use a set of random features to represent cause-effect 
samples, we term our method the \emph{Randomized Causation Coefficient}
(RCC).

\subsection{Setting up RCC}\label{sec:rcc}

In the following three sections we define the protocol to construct our causal
direction finder, RCC. To this end, we need to i) construct synthetic
distributions on two random variables with labeled causal structures, ii)
featurize those distributions into $m$-dimensional empirical mean embeddings,
and iii) train a binary classifier on embeddings and labels. 

\subsubsection{Synthesis of observational samples}
We setup the following generative model to synthesize training data for RCC.\@
Importantly, this generative model is independent from all the experiments that
follow.  We build $N = 10000$ observational samples $S_i = \{(x_{i,j},
y_{i,j}\}_{j=1}^{n_i}$, with $n_i = n = 1000$ for all $1 \leq i \leq n$. The
observational sample $S_i$ has a set of random hyper-parameters drawn from 
\begin{align*}
  k_i &\sim \text{RandomInteger}[1,10),\\
  m_i &\sim \text{Uniform}[0,10),\\
  s_i &\sim \text{Uniform}[1,10),\\
  v_i &\sim \text{Uniform}[0,10),\\
  d_i &\sim \text{RandomInteger}[4,10),\\
  f_i &\sim \text{RandomSpline}(d_i),
\end{align*}
where RandomSpline is a smoothing spline with $d_i$ knots sampled from 
Gaussian$(0,1)$. After sampling one set of random hyper-parameters, the pairs
$(x_{i,j}, y_{i,j})$ forming the observational sample $S_i$ follow 
the generative model
\begin{align*}
  x_{i,j} &\sim \text{GMM}(k_i,m_i,s_i),\\
  \epsilon_{i,j} &\sim \text{Gaussian}(0,v_i),\\
  y_{i,j} &\leftarrow f_i(x_{i,j},\epsilon_{i,j}),
\end{align*}
where $\text{GMM}(k,p_1,p_2,v)$ is a Gaussian Mixture Model of $k_i$ components
with mixing weights sampled from Uniform$[0,1)$ and normalized to sum to one,
component means sampled from $\text{Gaussian}(0,m_i^2)$, and 
variance magnitudes sampled from
$\text{Gaussian}(0,s_i^2)$.  We now have a collection $\{S_i\}_{i=1}^N$ of
observational samples $S_i = \{(x_{i,j},y_{i,j})\}_{j=1}^{n}$ with known causal
relationship $\bm x_i \to \bm y_i$, for all $1 \leq i \leq N$. To learn from these
data, we first have to featurize it into a vector representation compatible
with off-the-shelf binary classifiers.

\subsubsection{Featurization of observational samples}
After constructing each observational sample $S_i = \{(x_{i,j}, y_{i,j})\}_{j=1}^n$,
the \emph{featurized} training data for RCC is 
\begin{align*}
  D = \{&(M(\{(x_{i,j},y_{i,j})\}_{j=1}^n), +1),\nonumber\\
        &(M(\{(y_{i,j},x_{i,j})\}_{j=1}^n), -1)\}_{i=1}^N,
\end{align*}
where we assume that all observational samples have zero mean and unit
variance. Here the featurization map $M$ accepts an observational sample $S_i =
\{(x_{i,j},y_{i,j})\}_{i=1}^n$ and takes the form 
\begin{align}\label{eq:ourfeatz}
  M(S_i) = \frac{1}{n} \sum\nolimits_{j=1}^n \big(
    &(\cos(\langle w^x_k, x_{i,j} \rangle + b_k))_{k=1}^m,\\
    &(\cos(\langle w^y_k, y_{i,j} \rangle + b_k))_{k=1}^m,\nonumber\\
    &(\cos(\langle (w^x_k,w^y_k), (x_{i,j},y_{i,j}) \rangle + b_k))_{k=1}^m
  \big) \in \R^{3m}\nonumber,
\end{align}
where $w^x_k, w^y_k\sim \N(0,2\gamma)$, $b_k \sim \mathcal{U}[0,2\pi]$, and
$m=500$, for all $1 \leq k \leq m$.

This featurization is a randomized approximation of the empirical kernel mean
embedding associated to the Gaussian kernel \eqref{eq:gauss}. To improve
statistical efficiency, $M$ embeds the marginal distribution of $\bm
x_i$, the marginal distribution of $\bm y_i$, and the joint distribution of
$(\bm x_i, \bm y_i)$ separately. This separate embedding is also to facilitate the inference 
of causal asymmetries between marginal
and conditional distributions. In practice we concatenate the 
$M$ associated with all bandwidths $\gamma \in \{10^{-2}, \ldots, 10^2\}$,
following the multiple kernel learning strategy described in
Remark~\ref{remark:mkl}.

We are almost ready: the synthetic featurized observational data $D$ contains 
pairs of $3m-$dimensional real vectors $M(S_i)$ and binary labels
$l_i$; therefore, we can now use any standard binary classifier to
predict the cause-effect relation for a new observational sample $S$.

\subsubsection{What classifier to use?}

To classify the embeddings \eqref{eq:ourfeatz} into causal or anticausal, we
use the random forest implementation from Python's \texttt{sklearn-0.16-git},
with $1000$ trees. Random forests are the most competitive alternatives from
all the classifiers that we tested, including support vector machines, gradient
boosting machines, and neural networks. One possible reason for this is that
random forests, as a bagging ensemble, aim at reducing the predictive variance.
This is beneficial in our setup, since we know a priori that the test data will
come from a different distribution than the Mother distribution. In
terms of our theory, the random forest feature
map~\eqref{eq:random-forest-features} induces a valid kernel, over which we
perform linear classification. 

\subsection{Classification of T\"ubingen cause-effect pairs}\label{sec:tuebingen}

The \emph{T\"ubingen cause-effect pairs v0.8} is a collection of heterogeneous,
hand-collected, real-world cause-effect samples \citep{Mooij14}.  Figure
\ref{fig:tuebingen} plots the classification accuracy of RCC, IGCI (see
Section~\ref{sec:igci}), and ANM (see Section~\ref{sec:anm}) versus the
fraction of decisions that the algorithms are forced to take of the 82 scalar
T\"ubingen cause-effect pairs. Each algorithm sorts its decisions in 
decreasing order by confidence.  To compare these results to other
lower-performance methods, refer to \citet{Janzing12}. Overall, RCC surpasses
the state-of-the-art in these data, with a classification accuracy of $82.47\%$
when inferring the causal directions on all pairs. The confidence of RCC are
the random forest class probabilities. Computing the RCC statistic for the
whole T\"ubingen dataset takes under three seconds in a single 1.8GhZ
processor.

\begin{figure}[t]
  \begin{center}
    \includegraphics[width=\linewidth]{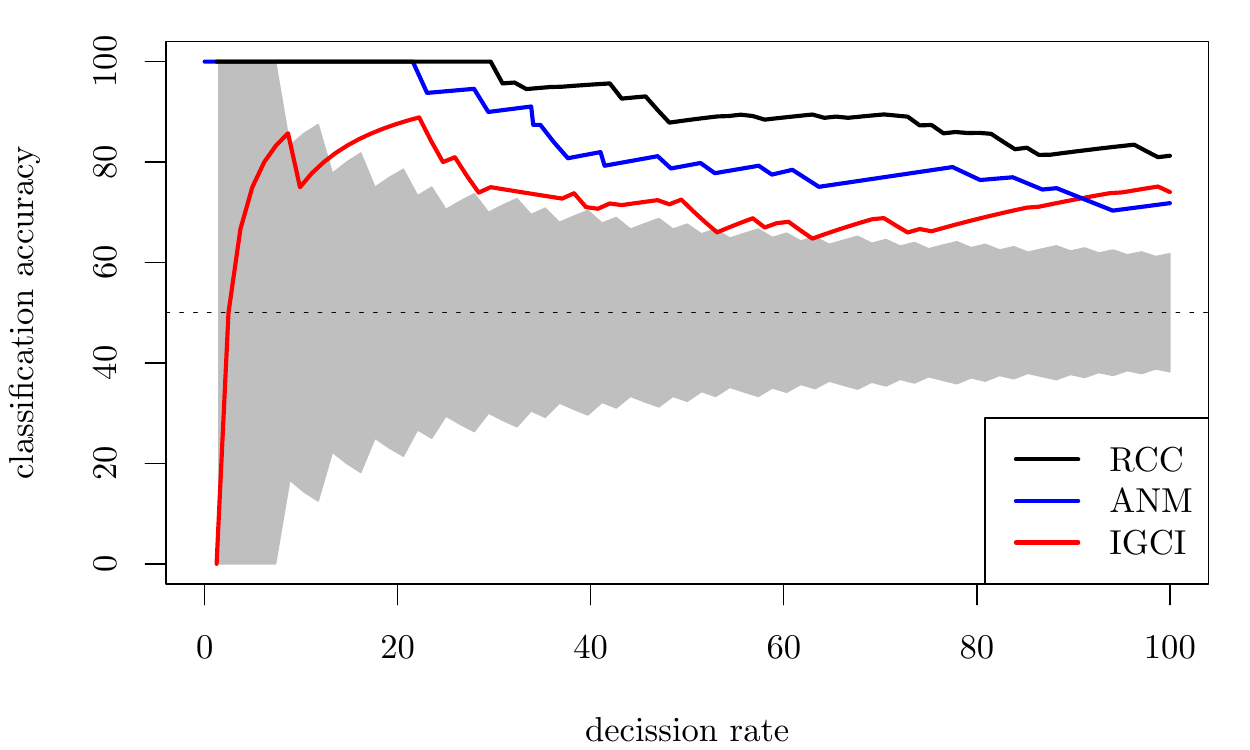} 
  \end{center}
  \caption[Results on T\"ubingen cause-effect pairs]{Accuracy of RCC, IGCI and
  ANM on the T\"ubingen cause-effect pairs, as a function of decision rate. The
  gray area depicts accuracies not statistically significant.}
  \label{fig:tuebingen}
\end{figure}

\subsection{Inferring the arrow of time}

We apply RCC to infer the arrow of time from causal
time series. More specifically, we assume access to a time series $
(x_{i})_{i=1}^{n}$, and our task is to infer whether $\bm x_i
\to \bm x_{i+1}$ or $\bm x_i \leftarrow \bm x_{i+1}$.

We compare RCC to the state-of-the-art of \citet{Peters09}, using the
same electroencephalography signals \citep{EEG} as in their original
experiment.  On the one hand, \citet{Peters09} construct two Auto-Regressive
Moving-Average (ARMA) models for each causal time series and time direction,
and prefers the solution under which the model residuals are independent from
the inferred cause. To this end, the method uses two parameters, chosen with
heuristics.  On the other hand, our approach makes no assumptions whatsoever
about the parametric model underlying the series, at the expense of requiring a
disjoint set of $N=10000$ causal time series for training. Our method matches
the best performance of \citet{Peters09}, with an accuracy of $82.66\%$.
 
\subsection{ChaLearn's challenge data}

The cause-effect challenges organized by \citet{Codalab14} provided $N =
16199$ training causal samples $S_i$, each drawn from the distribution of $\bm x_i
\times \bm y_i$, and labeled either ``$\bm x_i \to \bm y_i$'', ``$\bm x_i \leftarrow \bm y_i$'',
``$\bm x_i \leftarrow \bm z_i \to \bm y_i$'', or ``$\bm x_i \indep \bm y_i$''. The goal of the
competition was to develop a \emph{causation coefficient} which would predict
large positive values to causal samples following ``$\bm x_i \to \bm y_i$'', large
negative values to samples following ``$\bm x_i \leftarrow \bm y_i$'', and zero
otherwise.  Using these data, RCC obtained a test
\textit{bidirectional area under the curve score} \citep{Codalab14} of $0.74$
in one minute and a half.  The winner of the competition obtained a score of
$0.82$ in thirty minutes, and resorted to dozens of hand-crafted
features. Overall, RCC ranked third in the competition.

Partitioning these same data in different ways, we learned two related but
different binary classifiers. First, we trained one classifier to
\emph{detect latent confounding}, and obtained a test classification accuracy
of $80\%$ on the task of distinguishing ``$\bm x \to \bm y$ or $\bm x
\leftarrow \bm y$'' from ``$\bm x \leftarrow \bm z \to \bm y$''.  Second, we
trained a second classifier to \emph{measure dependence}, and obtained a test
classification accuracy of $88\%$ on the task of distinguishing between ``$\bm
x \indep \bm y$'' and ``else''. We consider this result to be a promising
direction to learn nontrivial statistical tests \emph{from} data.

\subsection{Reconstruction of causal DAGs}\label{sec:dagsexperiment}

We apply the strategy described in Section~\ref{sec:dags} to reconstruct the
causal DAGs of two multivariate datasets: \emph{autoMPG} and \emph{abalone}
\citep{UCI}. Once again, we resort to synthetic training data, generated in a
similar procedure to the one used in Section~\ref{sec:tuebingen}. Refer to
Section~\ref{sec:dagtrain} for details.

Regarding \emph{autoMPG}, in Figure~\ref{fig:auto}, we can see that 1) the
release date of the vehicle (AGE) causes the miles per gallon consumption
(MPG), acceleration capabilities (ACC) and horse-power (HP), 2) the weight of
the vehicle (WEI) causes the horse-power and MPG, and that 3) other
characteristics such as the engine displacement (DIS) and number of cylinders
(CYL) cause the MPG.\@ For \emph{abalone}, in Figure~\ref{fig:abalone}, we can
see that 1) the age of the snail causes all the other variables, 2) the partial
weights of its meat (WEA), viscera (WEB), and shell (WEC) cause the overall
weight of the snail (WEI), and 3) the height of the snail (HEI) is responsible
for other physically attributes such as its diameter (DIA) and length (LEN).

In Figures \ref{fig:auto} and \ref{fig:abalone}, the target variable for 
each dataset is shaded in gray. Our inference
reveals that the \emph{autoMPG} dataset is a \emph{causal} prediction task (the
features \emph{cause} the target), and that the \emph{abalone} dataset is an
\emph{anticausal} prediction task (the target \emph{causes} the features). This
distinction has implications when learning from these data (Section~\ref{sec:causation_learning}).

\begin{figure}
\begin{center}
\begin{tikzpicture}[node distance=1cm, auto,]
 \node[punkt,minimum height=3em,fill=gray!30!white] (MPG) {MPG};
 \node[punkt,minimum height=3em, below=of MPG] (AGE) {AGE};
 \node[punkt,minimum height=3em, left=of AGE] (ACC) {ACC};
 \node[punkt,minimum height=3em, left=of ACC] (WEI) {WEI};
 \node[punkt,minimum height=3em, above=of WEI] (HP)  {HP};
 \node[punkt,minimum height=3em, right=of AGE] (CYL) {CYL};
 \node[punkt,minimum height=3em, right=of MPG] (DIS) {DIS};
 \draw[pil] (WEI) -- (MPG);
 \draw[pil] (CYL) -- (MPG);
 \draw[pil] (DIS) -- (MPG);
 \draw[pil] (HP) -- (MPG);
 \draw[pil] (AGE) -- (MPG);
 \draw[pil] (AGE) -- (ACC);
 \draw[pil] (AGE) -- (HP);
 \draw[pil] (WEI) -- (HP);
\end{tikzpicture}
\end{center}
\caption{Causal DAG recovered from data \emph{autoMPG}.}
\label{fig:auto}
\end{figure}
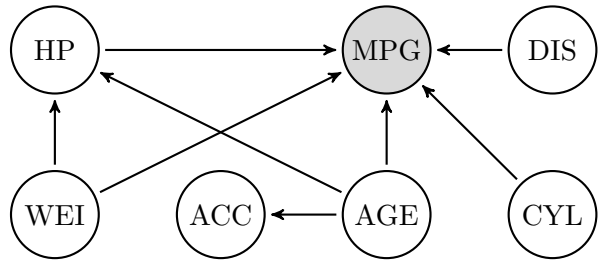

\begin{figure}
\begin{center}
\begin{tikzpicture}[node distance=1cm, auto,]
 \node[punkt,minimum height=3em, fill=gray!30!white] (AGE) {AGE};
 \node[punkt,minimum height=3em, left=of AGE] (WEC) {WEC};
 \node[punkt,minimum height=3em, left=of WEC] (WEB) {WEB};
 \node[punkt,minimum height=3em, above=of WEB] (WEA) {WEA};
 \node[punkt,minimum height=3em, above=of AGE] (LEN) {LEN};
 \node[punkt,minimum height=3em, right=of LEN] (DIA) {DIA};
 \node[punkt,minimum height=3em, right=of AGE] (HEI) {HEI};
 \node[punkt,minimum height=3em, right=of WEA] (WEI) {WEI};
 \draw[pil] (AGE) -- (LEN);
 \draw[pil] (AGE) -- (HEI);
 \draw[pil] (AGE) -- (WEI);
 \draw[pil] (AGE) -- (WEA);
 \draw[pil] (AGE) -- (WEC);
 \draw[pil] (WEC) -- (WEB);
 \draw[pil] (WEC) -- (WEI);
 \draw[pil] (WEA) -- (WEI);
 \draw[pil] (WEB) -- (WEI);
 \draw[pil] (HEI) -- (LEN);
 \draw[pil] (HEI) -- (DIA);
 \draw[pil] (AGE) -- (DIA);
 \draw[pil] (AGE) to[bend left] (WEB);
\end{tikzpicture}
\end{center}
\caption{Causal DAG recovered from data \emph{abalone}.}
\label{fig:abalone}
\end{figure}
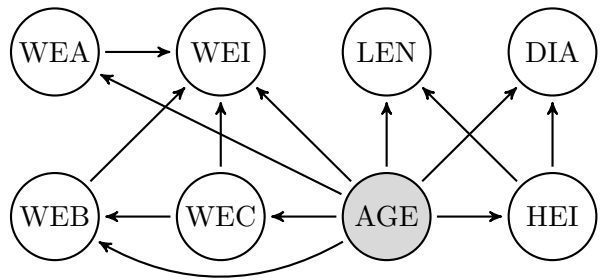

\section{Future research directions}
  \paragraph{Causation and optimal transport} The probabilistic account of
  causation allows for a mathematical characterization of change: given a pair
  of cause and effect distributions, causation is the operator mapping the
  cause distribution to the effect distribution. This is a common operation in
  the research field of \emph{optimal transportation}
  \citep{villani2003topics}, concerned with studying how to optimally map one
  distribution into another. However, the author is not aware of any
  cross-fertilization between the research fields of causal inference and
  optimal transportation.
  
  \paragraph{Causal regularization} Differentiable causal inference
  methods can act as \emph{causal regularizers}. In particular, one could use
  RCC or any other causal direction finder to promote learning (anti)causal
  features in unsupervised learning algorithms, or (anti)causal interventions
  (with respect to an effect of interest) in reinforcement learning
  environments.

\section{Discovering causal signals in images}\label{sec:visual-causation}

Imagine an image of a bridge over a river. On top of the
bridge, a car speeds through the right lane. Consider the question
\begin{center}
  \emph{``Is there a car in this image?''}
\end{center}
This is a question about the observable properties of the scene under
consideration, and modern computer vision algorithms excel at answering these
kinds of questions.  Excelling at this task is fundamentally about leveraging
correlations between pixels and image features across large datasets of
images.\footnote{Here and below, the term {\em correlation} is meant to include
the more general concept of {\em statistical dependence}. The term
\emph{feature} denotes, for instance, a numerical value from the image
representation of a convolutional neural network.} However, a more nuanced
understanding of images arguably requires the ability to \emph{reason about}
how the scene depicted in the image would change in response to interventions.
The list of possible interventions is long and complex but, as a first step, we
can reason about the intervention of removing an object.

To this end, consider the two counterfactual questions \emph{``What would the
scene look like if we were to remove the car?''} and \emph{``What would the
scene look like if we were to remove the bridge?''} On the one hand, the first
intervention seems rather benign. We could argue that the rest of the scene
depicted in the image (the river, the bridge) would remain the same if the car
were removed. On the other hand, the second intervention seems more severe. If
the bridge were removed from the scene, it would make little sense for us to
observe the car floating weightless over the river. Thus, we understand that
removing the bridge would have an effect on the cars located on top of it. 
Reasoning about these and similar counterfactuals allows to begin
asking questions of the form
\begin{center}
  \emph{``Why is there a car in this image?''}
\end{center}
This question is of course poorly defined, but the answer is linked to the
causal relationship between the bridge and the car. In our example, the
presence of the bridge {\em causes} the presence of the car, in the sense that
if the bridge were not there, then the car would not be either. Such {\em
interventional} semantics of what is meant by {\em causation} aligns with
current approaches in the literature \citep{pearl2009causality}.

In light of this exposition, it seems plausible that the objects in a scene
share asymmetric causal relationships. These causal relationships, in turn, may
differ significantly from the correlation structures that modern computer
vision algorithms exploit. For instance, most of the images of cars in a given
dataset may also contain roads.  Therefore, features of cars and features of
roads will be highly correlated, and therefore features of roads may be good car
predictors in an iid setting irrespective of the underlying causal structure
\citep{scholkopf12anti}. However, should a car sinking in the ocean be given a low
``car score'' by our object recognition algorithm because of its unusual
context?  The answer depends on the application. If the goal is to maximize the
average object recognition score over a test set that has the same distribution
as the training set, then we should use the context to make our decision.
However, if the goal is to reason about non-iid situations, or cases that may
require intervention, such as saving the driver from drowning in the ocean, we
should be robust and not refuse to believe that a car is a car just because of
its context.

While the correlation structure of image features may shift dramatically
between different data sets or between training data and test data, we expect
the causal structure of image features to be more stable. Therefore, object
recognition algorithms capable of leveraging knowledge of the cause-effect
relations between image features may exhibit better generalization to novel
test distributions.  For these reasons, the detection of causal signals in
images is of great interest.  However, this is a very challenging task: in
static image datasets we lack the arrow of time, face strong selection biases
(pictures are often taken to show particular objects), and 
randomized experiments (the gold standard to infer causation) are
unfeasible. Because of these reasons, our present interest is in
detecting causal signals in \emph{observational} data.

In the absence of any assumptions, the determination of causal relations
between random variables given samples from their joint distribution is
impossible in principle \citep{pearl2009causality,peters2014causal}. In particular, any
joint distribution over two random variables $A$ and $B$ is consistent with any
of the following three underlying causal structures: (i) $A$ causes $B$, (ii)
$B$ causes $A$, and (iii) $A$ and $B$ are both caused by an unobserved confounder $C$
\citep{reichenbach56}.  However, while the causal structure may not be
identifiable in principle, it may be possible to determine the structure in
practice.  For joint distributions that occur in the real world, the different
causal interpretations may not be equally likely. That is, the causal direction
between typical variables of interest may leave a detectable signature in their
joint distribution. In this work, we will exploit this insight to build a
classifier for determining the cause-effect relation between two random
variables from samples of their joint distribution.

Our experiments will show that the higher-order statistics of image datasets
can inform us about causal relations. To our knowledge, \emph{no prior work has
established, or even considered, the existence of such a signal.}

In particular, we make a first step towards the discovery of causation in
visual features by examining large collections of images of different
\emph{objects of interest} such as cats, dogs, trains, buses, cars, and people.
The locations of these objects in the images are given to us in the form of
bounding boxes. For each object of interest, we can distinguish between
\emph{object features} and \emph{context features}. By definition, object
features are those mostly activated inside the bounding box of the object of
interest. On the other hand, context features are those mostly found outside
the bounding box of the object of interest. Independently and in parallel, we
will distinguish between \emph{causal features} and \emph{anticausal features},
cf. \citep{scholkopf12anti}.  Causal features are those that \emph{cause the
presence of the object of interest in the image} (that is, those features that
cause the object's class label), while anticausal features are those
\emph{caused by the presence of the object in the image} (that is, those
features caused by the class label). Our hypothesis, to be validated
empirically, is 
\begin{hypothesis}\label{hyp:our-hypothesis}
  Object features and anticausal features are closely related. 
  Context features and causal features are not necessarily related. 
\end{hypothesis}
We expect Hypothesis~\ref{hyp:our-hypothesis} to be true because many of the
features caused by the presence of an object should be features of subparts of
the object and hence likely to be contained inside its bounding box (the
presence of a car causes the presence of the car's wheels).  However, the
context of an object may cause or be caused by its presence (road-like features
cause the presence of a car, but the presence of a car causes its shadow on a
sunny day).  Providing empirical evidence supporting
Hypothesis~\ref{hyp:our-hypothesis} would imply that (1) there exists a
relation between causation and the difference between objects and their
contexts, and (2) there exist observable causal signals within sets of static
images.

Our exposition is organized as follows.  Section~\ref{sec:ncc} proposes a new
algorithm, the Neural Causation Coefficient (NCC), for learning to infer
causation from a corpus of labeled data end-to-end using neural networks.
Section~\ref{sec:visual-causation2} makes use of NCC to distinguish between
causal and anticausal features. As hypothesized, we show a consistent
relationship between anticausal features and object features.

\begin{example}[Tanks in bad weather]
  The US Army was once interested in detecting the presence of camouflaged
  tanks in images. To this end, the Army trained a neural network on a
  dataset of 50 images containing camouflaged tanks, and 50 images not
  containing camouflaged tanks.  Unluckily, all the images containing tanks
  were taken in cloudy days, and all the images not containing tanks were
  taken in sunny days. Therefore, the resulting neural network turned out to
  be a ``weather classifier'', and its performance to detect tanks in new
  images was barely above chance \citep{yudkowsky2008artificial}.
\end{example}

\subsection{The neural causation coefficient}\label{sec:ncc}

To learn causal footprints from data, we follow Section~\ref{sec:rcc} and pose
cause-effect inference as a binary classification task. Our input patterns
$S_i$ are effectively scatterplots similar to those shown in
Figure~\ref{fig:anm}. That is, \emph{each data point is a bag of samples
$(x_{ij}, y_{ij}) \in \R^2$ drawn iid from a distribution $P(X_i,Y_i)$}.  The
class label $l_i$ indicates the causal direction between $X_i$ and $Y_i$.
\begin{align}
  D   &= \{(S_i, l_i)\}_{i=1}^n,\nonumber\\
  S_i &= \{(x_{ij},y_{ij})\}_{j=1}^{m_i} \sim P^{m_i}(X_i,Y_i),\nonumber\\
  l_i &= \begin{cases} 0 & \text{if $X_i \to Y_i$} \\ 1 & \text{if $X_i \leftarrow Y_i$} \end{cases}.
  \label{eq:causal-data}
\end{align}
Using data of this form, we will train a neural network to classify samples
from probability distributions as causal or anticausal. Since the input
patterns $S_i$ are not fixed-dimensional vectors, but bags of points, we borrow
inspiration from the literature on kernel mean embedding classifiers
\citep{Smola07Hilbert} and construct a feedforward neural network of the form
\begin{equation*}
  \text{NCC}(\{(x_{ij},y_{ij})\}_{j=1}^{m_i}) = \psi\left(\frac{1}{m_i} \sum_{j=1}^{m_i} \phi(x_{ij}, y_{ij})\right).
\end{equation*}
In the previous equation, $\phi$ is a \emph{feature map}, and the average over
all $\phi(x_{ij}, y_{ij})$ is the \emph{mean embedding} of the empirical
distribution $\frac{1}{m_i} \sum_{i=1}^{m_i} \delta_{(x_{ij},y_{ij})}$. The
function $\psi$ is a binary classifier that takes a fixed-length mean embedding
as input (Section~\ref{sec:distrib}). 

In kernel-based methods such as RCC (Section~\ref{sec:rcc}), $\phi$ is fixed a
priori and defined with respect to a nonlinear kernel \citep{Smola07Hilbert},
and $\psi$ is a separate classifier.  In contrast, our feature map $\phi :
\mathbb{R}^2 \to \mathbb{R}^h$ and our classifier $\psi : \mathbb{R}^h \to
\{0,1\}$ are both multilayer perceptrons, which are learned jointly from data.
Figure~\ref{fig:ncc} illustrates the proposed architecture, which we term the
Neural Causation Coefficient (NCC). In short, to classify a sample $S_i$ as
causal or anticausal, NCC maps each point $(x_{ij}, y_{ij})$ in the sample
$S_i$ to the representation $\phi(x_{ij},y_{ij}) \in \mathbb{R}^h$, computes
the embedding vector $\phi_{S_i} := \frac{1}{m_i} \sum_{j=1}^{m_i}
\phi(x_{ij},y_{ij})$ across all points $(x_{ij}, y_{ij})\in S_i$, and
classifies the embedding vector $\phi_{S_i} \in \mathbb{R}^h$ as causal or
anticausal using the neural network classifier $\psi$. Importantly, the
proposed neural architecture is not restricted to cause-effect inference, and
can be used to represent and learn from general distributions.

NCC has some attractive properties. First, predicting the cause-effect relation
for a new set of samples at test time can be done efficiently with a single
forward pass through the aggregate network.  The complexity of this operation
is linear in the number of samples.  In contrast, the computational
complexity of kernel-based additive noise model inference algorithms is cubic
in the number of samples $m_i$.  Second, NCC can be trained using mixtures of
different causal and anticausal generative models, such as linear,
non-linear, noisy, and deterministic mechanisms linking causes to their
effects.  This rich training allows NCC to learn a diversity of causal
footprints simultaneously.  Third, for differentiable activation functions,
NCC is a differentiable function. This allows us to embed NCC into larger
neural architectures or to use it as a regularization term to encourage the
learning of causal or anticausal patterns.

The flexibility of NCC comes at a cost. In practice, labeled cause-effect data
as in Equation~\eqref{eq:causal-data} is scarce and laborious to collect.
Because of this, we follow Section~\ref{sec:rcc} and train NCC on artificially
generated data.

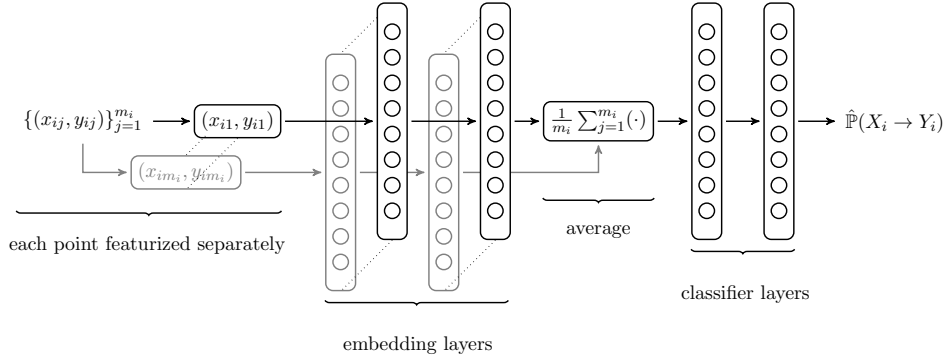
\begin{figure}
  \begin{center}
  \resizebox{\textwidth}{!}{
  \begin{tikzpicture}
    \node[] (input) at (-6,0) {$\lbrace (x_{ij}, y_{ij}) \rbrace_{j=1}^{m_i}$};
    \node[] (dummy) at (-6,-1) {};
    
    \node[cir] (inp1) at (-3,0) {$(x_{i1}, y_{i1})$};
    \node[cir,gray] (inp2) at (-4,-1) {$(x_{im_i}, y_{im_i})$};
    
    \draw[dotted, shorten <=25pt] (inp1.north) edge (inp2.north);
    \draw[dotted] (inp1.south) edge (inp2.south);
    
    \draw [arr] (input) edge (inp1);
    \path [draw, gray, arr] (input) |- (inp2);

    \drawlayer{emb22}{1}{-1}{gray}
    \node[cir] (mean) at (4,0) {$\frac{1}{m_i} \sum_{j=1}^{m_i} (\cdot)$};
    
    \path [draw, gray, arr] (emb22.east) -| (mean.south);
    \drawlayer{emb21}{2}{0}{}
    
    \drawlayer{emb12}{-1}{-1}{gray}
    \draw [arr,gray] (emb12) edge (emb22);
  
    \drawlayer{emb11}{0}{0}{}
    \draw [arr] (emb11) edge (emb21);
    
    \draw [arr,gray] (inp2) edge (emb12);
    \draw [arr] (inp1) edge (emb11);
    
    \draw[dotted, shorten <=12pt] (emb11.north) edge (emb12.north);
    \draw[dotted] (emb11.south) edge (emb12.south);
    \draw[dotted, shorten <=12pt] (emb21.north) edge (emb22.north);
    \draw[dotted] (emb21.south) edge (emb22.south);
    
    \draw[arr] (emb21) edge (mean);
    
    \drawlayer{clf1}{6.1}{0}{}
    \drawlayer{clf2}{7.5}{0}{}
  
    \node[] (prob) at (9.75,0) {$\hat{\mathbb{P}}(X_i \to Y_i)$};
  
    \draw [arr] (mean) edge (clf1);
    \draw [arr] (clf1) edge (clf2);
    \draw [arr] (clf2) edge (prob);
  
    \draw [brace] (clf1.west) -- (clf2.east);
    
    \draw [brace] (emb12.west) -- (2.25,-1);
  
    \draw [brace, decoration={raise=1.5cm}] (mean.west) -- (mean.east);
    
    \draw [brace, decoration={raise=1.75cm}] (input.west) -- (inp1.east);
    
    \node[below=1.5cm of mean] (meanlabel) {average};
    
    \node[below=0.75cm of clf2, xshift=-.65cm] (clflabel) {classifier layers};
    
    \node[below=0.75cm of emb22, xshift=-.5cm] (clflabel) {embedding layers};
    
    \node[below=1.75cm of input, xshift=1.25cm] (clflabel) {each point featurized separately};
  \end{tikzpicture}
  }
  \end{center}
  \caption{Scheme of the Neural Causation Coefficient (NCC) architecture.}
  \label{fig:ncc}
\end{figure}

\subsubsection{Synthesis of training data}
\label{sec:syntheticdata}

We will construct $n$ synthetic observational samples, where the $i$th observational
sample contains $m_i$ points. The points comprising the observational sample
$S_i = \{(x_{ij}, y_{ij})\}_{j=1}^{m_i}$ are drawn from an additive noise
model $y_{ij} \leftarrow f_i(x_{ij}) + v_{ij}e_{ij}$, for all $j=1, \ldots, m_i$.

The \emph{cause terms} $x_{ij}$ are drawn from a mixture of $k_i$ Gaussians
distributions. We construct each Gaussian by sampling its mean from
$\text{Gaussian}(0,r_i)$, its standard deviation from $\text{Gaussian}(0,s_i)$
followed by an absolute value, and its unnormalized mixture weight from
$\text{Gaussian}(0,1)$ followed by an absolute value. We sample $k_i \sim
\text{RandomInteger}[1,5]$ and $r_i, s_i \sim \text{Uniform}[0,5]$. We
normalize the mixture weights to sum to one. We normalize 
$\{x_{ij}\}_{j=1}^{m_i}$ to zero mean and unit variance.

The \emph{mechanism $f_i$} is a cubic Hermite spline with support
\begin{equation}\label{eq:support}
[\min(\{x_{ij}\}_{j=1}^{m_i})-\text{std}(\{x_{ij}\}_{j=1}^{m_i}),
\max(\{x_{ij}\}_{j=1}^{m_i})+\text{std}(\{x_{ij}\}_{j=1}^{m_i})],
\end{equation}
and $d_i$ knots drawn from $\text{Gaussian}(0,1)$, where $d_i \sim
\text{RandomInteger}(4,5)$. The noiseless effect terms
$\{f(x_{ij})\}_{j=1}^{m_i}$ are normalized to have zero mean and unit variance.

The \emph{noise terms $e_{ij}$} are sampled from $\text{Gaussian}(0,v_i)$,
where $v_i \sim \text{Uniform}[0,5]$.  To slightly generalize
Section~\ref{sec:rcc}, we allow for heteroscedastic noise: we multiply each
$e_{ij}$ by $v_{ij}$, where $v_{ij}$ is the value of a smoothing spline with
support defined in Equation~\eqref{eq:support} and $d_i$ random knots drawn
from $\text{Uniform}[0,5]$.  The noisy effect terms $\{y_{ij}\}_{j=1}^{m_i}$
are normalized to have zero mean and unit variance.

This sampling process produces a training set of $2n$ labeled observational samples 
\begin{align}
  D = \{(\{(x_{ij},y_{ij})\}_{j=1}^{m_i}, 0)\}_{i=1}^n \cup
      \{(\{(y_{ij},x_{ij})\}_{j=1}^{m_i}, 1)\}_{i=1}^n.\label{eq:minibatch}
\end{align}

\subsubsection{Training NCC}

We train NCC with two embedding layers and two classification layers followed
by a softmax output layer.  Each hidden layer is a composition of batch
normalization \citep{ioffe2015batch}, $100$ hidden neurons, a rectified linear
unit, and $25\%$ dropout \citep{srivastava2014dropout}. We train for $10000$
iterations using RMSProp \citep{tieleman2012lecture} with the default parameters, where
each minibatch is of the form given in Equation~\eqref{eq:minibatch} and has
size $2n = 32$. Lastly, we further enforce the symmetry $\mathbb{P}(X \to Y) =
1 - \mathbb{P}(Y \to X)$, by training the composite classifier 
\begin{equation}\label{eq:composite}
\tfrac12\left(1 -
\text{NCC}(\{(x_{ij},y_{ij})\}_{j=1}^{m_i}) +
\text{NCC}(\{(y_{ij},x_{ij})\}_{j=1}^{m_{i}}) \right)\,,
\end{equation}
where $\text{NCC}(\{(x_{ij},y_{ij})\}_{j=1}^{m_i})$ tends to zero if the
classifier believes in $X_i \to Y_i$, and tends to one if the classifier
believes in $X_i \leftarrow Y_i$.  We chose our parameters by monitoring the
validation error of NCC on a held-out set of $10000$ synthetic observational
samples. Using this held-out validation set, we cross-validated the percentage
of dropout over $\{0.1,0.25,0.3\}$, the number of hidden layers over $\{2,3\}$,
and the number of hidden units in each of the layers over $\{50,100,500\}$.

\subsubsection{Testing NCC}

We test the performance of NCC on the T\"ubingen dataset, version 1.0
\citep{Mooij14}.  This is a collection of one hundred heterogeneous,
hand-collected, real-world cause-effect observational samples that are widely
used as a benchmark in the causal inference literature \citep{Mooij14}.  The
NCC model with the highest synthetic held-out validation accuracy correctly
classifies the cause-effect direction of $79\%$ of the T\"ubingen dataset
observational samples. We leave a detailed comparison between RCC and NCC for
future work.

\subsection{Causal signals in sets of static images} \label{sec:visual-causation2}

We have all the necessary tools to explore the existence of causal signals in
sets of static images at our disposal. In the following, we describe the
datasets that we use, the process of extracting features from these datasets,
and the measurement of \emph{object scores}, \emph{context scores},
\emph{causal scores}, and \emph{anticausal scores} for the extracted features.
Finally, we validate Hypothesis~\ref{hyp:our-hypothesis} empirically.

\subsubsection{Datasets}

We conduct our experiments with the two datasets PASCAL VOC 2012
\cite{pascal-voc-2012} and Microsoft COCO \cite{lin2014microsoft}.  These
datasets contain heterogeneous images collected ``in the wild.'' Each image may
contain multiple objects from different categories. The objects may appear at
different scales and angles and may be partially visible or occluded.  In the
PASCAL dataset, we study all the twenty classes aeroplane, bicycle, bird, boat,
bottle, bus, car, cat, chair, cow, dining table, dog, horse, motorbike, person,
potted plant, sheep, sofa, train, and television. This dataset contains 11541
images.  In the COCO dataset, we study the same classes.  This selection
amounts to 99,309 images.  We preprocess the images to have a shortest side of
$224$ pixels, and then take the central $224 \times 224$ crop.

\subsubsection{Feature extraction}

We use the last hidden representation (before its nonlinearity) of a residual
deep convolutional neural network of 18 layers \cite{resnet} as a feature
extractor. This network was trained on the entire ImageNet dataset
\cite{resnet}. In particular, we denote by $f_j = f(x_j) \in \mathbb{R}^{512}$
the vector of real-valued features obtained from the image $x_j \in
\mathbb{R}^{3 \times 224 \times 224}$ using this network. 

Building on top of these features and using the images from the PASCAL dataset,
we train a neural network classifier formed by two hidden layers of $512$ units
each to distinguish between the $20$ classes under study. In particular, we
denote by $c_j = c(x_j) \in \mathbb{R}^{20}$ the vector of continuous log odds
(activations before the classifier nonlinearity) obtained from the image $x_j
\in \mathbb{R}^{3 \times 224 \times 224}$ using this classifier. We use
features before their nonlinearity and log odds instead of the class
probabilities or class labels because NCC has been trained on continuous data
with full support on $\mathbb{R}$.

\begin{figure}
  \begin{center}
  \begin{subfigure}{.25\textwidth}
    \centering
    \includegraphics[width=\textwidth]{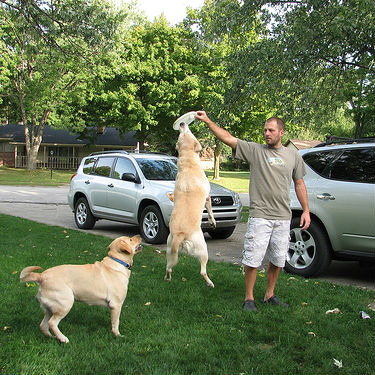}
    \caption{Original image $x_j$}
    \label{fig:images:original}
  \end{subfigure}
  \hspace{0.5cm}
  \begin{subfigure}{.25\textwidth}
    \centering
    \includegraphics[width=\textwidth]{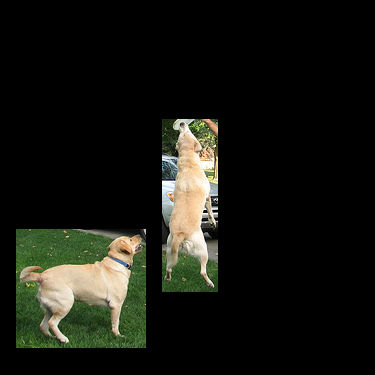}
    \caption{Object image $x^o_j$}
    \label{fig:images:object}
  \end{subfigure}
  \hspace{0.5cm}
  \begin{subfigure}{.25\textwidth}
    \centering
    \includegraphics[width=\textwidth]{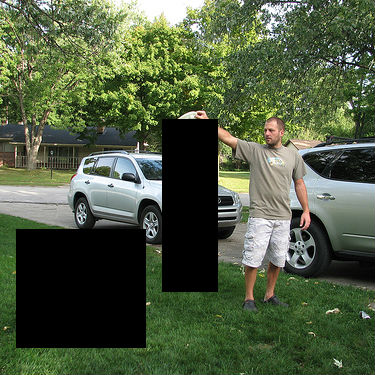}
    \caption{Context image $x^c_j$}
    \label{fig:images:context}
  \end{subfigure}%
  \end{center}
  \caption[Object and context blackout processes]{Blackout processes for object
  of interest ``dog''. Original images $x_j$ produce features $\{f_{jl}\}_l$
  and class-probabilities $\{c_{jk}\}_k$.  Object images $x^o_j$ produce
  features $\{f^o_{jl}\}_l$. Context images $x^c_j$ produce features
  $\{f_{jl}^c\}_l$. Blackout processes are performed after image normalization,
  in order to obtain true zero (black) pixels.}
  \label{fig:images}
\end{figure}

In the following we describe how to compute, for each feature $l = 1, \ldots,
512$, four different scores: its object score, context score, causal score, and
anticausal score. Importantly, the object/context scores are computed
independently from the causal/anticausal scores. For simplicity, the following
sections describe how to compute scores for a particular object of interest
$k$. However, our experiments will repeat this process for all the twenty
objects of interest. 

\subsubsection{Computing ``object'' and ``context'' feature scores}

We featurize each image $x_j$ in the COCO dataset in three different ways, for
all $j = 1 \ldots, m$.
First, we featurize the original image $x_j$ as $f_j := f(x_j)$.
Second, we blackout the context of the objects of interest $k$ in $x_j$ by
placing zero-valued pixels outside their bounding boxes. This produces the
object image $x^o_j$, as illustrated in Figure~\ref{fig:images:object}.  We
featurize $x^o_j$ as $f^o_j = f(x^o_j)$.
Third, we blackout the objects of interest $k$ in $x_j$ by placing zero-valued
pixels inside their bounding boxes. This produces the context image $x^c_j$, as
illustrated in Figure~\ref{fig:images:context}. We featurize $x^c_j$ as $f^c_j
= f(x^c_j)$.

Using the previous three featurizations we compute, for each feature $l = 1,
\ldots, 512$, its \emph{object score} $s^o_l = \frac{\sum_{j=1}^m \left|
f^c_{jl} - f_{jl} \right|}{\sum_{j=1}^m \left| f_{jl} \right|}$ and its
\emph{context score} $s^c_l = \frac{\sum_{j=1}^m \left| f^o_{jl} - f_{jl}
\right|}{\sum_{j=1}^m \left| f_{jl} \right|}$.  Intuitively, features with high
\emph{object scores} are those features that react violently when the object of
interest is removed from the image.

Furthermore, we compute the log odds for the presence of the object of interest
$k$ in the original image $x_j$ as $c_{jk} = c(x_j)_k$.

\subsubsection{Computing ``causal'' and ``anticausal'' feature scores}

For each feature $l$, we compute its \emph{causal score}
$1-\text{NCC}(\{(f_{jl}, c_{jk})\}_{j=1}^m)$, and its \emph{anticausal score}
$1-\text{NCC}(\{(c_{jk}, f_{jl})\}_{j=1}^m)$.  Because we will be examining one
feature at a time, the values taken by all other features will be an additional
source of noise to our analysis, and the observed dependencies will be much
weaker than in the synthetic NCC training data. To avoid detecting causation
between independent random variables, we train NCC with an augmented training
set: in addition to presenting each scatterplot in both causal directions as in
\eqref{eq:minibatch}, we pick a random permutation $\sigma$ to generate an
additional uncorrelated example $\{x_{i,\sigma(j)},y_{ij}\}_{j=1}^{m_i}$ with
label $\frac{1}{2}$. We use our best model of this kind which, for validation
purposes, obtains $79\%$ accuracy in the T\"ubingen dataset.

\subsection{Experiments}

Figure~\ref{fig:results} shows the mean and standard deviation of the object
scores and the context scores of the features with the top 1\% anticausal
scores and the top 1\% causal scores.
As predicted by Hypothesis~\ref{hyp:our-hypothesis}, object features are
related to anticausal features. In particular, the features with the highest
anticausal score exhibit a higher object score than the features with the
highest causal score.
This effect is consistent across all $20$ classes of interest when selecting
the top 1\% causal/anticausal features, and remains consistent across $16$ out
of $20$ classes of interest when selecting the top 20\% causal/anticausal
features.
These results indicate that anticausal features may be useful for detecting objects
in a robust manner, regardless of their context. As stated in
Hypothesis~\ref{hyp:our-hypothesis}, we could not find a consistent relationship
between context features and causal features. Remarkably, we remind the reader
that NCC was trained to detect the arrow of causation \emph{independently and
from synthetic data}.  As a sanity check, we did not obtain any similar results
when replacing the NCC with the correlation coefficient or the absolute value
of the correlation coefficient.

Although outside the scope of these experiments, we ran some preliminary experiments
to find causal relationships between objects of interest, by computing the NCC
scores between the log odds of different objects of interest.  The strongest
causal relationships that we found were ``bus causes car,'' ``chair causes
plant,'' ``chair causes sofa,'' ``dining table causes bottle,'' ``dining table
causes chair,'' ``dining table causes plant,'' ``television causes chair,'' and
``television causes sofa.''

\begin{figure}
  \begin{center}
  \includegraphics[width=\textwidth]{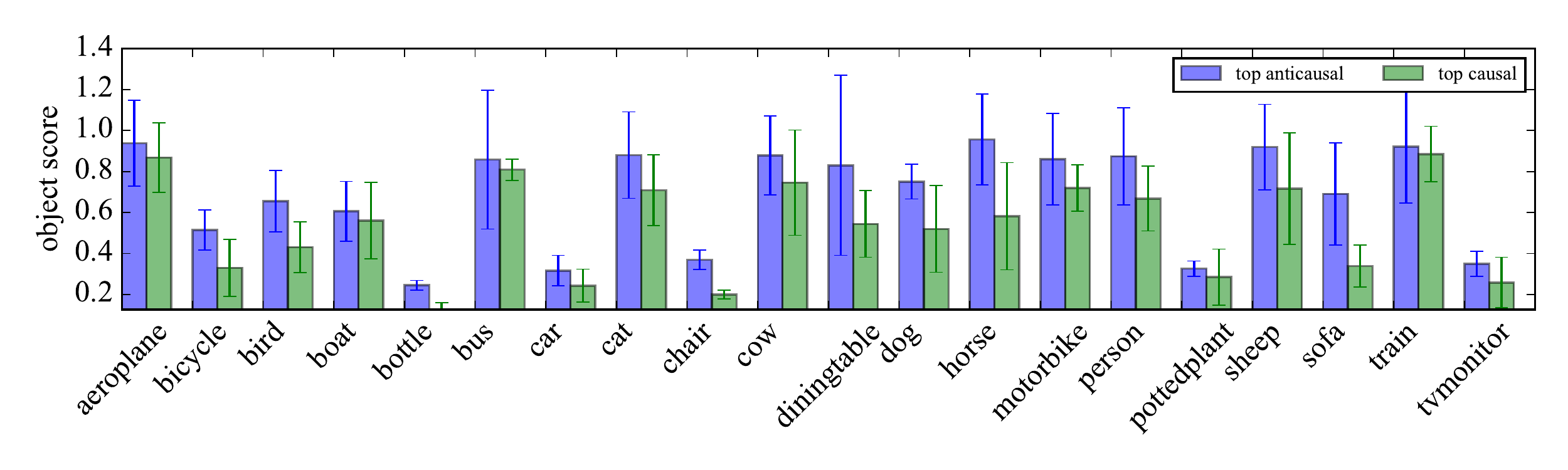}
  \includegraphics[width=\textwidth]{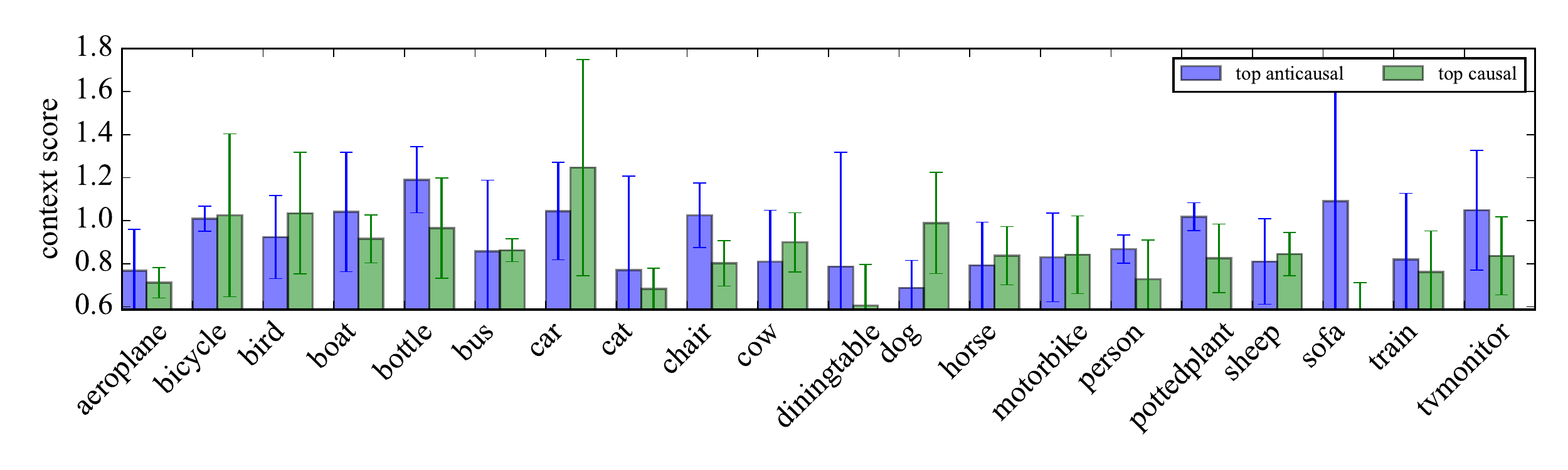}
  \end{center}
  \vspace{-0.5cm}
  \caption{Object and context scores for top anticausal and causal features.}
  \label{fig:results}
\end{figure}

Our experiments indicate the existence of statistically observable causal
signals within sets of static images.  However, further research is needed to
best capture and exploit causal signals for applications in image understanding
and robust object detection. In particular, we stress the importance of (1)
building large, real-world datasets to aid research in causal inference, (2)
extending data-driven techniques like NCC to causal inference of more than two
variables, and (3) exploring data with explicit causal signals, such as the
arrow of time in videos \cite{Pickup14}.

\section{Proofs}
For clarity, we omit bold fonts throughout this section.

\subsection{Distributional learning is measurable}
Let $(\Z, \tau_\Z)$ and $(\L, \tau_\L)$ be two separable topological spaces,
where we call $\Z$ the \emph{input space} and we call $\L := \{-1,1\}$ the
\emph{output space}.  Let $\B(\tau)$ be the Borel \hbox{$\sigma$-algebra}
induced by the topology $\tau$.  Let $P$ be an unknown probability measure on
$(\Z \times \L,\B(\tau_\Z)\otimes\B(\tau_\L))$.  Consider also the classifiers
$f \in \F_k$ and loss function $\ell$ to be measurable.

The first step to deploy our learning setup is to guarantee the
existence of a measure on the space $\muP \times \L$, where
$$\muP = \{\mu_k(P) : P \in \P \} \subseteq \H_k$$
is the set of kernel mean embeddings of the measures in $\P$.  The following
lemma provides this guarantee, which allows learning on ${\muP \times \L}$
throughout this chapter.

\begin{lemma}[Measurability of distributional learning]\label{lemma:measurability}
  Let $(\Z, \tau_\Z)$ and $(\L, \tau_\L)$ be two separable topological spaces.
  Let $\P$ be the set of all Borel probability measures on $(\Z,\B(\tau_\Z))$.
  Let $\muP = \{ \mu_k(P) : P \in \P \} \subseteq \H_k$, where $\mu_k$ is the
  kernel mean embedding \eqref{eq:meank} associated to some bounded continuous
  kernel function $k : \Z \times \Z \to \R$.  Then, there exists a measure on
  $\muP \times \L$.
  \begin{proof}
    Start by endowing $\P$ with the weak topology $\tau_\P$, such that the map
    \begin{equation}\label{eq:lmap}
      L(P) = \int_\Z f(z)\d P(z),
    \end{equation}
    is continuous for all $f \in C_b(\Z)$. This makes $(\P,\B(\tau_\P))$ a
    measurable space.
    
    First, we show that $\mu_k : (\P,\B(\tau_\P)) \to (\H_k,\B(\tau_\H))$ is
    Borel measurable. Note that $\H_k$ is separable due to the separability of
    $(\Z, \tau_\Z)$ and the continuity of $k$ \citep[Lemma 4.33]{Steinwart08}.
    The separability of $\H_k$ implies $\mu_k$ is Borel measurable if and only
    if it is weakly measurable \citep[Thm.  IV.22]{reed1972functional}.  Note
      that the boundedness and the continuity of $k$ imply $\H_k \subseteq
      C_b(\Z)$ \citep[Lemma 4.28]{Steinwart08}. Therefore, \eqref{eq:lmap}
      remains continuous for all $f\in \H_k$, which implies that $\mu_k$ is Borel
      measurable.
    
    Second, $\mu_k : (\P, \B(\tau_\P)) \to (\G, \B(\tau_\G))$ is Borel measurable,
    since the $\B(\tau_\G) = \{ A \cap \G : A \in \B(\H_k)\} \subseteq
    \B(\tau_\H)$, where $\B(\tau_\G)$ is the $\sigma$-algebra induced by the
    topology of $\G \in \B(\H_k)$ \citep{Szabo14}.
    
    Third, we show that $g : (\P \times \L, \B(\tau_\P) \otimes \B(\tau_\L)) \to
    (\G \times \L, \B(\tau_\G) \otimes \B(\tau_\L))$ is measurable. For that, it
    suffices to decompose $g(x,y) = (g_1(x,y),g_2(x,y))$ and show that $g_1$ and
    $g_2$ are measurable, as done by \citet{Szabo14}.
  \end{proof}
\end{lemma}

\subsection{Theorem \ref{thm:LeSong}}
\label{proof:LeSong}
The statement \citep[Theorem 27]{S06} assumed $f\in[0,1]$, but we let these
functions to take negative values.  This requires some minor changes of the
proof.  Using the well known dual relation between the norm in RKHS and
sup-norm of empirical process \citep[Theorem 28]{S06}, write:
\begin{equation}
\label{eq:RKHS-duality}
\|\mu_k(P)-\mu_k(P_S)\|_{\H_k}
=
\sup_{\|f\|_{\H_k}\leq 1}\left(
\E{z\sim P}{f(z)}
-
\frac{1}{n}\sum_{i=1}^n f(z_i)
\right).
\end{equation}
The sup-norm from the right hand side of the previous equation is real-valued
function of the iid random variables $z_1,\dots,z_n$, which we denote as
$F(z_1,\dots,z_n)$. This function $F$ satisfies the \emph{bounded difference}
condition (Theorem 14 of \citep{S06}). Using this fact, we fix all the values
$z_1,\dots,z_n$ except for $z_j$, which we replace with $z_j'$.  Using the
identity $|a-b| = (a-b)\I_{a>b} + (b-a)\I_{a\leq b}$, and
noting that if $\sup_x f(x) = f(x^*)$ then $\sup_x f(x) - \sup_x g(x) \leq
f(x^*) - g(x^*)$, write 
\begin{align*}
&|F(z_1,\dots,z_j',\dots,z_n) - F(z_1,\dots,z_j,\dots,z_n)|
\\
&\leq
\frac{1}{n}\bigl(f(z_j) - f(z_j')\bigr)\I_{F(z_1,\dots,z_j',\dots,z_n) > F(z_1,\dots,z_j,\dots,z_n)}\\
&+
\frac{1}{n}\bigl(f(z_j') - f(z_j)\bigr)\I_{F(z_1,\dots,z_j',\dots,z_n) \leq F(z_1,\dots,z_j,\dots,z_n)}.
\end{align*}
Since $|f(z) - f(z')|\in[0,2]$, we conclude that
\begin{align*}
&
|F(z_1,\dots,z_j',\dots,z_n) - F(z_1,\dots,z_j,\dots,z_n)|\\
&\leq 
\frac{2}{n}\I_{F(z_1,\dots,z_j',\dots,z_n) > F(z_1,\dots,z_j,\dots,z_n)}
+
\frac{2}{n}\I_{F(z_1,\dots,z_j',\dots,z_n) \leq F(z_1,\dots,z_j,\dots,z_n)} = \frac{2}{n}.
\end{align*}
Using McDiarmid's inequality (Theorem~\ref{thm:mcdiarmids}) with $c_i=2/n$ if
follows that, with probability at least $1-\delta$: 
\begin{align*}
\sup_{\|f\|_{\H_k}\leq 1}&\left(
\E{z\sim P}{f(z)}
-
\frac{1}{n}\sum_{i=1}^n f(z_i)
\right)\\
&\leq
\E{}{\sup_{\|f\|_{\H_k}\leq 1}\left(
\E{z\sim P}{f(z)}
-
\frac{1}{n}\sum_{i=1}^n f(z_i)
\right)}
+
\sqrt{\frac{2\log(1/\delta)}{n}}.
\end{align*}
Next, we use symmetrization (Theorem~\ref{thm:symm}) to upper bound the expected
value of the sup-norm of empirical process with twice the Rademacher complexity
of $\{f\in\H_k\colon \|f\|_{\H_k}\leq 1\}$. Finally, we upper bound this
Rademacher complexity using \citep[Lemma 22]{BM01}.

The statement \citep[Theorem 27]{S06} contains extra multiplicative factor 2 under the
logarithm, when compared to our result.  This is because we upper bound the
Rademacher complexity directly, but \citet{S06} upper bounds it instead in
terms of the empirical Rademacher complexity. This, in turn, requires the use
of McDiarmid's inequality together with the union bound.

\subsection{Theorem \ref{thm:risk-bound}}\label{sect:ProofRiskBound}
Start by decomposing the excess risk as:
\begin{align}
\notag
\Rp(\fnt) - \Rp(\f)
&=
\Rp(\fnt) - \Rpnt(\fnt)\\
\notag
&+
\Rpnt(\fnt) -\Rpnt(\f)\\
\notag
&+
\Rpnt(\f) - \Rp(\f)\\
\notag
&\leq
2\sup_{f\in \F_k}|\Rp(f) - \Rpnt(f)|\\
\notag
&=
2\sup_{f\in \F_k}|\Rp(f) - \Rpn(f) + \Rpn(f) - \Rpnt(f)|\\
&\leq
\label{eq:excess-bound1}
2\sup_{f\in \F_k}|\Rp(f) - \Rpn(f)|
+
2\sup_{f\in \F_k}|\Rpn(f) - \Rpnt(f)|,
\end{align}
where $\Rpnt(\fnt) -\Rpnt(\f) \leq 0$.  We now upper bound the two terms in
\eqref{eq:excess-bound1}.

To upper bound the first term, we must translate the quantities from our
distributional learning problem into the quantities from classical learning
theory, as discussed in Section \ref{sec:classic}.
To this end, let $\mu(\P)$ play the role of the input space $\Z$.  So, the
input objects are kernel mean embeddings of elements of $\P$.
According to Lemma~\ref{lemma:measurability}, there is a distribution defined
over $\mu(\P)\times\L$. This distribution plays the role of the data generating
distribution $P$ from classical learning theory. 
Finally, the iid data $\bigl\{\bigl(\mu_k(P_i),l_i\bigr)\bigr\}_{i=1}^n$
form the training sample.
Thus, using Theorem \ref{thm:classic} we get that, with probability not less
than $1-\delta/2$ with respect to the random training sample
$\bigl\{\bigl(\mu_k(P_i),l_i\bigr)\bigr\}_{i=1}^n$, 
\begin{equation}
\label{eq:excess-bound2}
\sup_{f\in\F_k}|\Rp(f) - \Rpn(f)|
\leq
2 L_\varphi \E{}{\sup_{f\in \F_k}\frac{1}{n}\left|\sum_{i=1}^n \sigma_i f(z_i)\right|} + B\sqrt{\frac{\log(2/\delta)}{2n}}.
\end{equation}
To upper bound the second term from \eqref{eq:excess-bound1}, write
\begin{align*}
\notag
\sup_{f\in \F_k}|\Rpn(f) - \Rpnt(f)| 
&=
\sup_{f\in \F_k}\left|
\frac{1}{n}
\sum_{i=1}^n\Bigl[\varphi\bigl(-l_if\bigl(\mu_k(P_i)\bigr)\bigr) - \varphi\bigl(-l_i f\bigl(\mu_k(P_{S_i})\bigr)\bigr)\Bigr]
\right| \\
\notag
&\leq
\sup_{f\in \F_k}
\frac{1}{n}
\sum_{i=1}^n
\left|
\varphi\bigl(-l_if\bigl(\mu_k(P_i)\bigr)\bigr) - \varphi\bigl(-l_i f\bigl(\mu_k(P_{S_i})\bigr)\bigr)
\right| \\
\notag
&\leq
L_\varphi\sup_{f\in \F_k}
\frac{1}{n}
\sum_{i=1}^n
\left|
f\bigl(\mu_k(P_i)\bigr) - f\bigl(\mu_k(P_{S_i})\bigr)
\right|,
\end{align*}
where we have used the Lipschitzness of the cost function $\varphi$.
Using the Lipschitzness of the functionals $f\in\F_k$ we obtain:
\begin{align}
\label{eq:sec3-proof-1}
\sup_{f\in \F_k}|\Rpn(f) - \Rpnt(f)| \leq
L_\varphi\sup_{f\in \F_k}
\frac{L_f}{n}
\sum_{i=1}^n
\|
\mu_k(P_i) - \mu_k(P_{S_i})
\|_{\H_k}.
\end{align}

We now use \ref{thm:LeSong} to upper bound every term in \eqref{eq:sec3-proof-1}. We then
combine these upper bounds using the union bound over $i=1,\dots,n$, and show
that for any fixed $P_1,\dots,P_n$, with probability not less than $1-\delta/2$
with respect to the random samples $\{S_i\}_{i=1}^n$, it follows that:
\begin{align}
L_\varphi\sup_{f\in \F}
\frac{L_f}{n}
&\sum_{i=1}^n
\|
\mu_k(P_i) - \mu_k(P_{S_i})
\|_{\H_k}\nonumber\\
&\leq
L_\varphi\sup_{f\in \F}
\frac{L_f}{n}
\sum_{i=1}^n
\left(
2\sqrt{\frac{\E{z\sim P}{k(z,z)}}{n_i}} + \sqrt{\frac{2\log\frac{2n}{\delta}}{n_i}}
\right).
\label{eq:excess-bound3}
\end{align}
The quantity $2n/\delta$ appears under the logarithm because we have used
Theorem~\ref{thm:LeSong} for every $i$, with $\delta' = \delta / (2n)$.
Combining \eqref{eq:excess-bound2} and \eqref{eq:excess-bound3} using the union
bound into \eqref{eq:excess-bound1}, we get that with probability not less than
$1-\delta$,
\begin{align*}
\Rp(\fnt) - \Rp(\f)
&\leq
4 L_\varphi R_n(\F)\\
&+ 2B\sqrt{\frac{\log(2/\delta)}{2n}}\\
&+
\frac{4L_\varphi L_{\F}}{n}
\sum_{i=1}^n
\left(
\sqrt{\frac{\E{z\sim P}{k(z,z)}}{n_i}} + \sqrt{\frac{\log\frac{2n}{\delta}}{2n_i}}
\right),
\end{align*}
where $L_{\F} = \sup_{f\in\F}L_f$.

\subsection{Theorem \ref{thm:lower-bound}}
\label{sect:ProofLowerBound}
Our proof is a simple combination of the duality equation
\eqref{eq:RKHS-duality} combined with the following lower bound on the suprema
of empirical process \citep[Theorem 2.3]{BM06}:
\begin{theorem}[Lower bound on supremum of empirical processes]
Let $F$ be a class of real-valued functions defined on a set $\Z$ such that
$\sup_{f\in F}\|f\|_{\infty}\leq 1$.  Let $z_1,\dots,z_n,z\in\Z$ be
iid according to some probability measure $P$ on $\Z$.  Set $\sigma^2_F =
\sup_{f\in F} \V{}{f(z)}.$ Then there are universal constants $c,c',$ and
$C$ for which the following holds:
\[
\E{}{\sup_{f\in F}\left|
\E{}{f(z)}
-
\frac{1}{n}\sum_{i=1}^n f(z_i)
\right|}
\geq
c\frac{\sigma_F}{\sqrt{n}}.
\]
Furthermore, for every integer $n\geq 1/\sigma^2_F$, with probability at least $c'$,
\[
\sup_{f\in F}\left|
\E{}{f(z)}
-
\frac{1}{n}\sum_{i=1}^n f(z_i)
\right|
\geq
C\E{}{\sup_{f\in F}\left|
\E{}{f(z)}
-
\frac{1}{n}\sum_{i=1}^n f(z_i)
\right|}.
\]
\end{theorem}
The constants $c,c',$ and $C$ appearing in the last result do not depend on any
other quantities from the statement, such as $n,\sigma^2_F$, as seen in the
proof provided by \citet{BM06}.

\subsection{Lemma \ref{lemma:approx}}
\label{proof:ApproxLemma}
\begin{proof}
Recall that Bochner's theorem, presented here as Theorem~\ref{thm:bochner},
allows to write any real-valued, shift-invariant kernel $k$ on $\Z \times \Z$
as
\begin{align*}
k(z,z') =2\int_{\Z}\int_{0}^{2\pi}\frac{1}{2\pi}p_k(w) \cos(\langle w, z\rangle+b)\cos(\langle w, z'\rangle+b)\, db\, dw,
\end{align*}
which was first presented as Equation~\ref{eq:bochner-2}. As explained in
Equation~\ref{eq:kernel-expectation}, this expression mimics the
expectation
\begin{align}
\label{eq:bochner-simple}
k(z,z')=2c_k\E{b,w}{\cos(\langle w, z\rangle+b)\cos(\langle w, z'\rangle+b)},
\end{align}
where $w \sim p(w)$, $b \sim \U[0,2\pi]$, and $c_k=\int_{\Z}p(w)dw<\infty$. 
Now let $Q$ be any probability distribution defined on $\Z$. 
Then, for any $z,w\in\Z$ and $b\in[0,2\pi]$, the function 
\[
g_{w,b}^z(\cdot) := 2 c_k \cos(\langle w, z\rangle+b)\cos(\langle w, \cdot\rangle+b)
\]
belongs to $L^2(Q)$. Moreover 
\begin{align*}
\notag
\|g_{w,b}^z(\cdot)\|^2_{L^2(Q)}&=\int_{\Z}\Bigl(2c_k\cos(\langle w, z\rangle+b)\cos(\langle w, t\rangle+b)\Bigr)^2 dQ(t)\\
&\leq
4c_k^2\int_{\Z}dQ(t) = 4c_k^2.
\end{align*}
For any fixed $x\in\Z$ and any random parameters $w\in\Z$ and $b\in[0,2\pi]$,
the function $g_{w,b}^z$ is a \emph{random variable} taking values in $L^2(Q)$,
which is a Hilbert Space. To study the concentration random variables in
Hilbert spaces, we appeal to \citep[Lemma 4]{Rahimi08}:
\begin{lemma}[Hoeffding inequality on Hilbert spaces]
\label{lemma:rahimi_hoeffding}
Let $v_1,\dots,v_m$ be iid random variables taking values in a ball of
radius $M$ centered around origin in a Hilbert space $H$. 
Then, for any $\delta > 0$, the following holds: 
\[
\left\|\frac{1}{m}\sum_{i=1}^m v_i - \E{}{\frac{1}{m}\sum_{i=1}^m v_i}\right\|_{H} \leq \frac{M}{m}\left(1+\sqrt{2\log(1/\delta)}\right).
\]
with probability higher than $1-\delta$ over the random sample $v_1,\dots,v_m$.
\end{lemma}
Equation \eqref{eq:bochner-simple} hints that if $w$ follows the distribution
of the normalized Fourier transform $\frac{1}{c_k}p_k$ and
$b\sim\mathcal{U}([0,2\pi])$, then $\E{w,b}{g_{w,b}^z(\cdot)} =
k(z,\cdot)$.  Moreover, we can show that any $h\in\H_k$ is also in $L^2(Q)$:
\begin{align}
\notag
\|h(\cdot)\|^2_{L^2(Q)}
&=\int_{\Z}\bigl(h(t)\bigr)^2 dQ(t)\\
\notag
&=\int_{\Z}\langle k(t,\cdot),h(\cdot)\rangle_{\H_k}^2 dQ(t)\\
\label{eq:hk-in-l2}
&\leq\int_{\Z}k(t,t)\|h\|_{\H_k}^2 dQ(t) \leq \|h\|_{\H_k}^2<\infty,
\end{align}
where we have used the reproducing property of $k$ in $\H_k$, the
Cauchy-Schwartz inequality, and the boundedness of $k$.  Thus, we conclude that
the function $k(z,\cdot) \in L^2(Q)$.

The previous reasoning illustrates that if we have a sample of iid data
$\{(w_i,b_i)\}_{i=1}^m$, then $\E{}{\frac{1}{m}\sum_{i=1}^m
g^z_{w_i,b_i}(\cdot)} = k(z,\cdot)$, where
$\{g^z_{w_i,b_i}(\cdot)\}_{i=1}^m$ are iid elements of $L^2(Q)$.  We
conclude by using Lemma \ref{lemma:rahimi_hoeffding} together with the union
bound over each element $z_i\in S$, expressed as:
\begin{align*}
\left\|\mu_k(P_S) - \frac{1}{n}\sum_{i=1}^n\hat{g}_m^{z_i}(\cdot)\right\|_{L^2(Q)}
&=
\left\|\frac{1}{n}\sum_{i=1}^nk(z_i,\cdot) - \frac{1}{n}\sum_{i=1}^n\hat{g}_m^{z_i}(\cdot)\right\|_{L^2(Q)}\\
&\leq
\frac{1}{n}\sum_{i=1}^n\left\|k(z_i,\cdot) - \hat{g}_m^{z_i}(\cdot)\right\|_{L^2(Q)}\\
&=
\frac{1}{n}\sum_{i=1}^n\left\|k(z_i,\cdot) - \frac{1}{m}\sum_{i=j}^m g^{z_i}_{w_j,b_j}(\cdot)\right\|_{L^2(Q)},
\end{align*}
where we have used the triangle inequality.
\end{proof}

\subsection{Excess risk for low dimensional representations}\label{sec:new-theorem}
For any $w,z\in\Z$ and $b\in[0,2\pi]$, define the function 
\begin{equation}
\label{eq:new-proof-cos}
g_{w,b}^z (\cdot) = 2c_k\cos(\langle w,z\rangle + b)\cos(\langle w,\cdot\rangle + b) \in L^2(Q),
\end{equation}
where $c_k = \int_{\Z}p_k(z) dz$ for $p_k\colon \Z\to\R$ is the Fourier
transform of $k$.  Sample $m$ pairs $\{(w_i,b_i)\}_{i=1}^m$ from
$\left(\frac{1}{c_k}p_k\right)\times \U[0,2\pi]$, and define the average 
\[
\hat{g}_m^z(\cdot) = \frac{1}{m}\sum_{i=1}^m g_{w_i,b_i}^z(\cdot) \in L^2(Q).
\]

Given a kernel function $k$, the sinusoids \eqref{eq:new-proof-cos} do not
necessarily belong its RKHS $\H_k$.  Since we are going to use such sinusoids
as training data, our classifiers should act on the more general space
$L^2(Q)$.  To this end, we redefine the set of classifiers introduced in the
Section \ref{sec:distrib} to be $\{\sig f\colon f\in\F_Q\}$, where $\F_Q$ is
the set of functionals mapping $L^2(Q)$ to $\R$.

Our goal is to find a function $f^*$ such that
\begin{equation}
\label{eq:new-proof-1}
f^* \in \arg\min_{f\in \F_Q} R_{\varphi}(f):= \arg\min_{f\in \F_Q} \E{(P,l)\sim \M}{\varphi\Bigl(-f\bigl(\mu_k(P)\bigr)l\Bigr)}.
\end{equation}
As described in Section \ref{proof:ApproxLemma}, the kernel
boundedness condition $\sup_{z\in\Z}k(z,z)\leq 1$ implies $\H_k\subseteq
L^2(Q)$.  In particular, for any $P\in\P$ it holds that $\mu_k(P)\in L^2(Q)$,
and thus \eqref{eq:new-proof-1} is well defined.

We will approximate \eqref{eq:new-proof-1} by empirical risk minimization.
This time we will replace the infinite-dimensional empirical mean
embeddings $\{\mu_k(P_{S_i})\}_{i=1}^n$ with low-dimensional representations
formed by random sinusoids \eqref{eq:new-proof-cos}.  Namely, we propose to use
the following estimator $\fnt^m$:
\begin{equation*}
\fnt^m \in \arg\min_{f\in \F_Q} \Rpnt^m(f):= \arg\min_{f\in \F_Q} \frac{1}{n}\sum_{i=1}^n\varphi\left(-f\left(\frac{1}{n_i}\sum_{z\in S_i}\hat{g}_m^{z}(\cdot)\right)l_i\right).
\end{equation*} The following result combines Theorem \ref{thm:risk-bound} and
Lemma \ref{lemma:approx} to provide an excess risk bound for $\fnt^m$, which
accounts for all sources of the errors introduced in the learning pipeline: $n$
training distributions, $n_i$ samples from the $i$th training distribution,
and $m$ random features to represent empirical mean embeddings.
\begin{theorem}[Excess risk of ERM on empirical kernel mean embeddings and random features]
\label{thm:new-theorem}
Let $\Z=\R^d$
and $Q$ be any probability distribution on $\Z$.
Consider the RKHS $\H_k$ associated with some bounded, continuous,
characteristic and shift-invariant kernel function $k$, such that $\sup_{z \in \Z} k(z,z)\leq 1$.
Consider a class $\F_Q$ of functionals mapping $L^2(Q)$ to~$\R$ with Lipschitz
constants uniformly bounded by $L_{Q}$.  Let $\varphi\colon \R\to\R^+$ be a
$L_{\varphi}$-Lipschitz function such that $\phi(z) \geq \I_{z>0}$.
Let $\varphi\bigl(-f(h) l\bigr) \leq B$ for every $f\in\F_Q$, $h \in L^2(Q)$, and
$l\in\L$.  
Then for any $\delta > 0$ the following holds:
\begin{align*}
\Rp(\fnt^m) - \Rp(\f)
&\leq
4 L_\varphi R_n(\F_Q) + 2B\sqrt{\frac{\log(3/\delta)}{2n}}\\
&+
\frac{4L_\varphi L_{Q}}{n}
\sum_{i=1}^n
\left(
\sqrt{\frac{\E{z\sim P_i}{k(z,z)}}{n_i}} + \sqrt{\frac{\log\frac{3n}{\delta}}{2n_i}}
\right)\\
&+
2\frac{L_{\varphi}L_Q}{n}\sum_{i=1}^n
\frac{2c_k}{\sqrt{m}}\left(1 + \sqrt{{2\log(3n\cdot n_i/\delta)}}\right)
\end{align*}
with probability not less than $1-\delta$ over all sources of
randomness, which are $\{(P_i,l_i)\}_{i=1}^n$, $\{S_i\}_{i=1}^n$, $\{(w_i,b_i)\}_{i=1}^m$.
\end{theorem}
\begin{proof}
We will proceed similarly to \eqref{eq:excess-bound1}. Decompose the excess risk as:
\begin{align}
\notag
\Rp(\fnt^m) - \Rp(\f)
&=
\Rp(\fnt^m) - \Rpnt^m(\fnt^m)\\
\notag
&+
\Rpnt^m(\fnt^m) -\Rpnt^m(\f)\\
\notag
&+
\Rpnt^m(\f) - \Rp(\f)\\
\notag
&\leq
2\sup_{f\in \F_Q}|\Rp(f) - \Rpnt^m(f)|\\
\notag
&=
2\sup_{f\in \F_Q}|\Rp(f) - \Rpn(f) + \Rpn(f) - \Rpnt(f) + \Rpnt(f) - \Rpnt^m(f)|\\
&\leq
\label{eq:new-proof-ex}
2\sup_{f\in \F_Q}|\Rp(f) - \Rpn(f)|\\
&+
2\sup_{f\in \F_Q}|\Rpn(f) - \Rpnt(f)|\nonumber\\
&+
2\sup_{f\in \F_Q}|\Rpnt(f) - \Rpnt^m(f)|\nonumber.
\end{align}
The first two terms of \eqref{eq:new-proof-ex} were upper bounded in Section
\ref{sect:ProofRiskBound}.
The upper bound of the second term (proved in Theorem \ref{thm:risk-bound}) 
relied on the assumption that functionals in $F_Q$ are Lipschitz on $\H_k$, with
respect to the RKHS norm.
When using bounded kernels, we have $\H_k\subseteq L^2(Q)$, which implies
$\|h\|_{L^2(Q)}\leq \|h\|_{\H_k}$ for any $h\in\H_k$ (see \eqref{eq:hk-in-l2}).
Thus, 
$$|f(h) - f(h')| \leq L_f \|h - h'\|_{L^2(Q)}\leq L_f \|h - h'\|_{\H_k}$$
for any $h,h'\in \H_k$.  This means that the assumptions of Theorem
  \ref{thm:risk-bound} hold, and we can safely apply it to upper bound the
  first two terms of \eqref{eq:new-proof-ex}.

The last step is to upper bound the third term in \eqref{eq:new-proof-ex}. To
this end, 
\begin{align*}
&\sup_{f\in \F_Q}|\Rpnt(f) - \Rpnt^m(f)|\\
&=\sup_{f\in \F_Q}\left| \frac{1}{n}\sum_{i=1}^n \varphi\left(-f\bigl(\mu_k(P_{S_i})\bigr)l_i\right) - 
\frac{1}{n}\sum_{i=1}^n\varphi\left(-f\left(\frac{1}{n_i}\sum_{z\in S_i}\hat{g}_m^{z}(\cdot)\right)l_i\right)\right|\\
&\leq
\frac{1}{n}\sum_{i=1}^n\sup_{f\in \F_Q}\left| \varphi\left(-f\bigl(\mu_k(P_{S_i})\bigr)l_i\right) - 
\varphi\left(-f\left(\frac{1}{n_i}\sum_{z\in S_i}\hat{g}_m^{z}(\cdot)\right)l_i\right)\right|\\
&\leq
\frac{L_{\varphi}}{n}\sum_{i=1}^n\sup_{f\in \F_Q}\left| f\bigl(\mu_k(P_{S_i})\bigr) - 
f\left(\frac{1}{n_i}\sum_{z\in S_i}\hat{g}_m^{z}(\cdot)\right)\right|\\
&\leq
\frac{L_{\varphi}}{n}\sum_{i=1}^n\sup_{f\in \F_Q}L_f
\left\| \mu_k(P_{S_i}) - \frac{1}{n_i}\sum_{z\in S_i}\hat{g}_m^{z}(\cdot)\right\|_{L^2(Q)}.
\end{align*}
We can now use Lemma \ref{lemma:approx} and the union bound over $i=1,\dots,n$ with $\delta' = \delta/n$.
This yields 
\[
\sup_{f\in \F_Q}|\Rpnt(f) - \Rpnt^m(f)|
\leq
\frac{L_{\varphi}L_Q}{n}\sum_{i=1}^n
\frac{2c_k}{\sqrt{m}}\left(1 + \sqrt{{2\log(n\cdot n_i/\delta)}}\right).
\]
with probability not less than $1-\delta$ over $\{(w_i,b_i)\}_{i=1}^m$.
\end{proof}

\section{Training and test protocols for Section~\ref{sec:dagsexperiment}}\label{sec:dagtrain}

The synthesis of training data for the experiments in
Section~\ref{sec:dagsexperiment} resembles the one in
Section~\ref{sec:tuebingen}. The main difference here is that, when trying to
infer the cause-effect relationship between two variables $\bm x_i$ and $\bm x_j$
embedded in a larger set of variables $\bm x = (\bm x_1, \ldots, \bm x_d)$, we have to take
into account the potential confounding effects of the variables $\bm x_k \subseteq
\bm x \setminus \{\bm x_i,\bm x_j\}$.
For the sake of simplicity, we will only consider one-dimensional confounding
effects, that is, scalar $\bm x_k$.

\subsection{Training phase}

To generate cause-effect pairs that exemplify every possible type of scalar
confounding, we generate data from the eight possible directed acyclic graphs
on three variables, depicted in Figure~\ref{fig:8dags}.

\begin{figure}
  \begin{center}
  \resizebox{\textwidth}{!}{
  \begin{tikzpicture}[node distance=0.8cm, auto,]
   \node[punkt, minimum height=2em             ] (X1) {};
   \node[punkt, minimum height=2em, right=of X1] (Y1) {};
   \node[punkt, minimum height=2em, right=of Y1] (Z1) {};
   
   \node[punkt, minimum height=2em, below=of X1] (X2) {};
   \node[punkt, minimum height=2em, right=of X2] (Y2) {};
   \node[punkt, minimum height=2em, right=of Y2] (Z2) {};
   \draw[pil] (X2) -- (Y2);
   
   \node[punkt, minimum height=2em, below=of X2] (X3) {};
   \node[punkt, minimum height=2em, right=of X3] (Y3) {};
   \node[punkt, minimum height=2em, right=of Y3] (Z3) {};
   \draw[pil] (X3) -- (Y3);
   \draw[pil] (Y3) -- (Z3);
   
   \node[punkt, minimum height=2em, right=of Z1] (X4) {};
   \node[punkt, minimum height=2em, right=of X4] (Y4) {};
   \node[punkt, minimum height=2em, right=of Y4] (Z4) {};
   \draw[pil] (X4) -- (Y4);
   \draw[pil] (Z4) -- (Y4);
   
   \node[punkt, minimum height=2em, below=of X4] (X5) {};
   \node[punkt, minimum height=2em, right=of X5] (Y5) {};
   \node[punkt, minimum height=2em, right=of Y5] (Z5) {};
   \draw[pil] (Y5) -- (X5);
   \draw[pil] (Y5) -- (Z5);
   
   \node[punkt, minimum height=2em, below=of X5] (X6) {};
   \node[punkt, minimum height=2em, right=of X6] (Y6) {};
   \node[punkt, minimum height=2em, right=of Y6] (Z6) {};
   \draw[pil] (X6) -- (Y6);
   \draw[pil] (Y6) -- (Z6);
   \draw[pil] (X6) to[bend left] (Z6);
   
   \node[punkt, minimum height=2em, right=of Z4] (X7) {};
   \node[punkt, minimum height=2em, right=of X7] (Y7) {};
   \node[punkt, minimum height=2em, right=of Y7] (Z7) {};
   \draw[pil] (X7) -- (Y7);
   \draw[pil] (Z7) -- (Y7);
   \draw[pil] (X7) to[bend left] (Z7);
   
   \node[punkt, minimum height=2em, below=of X7] (X8) {};
   \node[punkt, minimum height=2em, right=of X8] (Y8) {};
   \node[punkt, minimum height=2em, right=of Y8] (Z8) {};
   \draw[pil] (Y8) -- (X8);
   \draw[pil] (Y8) -- (Z8);
   \draw[pil] (X8) to[bend left] (Z8);
  \end{tikzpicture}
  }
  \end{center}
  \caption{The eight possible directed acyclic graphs on three variables.} 
  \label{fig:8dags}
\end{figure}
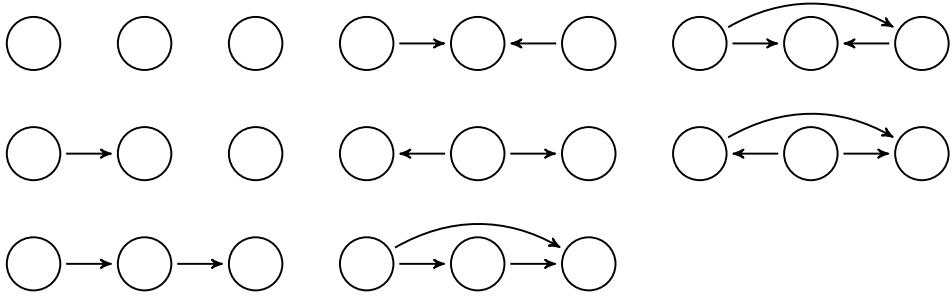

In particular, we will sample $N$ different causal DAGs $G_1, \ldots, G_N$,
where the $G_i$ describes the causal structure underlying $(\bm x_i, \bm y_i, \bm z_i)$.
Given $G_i$, we generate the sample set $S_i = \{(x_{i,j}, y_{i,j}, z_{i,j})\}_{j=1}^n$
according to the generative process described in Section~\ref{sec:tuebingen}.
Together with $S_i$, we annotate the triplet of labels $(l_{i,1}, l_{i,2}, l_{i,3})$, where according to $G_i$,
\begin{itemize}
  \item $l_{i,1} = +1$ if ``$\bm x_i \to \bm y_i$'', $l_{i,1} = -1$ if ``$\bm x_i \leftarrow \bm y_i$'', and $l_{i,1} = 0$ else.
  \item $l_{i,2} = +1$ if ``$\bm y_i \to \bm z_i$'', $l_{i,2} = -1$ if ``$\bm y_i \leftarrow \bm z_i$'', and $l_{i,2} = 0$ else.
  \item $l_{i,3} = +1$ if ``$\bm x_i \to \bm z_i$'', $l_{i,1} = -1$ if ``$\bm x_i \leftarrow \bm z_i$'', and $l_{i,1} = 0$ else.
\end{itemize}
Then, we add the following six elements to our training set:
\begin{align*}
(\{(x_{i,j},y_{i,j}, z_{i,j})\}_{j=1}^n,+l_{i,1}),\\
(\{(y_{i,j},z_{i,j}, x_{i,j})\}_{j=1}^n,+l_{i,2}),\\
(\{(x_{i,j},z_{i,j}, y_{i,j})\}_{j=1}^n,+l_{i,3}),\\
(\{(y_{i,j},x_{i,j}, z_{i,j})\}_{j=1}^n,-l_{i,1}),\\
(\{(z_{i,j},y_{i,j}, x_{i,j})\}_{j=1}^n,-l_{i,2}),\\
(\{(z_{i,j},x_{i,j}, y_{i,j})\}_{j=1}^n,-l_{i,3}),
\end{align*}
for all $1 \leq i \leq N$.  Therefore, our training set will consist on $6N$
sample sets and their paired labels. At this point, and given any sample
$\{(u_{i,j}, v_{i,j}, w_{i,j})\}_{j=1}^n$ from the training set, we propose to
use as feature vectors the concatenation of the $m-$dimensional empirical
kernel mean embeddings \eqref{eq:meank3} of $\{u_{i,j}\}_{j=1}^n$,
$\{v_{i,j}\}_{j=1}^n$, and $\{(u_{i,j}, v_{i,j}, w_{i,j})\}_{j=1}^n$.

\subsection{Test phase}
In order to estimate the causal graph underlying the test sample set $S$, we
compute three $d \times d$ matrices $M_\to$, $M_{\indep}$, and
$M_{\leftarrow}$. Each of these three matrices will contain, at their
coordinates $i,j$, the class probabilities of the labels ``$\bm x_i \to \bm x_j$'',
``$\bm x_i \indep \bm x_j$'', and ``$\bm x_i \leftarrow \bm x_j$'', when
voting over all possible scalar confounders $\bm x_k$.  Using these matrices,
we estimate the underlying causal graph by selecting the type of each edge
(forward, backward, or no edge) to be the one with maximal probability from the
three matrices, and according to our classifier. As a post-processing step, we
prune the least-confident edges until the derived graph is a DAG.

Note that our binary classifier is taught to predict the existence of an arrow
in a large graph by observing only a small subset (three nodes) of such graph.
Therefore, our binary classifier is taught to ignore arrows due to confounding,
and to predict only arrows due to direct causal relationships.

  \chapter[Conclusion and future directions]{Conclusion\\and future
directions}\label{chapter:conclusion}
\vspace{-1.25cm}
  \emph{This chapter contains novel material. In particular, we introduce three
  directions for future research in artificial intelligence:
  machines-teaching-machines paradigms  (Section~\ref{sec:iclr}, \citet{dlp-distillation}), the
  supervision continuum (Section~\ref{sec:continuum}), and probabilistic
  convexity (Section~\ref{sec:nonconvex}).}
\vspace{1.25cm}

\noindent Learning machines excel at prediction, one integral part of intelligence.
But intelligent behaviour must complete prediction with reasoning,
and reasoning requires mastering causal inference. To summarize this
thesis bluntly,
\begin{center}
  \emph{dependence and causation are learnable from observational data.}
\end{center}
Such conclusion further motivates solving the dilemma introduced in this
thesis, namely:
\begin{center}
  \emph{causal inference is key to intelligence, yet ignored by learning algorithms.}
\end{center}

Prediction studies single probability distributions. In opposition, causation
bridges different but related probability distributions, let them be the
training and testing distributions of a learning problem; the multiple
distributions involved in multitask, domain adaptation, and transfer learning;
the changing distributions governing a reinforcement or online learning
scenario; or the different environments over which we plan our actions and
anticipate their outcomes. The differences between these different but related
distributions are often \emph{causal leaps of faith}, used to answer what could
had been, but it never was. The ability to use these causal leaps to our
advantage is what makes us reasoning, creative, intelligent, human agents.
These causal links are the same connections that we use to tie different
learning problems together, transform one piece of knowledge into another, and
more generally, make learning a holistic experience rather than multiple
independent tasks.  Thus, the development of methods able to discover causal
structures from data, and the use of these structures in machine learning is
one necessary step towards machine reasoning and artificial intelligence. 

The last chapter of this thesis is a reflection on what I consider three novel
and important frontiers in artificial intelligence: machine-teaching-machines
paradigms, theory of nonconvex optimization, and the supervision continuum. The
following exposition relies on unpublished work, not necessarily related to
causation, and the reader should understand this chapter as a collection of
conjectures that are currently under investigation. 

\section{Machines-teaching-machines paradigms}\label{sec:iclr}
Humans learn much faster than machines.
\citet{VapIzm15} illustrate this discrepancy with the Japanese proverb
\begin{center}
\emph{better than a thousand days of diligent study is one day with a great
teacher}.
\end{center}
Motivated by this insight, the authors incorporate an
``intelligent teacher'' into machine learning. Their solution 
is to consider training data formed by a collection of triplets
\begin{equation*}
 \{(x_1, x^\star_1, y_1), \ldots, (x_n, x^\star_n, y_n) \} \sim P^n(x,x^\star,y).
\end{equation*}
Here, each $(x_i, y_i)$ is a feature-label pair, and the novel element
$x^\star_i$ is additional information about the example $(x_i, y_i)$ provided
by an intelligent teacher, such as to support the learning process.
Unfortunately, the learning machine will not have access to the teacher
explanations $x^\star_i$ at test time. Thus, the framework of \emph{learning
using privileged information} \citep{Vapnik09,VapIzm15} studies how to leverage
these explanations $x^\star_i$ at training time, to build a classifier for test
time that outperforms those built on the regular features $x_i$
alone. As an example, $x_i$ could be the image of a biopsy, $x^\star_i$ the
medical report of an oncologist when inspecting the image, and $y_i$ a binary
label indicating whether the tissue shown in the image is cancerous or healthy.

The previous exposition finds a mathematical justification in VC theory
\citep{Vapnik98}, which characterizes the speed at which machines learn using
two ingredients: the capacity or flexibility of the machine, and the amount of
data that we use to train it. Consider a binary classifier $f$ belonging to a
function class $\F$ with finite VC-Dimension
$|\F|_\textrm{VC}$. Then, with probability $1-\delta$, the
\emph{expected error} $R(f)$ is upper bounded by
\begin{equation*}
  R(f) \leq R_n(f) + O\left(\left(\frac{|\F|_{\textrm{VC}}-\log
  \delta}{n}\right)^{\alpha}\right),
\end{equation*}
where $R_n(f)$ is the training error over $n$ data, and $\frac{1}{2} \leq
\alpha \leq 1$.  For difficult (\emph{not separable}) problems the exponent is
$\alpha = \frac{1}{2}$, which translates into machines learning at a
\emph{slow} rate of $O(n^{-1/2})$.  On the other hand, for easy
(\emph{separable}) problems, i.e., those on which the machine $f$ makes 
no training errors, the exponent is $\alpha = 1$, which translates into machines
learning at a \emph{fast} rate of $O(n^{-1})$.  The difference between these
two rates is huge: the $O(n^{-1})$ learning rate potentially only requires
$1000$ examples to achieve the accuracy for which the $O(n^{-1/2})$ learning
rate needs $10^6$ examples.  So, given a student who learns from a fixed amount
of data $n$ and a function class $\F$, a good teacher can try to ease
the problem at hand by accelerating the learning rate from $O(n^{-1/2})$ to
$O(n^{-1})$.

Vapnik's \emph{learning using privileged information} is one example of what
we call \emph{machines-teaching-machines}: the paradigm where machines learn
from other machines, in addition to training data. Another seemingly unrelated
example is \emph{distillation} \citep{Hinton15},\footnote{Distillation 
relates to \emph{model compression} \citep{Bucilua06,Ba14}. We will
adopt the term \emph{distillation} throughout this section.} where a simple
machine learns a complex task by imitating the solution of a flexible machine.
In a wider context, the machines-teaching-machines paradigm is one step toward
the definition of \emph{machine reasoning} of \citet{bottou2014machine},
``the algebraic manipulation of previously acquired knowledge to
answer a new question''. In fact, recent state-of-the-art systems compose
data and supervision from multiple sources, such as object recognizers reusing
convolutional neural network features \citep{oquab2014learning}, and natural
language processing systems operating on vector word representations extracted
from unsupervised text corpora \citep{mikolov2013efficient}.

In the following, we frame Hinton's distillation and Vapnik's privileged
information as two instances of the same machines-teaching-machines paradigm,
termed \emph{generalized distillation}. The analysis of generalized
distillation sheds light to applications in semisupervised learning,
domain adaptation, transfer learning, Universum learning \citep{Weston06},
reinforcement learning, and curriculum learning \citep{bengio2009curriculum};
some of them discussed in our numerical simulations.

\subsection{Distillation}
  We focus on $c$-class classification, although the same ideas apply to
  regression. Consider the data
  \begin{equation}
  \label{eq:data1}
    \{(x_i, y_i)\}_{i=1}^n \sim P^n(x,y), \,\, x_i \in \R^d, \,\, y_i
  \in \Delta^c.
  \end{equation}
  Here, $\Delta^c$ is the set of $c$-dimensional probability vectors.  Using
  \eqref{eq:data1}, we target learning the representation 
  \begin{equation}\label{eq:obj1}
    f_t = \argmin_{f \in \F_t} \frac{1}{n} \sum_{i=1}^n \ell(y_i, \sigma(f(x_i))) + \Omega(\|f\|),
  \end{equation}
  where $\F_t$ is a class of functions from $\R^d$ to
  $\R^c$, the function $\sigma : \R^c \to \Delta^c$ is the
  softmax operation
  \begin{equation*}
  \sigma(z)_k = \frac{e^{z_k}}{\sum_{j=1}^c e^{z_j}},
  \end{equation*}
  for all $1 \leq k \leq c$, the function $\ell : \Delta^c \times \Delta^c \to \R_+$
    is the cross-entropy loss
  \begin{equation*}
    \ell(y,\hat{y}) = -\sum_{k=1}^c y_k \log \hat{y}_k,
  \end{equation*}
  and $\Omega : \R \to \R$ is an increasing function which serves as a regularizer.

  When learning from real world data such as high-resolution images, $f_t$ is
  often an ensemble of large deep convolutional neural networks
  \citep{Lecun98}.  The computational cost of predicting new examples at test
  time using these ensembles is often prohibitive for production systems.  For
  this reason, \citet{Hinton15} propose to \emph{distill} the
  learned representation $f_t \in \F_t$ into
  \begin{equation}\label{eq:obj2}
    f_s = \argmin_{f\in \F_s} \frac{1}{n} \sum_{i=1}^n \Big[(1-\lambda)\ell(y_i,
    \sigma({f}(x_i))) + \lambda \ell(s_i,
    \sigma({f}(x_i)))\Big],
  \end{equation}
  where
  \begin{equation*}
    s_i = \sigma(f_t(x_i)/T) \in \Delta^c
  \end{equation*}   
  are the \emph{soft predictions} from $f_t$ about the training data, and
  $\F_s$ is a function class  simpler than $\F_t$. The
  temperature parameter $T > 0$ controls how much do we want to soften or
  smooth the class-probability predictions from $f_t$, and the imitation
  parameter $\lambda \in [0,1]$ balances the importance between imitating the
  soft predictions $s_i$ and predicting the true hard labels $y_i$. Higher
  temperatures lead to softer class-probability predictions $s_i$. In turn,
  softer class-probability predictions reveal label dependencies which would be
  otherwise hidden as extremely large or small numbers. After distillation, we
  can use the simpler ${f_s} \in \F_s$ for faster prediction at test
  time.
  
\subsection{Privileged information}
We now turn back to Vapnik's problem of learning in the company of an
intelligent teacher, as introduced in the opening of this section. The question at
hand is: How can we leverage the privileged information $x^\star_i$ to build a
better classifier for test time?  One na\"ive way to proceed would be to estimate the
privileged representation $x^\star_i$ from the regular representation $x_i$, and then use
the union of regular and \emph{estimated} privileged representations as our
test-time feature space.  But this may be a cumbersome endeavour: in the
example of biopsy images $x_i$ and medical reports $x^\star_i$, it is
reasonable to believe that predicting reports from images is more
complicated than classifying the images into cancerous or healthy.

Alternatively, we propose to use distillation to extract useful knowledge from
privileged information.  The proposal is as follows. First, learn a teacher
function $f_t \in \F_t$ by solving \eqref{eq:obj1} using the data
$\{(x^\star_i, y_i)\}_{i=1}^n$.  Second, compute the teacher soft labels $s_i =
\sigma(f_t(x^\star_i)/T)$, for all $1\leq i\leq n$ and some temperature
parameter $T > 0$.  Third, distill $f_t \in \F_t$ into $f_s \in
\F_s$ by solving \eqref{eq:obj2} using both the hard labeled data
$\{(x_i, y_i)\}_{i=1}^n$ and the softly labeled data $\{(x_i,s_i)\}_{i=1}^n$.

\subsubsection{Comparison to prior work}

\citet{Vapnik09,VapIzm15} offer two strategies to learn using privileged information:
similarity control and knowledge transfer. Let us briefly compare them to 
our distillation-based proposal.

The motivation behind \emph{similarity control} is that SVM
classification is separable after we correct for the \emph{slack values}
$\xi_i$, which measure the degree of misclassification of training data points $x_i$
\citep{Vapnik09}.
Since separable classification admits
$O(n^{-1})$ fast learning rates, it would be ideal to have a teacher that could
supply slack values to us.  Unluckily, it seems quixotic to aspire for a teacher able
to provide with abstract floating point number slack values. Perhaps it is more
realistic to assume instead that the teacher can provide with some rich,
high-level representation useful to estimate the sought-after slack values.  This
reasoning crystallizes into the SVM+ objective function from \citep{Vapnik09}:
\begin{align*}
  L(w,w^\star,b,b^\star, \alpha, \beta) &=
  \underbrace{\frac{1}{2}\|w\|^2 + \sum_{i=1}^n \alpha_i - \sum_{i=1}^n
  \alpha_i y_i f_i}_{\textrm{separable SVM objective}}\\
  &+ \underbrace{\frac{\gamma}{2} \|w^\star\|^2 + \sum_{i=1}^n (\alpha_i + \beta_i
  - C) {f^\star_i}}_{\textrm{corrections from teacher}},
\end{align*}
where $f_i := \langle w, x_i \rangle +b$ is the decision boundary at $x_i$, and
$f^\star_i := \langle w^\star, x^\star_i \rangle +b^\star$ is the teacher
correcting function at the same location. The SVM+ objective function matches
the objective function of not separable SVM when we replace the correcting
functions $f^\star_i$ with the slacks $\xi_i$. Thus, skilled teachers provide
with privileged information $x^\star_i$ highly informative about the slack values
$\xi_i$. Such privileged information allows for simple correcting functions
$f^\star_i$, and the easy estimation of these correcting functions is a proxy
to $O(n^{-1})$ fast learning rates. Technically, this amounts to
saying that a teacher is helpful whenever the capacity of her correcting
functions is much smaller than the capacity of the student decision
boundary.

In \emph{knowledge transfer} \citep{VapIzm15} the teacher fits a function
$f_t(x^\star) = \sum_{j=1}^m \alpha^\star_j k^\star(u^\star_j,x^\star)$ on the
input-output pairs $\{(x^\star_i,y_i)\}_{i=1}^n$ and $f_t \in \F_t$,
to find the best reduced set of prototype or basis points
$\{u^\star_j\}_{j=1}^m$.  Second, the student fits one function $g_j$ per set
of input-output pairs $\{(x_i, k^\star(u^\star_j,x^\star_i))\}_{i=1}^n$, for
all $1 \leq j \leq m$. Third, the student fits a new vector of coefficients
$\alpha \in \R^m$ to obtain the final student function $f_s(x) =
\sum_{j=1}^m \alpha_j g_j(x)$, using the input-output pairs
$\{(x_i,y_i)\}_{i=1}^n$ and $f_s \in \F_s$.  Since the representation
$x^\star_i$ is intelligent, we assume that the function class $\F_t$
has small capacity, and thus allows for accurate estimation under small
sample sizes.

Distillation differs from similarity control in three ways. First, unlike SVM+, distillation 
is not restricted to SVMs.  Second, while the SVM+ solution contains twice the amount
of parameters than the original SVM, the user can choose a priori the amount of parameters in the distilled
classifier. Third, SVM+ learns the teacher
correcting function and the student decision boundary simultaneously,
but distillation proceeds sequentially: first with the teacher, then with the
student.  On the other hand, knowledge transfer is closer in spirit to
distillation, but the two techniques differ: while similarity control relies on a
student that purely imitates the hidden representation of a low-rank kernel
machine, distillation is a trade-off between imitating soft predictions and hard
labels, using arbitrary learning algorithms.

The framework of learning using privileged information enjoys theoretical
analysis \citep{Pechyony10a}, equivalence analysis to weighted learning
\citep{lapin2014learning}, and multiple applications that include ranking
\citep{Sharmanska13}, computer vision \citep{Sharmanska14,dlp-rca}, clustering
\citep{Feyereisl12}, metric learning \citep{Fouad13}, Gaussian process
classification \citep{Hernandez14-lupi}, and finance
\citep{ribeiro2010financial}.

\subsection{Generalized distillation}\label{sec:gendistillation}
We now have all the necessary background to describe \emph{generalized
distillation}. To this end, consider the data
$\{(x_i,x^\star_i,y_i)\}_{i=1}^n$. Then, the process of {generalized
distillation} is as follows:
\begin{enumerate}
  \item Learn teacher $f_t\in\F_t$ using the input-output pairs
  $\{(x^\star_i, y_i)\}_{i=1}^n$ and Eq.~\ref{eq:obj1}.
  \item Compute teacher soft labels $\{\sigma(f_t(x^\star_i)/T)\}_{i=1}^n$,
  using temperature parameter $T>0$.
  \item Learn student $f_s\in\F_s$ using the input-output pairs
  $\{(x_i, y_i)\}_{i=1}^n$, $\{(x_i, s_i)\}_{i=1}^n$, Eq. 
  \ref{eq:obj2}, and imitation parameter $\lambda \in [0,1]$.
\end{enumerate}

We say that generalized distillation reduces to \emph{Hinton's distillation} if
$x^\star_i = x_i$ for all $1 \leq i \leq n$ and $|\F_s|_{\textrm{C}}
\ll |\F_t|_{\textrm{C}}$, where $|\cdot|_C$ is an appropriate function
class capacity measure. Conversely, we say that generalized distillation
reduces to \emph{Vapnik's learning using privileged information} if $x^\star_i$
is a privileged description of $x_i$, and $|\F_s|_{\textrm{C}} \gg
|\F_t|_{\textrm{C}}$.

This comparison reveals a subtle difference between Hinton's distillation and
Vapnik's privileged information.  In Hinton's distillation, $\F_t$ is
\emph{flexible}, for the teacher to exploit her \emph{general purpose}
representation $x^\star_i = x_i$ to learn intricate patterns from \emph{large}
amounts of labeled data. In Vapnik's privileged information, $\F_t$ is
\emph{simple}, for the teacher to exploit her \emph{rich} representation
$x^\star_i \neq x_i$ to learn intricate patterns from \emph{small} amounts of
labeled data.  The space of privileged information is thus a specialized space,
one of ``metaphoric language''. In our running example of biopsy images, the
space of medical reports is much more specialized than the space of pixels,
since the space of pixels can also describe buildings, animals, and other
unrelated concepts.  In any case, the teacher must develop a language that
effectively communicates information to help the student come up with better
representations. The teacher may do so by incorporating invariances, or biasing
them towards being robust with respect to the kind of distribution shifts that
the teacher may expect at test time.  In general, having a teacher is one
opportunity to learn characteristics about the decision boundary which are not
contained in the training sample, in analogy to a good Bayesian prior.

\subsubsection{Why does generalized distillation work?}
Recall our three actors: the student function $f_{s} \in \F_s$, the
teacher function $f_{t} \in \F_t$, and the real target function of
interest to both the student and the teacher, $f \in \F$. For simplicity, consider \emph{pure
distillation} (set the imitation parameter to $\lambda = 1$).
Furthermore, we will place some assumptions about how the student, teacher, and
true function interplay when learning from $n$ data.  First, assume that the
student may learn the true function at a slow rate
\begin{equation*}
  R(f_{s}) - R(f) \leq O\left(\frac{|\F_s|_\textrm{C}}{\sqrt{n}}\right) + \varepsilon_s,
\end{equation*}
where the $O(\cdot)$ term is the estimation error, and $\varepsilon_s$ is the
approximation error of the student function class $\F_s$ with respect
to $f \in \F$. Second, assume that the better representation of the
teacher allows her to learn at the fast rate 
\begin{equation*}
  R(f_{t}) - R(f) \leq O\left(\frac{|\F_t|_\textrm{C}}{n}\right) +
  \varepsilon_t,
\end{equation*}
where $\varepsilon_t$ is the approximation error of the teacher function class
$\F_t$ with respect to $f \in \F$. Finally, assume that when
the student learns from the teacher, she does so at the rate
\begin{equation*}
  R(f_{s}) - R(f_{t}) \leq
  O\left(\frac{|\F_s|_\textrm{C}}{n^{\alpha}}\right) + \varepsilon_l,
\end{equation*}
where $\varepsilon_l$ is the approximation error of the student function class
$\F_s$ with respect to $f_t \in \F_t$, and $\frac{1}{2} \leq
\alpha \leq 1$. Then, the rate at which the student learns the true function
$f$ admits the alternative expression 
\begin{align*}
  R(f_{s})-R(f) &= R(f_{s})-R(f_{t})+R(f_{t})-R(f)\\
                  &\leq
                  O\left(\frac{|\F_s|_\textrm{C}}{n^{\alpha}}\right)
                  + \varepsilon_l +
                  O\left(\frac{|\F_t|_\textrm{C}}{n}\right) +
                  \varepsilon_t\\
                  &\leq O\left(\frac{|\F_s|_\textrm{C} +
                  |\F_t|_\textrm{C}}{n^{\alpha}}\right) +
                  \varepsilon_l + \varepsilon_t,
\end{align*}
where the last inequality follows because $\alpha \leq 1$. Thus, the question
at hand is to argue, for a given learning problem, if the inequality
\begin{equation*}
  O\left(\frac{|\F_s|_\textrm{C} +
  |\F_t|_\textrm{C}}{n^{\alpha}}\right) + \varepsilon_l +
  \varepsilon_t \leq O\left(\frac{|\F_s|_\textrm{C}}{\sqrt{n}}\right)
  + \varepsilon_s
\end{equation*}
holds.  The inequality highlights that the benefits of learning with a teacher
arise due to i) the capacity of the teacher being small, ii) the approximation
error of the teacher being smaller than the approximation error of the student,
and iii) the coefficient $\alpha$ being greater than $\frac{1}{2}$.
Remarkably, these factors embody the assumptions of privileged information from
\citet{VapIzm15}. The inequality is also reasonable under the main assumption
in \citep{Hinton15}, which is $\varepsilon_s \gg \varepsilon_t +
\varepsilon_l$.  Moreover, the inequality highlights that the teacher is most
helpful in low data regimes; for instance, when working with small datasets, or
in the initial stages of online and reinforcement learning.

We believe that the ``$\alpha > \frac{1}{2}$ case'' is a general situation,
since soft labels (dense vectors with a real number of information per class)
contain more information than hard labels (one-hot-encoding vectors with one
bit of information per class) per example, and should allow for faster
learning. This additional information, also understood as label uncertainty, 
relates to the acceleration in SVM+ due to the knowledge of slack values. Since a
good teacher smoothes the decision boundary and instructs the student to fail
on difficult examples, the student can focus on the remaining body of data.
Although this translates into the unambitious ``whatever my
teacher could not do, I will not do'', the imitation parameter $\lambda \in [0,1]$
in \eqref{eq:obj2} allows to follow this rule safely, and fall back to regular
learning if necessary.

\subsubsection{Extensions}
\paragraph{Semi-supervised learning} We now extend generalized distillation to
the situation where examples lack regular features, privileged features,
labels, or a combination of the three. In the following, we denote missing
elements by $\square$. For instance, the example $(x_i, \square, y_i)$ has no
privileged features, and the example $(x_i,x^\star_i,\square)$ is missing its
label. Using this convention, we introduce the \emph{clean subset} notation
\begin{equation*}
  c(S) = \{ v : v \in S,  v_i \neq \square \,\, \forall i \}.
\end{equation*}
Then, semisupervised generalized distillation walks the same three steps as
generalized distillation, enumerated at the beginning of
Section~\ref{sec:gendistillation}, but uses the appropriate clean subsets
instead of the whole data.  For example, the semisupervised extension of
distillation allows the teacher to prepare soft labels for all
the unlabeled data $c(\{(x_i,x^\star_i)\}_{i=1}^n)$. These additional
soft-labels are additional information available to the student to learn the teacher
representation $f_t$.

\paragraph{Learning with the Universum} The unlabeled data
$c(\{x_i,x^\star_i\}_{i=1}^n)$ can belong to one of the classes of interest, or be
\emph{Universum} data \citep{Weston06}. Universum data may have
labels: in this case, one can exploit these additional labels by i) training a
teacher that distinguishes amongst all classes (those of interest and those 
from the Universum), ii) computing soft class-probabilities only for the
classes of interest, and iii) distilling these soft probabilities into a
student function.

\paragraph{Learning from multiple tasks} Generalized distillation applies to
some domain adaptation, transfer learning, or multitask learning scenarios. On
the one hand, if the multiple tasks share the same labels $y_i$ but differ in
their input modalities, the input modalities from the source tasks are
privileged information.  On the other hand, if the multiple tasks share the
same input modalities $x_i$ but differ in their labels, the labels from the
source tasks are privileged information. In both cases, the regular student
representation is the input modality from the target task.

\paragraph{Curriculum and reinforcement learning} We conjecture that the
uncertainty in the teacher soft predictions is a mechanism to rank
the difficulty of training examples, and use these ranks for curriculum
learning \citep{bengio2009curriculum}. Furthermore, distillation 
resembles imitation, a technique that learning agents could exploit in
\emph{reinforcement learning} environments.

\subsubsection{A causal perspective on generalized distillation}\label{sec:causal}

The assumption of \emph{independence of cause and mechanisms} states that ``the
probability distribution of a cause is often independent from the process
mapping this cause into its effects'' \citep{scholkopf12anti}.  Under this
assumption, for instance, \emph{causal learning problems} ---i.e., those where
the features cause the labels--- do not benefit from semisupervised learning,
since by the independence assumption, the marginal distribution of the features
contains no information about the function mapping features to labels.
Conversely, \emph{anticausal learning problems} ---those where the labels cause
the features--- may benefit from semisupervised learning. 

Causal implications also arise in generalized distillation.  First, if the
privileged features $x^\star_i$ only add information about the marginal
distribution of the regular features $x_i$, the teacher should be able to help
only in anticausal learning problems. Second, if the teacher provides
additional information about the conditional distribution of the labels $y_i$
given the inputs $x_i$, it should also help in the causal setting. We will confirm this hypothesis in the next section.

\subsection{Numerical simulations}
We now present some experiments to illustrate when the
distillation of privileged information is effective, and when it is not. 

We start with four synthetic experiments, designed to minimize modeling
assumptions and to illustrate different prototypical types of privileged
information. These are simulations of logistic regression models repeated over
$100$ random partitions, where we use $n_\text{tr} = 200$ samples for training,
and $n_\text{te} = 10,000$ samples for testing.  The dimensionality of the
regular features $x_i$ is $d=50$, and the involved separating hyperplanes
$\alpha \in \R^d$ follow the distribution $\N(0,I_d)$.  For each
experiment, we report the test accuracy when i) using the teacher explanations
$x^\star_i$ at both train and test time, ii) using the regular features $x_i$
at both train and test time, and iii) distilling the teacher explanations into
the student classifier with $\lambda = T = 1$. 

\paragraph{1. Clean labels as privileged information.} We sample triplets $(x_i,
x^\star_i, y_i)$ from:
\begin{align*}
  x_i       &\sim \N(0,I_d)\\
  x^\star_i &\leftarrow \langle \alpha, x_i \rangle\\
  \varepsilon_i &\sim \N(0,1)\\
  y_i       &\leftarrow \I((x^\star_i + \varepsilon_i) > 0).
\end{align*}
Here, each teacher explanation $x^\star_i$ is the exact distance to the
decision boundary for each $x_i$, but the data labels $y_i$ are corrupt. This
setup aligns with the assumptions about slacks in the similarity control
framework of \citet{Vapnik09}. We obtained a privileged test classification
accuracy of $96 \pm 0\%$, a regular test classification accuracy of $88\pm
1\%$, and a distilled test classification accuracy of $95\pm
1\%$. This illustrates that distillation of privileged information is an
effective mean to detect outliers in label space.

\paragraph{2. Clean features as privileged information} We sample 
triplets $(x_i, x^\star_i, y_i)$ from:
\begin{align*}
  x^\star_i     &\sim \N(0,I_d)\\
  \varepsilon_i &\sim \N(0,I_d)\\
  x_i     &\leftarrow x^\star_i + \varepsilon\\
  y_i       &\leftarrow \I \left(\langle \alpha, x^\star_i \rangle > 0\right).
\end{align*}
In this setup, the teacher explanations $x^\star_i$ are clean versions of the
regular features $x_i$ available at test time.  We obtained a privileged test
classification accuracy of $90 \pm 1\%$, a regular test classification accuracy
of $68\pm 1\%$, and a distilled test classification accuracy of
$70\pm 1\%$. This improvement is not statistically significant. This is because
the intelligent explanations $x^\star_i$ are independent from the noise
$\varepsilon_i$ polluting the regular features $x_i$.  Therefore, there exists
no additional information transferable from the teacher to the student.

\paragraph{3. Relevant features as privileged information} We sample triplets
$(x_i, x^\star_i, y_i)$
from:
\begin{align*}
  x_i       &\sim \N(0,I_d)\\
  x^\star_i &\leftarrow x_{i,J}\\
  y_i       &\leftarrow \I(\langle \alpha_J, x^\star_i\rangle > 0),
\end{align*}
where the set $J$, with $|J| = 3$, is a subset of the variable indices $\{1,
\ldots, d\}$ chosen at random but common for all samples. In another words, the
teacher explanations indicate the values of the variables relevant for
classification, which translates into a reduction of the
dimensionality of the data that we have to learn from.  We obtained a
privileged test classification accuracy of $98 \pm 0\%$, a regular test
classification accuracy of $89\pm 1\%$, and a distilled test
classification accuracy of $97\pm 1\%$. This illustrates that distillation on
privileged information is an effective tool for feature selection.

\paragraph{4. Sample-dependent relevant features as privileged information}
Sample triplets
\begin{align*}
  x_i       &\sim \N(0,I_d)\\
  x^\star_i &\leftarrow x_{i,{J_i}}\\
  y_i       &\leftarrow \I(\langle \alpha_{J_i}, x^\star_i\rangle > 0),
\end{align*}
where the sets $J_i$, with $|J_i| = 3$ for all $i$, are a subset of the
variable indices $\{1, \ldots, d\}$ chosen at random for each sample
$x^\star_i$. One interpretation of such model is the one of bounding boxes in
computer vision: each high-dimensional vector $x_i$ would be an image, and each
teacher explanation $x^\star_i$ would be the pixels inside a bounding box
locating the concept of interest \citep{Sharmanska13}.
We obtained a privileged test classification accuracy of $96 \pm 2\%$, a
regular test classification accuracy of $55\pm 3\%$, and a distilled
test classification accuracy of $0.56\pm 4\%$. Note that
although the classification is linear in $x^\star$, this is not the case in
terms of $x$. Therefore, although we have misspecified the function class
$\F_s$ for this problem, the distillation approach did not deteriorate
the final performance.

The previous four experiments set up causal learning problems. In the second
experiment, the privileged features $x^\star_i$ add no information about the
target function mapping the regular features to the labels, so the causal
hypothesis from Section~\ref{sec:causal} justifies the lack of improvement. The
first and third experiments provide privileged information that adds
information about the target function, and therefore is beneficial to distill
this information. The fourth example illustrates that 
the privileged features adding information about
the target function is not a sufficient condition for improvement.

\paragraph{5. MNIST handwritten digit image classification} The privileged
features are the original 28x28 pixels MNIST handwritten digit images
\citep{MNIST}, and the regular features are the same images downscaled to 7x7
pixels. We use $300$ or $500$ samples to train both the teacher and the
student, and test their accuracies at multiple levels of temperature and
imitation on the full test set. Both student and teacher are neural networks of
composed by two hidden layers of $20$ rectifier linear units and a softmax
output layer (as in the remaining experiments).
Figure~\ref{fig:mnist} summarizes the results of this experiment, where we see
a significant improvement in classification accuracy when distilling the
privileged information, with respect to using the regular features alone. As
expected, the benefits of distillation diminished as we further increased the
sample size.

\begin{figure}
  \begin{subfigure}{0.5\textwidth}
    \begin{center}
    \includegraphics[width=0.85\textwidth]{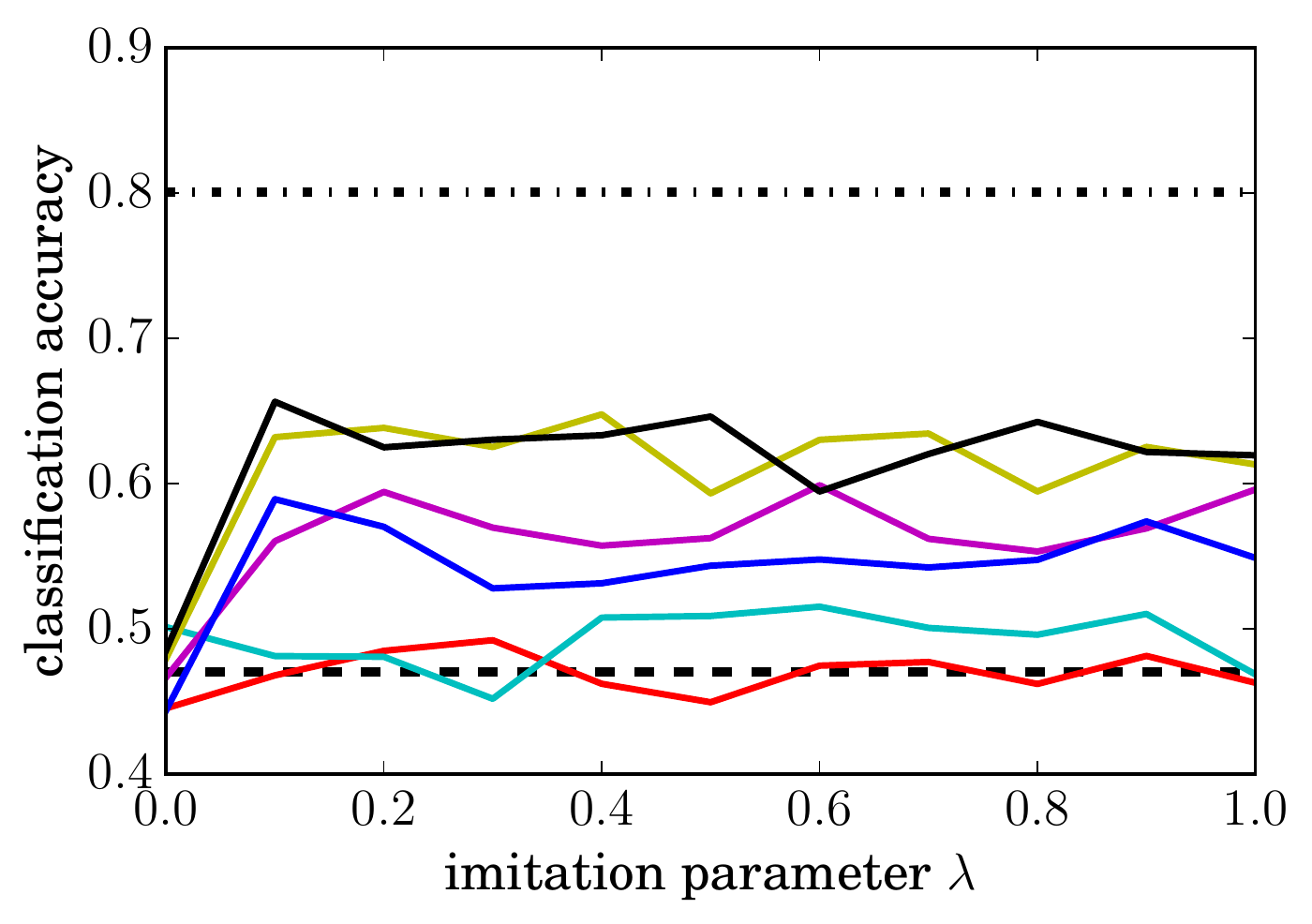}
    \end{center}
  \end{subfigure}
  \begin{subfigure}{0.5\textwidth}
    \begin{center}
    \includegraphics[width=0.85\textwidth]{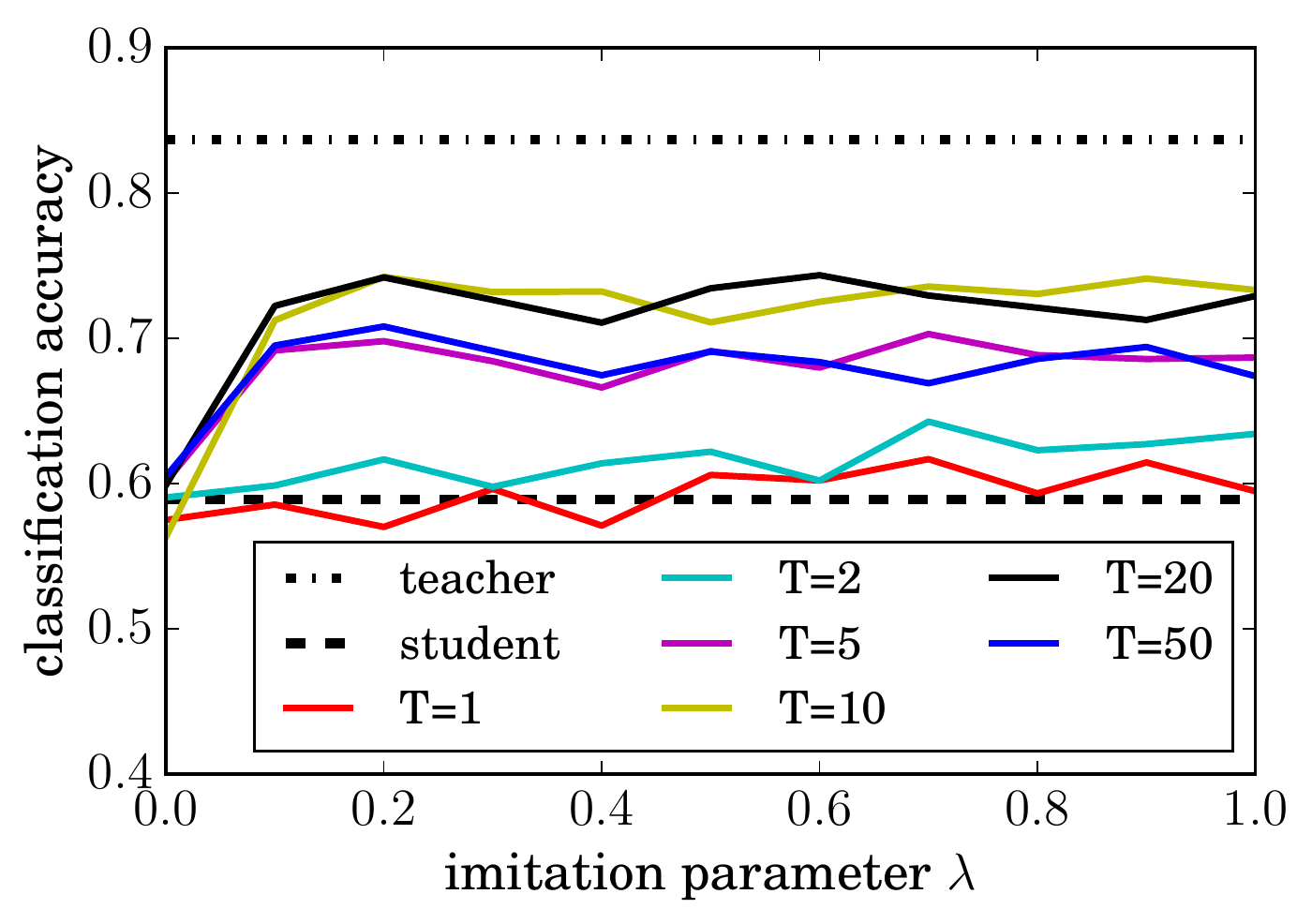}
    \end{center}
  \end{subfigure}
  \caption{Distillation results on MNIST for 300 and 500 samples.}
  \label{fig:mnist}
\end{figure}

\paragraph{6. Semisupervised learning} We explore the semisupervised capabilities
of generalized distillation on the CIFAR10 dataset \citep{CIFAR10}. Here, the
privileged features are the original 32x32 pixels CIFAR10 color images, and the
regular features are the same images when polluted with additive Gaussian
noise. We provide labels for $300$ images, and unlabeled privileged and regular
features for the rest of the training set. Thus, the teacher trains on $300$
images, but computes the soft labels for the whole training set of $50,000$
images. The student then learns by distilling the $300$ original hard labels
and the $50,000$ soft predictions. As seen in Figure~\ref{fig:others}, the soft
labeling of unlabeled data results in a significant improvement with respect to
pure student supervised classification. Distillation on the $300$ labeled
samples did not improve the student performance. This illustrates the
importance of semisupervised distillation in this data.  We believe that the
drops in performance for some distillation temperatures are due to the lack of
a proper weighting between labeled and unlabeled data in \eqref{eq:obj2}.

\paragraph{7. Multitask learning} The SARCOS
dataset\footnote{\url{http://www. gaussianprocess.org/gpml/data/}}
characterizes the 7 joint torques of a robotic arm given 21 real-valued
features.  Thus, this is a multitask learning problem, formed by 7 regression
tasks. We learn a teacher on $300$ samples to predict each of the 7 torques
given the other 6, and then distill this knowledge into a student who uses as
her regular input space the 21 real-valued features.  Figure~\ref{fig:others}
illustrates the performance improvement in mean squared error when using
generalized distillation to address the multitask learning problem. When
distilling at the proper temperature, distillation allowed the student to match
her teacher performance.  

\begin{figure}
  \begin{subfigure}{0.5\textwidth}
    \begin{center}
    \includegraphics[width=.85\textwidth]{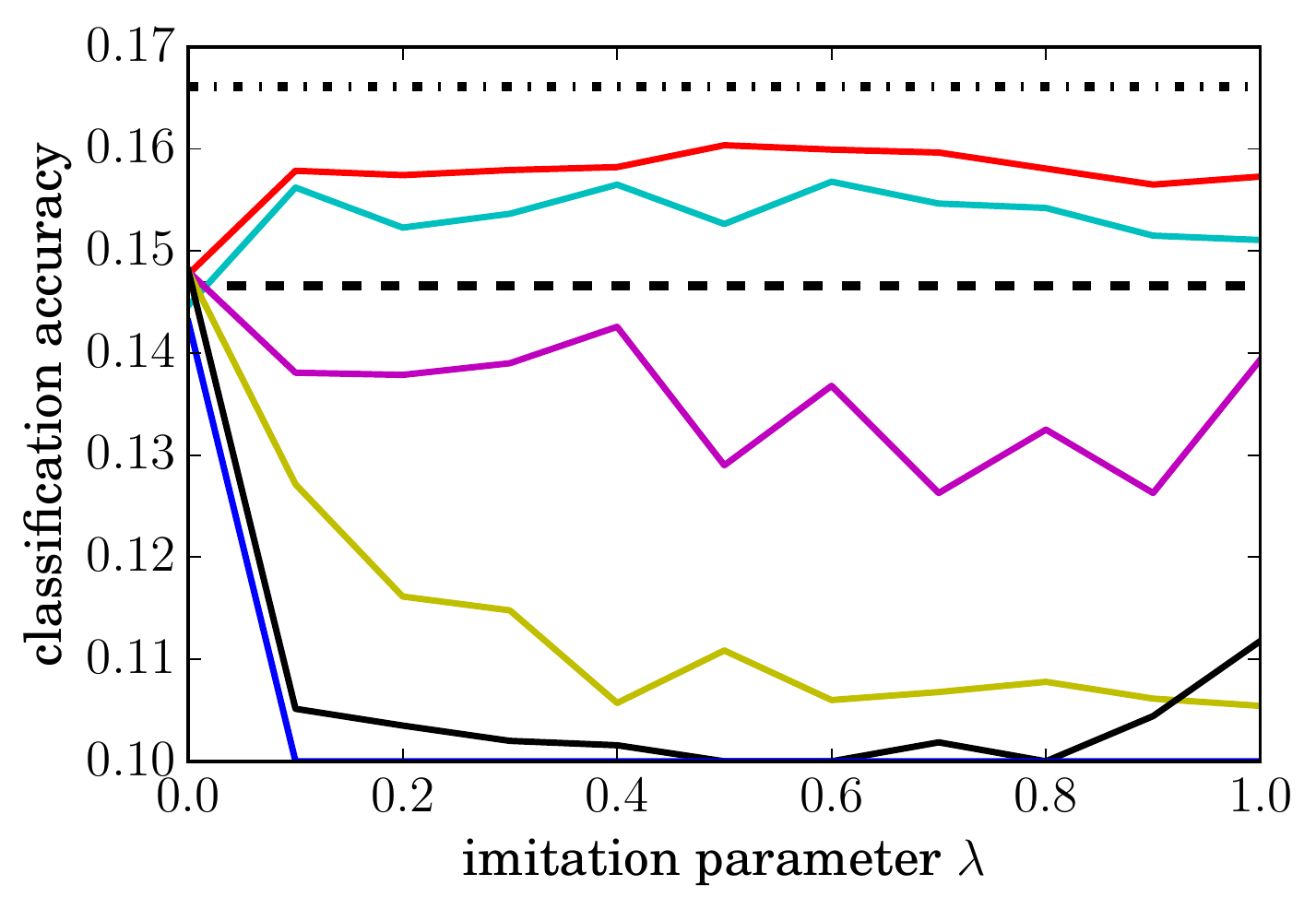}
    \end{center}
  \end{subfigure}
  \begin{subfigure}{0.5\textwidth}
    \begin{center}
    \includegraphics[width=.85\textwidth]{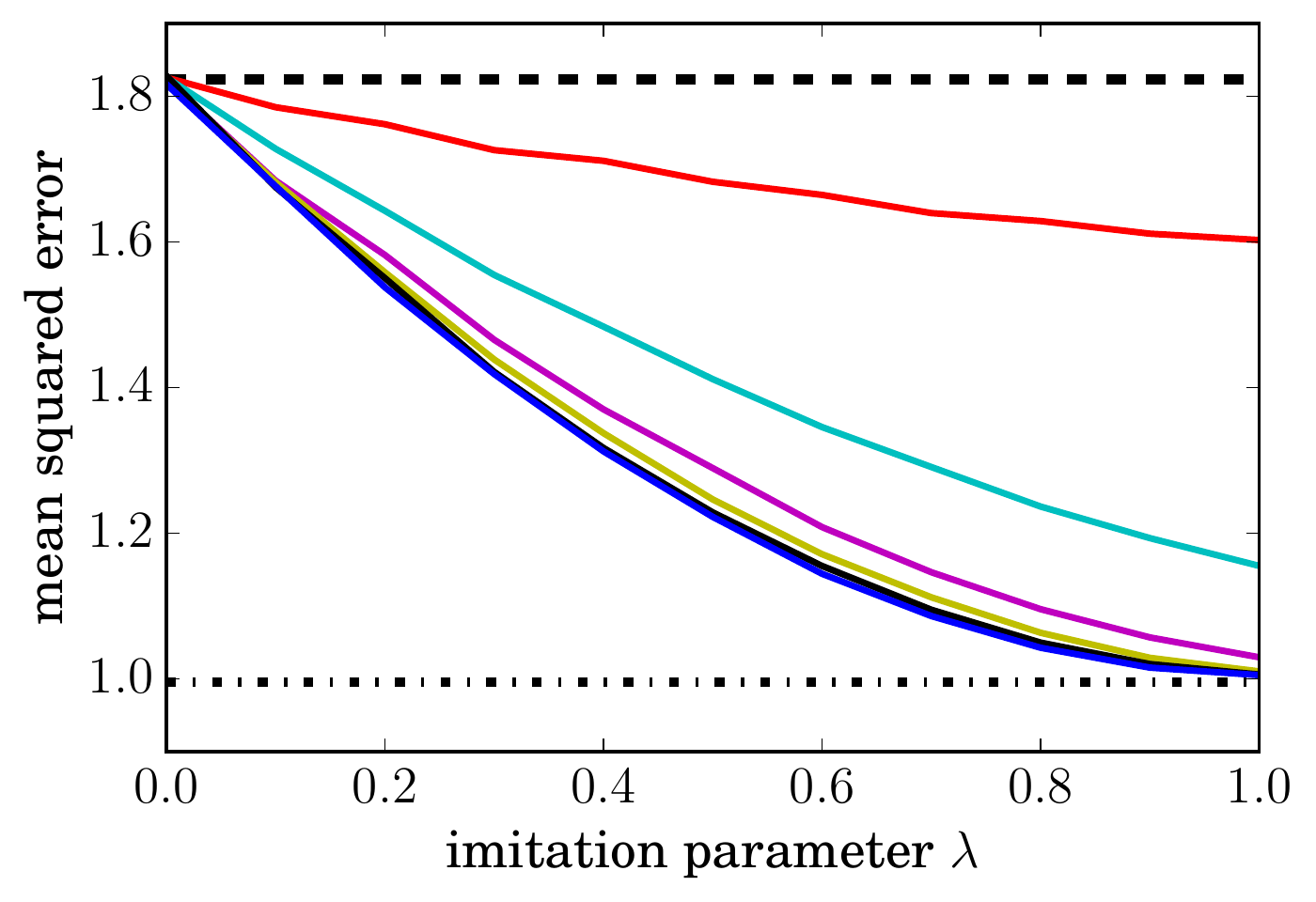}
    \end{center}
  \end{subfigure}
  \caption{Distillation results on CIFAR 10 and SARCOS.}
  \label{fig:others}
\end{figure}

\subsubsection{Machine adapters}

Consider a machine $f : \Rd \to \R^c$ trained on some task, and a collection of
unlabeled data $\{x_i\}_{i=1}^n$, $x_i \in \R^q$, related to a new but related
task. For instance, we may exploit the knowledge contained in $f$ by learning an
\emph{adapter} $a : \R^{q} \to \R^{d}$ that minimizes the reconstruction loss
\begin{equation*}
  L(a,g; x, f) = \frac{1}{n} \sum_{i=1}^n \| x_i - g(f(a(x_i)))\|,
\end{equation*}
where $g : \R^c \to \R^q$. The resulting machine $f(a(x))$ would simultaneously
contain knowledge from $f$ (for instance, high-level visual features) and most
of the information from the new data $\{x_i\}_{i=1}^n$. Alternatively, one
could also train the adapter $a$ by using labeled data and a supervised
objective, or a generative adversarial network (see Section~\ref{sec:gans}), or
an unsupervised objective function on the output statistics of $f$.

\section{Theory of nonconvex optimization}\label{sec:nonconvex}
When learning a function $g$ using empirical risk minimization over a function
class $\F$ and dataset $D\sim P^n$, the error of a computed solution $\tilde{f}
\in \F$ is
\begin{align*}
  \Ex &= \E{}{R(\tilde{f}) - R(g)}\\
  &= \E{}{R(\tilde{f})-R(\hat{f})} + \E{}{R(\hat{f})-R(f^\star)}
  + \E{}{R({f^\star}) - R(g)} \\
  &= \Ex_{\text{opt}} + \Ex_{\text{est}} +
  \Ex_{\text{app}}.
\end{align*}
First, the term $\Ex_{\text{app}}$ is the \emph{approximation error} due to the
difference between the expected risk minimizer $f^\star \in \F$ and the target
function $g$. Second, the term $\Ex_{\text{est}}$ is the \emph{estimation
error} due to the difference between the expected risk minimizer $f^\star$ and
the empirical risk minimizer $\hat{f} \in \F$. Third, the term $\Ex_\text{opt}$
is the \emph{optimization error} due to the difference between the empirical
risk minimizer $\hat{f}$ and the computed solution $\tilde{f}$.  
Optimization errors arise due to the imprecisions of the numerical
computation of $\tilde{f}\in\F$, such as the local minima of nonconvex
empirical risk minimization problems and limited computational budgets.

Observe that if $g \in \F$ or $\F$ is universally consistent, then
$\Ex_\text{app} = 0$. Second, the generalization error $\Ex_\text{est}$ is
inversely proportional to the amount of available training data, and directly
proportional to the flexibility of $\F$ as measured, for instance, using
Rademacher complexities \citep{BM01} or stability criteria \citep{hardtrecht}.
For convex learning problems and gradient-based numerical optimization
routines (Section~\ref{sec:numerical-optimization}), the optimization error
$\Ex_\text{opt}$ is inversely proportional to the number of iterations
\citep{bousquet2008tradeoffs}. 
However, for general nonconvex learning problems, such as deep or convolutional
neural networks \citep{dlbook}, we have no guarantees about the optimization
error, that is, the difference between $\hat{f}$ and $\tilde{f}$. This is a big
caveat: nonconvex empirical risk minimization is
NP-hard, and the theory of empirical risk minimization only holds if we can find the 
empirical risk minimizer (Section~\ref{sec:learning-theory}). 

Let us exemplify the goal of this section by using the
language of neural networks.  To this end, assume data $\mathcal{D} = \{(x_i,
g(x_i))\}_{i=1}^n$, where $g$ is a neural network with $h$ hidden layers of $w$
neurons each. Using empirical risk minimization over the data $\mathcal{D}$ and
a neural network with $H \geq h$ hidden layers of $W \geq w$ neurons each, we obtain the
solution $\tilde{f}$. Because of nonconvexity, the solution $\tilde{f}$ may be
worse than the empirical risk minimizer $\hat{f}$. Also, since we are in a
realizable learning situation, the empirical risk minimizer has zero
approximation error. Therefore, we are interested in characterizing the
optimization error as the tail probability 
\begin{equation}\label{eq:nnet-tail}
  \Pr\left(R(\tilde{f}) > t\right) \leq g(n,w,h,W,H),
\end{equation}
where the randomness is due to the random initialization of the neural network
parameters provided to the gradient-based optimizer.  We propose to study
\eqref{eq:nnet-tail} by sketching two novel concepts: convexity generalizations
and continuation methods.

\subsection{Convexity generalizations}

One way to study nonconvex functions is to compare the quality of their local
minima. We do this by introducing two generalizations of convexity:
$\alpha$-convexity and $\varepsilon$-convexity. The first one,
$\alpha$-convexity, measures how much does the quality of two random local
minima of $f$ differ.

\begin{definition}[$\alpha$-convexity]
  A function $f : \Omega \to [0,1]$ is $\alpha$-convex
  if, for two local minima $w, w' \in \Omega$, it follows that
  \begin{equation*}
    \Pr_{w,w'}\left(\left| f(w) - f(w') \right| > t\right) \leq C_f \exp(-c_f t^2),
  \end{equation*}
  for some constants $C_f,c_f > 0$. 
\end{definition}

Therefore, the local minima of functions with an $\alpha$-convexity profile
that decays fast will be similar in value, and in particular, similar in value
to the global minima.  Alternatively, $\varepsilon$-convexity measures how much
does a differentiable function $f$ depart from a convex function. 

\begin{definition}[$\varepsilon$-convexity]
  A differentiable function $f : \Omega \to [0,1]$ is $\varepsilon$-convex if,
  for all $w, w' \in \Omega$, it follows that
  \begin{equation*}
    \Pr_{w,w'}\left(f(w') - f(w)-\nabla f(w)^\top(w'-w) > t \right) \leq C_f \exp(-c_f t^2),
  \end{equation*}
  for some constants $C_f,c_f > 0$.
\end{definition}

Similar definitions for $\varepsilon$-convexity follow by using zero-order or
second-order conditions. We conjecture that optimizing a function $f$ with a
$\varepsilon$-convexity profile that decays fast will be similar to optimizing
a convex function; this may translate into guarantees about the relationship
between the local and global minima of $f$.

\subsubsection{The convexity of deep neural networks}
We hope that $\alpha$-convexity and $\varepsilon$-convexity will aid the
investigation of the loss surface of multilayer neural networks
\citep{choromanska2014loss}. We believe this because of two intuitions.  First,
the local minima of large neural networks have better value than the global
minima of small neural networks.  This should translate into a good
$\alpha$-convexity profile, and the fast decay of the tail probability
\eqref{eq:nnet-tail}. In practice, to obtain the quality of the empirical risk
minimizer from a set of small neural networks, practitioners simply train to
local optimality a large neural network.  Second, large neural networks are
highly redundant \citep{denil2013predicting}.  Thus, it is not critical to
misconfigure some of the parameters of these networks, since we can leverage
the redundancy provided by the remaining parameters to keep descending down the
loss surface.  Actually, it is known that the amount of local minima decreases
exponentially with higher optimization dimensionality
\citep{dauphin2014identifying}, and that the challenges of high-dimensional
nonconvex optimization are mostly due to saddle points.  Our
intuition is that these thoughts relate to the $\alpha$-convexity and
$\varepsilon$-convexity profiles of neural network empirical risk minimization,
as well as to the tail probability \eqref{eq:nnet-tail}.  To turn intuition
into mathematics, it would be desirable to obtain expressions for the
$\varepsilon$-convexity and the $\alpha$-convexity of deep neural networks in
terms of their number of their hidden layers and neurons.  These results would
be a remarkable achieving, and would provide deep neural networks with the
necessary theory for their empirical risk minimization.

\subsection{Continuation methods}

Continuation methods \citep{mobahi2015theoretical} tackle nonconvex
optimization problems by first solving an easy optimization problem, and then
progressively morphing this easy problem into the nonconvex
problem of interest.  Here, we propose a simple way of
implementing continuation methods in neural networks. Our proposal is based on
two observations. First, the only nonlinear component in neural networks is
their activation function $\sigma : \R \to \R$. Second, for linear activation
functions, we can solve neural networks optimally \citep{Baldi89}.  Therefore,
let us replace the activation functions $\sigma$ in a neural network with
\begin{equation*}
  \sigma_\alpha(z) = (1-\alpha) z + \alpha \sigma(z),
\end{equation*}
where $0 \leq \alpha \leq 1$. For $\alpha = 0$, the neural network is linear.
For $\alpha = 1$, the neural network is nonlinear. For $0 < \alpha < 1$, the
neural network has an intermediate degree of nonlinearity. The continuation
scheme would be to first minimize our neural network equipped with activation
functions $\sigma_0$, and then reuse the solution to minimize the same neural
network with activation functions $\sigma_\epsilon, \sigma_{2\epsilon}, \ldots,
\sigma_1$, for some small $0 < \epsilon < 1$.

We believe that investigating the quality of neural networks obtained with a
continuation method like the one described above is an interesting research
question. How different are two solutions obtained with this continuation
method? How do these solutions compare to the solutions obtained from usual
backpropagation? Can we relate the solutions obtained using our
continuation method to the global minima, under additional assumptions and for
very small $\epsilon$? 

\section{The supervision continuum}\label{sec:continuum}

The mathematical difference between supervised and unsupervised learning is
subtle: in the end, both are the minimization of a loss function. For instance,
in the supervised task of \emph{classification} we ``learn a function $f : \Rd
\to \Delta^c$ using the loss $\ell_{\text{sup}}$ and the \emph{labeled} data
$\{(x_i,y_i)\}_{i=1}^n$''.  On the other hand, in the unsupervised task of
\emph{clustering} we ``learn a function $f : \Rd \to \Delta^c$ using the loss
$\ell_{\text{unsup}}$ and the \emph{unlabeled} data $\{x_i\}_{i=1}^n$''.

The main difference between the previous is that in supervised learning we have
a clear picture about how $\ell_{\text{sup}}$ should look like, but in
unsupervised learning, the shape of $\ell_\text{unsup}$ depends on the type of
learning tasks that we expect to confront in the future. Supervised and
unsupervised learning are the two ends of the \emph{supervision continuum}.
Everything in between is a situation where our training data is a mixture of
labeled and unlabeled examples. We formalize this by writing the examples
comprising our data as 
\begin{equation*}
  (x_i,y_i) \in (\R^d \cup \square) \times (\Delta^c \cup \square),
\end{equation*}
where $x_i = \square$ or $y_i = \square$ means ``not available''.

As the percentage of labeled examples in our data grows, so does the
\emph{supervision level} of the learning problem at hand. So, supervision is
not a matter of two extremes, but characterized as a \emph{continuum}.
Therefore, it makes sense to ask if there exists a single learning machine that
can deal efficiently with the whole supervision continuum, or if we need
fundamentally different algorithms to deal with different levels of
supervision.  One example of an algorithm dealing with the supervision
continuum is the \emph{ladder network} of \citet{rasmus2015semi}, which mixes a
cross-entropy objective for labeled examples with a reconstruction objective
for unlabeled examples.  However, it would be interesting to develop
unsupervised objectives alternative to reconstruction error, which do not
involve learning a whole complicated decoder function (for instance, favour low
density decision boundaries or large margin for unlabeled samples as in
transductive learning).

The supervision continuum extends to multitask learning. To see this, write the
examples comprising our data as
\begin{equation*}
  (x_i,y_i,t_i) \in (\R \cup \square)^{d_{t_i}} \times (\Delta \cup
  \square)^{c_{t_i}} \times (\mathbb{Z} \cup \square),
\end{equation*}
where $t_i \in (\mathbb{Z}, \square)$ is the ``task identification number for
the $i$-th example'', and two examples $(x_i, y_i, t_i)$ and $(x_j, y_j, t_j)$
may have inputs and outputs defined on different spaces.  Said differently, we
may not know to which task some of the examples in our data belong.  This is
similar to human learning: we are constantly presented with a stream of data,
that we exploit to get better at different but related learning tasks. However,
in many cases these tasks are not explicitly identified. How can a machine deal
with this additional continuum of supervision?

  \addcontentsline{toc}{chapter}{Bibliography}
  \bibliographystyle{thesis}
  \bibliography{thesis}
\end{document}